\documentclass[11pt, a4paper, final]{book}
\usepackage{amsmath,amsthm,amssymb,amsfonts}
\usepackage{svn-multi}
\svnid{$Id$}
\usepackage{prelim2e}

\usepackage[hyperindex=true,
			bookmarks=true,
            pdftitle={}, pdfauthor={Xi Yang},
            colorlinks=false,
            pdfborder=0,
            pagebackref=false,
            citecolor=blue,
            plainpages=false,
            pdfpagelabels,
            pagebackref=true,
            hyperfootnotes=false]{hyperref}
\usepackage[all]{hypcap}
\usepackage[palatino]{anuthesis}
\usepackage{afterpage}
\usepackage{graphicx}
\usepackage{thesis}
\usepackage[square]{natbib}
\usepackage[normalem]{ulem}
\usepackage[table]{xcolor}
\usepackage{makeidx}
\usepackage{cleveref}
\usepackage{float}
\usepackage[T1]{fontenc}
\usepackage[scaled]{beramono}

\usepackage{multirow}
\usepackage{mathtools}

\usepackage{url}

\renewcommand*{\backref}[1]{}
\renewcommand*{\backrefalt}[4]{
  \ifcase #1 %
  \or
    (cited on page #2)%
  \else
    (cited on pages #2)%
  \fi
}

\allowdisplaybreaks
\usepackage{tikz}
\usepackage{pgfplots}
\usepackage{olo}
\usepackage{float}
\usepackage{caption}

\usetikzlibrary{matrix,positioning,calc}
\tikzset{snake it/.style={-stealth,
		decoration={snake, 
			amplitude = .4mm,
			segment length = 2mm,
			post length=0.9mm},decorate}}
\pgfplotsset{compat=newest}
\usepgfplotslibrary{fillbetween}
\usetikzlibrary{intersections}

\renewcommand{\hat}[1]{\widehat{{#1}}}

\renewcommand{\bar}[1]{\overline{{#1}}}
\newcommand{\eexp}{\mathrm{e}^\epsilon}

\newcommand{\learningtaskbasicext}{\begin{tikzpicture}
	\matrix (m) [matrix of nodes, ampersand replacement=\&, row sep=5em,
	column sep=7em, text height=1.5ex, text depth=0.25ex]
	{ $\Theta$ \& $\cO$ \& $\cA$  \\ };
	\path[->]
	(m-1-1) edge [snake it] node[above] {$\varepsilon$} node[below] {Experiment} (m-1-2);
	\path[->]
	(m-1-2) edge [snake it] node[above] {$A$} node[below] {Decision rule} (m-1-3);
	\path[->]
	(m-1-1) edge [bend right=90, looseness=0.5, snake it] node[below] {$T := A \circ \varepsilon$} (m-1-3);
	\end{tikzpicture}}

\newcommand{\multicpetaskext}{\begin{tikzpicture}
	\matrix (m) [matrix of nodes, ampersand replacement=\&, row sep=5em,
	column sep=5em, text height=1.5ex, text depth=0.25ex]
	{ $\bs{k}$ \& $\cO$ \& $\bs{k}$  \\ };
	\path[->]
	(m-1-1) edge [snake it] node[above] {$\varepsilon$} (m-1-2);
	\path[->]
	(m-1-2) edge [snake it] node[above] {$A$} (m-1-3);
	\path[->]
	(m-1-1) edge [bend right=90, looseness=0.5, snake it] node[below] {$T := A \circ \varepsilon$} (m-1-3);
	\end{tikzpicture}}

\newcommand{\learningtask}{\begin{tikzpicture}
	\matrix (m) [matrix of nodes, ampersand replacement=\&, row sep=5em,
	column sep=5em, text height=1.5ex, text depth=0.25ex]
	{ $\Theta$ \& $\cO^n$ \& $\cA$  \\ };
	\path[->]
	(m-1-1) edge [snake it] node[above] {$\varepsilon_n$} (m-1-2);
	\path[->]
	(m-1-2) edge [snake it] node[above] {$A$} (m-1-3);
	\end{tikzpicture}}

\newcommand{\supervisedlearningtask}{\begin{tikzpicture}
	\matrix (m) [matrix of nodes, ampersand replacement=\&, row sep=5em,
	column sep=5em, text height=1.5ex, text depth=0.25ex]
	{ $\Theta$ \& $\br{\cX \times \cY}^n$ \& $\cF$  \\ };
	\path[->]
	(m-1-1) edge [snake it] node[above] {$\varepsilon_n$} (m-1-2);
	\path[->]
	(m-1-2) edge [snake it] node[above] {$\hat{f}$} (m-1-3);
	\end{tikzpicture}}

\newcommand{\supervisedbinaryclassificationtask}{\begin{tikzpicture}
	\matrix (m) [matrix of nodes, ampersand replacement=\&, row sep=5em,
	column sep=5em, text height=1.5ex, text depth=0.25ex]
	{ $\Theta$ \& $\br{\cX \times \bc{-1,1}}^n$ \& $\cF$  \\ };
	\path[->]
	(m-1-1) edge [snake it] node[above] {$\varepsilon_n$} (m-1-2);
	\path[->]
	(m-1-2) edge [snake it] node[above] {$\hat{f}$} (m-1-3);
	\end{tikzpicture}}

\newcommand{\specificlearningtask}{\begin{tikzpicture}
	\matrix (m) [matrix of nodes, ampersand replacement=\&, row sep=5em,
	column sep=5em, text height=1.5ex, text depth=0.25ex]
	{ $\Theta_{h,\cF}$ \& $\br{\cX \times \bc{-1,1}}^n$ \& $\cF$  \\ };
	\path[->]
	(m-1-1) edge [snake it] node[above] {$\varepsilon_n$} (m-1-2);
	\path[->]
	(m-1-2) edge [snake it] node[above] {$\hat{f}$} (m-1-3);
	\end{tikzpicture}}

\newcommand{\corrspecificlearningtask}{\begin{tikzpicture}
	\matrix (m) [matrix of nodes, ampersand replacement=\&, row sep=5em,
	column sep=5em, text height=1.5ex, text depth=0.25ex]
	{ $\Theta_{h,\cF}$ \& $\br{\cX \times \bc{-1,1}}^n$ \& $\br{\cX \times \bc{-1,1}}^n$ \& $\cF$  \\ };
	\path[->]
	(m-1-1) edge [snake it] node[above] {$\varepsilon_n$} (m-1-2);
	\path[->]
	(m-1-2) edge [snake it] node[above] {$T_{1:n}$} (m-1-3);
	\path[->]
	(m-1-3) edge [snake it] node[above] {$\hat{f}$} (m-1-4);
	\end{tikzpicture}}

\newcommand{\parametertask}{\begin{tikzpicture}
	\matrix (m) [matrix of nodes, ampersand replacement=\&, row sep=5em,
	column sep=5em, text height=1.5ex, text depth=0.25ex]
	{ $\Theta$ \& $\cO^n$ \& $\Theta$  \\ };
	\path[->]
	(m-1-1) edge [snake it] node[above] {$\varepsilon_n$} (m-1-2);
	\path[->]
	(m-1-2) edge [snake it] node[above] {$\hat{\theta}$} (m-1-3);
	\end{tikzpicture}}

\newcommand{\delparametertask}{\begin{tikzpicture}
	\matrix (m) [matrix of nodes, ampersand replacement=\&, row sep=5em,
	column sep=5em, text height=1.5ex, text depth=0.25ex]
	{ $B$ \& $\cO^n$ \& $B$  \\ };
	\path[->]
	(m-1-1) edge [snake it] node[above] {$\varepsilon_n$} (m-1-2);
	\path[->]
	(m-1-2) edge [snake it] node[above] {$\hat{b}$} (m-1-3);
	\end{tikzpicture}}

\newcommand{\constrainedlearningtask}{\begin{tikzpicture}
	\matrix (m) [matrix of nodes, ampersand replacement=\&, row sep=5em,
	column sep=5em, text height=1.5ex, text depth=0.25ex]
	{ $\Theta$ \& $\cO^n$ \& $\hat{\cO}^n$ \& $\cA$ \\ };
	\path[->]
	(m-1-1) edge [snake it] node[above] {$\varepsilon_n$} (m-1-2);
	\path[->]
	(m-1-2) edge [snake it] node[above] {$T_{1:n}$} (m-1-3);
	\path[->]
	(m-1-3) edge [snake it] node[above] {$A$} (m-1-4);
	\path[->]
	(m-1-1) edge [bend right=90, looseness=0.5, snake it] node[below] {$\tilde{\varepsilon}_n := \br{T \circ \varepsilon}_n$} (m-1-3);
	\end{tikzpicture}}

\newcommand{\localprivacytask}{\begin{tikzpicture}
	\matrix (m) [matrix of nodes, ampersand replacement=\&, row sep=5em,
	column sep=5em, text height=1.5ex, text depth=0.25ex]
	{ $\Theta$ \& $\cX$ \& $\cZ$ \\ };
	\path[->]
	(m-1-1) edge [snake it] node[above] {$\varepsilon$} (m-1-2);
	\path[->]
	(m-1-2) edge [snake it] node[above] {$T$} (m-1-3);
	\end{tikzpicture}}

\newcommand{\simplelocalprivacytask}{\begin{tikzpicture}
	\matrix (m) [matrix of nodes, ampersand replacement=\&, row sep=5em,
	column sep=5em, text height=1.5ex, text depth=0.25ex]
	{ $\Theta$ \& $\cX$ \& $\cX$ \\ };
	\path[->]
	(m-1-1) edge [snake it] node[above] {$\varepsilon$} (m-1-2);
	\path[->]
	(m-1-2) edge [snake it] node[above] {$T$} (m-1-3);
	\end{tikzpicture}}

\newcommand{\databaseprivacytask}{\begin{tikzpicture}
	\matrix (m) [matrix of nodes, ampersand replacement=\&, row sep=5em,
	column sep=5em, text height=1.5ex, text depth=0.25ex]
	{ $\Theta$ \& $\cX$ \& $\cY$ \& $\cZ$ \\ };
	\path[->]
	(m-1-1) edge [snake it] node[above] {$\varepsilon$} (m-1-2);
	\path[->]
	(m-1-2) edge node[above] {$f$} (m-1-3);
	\path[->]
	(m-1-3) edge [snake it] node[above] {$H$} (m-1-4);
	\end{tikzpicture}}

\newcommand{\simplelocalprivacytaskccc}{\begin{tikzpicture}
	\matrix (m) [matrix of nodes, ampersand replacement=\&, row sep=5em,
	column sep=5em, text height=1.5ex, text depth=0.25ex]
	{ $\Theta$ \& $\cV^n$ \& $\cV^n$ \\ };
	\path[->]
	(m-1-1) edge [snake it] node[above] {$\varepsilon_n$} (m-1-2);
	\path[->]
	(m-1-2) edge [snake it] node[above] {$T_{1:n}$} (m-1-3);
	\end{tikzpicture}}

\usepackage{algorithm}
\usepackage{algorithmic}

\usepackage{notation}

\newcommand{\minimize}{\argmin\limits_{x \in \Omega}}

\newcommand{\maximizevalue}{\sup\limits_{x \in \Omega}}
\newcommand{\maximizepairvalue}{\sup\limits_{x,y \in \Omega}}

\newcommand{\bregmanadaregpsi}[2]{\mathcal{B}_{\psi} \del{#1,#2}}

\newcommand{\bregmangunc}[2]{\mathcal{B}_{g} \del{#1,#2}}
\newcommand{\bregmanadareg}[2]{\mathcal{B}_{r_{0:t}} \del{#1,#2}}
\newcommand{\bregmanadaregcumalative}[3]{\mathcal{B}_{r_{0:#3}} \del{#1,#2}}
\newcommand{\bregmanadaregsingle}[3]{\mathcal{B}_{r_{#3}} \del{#1,#2}}

\newcommand{\regderivative}[1]{\nabla r_{0:t} \del{#1}}
\newcommand{\psiderivative}[1]{\nabla \psi \del{#1}}

\newcommand{\stronghlambda}{H_{1:t}+\lambda_{1:t}}

\newcommand{\stronghlambdaminus}{H_{1:t}+\lambda_{1:t-1}}

\newcommand{\compositefunc}[2]{f_t \del{#1} + \psi_t \del{#2}}

\makeatletter
\newcommand{\newreptheorem}[2]{\newtheorem*{rep@#1}{\rep@title} 
	\newenvironment{rep#1}[1]{\def\rep@title{#2 \ref*{##1}}\begin{rep@#1}}{\end{rep@#1}}
}
\makeatother

\newreptheorem{lemma}{Lemma}
\newreptheorem{theorem}{Theorem}
\newreptheorem{claim}{Claim} 

\usepackage{enumerate}
\usepackage{stmaryrd}

\usepackage[T1]{fontenc}
\usepackage{lmodern}

\newcommand{\RR}{\mathbb{R}} 
\newcommand{\ptil}{\tilde{p}}

\newcommand{\zbar}{\bar{z}}

\newcommand{\Lubar}{\underline{L}}

\newcommand{\Deltatil}{\tilde{\Delta}}
\newcommand{\elltil}{\tilde{\ell}}
\newcommand{\psitil}{\tilde{\psi}}
\newcommand{\Lubartil}{\tilde{\underline{L}}}

\newcommand{\where}{{\quad \text{where} \quad}}
\newcommand{\one}{\mathbf{1}}  

\usepackage{booktabs}
\usepackage{relsize}
\usepackage{xspace}
\usepackage{subfigure}
\usepackage{listings}
\definecolor{tableheadcolor}{rgb}{0.8,0.8,1.0}
\definecolor{tablealtcolor}{rgb}{0.9,0.9,0.95}

\definecolor{todocolor}{rgb}{0.8,0.8,1.0}
\definecolor{fixcolor}{rgb}{1,0.8,0.8}
\definecolor{commentcolor}{rgb}{0.8,1.0,0.8}

\usepackage[color=todocolor, colorinlistoftodos]{todonotes}

\newcommand{\textjava}[1]{{\lstset{basicstyle=\ttfamily}\lstinline@#1@}}
\newcommand{\textjavafn}[1]{{\lstset{basicstyle=\footnotesize\ttfamily}\lstinline@#1@}}
\usepackage{setspace}
\usepackage{ifthen}

\long\def\sfootnote[#1]#2{\begingroup%
\def\thefootnote{\fnsymbol{footnote}}\footnote[#1]{#2}\endgroup}
\newcommand{\ie}{i.e., }

\newcommand{\doi}[1]{\href{http://dx.doi.org/#1}{\nolinkurl{doi:#1}}}
\newcommand{\ignore}[1]{}

\title{Transitions, Losses, and Re-parameterizations: Elements of Prediction Games}
\author{Parameswaran Kamalaruban}
\date{\today}

\renewcommand{\thepage}{\roman{page}}

\makeindex
\begin{document}
\pagestyle{empty}
\thispagestyle{empty}

\begin{titlepage}
  \enlargethispage{2cm}
  \begin{center}
    \makeatletter
    \Huge\textbf{\@title} \\[.4cm]
    \Huge\textbf{\thesisqualifier} \\[2.5cm]
    \huge\textbf{\@author} \\[9cm]
    \makeatother
    \LARGE A thesis submitted for the degree of \\
    Doctor of Philosophy at \\
    The Australian National University \\[2cm]
    \thismonth
  \end{center}
\end{titlepage}

\vspace*{14cm}
\begin{center}
  \makeatletter
  \copyright\ \@author{} 2017
  \makeatother
\end{center}
\noindent
\begin{center}
  \footnotesize{~} 
\end{center}
\noindent

\newpage
\chapter*{Declaration}
\addcontentsline{toc}{chapter}{Declaration}


The results in the thesis were produced under the supervision of Bob Williamson, and partly in collaboration with Xinhua Zhang. However, the majority of the work, approximately $70-80\%$, is my own. The main contributions of this thesis are three related parts. The first part of the thesis on the asymmetric learning problems is based on intensive discussions and technical advice from Bob Williamson. The main technical results in the second part on exp-concave loss functions appeared as a conference paper with Bob Williamson and Xinhua Zhang [1]. The results were discussed with my supervisors Bob Williamson and Xinhua Zhang, who gave me advice and direction. The results on the acceleration of online optimization methods are work in progress and contained in an unpublished manuscript [3]. This third part of the thesis is based on some technical advice from Tim van Erven. 

Some of the material in the thesis has already been published elsewhere in collaboration with others, the details of which are (Conference Proceedings and Preprints):
\begin{enumerate}
	\item \textbf{Parameswaran Kamalaruban}, Robert Williamson, and Xinhua Zhang. Exp-Concavity of Proper Composite Losses. In \textit{Proceedings of The 28th Conference on Learning Theory,} pages 1035--1065, 2015.
	\item Kush Bhatia, Prateek Jain, \textbf{Parameswaran Kamalaruban}, and Purushottam Kar. Efficient and Consistent Robust Time Series Analysis. \textit{arXiv:1607.00146 [cs.LG],} 2016. URL \url{http://arxiv.org/abs/1607.00146}.
	\item \textbf{Parameswaran Kamalaruban}. Improved Optimistic Mirror Descent for Sparsity and Curvature. \textit{arXiv:1609.02383 [cs.LG],} 2016. URL \url{http://arxiv.org/abs/1609.02383}.
\end{enumerate}

\vspace*{3cm}

\hspace{8cm}\makeatletter\@author\makeatother\par
\hspace{8cm}\today

\cleardoublepage
\pagestyle{empty}
\vspace*{7cm}
\begin{center}
To my parents.
\end{center}

\cleardoublepage
\pagestyle{empty}
\chapter*{Acknowledgments}
\addcontentsline{toc}{chapter}{Acknowledgments}

I would like to express my gratitude to all the people whose help, advice, and support made significant contributions to this thesis.

First, I would like to warmly acknowledge the valuable guidance and the continuous support of my primary supervisor, Bob Williamson. His wide perspective and insight were instrumental in steering my research over the course of my studies. I have learned so much from him and have thoroughly enjoyed our interactions. A most special thanks go to Xinhua Zhang, who is my co-supervisor. I am very grateful for his constant support, availability, patience, for his thoughtful comments, sharp insights, constructive critics, and discussions. 

I would like to thank the Australian Government for its great support for research. I was very kindly supported by both the Australian National University and Data 61 (then NICTA). I thank both for creating a fantastic environment for research. I was lucky to have helpful colleagues for technical and general discussions, among them Aditya Krishna Menon, Richard Nock, and Brendan van Rooyen. 

My special thanks to Tim van Erven, and Prateek Jain for hosting me kindheartedly while visiting their research groups. Thank you to both of them for stimulating technical discussions.   

My very heartfelt thanks go to my friends in Canberra who shared a lot of laughter, debates, and ideas. Special thanks to my long-time friend and house-mate Ajanthan, and I treasure much of our friendly conversations on various topics.  

I would like to express my deep gratitude to my siblings (Nathan and Manju), and my long-time friends in Sri Lanka (Aravinthan, Pathmayogan, Prakash, and Manorathan). I still feel touched by the great trust and affection they showed me. 

And finally, deepest felt thanks to my parents for their unconditional love, and without their input, I would not be the person I am today!

\cleardoublepage
\pagestyle{headings}
\chapter*{Abstract}
\addcontentsline{toc}{chapter}{Abstract}
\vspace{-1em}
This thesis presents some geometric insights into three different types of two player prediction games -- namely general learning task, prediction with expert advice, and online convex optimization. These games differ in the nature of the opponent (stochastic, adversarial, or intermediate), the order of the players' move, and the utility function. The insights shed some light on the understanding of the intrinsic barriers of the prediction problems and the design of computationally efficient learning algorithms with strong theoretical guarantees (such as generalizability, statistical consistency, and constant regret etc.). The main contributions of the thesis are:
\begin{itemize}
	\item Leveraging concepts from statistical decision theory, we develop a necessary toolkit for formalizing the prediction games mentioned above and quantifying the objective of them. 
	\item We investigate the cost-sensitive classification problem which is an instantiation of the general learning task, and demonstrate the hardness of this problem by producing the lower bounds on the minimax risk of it. 
	
	Then we analyse the impact of imposing constraints (such as corruption level, and privacy requirements etc.) on the general learning task. This naturally leads us to further investigation of strong data processing inequalities which is a fundamental concept in information theory.
	
	Furthermore, by extending the hypothesis testing interpretation of standard privacy definitions, we propose an asymmetric (prioritized) privacy definition.
	\item We study efficient merging schemes for prediction with expert advice problem and the geometric properties (\textit{mixability} and \textit{exp-concavity}) of the loss functions that guarantee constant regret bounds. As a result of our study, we construct two types of link functions (one using calculus approach and another using geometric approach) that can re-parameterize any binary mixable loss into an exp-concave loss. 
	\item We focus on some recent algorithms for online convex optimization, which exploit the \textit{easy nature of the data} (such as sparsity, predictable sequences, and curved losses) in order to achieve better regret bound while ensuring the protection against the worst case scenario. We unify some of these existing techniques to obtain new update rules for the cases when these easy instances occur together, and analyse the regret bounds of them.
\end{itemize} 

\cleardoublepage
\pagestyle{headings}
\markboth{Contents}{Contents}
\tableofcontents
\listoffigures

\mainmatter

\chapter{Introduction}
\label{cha:intro}


A well-posed \textit{learning problem} can be stated as follows: A \textit{learning algorithm} is said to learn from \textit{experience} $E$ with respect to some \textit{task} $T$ and some \textit{performance measure} $P$, if its performance on $T$, as measured by $P$, improves with experience $E$ (\cite{mitchell1997machine}). Pattern recognition, regression estimation and density estimation are the three main learning problems described by \cite{vapnik1998statistical}. 

Developing learning algorithms is very challenging in complicated problem settings with very high dimensional datasets. These challenges are both \textit{theoretical} (tight error bounds relative to the best hypothesis in the benchmark class, generalizability, and statistical consistency) and \textit{computational} (efficient formulation of the optimization problem, optimal memory usage and running time). Generally, in the machine learning literature, these two challenges are considered independently. Understanding the connection between these two aspects of the learning problem to better understand the problem itself and to develop efficient learning algorithms, is an important and challenging research topic.

Several important problems in machine learning and statistics can be viewed as a two player prediction game between a decision maker and nature. This thesis presents some geometric insights into three different types of two player prediction games - namely general learning task, prediction with expert advice, and online convex optimization. These games differ in the nature of the opponent (stochastic, adversarial, or intermediate), the order of the players' move, mode of the game (batch or sequential), and the utility function. These insights shed some light on the understanding of the intrinsic barriers of the prediction problems and the design of computationally efficient learning algorithms with strong theoretical guarantees (such as generalizability, statistical consistency, and constant regret etc.).

There are many different objects which help us understanding the learning problems better. These include loss function, regularizer, information, risk measure, regret, and divergence. Systematically studying various representations (weighted average of primitive elements, variational and dual) of these objects and connections between them proves very useful in developing modular based solutions to learning problems (\cite{reid2011information}). Certain properties of these objects are necessary for strong theoretical guarantees, whereas some other properties are useful in developing computationally efficient learning algorithms. Thus by studying the geometric characterization of the problem w.r.t.\ these notions, we may be able to design solutions which are computationally efficient as well as having strong theoretical guarantees.

The rest of this chapter provides the background to, and a road map for, the rest of this thesis. 

\section{Thesis Outline}

Chapter~\ref{cha:decision} introduces the general learning task which covers many practical problems in machine learning and statistics as special instantiations. The goal of the learner is to find the functions which reflect relationships in data and thus best explain unseen data. Using the decision theoretic concepts, we set up an abstract language of transitions to formalize this general learning task. Then we define several quantities associated with the performance of a learning algorithm for this task such as conditional risk and full risk, and some measures of the hardness of the task such as minimum Bayes risk and minimax risk. 

Next we consider a specific instantiation of the general learning task - namely multi-class probability estimation problem. Finally we discuss the binary class probability estimation problem or classification (which is an instantiation of the multi-class probability estimation problem with $m=2$) in detail. 

The next three chapters contain the contributions of this thesis.

Chapter~\ref{cha:markov} mainly deals with the cost-sensitive classification problem, which is also an instantiation of the general learning task. This problem plays a crucial role in mission critical machine learning applications. We study the hardness of this problem and emphasize the impact of cost terms on the hardness.  

Chapter~\ref{cha:markov} investigates the intrinsic barriers of the general learning task subject to constraints such as privacy, noisy transmission (with minimum corruption level), and resource limitation. This naturally leads us to the investigation of strong data processing inequalities. Despite extensive investigation tracing back to the 1950's, the geometric insights of strong data processing inequalities are still not fully understood. A comprehensive survey paper providing an overview of strong data processing inequalities was written by \cite{Raginsky2014}. We continue existing investigations on strong data processing inequalities, and make a significant progress in the direction of filling this gap by focusing on the weighted integral representation of $f$-divergences. This guides us in the channel design for cost-sensitive constrained problems. Furthermore we propose a cost-sensitive privacy definition by extending the standard local privacy definitions, and provide a hypothesis testing based interpretation for it. 

Chapter~\ref{cha:mixability} considers the classical problem of \textit{prediction with expert advice} (\cite{cesa2006prediction}), in which the goal of the learner is to predict as well as the best expert in a given pool of experts, on any sequence of $T$ outcomes. This framework encompasses several applications as special cases (\cite{vovk1995game}) such as classifier aggregation, weather prediction etc. The regret bound of the learner depends on the merging scheme used to merge the experts' predictions and the nature of the loss function used to measure the performance. This problem has been widely studied and $O(\sqrt{T})$ and $O(\log{T})$ regret bounds can be achieved for convex losses (\cite{zinkevich2003online}) and strictly convex losses with bounded first and second derivatives (\cite{hazan2007logarithmic}) respectively. In special cases like the Aggregating Algorithm (\cite{vovk1995game}) with mixable losses and the Weighted Average Algorithm (\cite{kivinen1999averaging}) with exp-concave losses, it is possible to achieve $O(1)$ regret bounds. 

Even though exp-concavity trivially implies mixability, the converse implication is not true in general. Thus by understanding the underlying relationship between these two notions we can gain the best of both algorithms (strong theoretical performance guarantees of the Aggregating Algorithm and the computational efficiency of the Weighted Average Algorithm). We study the general conditions on mixable losses under which they can be transformed into an exp-concave loss through a suitable \textit{link function}. Under mild conditions, we construct two types of link functions (one using calculus approach and another using geometric approach) that can re-parameterize any binary mixable loss into an exp-concave loss. 

Chapter~\ref{cha:oco} focuses on the \textit{online convex optimization} problem which plays a key role in machine learning as it has interesting theoretical implications and important practical applications especially in the large scale setting where computational efficiency is the main concern (\cite{shalev2011online}). Early approaches to this problem were conservative, in which the main focus was protection against the worst case scenario. But recently several algorithms have been developed for tightening the regret bounds in \textit{easy data} instances such as sparsity (\cite{duchi2011adaptive}), predictable sequences (\cite{chiang2012online}), and curved losses (strongly-convex, exp-concave, mixable etc.) (\cite{hazan2007adaptive}). 

We unify some of these existing techniques to obtain new update rules for the cases when these easy instances occur together. First we analyse an adaptive and optimistic update rule which achieves tighter regret bound when the loss sequence is sparse and predictable. Then we explain an update rule that dynamically adapts to the curvature of the loss function and utilizes the predictable nature of the loss sequence as well. Finally we extend these results to composite losses.

Finally, Chapter~\ref{cha:conclussion} contains the conclusion of this thesis, and a discussion of possibilities for further research. Chapter~\ref{cha:conclussion} concludes and contains a summary of the key contributions of this thesis.

The following work was completed during the thesis: \cite{rtse2016}. In this work, we present and analyze a polynomial-time algorithm for consistent estimation of regression coefficients under adversarial corruptions. But it has been excluded from the thesis as it does not fit as well with our theme.

Some definitions are repeated, and there are slight variations in notation for each chapter. Ultimately, there is no single best notational system, the effort has been placed into using the notation that best suits the contents of the chapter.

\chapter{Elements of Decision and Information Theory}
\label{cha:decision}


The focus of this chapter is the abstract formulation of the general learning task, where a decision maker uses observations from experiments to inform her decisions. We present a rigorous mathematical language for making decisions under uncertainty, and quantifying the hardness of the problem. The concepts or results that we review here are based upon both classical works in decision theory \citep{blackwell1951comparison,degroot1962uncertainty,le1964sufficiency,von2007theory,wald1992statistical} as well as recent contributions \citep{dawid2007geometry,grunwald2004game,le2012asymptotic,reid2011information,torgersen1991comparison}. They serve as a necessary background for the rest of the thesis.

The chapter proceeds as follows. In section~\ref{sec:general-learning-task} we introduce the general learning task, formalize it using the language of transitions, and define some decision theoretic measures associated with the hardness of the problem. Then in section~\ref{sec:mutli-cpe} we study the multi-class probability estimation problem, which is a specific instantiation of the general learning task. Finally in section~\ref{sec:binary-cpe} we review more specific problem of binary class probability estimation in detail, by introducing several decision theoretic notions associated with it. 

\section{Notation and General Definitions}
We require the following notation and definitions for chapter~\ref{cha:decision} and chapter~\ref{cha:markov}. Other notation will be developed as necessary.


\paragraph*{Vectors and Matrices}
The real numbers are denoted $\mathbb{R}$, the non-negative reals $\mathbb{R}_{+}$ and the extended reals $\bar{\mathbb{R}} = \mathbb{R} \cup \bc{\infty}$; the rules of arithmetic with extended real numbers and the need for them in convex analysis are explained by \cite{rockafellar1970convex}. The integers and non-negative integers are denoted by $\mathbb{Z}$ and $\mathbb{Z}_{+}$ respectively.  A superscript prime, $A'$ denotes transpose of the matrix or vector $A$, except when applied to a real-valued function where it denotes derivative ($f'$). We denote the matrix multiplication of compatible matrices $A$ and $B$ by $A \cdot B$, so the inner product of two vectors $x,y \in \mathbb{R}^n$ is $x' \cdot y$. Let $[n]:=\{1,\dots,n\}$, and the $n$-simplex $\Delta^n:=\{(p_1,\dots,p_n)': 0 \leq p_i \leq 1, \forall{i \in [n]}, \mathrm{and} \sum_{i \in [n]} p_i = 1\}$. If $x$ is an $n$-vector, $A=\mathrm{diag}(x)$ is the $n \times n$ matrix with entries $A_{ii}=x_i$ , $i \in [n]$ and $A_{ij}=0$ for $i \neq j$. We use $e_i^n$ to denote the $i$th $n$-dimensional unit vector, $e_i^n=(\underbrace{0,\dots,0}_{i-1},1,\underbrace{0,\dots,0}_{n-i})'$ when $i \in [n]$, and define $e_i^n=0_n$ when $i > n$. The $n$-vector $\vone_n:=(\underbrace{1,\dots,1}_{n})'$. 

\paragraph*{Convexity}
A set $\cS \subseteq \bR^d$ is said to be \textit{convex} if for all $\lambda \in [0,1]$ and for all points $s_1,s_2 \in \cS$ the point
$\lambda s_1 + (1-\lambda) s_2 \in \cS$. A function $\phi:\cS \rightarrow \bR$ defined on a convex set $\cS$ is said to be a (proper) convex function if for all $\lambda \in [0,1]$ and points $s_1,s_2 \in \cS$ the function $\phi$ satisfies
\[
\phi\br{\lambda s_1 + (1-\lambda) s_2} \leq \lambda \phi\br{s_1} + (1-\lambda) \phi\br{s_2} .
\]
A function is said to be concave if $-\phi$ is convex.

Given a finite set $S$ and a weight vector $w$, the \textit{convex combination} of the elements of the set w.r.t the weight vector is denoted by $\mathrm{co}_w S$, and the \textit{convex hull} of the set which is the set of all possible convex combinations of the elements of the set is denoted by $\mathrm{co} S$ (\cite{rockafellar1970convex}). If $S,T \subset \RR^n$, then the \textit{Minkowski sum} $S \varoplus T := \{ s+t: s  \in S, t \in T \}$. 

\paragraph*{The Perspective Transform and the Csisz{\'a}r Dual} 
When $\phi:\bR_+ \rightarrow \bar{\bR}$ is convex, the perspective transform of $\phi$ is defined for $\tau \in \bR_+$ via
\[
I_\phi\br{s,\tau} := \begin{cases}
\tau \phi\br{s/\tau} & \tau > 0, s > 0\\
0 & \tau = 0, s = 0\\
\tau \phi\br{0} & \tau > 0, s = 0\\
s \phi'_\infty & \tau = 0, s > 0 ,
\end{cases}
\]
where $\phi\br{0} := \lim_{s \to 0} \phi\br{s} \in \bar{\bR}$ and $\phi'_\infty$ is the \textit{slope at infinity} defined as
\[
\phi'_\infty := \lim_{s \to +\infty}{\frac{\phi\br{s_0 + s} - \phi\br{s_0}}{s}} = \lim_{s \to +\infty}{\frac{\phi\br{s}}{s}}
\]
for every $s_0 \in \cS$ where $\phi\br{s_0}$ is finite. This slope at infinity is only finite when $\phi\br{s} = O\br{s}$, that is, when $\phi$ grows at most linearly as $s$ increases. When $\phi'_\infty$ is finite it measures the slope of the linear asymptote. The function $I_\phi:\left[ 0, \infty \right)^2 \rightarrow \bar{\bR}$ is convex in both arguments \cite{hiriart2013convex} and may take on the value $+\infty$ when $s$ or $\tau$ is zero. It is introduced here because it will form the basis of the $f$-divergences.

The perspective transform can be used to define the \textit{Csisz{\'a}r dual} $\phi^\diamond: \left[ 0, \infty \right) \rightarrow \bar{\bR}$ of a convex
function $\phi:\bR_+ \rightarrow \bar{\bR}$ by letting
\[
\phi^\diamond \br{\tau} := I_\phi \br{1,\tau} = \tau \phi \br{\frac{1}{\tau}}
\]
for all $\tau \in (0,\infty)$ and $\phi^\diamond \br{0} := \phi'_\infty$. The original $\phi$ can be recovered from $I_\phi$ since $\phi\br{s} = I_\phi \br{s,1}$.

The convexity of the perspective transform $I_\phi$ in both its arguments guarantees the convexity of the dual $\phi^\diamond$. Some simple algebraic manipulation shows that for all $s,\tau \in \bR_+$ 
\[
I_\phi\br{s, \tau} = I_{\phi^\diamond}\br{\tau , s}.
\]
This observation leads to a natural definition of symmetry for convex functions. We will call a convex function $\diamond$-symmetric (or simply symmetric when the context is clear) when its perspective transform is symmetric in its arguments. That is, $\phi$ is $\diamond$-symmetric when $I_\phi\br{s, \tau} = I_{\phi}\br{\tau , s}$ for all $s,\tau \in \left[ 0, \infty \right)$. Equivalently, $\phi$ is $\diamond$-symmetric if and only if $\phi^\diamond = \phi$.

\paragraph*{Probabilities and Expectations}
Let $\Omega$ be a measurable space and let $\mu$ be a probability measure on $\Omega$. $\Omega^n$ denotes the product space $\Omega \times \cdots \times \Omega$ endowed with the product measure $\mu^n$. The notation $\textsf{X} \sim \mu$ means $\textsf{X}$ is randomly drawn according to the distribution $\mu$. $\bP_\mu\bs{E}$ and $\Ee{\textsf{X} \sim \mu}{f\br{\textsf{X}}}$ will denote the probability of a statistical event $E$ and the expectation of a random variable $f\br{\textsf{X}}$ with respect to $\mu$ respectively. We will use capital letters $\textsf{X},\textsf{Y},\textsf{Z},\dots$ for random variables and lower-case letters $x,y,z,\dots$ for their observed values in a particular instance. We will denote by $\cP\br{\cX}$ the set of all probability distributions on an alphabet $\cX$ and by $\cP_*\br{\cX}$ the subset of $\cP\br{\cX}$ consisting of all strictly positive distributions. 


\paragraph*{Metric Spaces}
The Hamming distance on $\bR^n$ is defined as
\begin{equation}
\label{eq-hamming-distance}
\rho_\mathrm{Ha}\br{x,x'} := \sum_{i=1}^{n}{\ind{x_i \neq x_i'}} ,
\end{equation}
where $\ind{P}=1$ if $P$ is true and $\ind{P}=0$ otherwise. Define the $p$-norm of $x \in \bR^n$ as
\begin{equation}
\label{eq-lp-norm}
\norm{x}_p := \br{\sum_{i=1}^{n}{\abs{x_i}^p}}^{1/p} .
\end{equation}
Let $\ell_p^n = \br{\bR^n , \norm{\cdot}_p}$ and $B_p^n$ denote the unit ball of $\ell_p^n$. $\ell_{\infty}^n$ is $\bR^n$ endowed with the norm
\begin{equation}
\label{eq-infty-norm}
\norm{x}_\infty := \sup_{1 \leq i \leq n}{\abs{x_i}} .
\end{equation}
Let $L_\infty \br{\Omega}$ be the set of bounded functions on $\Omega$ with respect to the norm
\begin{equation}
\label{eq-infty-norm-func}
\norm{f}_\infty := \sup_{\omega \in \Omega}{\abs{f \br{\omega}}}
\end{equation}
and denote its unit ball by $B\br{L_\infty \br{\Omega}}$. For a probability measure $\mu$ on a measurable space $\Omega$ and $1 \leq p \leq \infty$, let $L_p \br{\mu}$ be the space of measurable functions on $\Omega$ with a finite norm
\begin{equation}
\label{eq-lp-norm-func}
\norm{f}_{L_p \br{\mu}} := \br{\int{\abs{f}^p d\mu}}^{1/p} .
\end{equation}
$\mathcal{Y}^{\mathcal{X}}$ represents the set of all measurable functions $f:\mathcal{X}\rightarrow\mathcal{Y}$. For a set $\cX$ define the functions $id_{\cX}(x) = x$, and $1_{\cX}(x) = 1$. The set of all real-valued measurable functions on $\cX$ is denoted by $\bR^\cX$; $\bR_{++}^{\cX}$ and ${\bR}_{+}^{\cX}$ are the subsets of $\bR^{\cX}$ consisting of all strictly positive and nonnegative measurable functions, respectively. Define $\bar{c} := 1-c$, for $c \in \bs{0,1}$. We write $x \wedge y := \min \br{x,y}$. A mapping $t \mapsto \mathrm{sign}\br{t}$ is defined by
\[
\mathrm{sign}\br{t} = \begin{cases}
1 & \text{if } t \geq 0 \\
-1 & \text{otherwise}
\end{cases} .
\] 
Throughout this thesis all absolute constants are denoted by $c$, $C$, or $K$.

\section{General Learning Task}
\label{sec:general-learning-task}

A \textit{general learning task} in statistical decision theory can be viewed as a two player game between the \textit{decision maker} (statistician or learner) and \textit{nature} (environment or opponent) as follows: Given the parameter space $\Theta$, observation space $\cO$, and decision space $\cA$, and the loss function $\ell : \Theta \times \cA \rightarrow \bR_+$,
\begin{itemize}
	\item Nature chooses $\theta \in \Theta$, and generates the data $\textsf{O} \sim P_\theta \in \cP\br{\cO}$, where $P_\theta$ is the distribution determined by the parameter $\theta$,
	\item the decision maker observes the data $\textsf{O}$, makes her own decision $a \in \cA$ (deterministic or stochastic), and incurs loss with $\ell\br{\theta,a}$.
\end{itemize} 

Throughout the thesis we assume $\Theta$ to be finite and $\cA$ to be closed, compact, set in order to provide a clear presentation by avoiding the measure theoretic complexities. This ensures that infimum of all the quantities defined can be replaced by minimum. Note that all the results presented in the thesis are applicable to general cases as well, under suitable regularity assumptions. \cite{torgersen1991comparison} (Theorem 6.2.12) shows how results for finite $\Theta$ can be extended to those for infinite $\Theta$.

In order to formalize the above game, we develop an \textit{abstract language} using the decision theoretic concepts. We start with the central object of this language called a \textit{transition}.

\subsection{Markov Kernel} 
We define a \textit{Markov kernel} (also known as a  \textit{transition} or a \textit{channel}) as follows:
\begin{definition}[\citep{le2012asymptotic,torgersen1991comparison}]
	A Markov kernel from a finite set $\cX$ to a finite set $\cY$ (denoted by $T: \cX \rightsquigarrow \cY$) is a function from $\cX$ to $\cP\br{\cY}$, the set of probability distributions on $\cY$.
\end{definition}
A \textit{Markov kernel} $T: \cX \rightsquigarrow \cY$ acts on probability distributions $\mu \in \cP\br{\cX}$ by
\[
T \circ \mu := \Ee{\textsf{X} \sim \mu}{T\br{\textsf{X}}} \in \cP\br{\cY}
\]
or on functions $f \in \bR^{\cY}$ by
\[
\br{T f} \br{x} := \Ee{Y \sim T\br{x}}{f\br{Y}} , \quad x \in \cX .
\]
The composition of two Markov kernels $T_1:\cX \rightsquigarrow \cY$ and $T_2:\cY \rightsquigarrow \cZ$, denoted by $T_2 T_1 : \cX \rightsquigarrow \cZ$, is defined by
\[
T_2 T_1 f ~=~ T_1 \br{T_2 f} , \quad f \in \bR^{\cZ} .
\]
Denote the set of all Markov kernels from $\cX$ to $\cY$ by $\cM\br{\cX,\cY}$. If $\cX$ and $\cY$ are finite, we can represent the distributions $P \in \cP\br{\cX}$ by vectors in $\bR^{\abs{\cX}}$, Markov kernels $T: \cX \rightsquigarrow \cY$ by column stochastic matrices ($\abs{\cY} \times \abs{\cX}$ positive matrices where the sum of all entries in each column is equal to $1$), and composition by matrix multiplication. We can also verify that $\cM\br{\cX,\cY}$ is a closed convex subset of $\bR^{\abs{\cY} \times \abs{\cX}}$, the set of all $\abs{\cY} \times \abs{\cX}$ matrices. Note that the transition $T:\cX \rightsquigarrow \cY$ induces a class of probability measures $\cP_T \br{\cY}:= \bc{P_x := T\br{x} \in \cP\br{\cY} : x \in \cX}$. For a transition $T: \cX \rightsquigarrow \cY$, define $T\br{y \mid x} := \bP_{T\br{x}}\bs{\textsf{Y}=y}$, where $T\br{x} \in \cP_T \br{\cY}$. 

A function $f:\cX \rightarrow \cY$ induces a Markov kernel $F: \cX \rightsquigarrow \cY$ with $F(x) = \delta_{f(x)}$, a point mass distribution on $f(x)$. For every measure space $\cX$, there are two special Markov kernels, the \textit{completely informative} Markov kernel induced from the identity function $\mathrm{id}_{\cX} : \cX \rightarrow \cX$ (where $\mathrm{id}_{\cX} \br{x} = x$), and the \textit{completely uninformative} Markov kernel induced from the function $\bullet_{\cX} : \cX \rightarrow \bullet$ (where $\bullet_{\cX} \br{x} = \bullet , \, \forall{x \in \cX}$ and $\bullet \in \cY$).

Given $\mu \in \cP\br{\cX}$, and $T:\cX \rightsquigarrow \cY$, let $D:= \mu \otimes T \in \cP\br{\cX \times \cY}$ denotes the joint probability measure of $\br{\textsf{X},\textsf{Y}} \in \cX \times \cY$ with $\bP_D \bs{\textsf{X} = x} = \bP_\mu \bs{\textsf{X} = x}$, and $\bP_D \bs{\textsf{Y} = y \mid \textsf{X} = x} = \bP_{T\br{x}} \bs{\textsf{Y} = y}$.

We will now use this abstract language of transitions to formulate the general learning task introduced in the beginning of this section. This will enable us to analyse the intrinsic barriers or capacity of the task in a more generic way. Later, by using appropriate instantiations, we will derive important practical problems in machine learning and statistics.

\subsection{Decision Theoretic Notions}
The general learning task described above can be represented by the following transition diagram:
\begin{equation}
\label{eq:general-learning-task-pic}
\learningtaskbasicext ,
\end{equation}
where
\begin{itemize}
	\item \textit{Experiment} (denoted by $\varepsilon:\Theta \rightsquigarrow \cO$) is a Markov kernel from the parameter space $\Theta$ to the observation space $\cO$. If the true hypothesis is $\theta \in \Theta$, then the observed data is distributed according the probability measure $\varepsilon \br{\theta}$. The class of probability measures associated with this experiment is given by $\cP_\varepsilon := \bc{P_\theta := \varepsilon \br{\theta} : \theta \in \Theta}$.
	\item Stochastic \textit{Decision rule} (denoted by $A:\cO \rightsquigarrow \cA$) is a Markov kernel from the observation space $\cO$ to the action space $\cA$. Upon observing data $o \in \cO$, the learner will choose an action in $\cA$ according to the distribution $A\br{o}$.
\end{itemize}

\begin{remark}
	We will depict the transitions (experiment and decision rule) associated with the learning task in a transition diagram, and we call it the `transition diagram representation of the learning task' throughout the thesis.
\end{remark}

\paragraph*{Loss and Regret:}
The quality of the composite relation $T := A \circ \varepsilon: \Theta \rightsquigarrow \cA$ is measured by a loss function 
\begin{equation}
\label{loss-function-eq}
\ell : \Theta \times \cA \ni \br{\theta,a} \mapsto \ell \br{\theta,a} \in \bR_+ .
\end{equation} 
The general learning task can more compactly be represented as the pair $(\ell,\varepsilon)$ where $\cA,\Theta,\cO$ can be inferred from the type signatures of $\ell$ and $\varepsilon$. We usually encounter the loss relative to the best action defined formally as the \textit{regret} 
\begin{equation}
	\label{regret-eq}
	\Delta \ell : \Theta \times \cA \ni \br{\theta,a} ~\mapsto~\Delta \ell \br{\theta,a} := \ell \br{\theta,a} - \inf_{a' \in \cA}{\ell \br{\theta,a'}} \in \bR_+ .
\end{equation}

\paragraph*{Conditional Risk:}
The quality of the final action chosen by the decision maker when
they use the composite relation $T: \Theta \rightsquigarrow \cA$ (in fact the stochastic decision rule $A:\cO \rightsquigarrow \cA$ for a given experiment $\varepsilon: \Theta \rightsquigarrow \cO$) can be evaluated using the notion of \textit{conditional risk} (defined with an overloaded notation for the loss):
\begin{equation}
	\label{conditional-risk-eq}
	\ell : \Theta \times \cM\br{\Theta , \cA} \ni \br{\theta,T} ~\mapsto~ \ell \br{\theta,T} := \Ee{\textsf{A} \sim T\br{\theta}}{\ell \br{\theta,\textsf{A}}} \in \bR_+ ,
\end{equation}
where the term inside the expectation is the loss \eqref{loss-function-eq} of a random variable $\textsf{A}$ when the true parameter is $\theta$. We use the overloaded notation with a reason, which will become clear in section~\ref{sec:mutli-cpe}. Similarly we can define the conditional risk in terms of regret as follows (again with an overloaded notation for the regret): 
\begin{equation}
	\label{conditional-risk-regret-eq}
	\Delta \ell : \Theta \times \cM\br{\Theta , \cA} \ni \br{\theta,T} ~\mapsto~ \Delta \ell \br{\theta,T} := \Ee{\textsf{A} \sim T\br{\theta}}{\Delta \ell \br{\theta,\textsf{A}}} \in \bR_+ ,
\end{equation}
where the term inside the expectation is the regret \eqref{regret-eq} of a random variable $\textsf{A}$ when the true parameter is $\theta$.

For any fixed (unknown) parameter $\theta \in \Theta$, we can calculate the conditional risk of any composite relation $T$, and the goal is to find an optimal composite relation (in fact an optimal stochastic decision rule for a given experiment). Two main approaches to find the best composite relation (or the best decision rule) are:
\begin{itemize}
	\item \textit{Bayesian approach} (average case analysis), which is more appropriate if the decision maker has some intuition about $\theta$, given in the form of a prior probability distribution $\pi$, and
	\item \textit{Minimax approach} (worst case analysis), which is more appropriate if the decision maker has no prior knowledge concerning $\theta$.
\end{itemize}
The \textit{conditional Bayesian risk} and \textit{conditional max risk} are defined as,
\begin{align}
	L_\ell : \cP\br{\Theta} \times \cM\br{\Theta,\cA} \ni \br{p,T} ~\mapsto~& L_\ell \br{p,T} := \Ee{\textsf{Y} \sim p}{\ell \br{\textsf{Y},T}} \in \bR_+ , \text{ and }\label{conditional-bayes-risk} \\
	L_\ell^\star : \cM\br{\Theta,\cA} \ni T ~\mapsto~& L_\ell^\star \br{T} := \sup_{\theta \in \Theta}{\ell \br{\theta,T}} \in \bR_+ , \label{conditional-max-risk} 
\end{align}
respectively. We measure the difficulty of the general learning task by the \textit{conditional minimum Bayesian risk} and \textit{conditional minimax risk} defined as,
\begin{align}
	\underline{L}_\ell : \cP\br{\Theta} \ni p ~\mapsto~& \underline{L}_\ell \br{p} := \inf_{T \in \cM\br{\Theta,\cA}}{L_\ell \br{p,T}} \in \bR_+ , \text{ and }\label{conditional-minbayes-risk} \\
	\underline{L}_\ell^\star : \cdot ~\mapsto~& \underline{L}_\ell^\star := \inf_{T \in \cM\br{\Theta,\cA}}{L_\ell^\star \br{T}} \in \bR_+ , \label{conditional-minimax-risk} 
\end{align}
respectively.

\begin{remark}
	\label{loss-regret-based-risk}
	By replacing $\ell$ by $\Delta \ell$ in \eqref{conditional-bayes-risk}, \eqref{conditional-max-risk}, \eqref{conditional-minbayes-risk}, and \eqref{conditional-minimax-risk} , we obtain $L_{\Delta \ell}$, $L_{\Delta \ell}^\star$, $\underline{L}_{\Delta \ell}$ and $\underline{L}_{\Delta \ell}^\star$ respectively. One can do this transformation for all the concepts that we introduce below and obtain the `regret' based notions.
\end{remark}

\paragraph*{Full Risk:}

In the conditional quantities defined above, we have abstracted away the observation space $\cO$ (i.e. no data setting). Now we consider the practical scenario with observations, and define the \textit{full risk} of a stochastic decision rule $A:\cO \rightsquigarrow \cA$ as follows 
\begin{multline}
	\bL_{\ell} : \Theta \times \cM\br{\Theta,\cO} \times \cM\br{\cO,\cA} \ni \br{\theta,\varepsilon,A} ~\mapsto~ \\ \bL_{\ell} \br{\theta,\varepsilon,A} := \ell\br{\theta , A \circ \varepsilon} = \Ee{\textsf{O} \sim \varepsilon\br{\theta}}{\Ee{\textsf{A} \sim A\br{\textsf{O}}}{\ell \br{\theta,\textsf{A}}}} \in \bR_+ , \label{full-risk}
\end{multline}
where $\ell\br{\theta , A \circ \varepsilon}$ is the conditional risk \eqref{conditional-risk-eq} of the composite relation $A \circ \varepsilon$. Note that $A:\cO \rightsquigarrow \cA$ is a function of the observation in $\cO$ which is distributed according the probability distribution associated with a parameter in $\Theta$.

As in the conditional case, we define the \textit{full Bayesian risk}, \textit{full minimum Bayesian risk}, \textit{full max risk}, and \textit{full minimax risk} as follows:
\begin{align*}
	\cR_{\ell} : \cP\br{\Theta} \times \cM\br{\Theta,\cO} \times \cM\br{\cO,\cA} \ni \br{\pi,\varepsilon,A} ~\mapsto~& \cR_{\ell} \br{\pi,\varepsilon,A} := \\
	& \Ee{\textsf{Y} \sim \pi}{\bL_{\ell} \br{\textsf{Y},\varepsilon,A}} \in \bR_+ , \\
	\underline{\cR}_{\ell} : \cP\br{\Theta} \times \cM\br{\Theta,\cO} \ni \br{\pi,\varepsilon} ~\mapsto~& \underline{\cR}_{\ell} \br{\pi,\varepsilon} := \\
	& \inf_{A \in \cM\br{\cO,\cA}}{\cR_{\ell} \br{\pi,\varepsilon,A}} \in \bR_+ , \\
	\cR_{\ell}^\star : \cM\br{\Theta,\cO} \times \cM\br{\cO,\cA} \ni \br{\varepsilon,A} ~\mapsto~& \cR_{\ell}^\star \br{\varepsilon,A} := \\
	& \sup_{\theta \in \Theta}{\bL_{\ell} \br{\theta,\varepsilon,A}} \in \bR_+ , \text{ and }\\
	\underline{\cR}_{\ell}^\star : \cM\br{\Theta,\cO} \ni \varepsilon ~\mapsto~& \underline{\cR}_{\ell}^\star \br{\varepsilon} := \\
	& \inf_{A \in \cM\br{\cO,\cA}}{\cR_{\ell}^\star \br{\varepsilon,A}} \in \bR_+ 
\end{align*}
respectively.

Let $\textsf{Y}$ and $\textsf{O}$ be random variables over $\Theta$ and $\cO$ respectively. Also let $\theta \in \Theta$ and $o \in \cO$. The experiment $\varepsilon$ (in \eqref{eq:general-learning-task-pic}) and a prior $\pi$ on $\Theta$ induces a joint probability measure $D$ on $\Theta \times \cO$ and thus a transition $\eta_D : \cO \rightsquigarrow \Theta$ (given by $\eta_D \br{\theta \mid o} := \bP_D \bs{\textsf{Y} = \theta \mid \textsf{O}=o}$) and a marginal distribution $M_D$ on $\cO$ (given by $M_D \br{o} := \bP_D \bs{\textsf{O}=o}$). That is if $\Theta \times \cO \ni \br{\textsf{Y},\textsf{O}} \sim D$, then we have
\begin{align*}
	\bP_D\bs{\textsf{Y}=\theta , \textsf{O}=o} ~=~& \bP_D\bs{\textsf{Y}=\theta} \cdot \bP_D\bs{\textsf{O}=o \mid \textsf{Y}=\theta} ~=~ \pi\br{\theta} \cdot \varepsilon\br{o\mid \theta} \\
	~=~& \bP_D\bs{\textsf{O}=o} \cdot \bP_D\bs{\textsf{Y}=\theta \mid \textsf{O}=o} ~=~ M_D\br{o} \cdot \eta_D\br{\theta \mid o} .
\end{align*}
Thus we can use the pairs $\br{\pi,\varepsilon}$ and $\br{M,\eta}$ interchangeably. We can define the full Bayesian risk and full minimum Bayesian risk in terms of $\br{M,\eta}$ as follows:
\begin{align*}
	\hat{\cR}_{\ell} : \cP\br{\cO} \times \cM\br{\cO,\Theta} \times \cM\br{\cO,\cA} \ni \br{M,\eta,A} ~\mapsto~& \hat{\cR}_{\ell} \br{M,\eta,A} := \\
	& \Ee{\textsf{O} \sim M}{\Ee{\textsf{Y} \sim \eta\br{\textsf{O}}}{\ell\br{\textsf{Y},A\br{\textsf{O}}}}} \in \bR_+ \\
	\underline{\hat{\cR}}_{\ell} : \cP\br{\cO} \times \cM\br{\cO,\Theta} \ni \br{M,\eta} ~\mapsto~& \underline{\hat{\cR}}_{\ell} \br{M,\eta} := \\ 
	& \inf_{A \in \cM\br{\cO,\cA}}{\hat{\cR}_{\ell} \br{M,\eta,A}} \in \bR_+ 
\end{align*}
At this point we note the following facts:
\begin{itemize}
	\item Since 
	\[
	\Ee{\br{\textsf{Y} , \textsf{O}} \sim D}{\ell\br{\textsf{Y},A\br{\textsf{O}}}} = \Ee{\textsf{O} \sim M}{\Ee{\textsf{Y} \sim \eta\br{\textsf{O}}}{\ell\br{\textsf{Y},A\br{\textsf{O}}}}} = \Ee{\textsf{Y} \sim \pi}{\Ee{\textsf{O} \sim \varepsilon\br{\textsf{Y}}}{\ell\br{\textsf{Y},A\br{\textsf{O}}}}}
	\] 
	we have that $\hat{\cR}_{\ell} \br{M,\eta,A} = \cR_{\ell} \br{\pi,\varepsilon,A}$ and $\underline{\hat{\cR}}_{\ell} \br{M,\eta} = \underline{\cR}_{\ell} \br{\pi,\varepsilon}$.
	\item $\hat{\cR}_{\ell} \br{M,\eta,A} = \Ee{\textsf{O} \sim M}{L_\ell \br{\eta\br{\textsf{O}},A\br{\textsf{O}}}}$ and $\underline{\hat{\cR}}_{\ell} \br{M,\eta} = \Ee{\textsf{O} \sim M}{\underline{L}_\ell \br{\eta\br{\textsf{O}}}}$.
\end{itemize}

By using the minimax theorem (\cite{komiya1988elementary}), we obtain the following result that relates the full minimum Bayesian risk and the full minimax risk.
\begin{theorem}
	\label{minimax-bayesian-comparison-theorem}
	Let $\Theta$ to be finite and $\cA$ to be closed, compact, set with $\ell$ a continuous function. Then for all experiments $\varepsilon$,
	\[
	\underline{\cR}_\ell^\star \br{\varepsilon} = \sup_{\pi \in \cP\br{\Theta}}{\underline{\cR}_{\ell} \br{\pi , \varepsilon}} .
	\]
\end{theorem}

\subsection{Repeated and Parallelized Transitions} 
\label{subsec:rep-exp}
Transitions can be \textit{repeated}. For $P,Q \in \cP\br{\cX}$, denote the product distribution by $P \otimes Q$. For any transition $T \in \cM\br{\cX , \cY}$ we denote the \textit{repeated transition} $T_n \in \cM\br{\cX , \cY^n}, n \in \bZ_+$, with,
\begin{equation}
\label{repeated-exp-eq}
T_n \br{x} = T \br{x} \otimes \cdots \otimes T \br{x} = T \br{x}^n , 
\end{equation}
the $n$-fold product of $T \br{x}$. Note that the transition $T_n$ induces a probability space $\cP_{T_n} \br{{\cY}^n}:= \bc{P_x^n := T\br{x}^n \in \cP\br{\cY}^n : x \in \cX}$.


Transitions can also be combined in parallel. If $T_i \in \cM\br{\cX_i , \cY_i} , i \in \bs{k}$, are transitions then denote,
\begin{equation}
\label{paralle-multi-channel-eq}
\bigotimes_{i=1}^k T_i \in \cM\br{\times_{i=1}^k \cX_i , \times_{i=1}^k \cY_i}
\end{equation}
with $\bigotimes_{i=1}^k T_i \br{x} = T_1 \br{x_1} \otimes \cdots \otimes T_n \br{x_n}$. For any transition $T \in \cM\br{\cX , \cY}$ we denote the \textit{parallelized transition} $T_{1:n} \in \cM\br{\cX^n , \cY^n}, n \in \bZ_+$, with, 
\begin{equation}
\label{paralle-same-channel-eq}
T_{1:n} \br{x} = \bigotimes_{i=1}^n T \br{x}.
\end{equation}

\section{Multi-Class Probability Estimation Problem}
\label{sec:mutli-cpe}

We will now consider the special case when the prediction space is $\Theta = \bs{k}$, and the action space is also $\cA = \bs{k}$. In this case, the loss function is written as 
\begin{equation}
\label{multi-cpe-loss}
\ell : [k] \times [k] \ni \br{y,\hat{y}} ~\mapsto~ \ell \br{y,\hat{y}} \in \bR_+ .
\end{equation}
The resulting problem is called the $k$-\textit{class probability estimation (CPE)} problem $\br{\ell,\varepsilon}$ and can be represented by the following transition diagram: 
\begin{equation}
\label{multi-cpe-task-transition}
\multicpetaskext .
\end{equation}

Define $T := A \circ \varepsilon : \bs{k} \rightsquigarrow \bs{k}$. As in the general learning problem, we define the conditional risk as follows (with overloaded notation):
\begin{equation}
\label{conditional-risk-cpe-standard}
\ell : [k] \times \cM\br{[k],[k]} \ni \br{y,T} ~\mapsto~ \ell \br{y,T} := \Ee{\textsf{Y} \sim T \br{y}}{\ell \br{y,\textsf{Y}}} \in \bR_+ ,
\end{equation}
where term inside the expectation is the loss \eqref{multi-cpe-loss} of a random variable $\textsf{Y}$ given that the actual parameter is $y$. In this setting, it is common in the literature to refer the conditional risk as the loss function of the problem (it is also said to be \textit{multi-CPE loss}), and that's why we purposefully use overloaded notation for them. In fact, in Chapter~\ref{cha:mixability} we call the conditional risk as the loss function of prediction with expert advice problem.

The conditional Bayesian risk and the conditional minimum Bayesian risk of this $k$-class probability estimation problem can be written as follows:
\begin{align}
L_\ell : \Delta^k \times \cM\br{[k],[k]} \ni \br{p,T} ~\mapsto~& L_\ell \br{p,T} := \Ee{\textsf{Y} \sim p}{\ell \br{\textsf{Y},T}} \in \bR_+ , \text{ and } \label{conditional-bayes-risk-cpe-standard}\\
\underline{L}_\ell : \Delta^k \ni p ~\mapsto~& \underline{L}_\ell \br{p} := \inf_{T \in \cM\br{[k],[k]}}{L_\ell \br{p,T}} \in \bR_+ \label{min-conditional-bayes-risk-cpe-standard}
\end{align}
respectively.

\section{Binary Experiments}
\label{sec:binary-cpe}

In this section we consider the $k$-class probability estimation problem with $k=2$. Such a problem is known as a binary experiment. Here we review some important notions associated with the binary experiments such as loss, risk, ROC (Receiver Operating Characteristic) curves, information, and distance or divergence between probability distributions.


For consistency with much of the literature, we let $\Theta = \bc{1,2}$, $P = \varepsilon \br{1}$ , and $Q = \varepsilon \br{2}$. Thus a binary experiment can be simply represented $\br{P,Q}$. The densities of $P$ and $Q$ with respect to some third reference distribution $M$ over $\cO$ will be defined by $dP = p dM$ and $dQ = q dM$ respectively. A central statistic in the study of binary experiments and statistical hypothesis testing is the likelihood ratio $dP/dQ$.  

\subsection{Hypothesis Testing}
\label{sec:stat-tests}
In the context of a binary experiment $\br{P,Q}$, a \textit{statistical test} is any function that assigns each instance $o \in \cO$ to either $P$ or $Q$. We will use the labels $1$ and $2$ for $P$ and $Q$ respectively and so a statistical test is any function $r:\cO \rightarrow \bc{1,2}$. The \textit{classification rates} defined by a given test $r$ are:
\begin{enumerate}
	\item True positive rate $\mathrm{TP}_r := P\br{\cO_r^1}$
	\item True negative rate $\mathrm{TN}_r := Q\br{\cO_r^2}$
	\item False positive rate $\mathrm{FP}_r := Q\br{\cO_r^1}$
	\item False negative rate $\mathrm{FN}_r := P\br{\cO_r^2}$
\end{enumerate}
where $\cO_r^1 := \bc{o \in \cO : r\br{o} = 1}$ and $\cO_r^2 := \bc{o \in \cO : r\br{o} = 2}$. Since $P$ and $Q$ are distributions over $\cO = \cO_r^1 \cup \cO_r^2$ and the positive and negative sets are disjoint we have that $\mathrm{TP}+\mathrm{FN} = 1$ and $\mathrm{FP}+\mathrm{TN} = 1$.

For a given binary experiment $\br{P,Q}$, we define the following important quantities or notions associated with a statistical test $r$: 
\begin{itemize}
	\item The \textit{power} $\beta_r := \mathrm{TP}_r$.
	\item The \textit{size} $\alpha_r := \mathrm{FP}_r$.
	\item A test $r$ is said to be the \textit{most powerful} (MP) test of size $\alpha \in [0,1]$ if, $\alpha_r = \alpha$ and for all other tests $r'$ such that $\alpha_{r'} \leq \alpha$ we have $1-\beta_r \leq 1-\beta_{r'}$. 
	\item The \textit{Neyman-Pearson function for the dichotomy $(P,Q)$} (\cite{torgersen1991comparison})
	\[
	\beta\br{\alpha} = \beta\br{\alpha,P,Q} := \sup_{r \in \bc{1,2}^\cO}\bc{\beta_r : \alpha_r \leq \alpha} .
	\]
\end{itemize}

\subsection{ROC curves}
Often, statistical tests are obtained by applying a threshold $\tau_0$ to a real-valued \textit{test statistic} $\tau : \cO \rightarrow \bR$. In this case, the statistical test is $r\br{o} = 2 - \ind{\tau\br{o} \geq \tau_0}$. This leads to parameterized forms of prediction sets $\cO_\tau^y \br{\tau_0} := \cO_{\ind{\tau \geq \tau_0}}^y$ for $y \in \bc{1,2}$, and the classification rates $\mathrm{TP}_\tau \br{\tau_0}$, $\mathrm{FP}_\tau \br{\tau_0}$,
$\mathrm{FN}_\tau \br{\tau_0}$, and $\mathrm{TN}_\tau \br{\tau_0}$ which are defined analogously. By varying the threshold parameter a range of classification rates can be achieved. This observation leads to a well known graphical representation of test statistics known as the \textit{receiver operating characteristic (ROC) curve}. 

An ROC curve for the test statistic $\tau$ is simply a plot of the true positive rate of these classifiers as a function of their false positive rate as the threshold $\tau_0$ varies over $\bR$. Formally,
\[
\mathrm{ROC}(\tau) := \bc{\br{\mathrm{FP}_\tau\br{\tau_0},\mathrm{TP}_\tau\br{\tau_0}} : \tau_0 \in \bR} \subset [0,1]^2 . 
\]
A graphical example of an ROC curve is shown as the solid black line in Figure~\ref{roc-figure}.

The Neyman-Pearson lemma (\cite{neyman1933problem}) shows that for a fixed experiment $\br{P,Q}$, the likelihood ratio $\tau^\star \br{o}=dP/dQ \br{o}$ is the most powerful test statistic for each choice of threshold $\tau_0$. This guarantees that the ROC curve for the likelihood ratio $\tau^\star = dP/dQ$ will lie above, or \textit{dominate}, that of any other test statistic $\tau$ as shown in Figure~\ref{roc-figure}. This is an immediate consequence of the likelihood ratio being the most powerful test since for each false positive rate (or size) $\alpha$ it will have the largest true positive rate (or power) $\beta$ of all tests (\cite{eguchi2001recent}). Thus $\mathrm{ROC}(dP/dQ)$ is the maximal ROC curve. 


\begin{figure}
	\centering
	\begin{tikzpicture}
	\begin{axis}[domain=0:1, xlabel={False Positive Rate (FP)}, ylabel={True Positve Rate (TP)}, axis x line=bottom, axis y line=left, ymin=0, ymax=1.15, xmin=0, xmax=1.15, xtick={1}, ytick={1}]
	\addplot[color=red,dashed] {x} ;
	
	\draw [thick, black] plot [smooth, tension=0.7] coordinates { (0,0) (0.15,0.2) (0.25,0.5) (0.5,0.7) (0.7,0.9) (1,1)};
	\draw [black] (0.3,0.55) node[above] {$\tau$};
	
	\draw [thick, blue] plot [smooth, tension=0.7] coordinates { (0,0) (0.2,0.8) (1,1)};
	\draw [black] (0.3,0.85) node[above] {$\tau^\star$};
	
	\end{axis}
	\end{tikzpicture}
	\caption[ROC curve for an arbitrary statistical test $\tau$, an optimal statistical test $\tau^\star$, and an uninformative statistical test]{ROC curve for (a) an arbitrary statistical test $\tau$	(middle, black curve), (b) an optimal statistical test $\tau^\star$ (top, blue curve), and (c) an uninformative statistical test (dashed red line).\label{roc-figure}}
\end{figure}
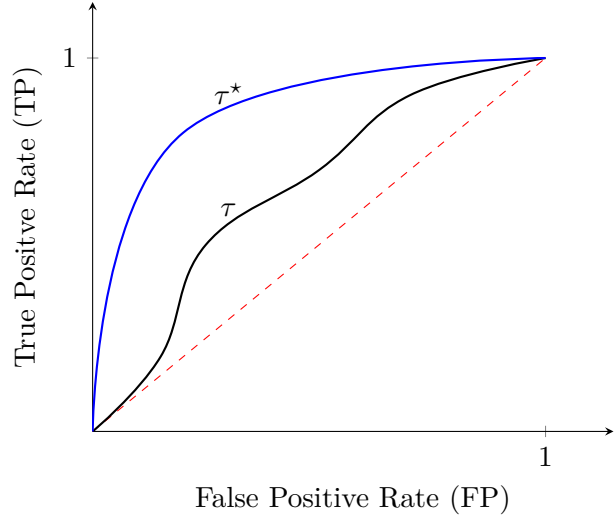

\subsection{$f$-Divergences}

The hardness of the binary classification problem depends on the distinguish-ability of the two probability distributions associated with it. The class of $f$-divergences (\citep{ali1966general,csiszar1972class}) provide a rich set of relations that can be used to measure the separation of the distributions in a binary experiment.

\begin{definition} 
	\label{binary-fdiv-def}
	Let $f: \br{0,\infty} \rightarrow \bR$ be a convex function with $f(1) = 0$. For all distributions $P, Q \in \cP\br{\cO}$ the $f$-divergence between $P$ and $Q$ is,
	\[
	\bI_f \br{P,Q} = \Ee{Q}{f\br{\frac{dP}{dQ}}} = \int_{\cO}^{} {f\br{\frac{dP}{dQ}} dQ}
	\]
	when $P$ is absolutely continuous with respect to $Q$ and equals $\infty$ otherwise.
\end{definition}

The behavior of $f$ is not specified at the endpoints of $\br{0,\infty}$ in the above definition. This is remedied via the perspective transform of $f$, which defines the limiting behavior of $f$. Given convex $f: \br{0,\infty} \rightarrow \bR$ such that $f(1) = 0$ the \textit{$f$-divergence of $P$ from $Q$} is
\begin{equation}
\label{fdivpersdef}
\bI_f \br{P,Q} := \Ee{M}{I_f \br{p,q}} = \Ee{O \sim M}{I_f \br{p\br{O},q\br{O}}} ,
\end{equation}
where $I_f$ is the perspective transform of $f$.

Many commonly used divergences in probability, mathematical statistics and information theory are special cases of $f$-divergences. For example: 
\begin{enumerate}
	\item The Kullback-Leibler divergence (with $\mathrm{KL} \br{u} = u \log{u}$)
	\[
	\bI_{\mathrm{KL}}\br{P , Q} = D\br{P \mid\mid Q} = \Ee{Q}{\frac{dP}{dQ} \log{\frac{dP}{dQ}}}
	\] 
	\item The total variation distance (with $\mathrm{TV} \br{u} = \abs{u-1}$) 
	\[
	\bI_{\mathrm{TV}}\br{P , Q} = d_{\mathrm{TV}}\br{P,Q} = \Ee{Q}{\card*{\frac{dP}{dQ} - 1}} . 
	\]
	Also for general measures $\mu$ and $\nu$ on $\cO$, we define $d_{\mathrm{TV}}\br{\mu,\nu} = \int{\abs{d\mu - d\nu}}$. 
	\item The $\chi^2$-divergence (with $\chi^2 \br{u} = (u-1)^2$) 
	\[
	\bI_{\chi^2}\br{P , Q} = \chi^2 \br{P \mid\mid Q} = \Ee{Q}{\br{\frac{dP}{dQ} - 1}^2}
	\]
	\item The squared Hellinger distance (with $\mathrm{He}^2 \br{u} = (\sqrt{u}-1)^2$)
	\[
	\bI_{\mathrm{He}^2}\br{P , Q} = \mathrm{He}^2 \br{P , Q} = \Ee{Q}{\br{\sqrt{\frac{dP}{dQ}} - 1}^2}
	\]
\end{enumerate}

We note the following properties of $f$-divergences:
\begin{itemize}
	\item $\bI_f \br{P,Q} \geq 0$ for all $P$ and $Q$ by Jensen's inequality
	\item $\bI_f \br{Q,Q} = 0$ for all distributions $Q$ since $f(1) = 0$
	\item $\bI_f \br{P,Q} = \bI_{f^\diamond} \br{Q,P}$ for all distributions $P$ and $Q$ (where $f^\diamond$ is the Csisz{\'a}r dual of $f$) due to the symmetry of the perspective $I_f$. An $f$-divergence is symmetric if $\bI_f \br{P,Q} = \bI_f \br{Q,P}$ for all $P,Q$.
	\item Let $f: \br{0,\infty} \rightarrow \bR$ be a convex function. Then for each $a,b \in \bR$ the convex function $g(x):=f(x)+a x + b$ satisfies $\bI_g \br{P,Q} = \bI_f \br{P,Q}$ for all $P$ and $Q$.
	\item The weak data processing theorem states that for all sets $\cO , \hat{\cO}$, all transitions $T \in \cM \br{\cO , \hat{\cO}}$, all distributions $P, Q \in \cP\br{\cO}$ and all $f$-divergences,
	\[
	\bI_f\br{T \circ P, T \circ Q} \leq \bI_f \br{P, Q}.
	\]
	Intuitively, adding noise never makes it easier to distinguish $P$ and $Q$. 
\end{itemize}

\begin{remark} 
	Here we give a more general definition for $f$-divergence. Let $\phi:[0,\infty)^k \rightarrow \bR$ be a convex function with $\phi\br{\vone_k} = 0$, for some $k \in \bZ_+$. For all experiments $\varepsilon: \bs{k} \rightsquigarrow \cX$ (with the parameter space $\bs{k}$ and the observation space $\cX$) the $f$-divergence of the experiment $\varepsilon$ is, 
	\begin{equation}
	\label{multi-f-div}
	\bI_{\phi} \br{\varepsilon} ~:=~ \Ee{\textsf{X} \sim \varepsilon\br{k}}{\phi\br{t\br{\textsf{X}}}},
	\end{equation}
	where $t:\cX \rightarrow [0,\infty)^k$ is given by 
	\[
	t \br{x} ~:=~ \br{\frac{d \varepsilon\br{x \mid 1}}{d \varepsilon\br{x \mid k}},\cdots,\frac{d \varepsilon\br{x \mid i}}{d \varepsilon\br{x \mid k}},\cdots,1}, \text{ for } x \in \cX .
	\]
	By defining $f\br{t} := \phi\br{t,1}$ (with $k=2$), we recover the binary $f$-divergence (Definition~\ref{binary-fdiv-def}) for binary experiments $\varepsilon:\bs{2} \rightsquigarrow \cX$. 
\end{remark} 

\paragraph*{Integral Representations of $f$-divergences: }
Representation of $f$-divergences and loss functions as weighted average of \textit{primitive} components (in the sense that they can be used to express other measures but themselves cannot be so expressed) is very useful in studying certain geometric properties of them using the weight function behavior. The following restatement of a theorem by \cite{liese2006divergences} provides such a representation for any $f$-divergence (confer \cite{reid2011information} for a proof):
\begin{theorem}
	\label{fdiv-weight-theorem}
	Define $\bar{c} := 1-c$, for $c \in \bs{0,1}$, and let $f$ be convex such that $f(1) = 0$. Then the $f$-divergence between $P$ and $Q$ can be written in a weighted integral form as follows:
	\begin{equation}
	\label{weighted-int-rep-fdiv}
	\bI_f \br{P,Q} = \int_{0}^{1}{\bI_{f_c}\br{P,Q} \gamma_f\br{c} dc} , 
	\end{equation}
	where 
	\begin{equation}
	\label{primitive-fdiv-tent-func}
	f_c (t) = \bar{c} \wedge c - \bar{c} \wedge (c t)
	\end{equation}
	and 
	\begin{equation}
	\label{fdiv-weight}
	\gamma_f\br{c} := \frac{1}{c^3} f'' \br{\frac{\bar{c}}{c}} .
	\end{equation}
\end{theorem}
For $c \in \bs{0,1}$, the term $\bI_{f_c} (P,Q)$ in \eqref{weighted-int-rep-fdiv} is called the \textit{$c$-primitive $f$-divergence} and can be written as 
\begin{align}
\bI_{f_c} (P,Q) ~=~& \int{\bc{\bar{c} \wedge c - \bar{c} \wedge \br{c \frac{dP}{dQ}}} dQ} \label{eq-primitive-f-div-1}\\
~=~& \bar{c} \wedge c - \int{\bar{c} dQ \wedge c dP} \label{eq-primitive-f-div-2} \\
~=~& \bar{c} \wedge c - \frac{1}{2} + \frac{1}{2} \int{\abs{c dP - \bar{c} dQ}} \label{eq-primitive-f-div-3} \\
~=~& \frac{1}{2} d_{\mathrm{TV}} \br{c P , \bar{c} Q} - \frac{1}{2} \abs{1-2 c} , \label{eq-primitive-f-div-4}
\end{align}
where the first equality \eqref{eq-primitive-f-div-1} is due to the definition of $f$-divergence and \eqref{primitive-fdiv-tent-func}, and the third equality \eqref{eq-primitive-f-div-3} is due to the following observation:
\begin{align*}
\int{\abs{p-q}} ~=~& \int_{q \geq p}{q-p} + \int_{q < p}{p-q} \\
~=~& \int_{q \geq p}{q} + \int_{q < p}{p} - \int{p \wedge q} \\
~=~& 1 - \int_{q < p}{q} + 1 - \int_{q \geq p}{p} - \int{p \wedge q} \\
~=~& 2 - 2 - \int{p \wedge q} . \\
\end{align*}
%

\paragraph*{Comparison between $f$-Divergences: }
Consider the problem of maximizing or minimizing an $f$-divergence between two probability measures subject to a finite number of constraints on other $f$-divergences. Given divergences $\bI_f$ and $\bI_{f_i} , i \in [m]$ and nonnegative real numbers
$\alpha_1,\dots,\alpha_m$, let
\begin{align*}
U \br{\alpha_1,\dots,\alpha_m} ~:=~& \sup_{P,Q} \bc{\bI_f \br{P,Q} : \bI_{f_i} \br{P,Q} \leq \alpha_i , \forall{i \in [m]}} , \text{ and } \\
L \br{\alpha_1,\dots,\alpha_m} ~:=~& \inf_{P,Q} \bc{\bI_f \br{P,Q} : \bI_{f_i} \br{P,Q} \geq \alpha_i , \forall{i \in [m]}} ,
\end{align*}
where the probability measures on the right hand sides above range over all possible measurable spaces. These large infinite-dimensional optimization problems can all be reduced to optimization problems over small finite dimensional spaces as shown in the following theorem~\ref{divopttheo}. 

Define 
\begin{align*}
U_n \br{\alpha_1,\dots,\alpha_m} ~:=~& \sup_{P,Q \in \cP\br{\bs{n}}} \bc{\bI_f \br{P,Q} : \bI_{f_i} \br{P,Q} \leq \alpha_i , \forall{i \in [m]}} , \text{ and } \\
L_n \br{\alpha_1,\dots,\alpha_m} ~:=~& \inf_{P,Q \in \cP\br{\bs{n}}} \bc{\bI_f \br{P,Q} : \bI_{f_i} \br{P,Q} \geq \alpha_i , \forall{i \in [m]}} ,
\end{align*}
where $\cP\br{\bs{n}}$ denotes the space of all probability measures defined on the finite set $[n]$. 

\begin{theorem}[\cite{guntuboyina2014sharp}]
	\label{divopttheo}
	For every $\alpha_1,\dots,\alpha_m \geq 0$, we have
	\[
	U \br{\alpha_1,\dots,\alpha_m} = U_{m+2} \br{\alpha_1,\dots,\alpha_m}
	\]
	Further if $\alpha_1,\dots,\alpha_m$ are all finite, then
	\[
	L \br{\alpha_1,\dots,\alpha_m} = L_{m+2} \br{\alpha_1,\dots,\alpha_m} .
	\]	
	Suppose that $\bI_f$ is an arbitrary $f$-divergence and that all divergences $\bI_{f_i} , i  \in [m]$ are $c$-primitive $f$-divergences \eqref{eq-primitive-f-div-4}. Then
	\[
	L \br{\alpha_1,\dots,\alpha_m} = L_{m+1} \br{\alpha_1,\dots,\alpha_m} .
	\]
\end{theorem}

Now we introduce a closely related concept - namely the \textit{joint range}.

\begin{definition}[Joint Range]
	Consider two $f$-divergences $\bI_f \br{P,Q}$ and $\bI_g \br{P,Q}$. Their joint
	range is a subset of $\bR^2$ defined by
	\begin{align*}
	J ~:=~& \bc{\br{\bI_f \br{P,Q} , \bI_g \br{P,Q}}: P,Q \in \cP\br{\cX} \text{ where } \cX \text{ is some measurable space}} , \\
	{J}_k ~:=~& \bc{\br{\bI_f \br{P,Q} , \bI_g \br{P,Q}}: P,Q \in \cP\br{[k]}} .
	\end{align*}
\end{definition}

The region ${J}$ seems difficult to characterize since we need to consider $P,Q$ over all measurable spaces; on the other hand, the region ${J}_k$ for small $k$ is easy to obtain. The following theorem relates these two regions (${J}$ and ${J}_k$). 

\begin{theorem}[\cite{harremoes2011pairs}]
	\label{joint-range-theo}
	${J} = \text{conv}\br{{J}_2} .$
\end{theorem}
By Theorem~\ref{joint-range-theo}, the region $J$ is no more than the convex hull of ${J}_2$. In certain cases, it is easy to obtain a parametric formula of ${J}_2$. In those cases, we can systematically prove several important inequalities between two $f$-divergences via their joint range. For example using the joint range between the total variation and Hellinger divergence, it can be shown that (\citep{Tsybakov2009nonparbook,yurilecnotes}):
\begin{equation}
\label{joint-range-hell-var}
\frac{1}{2} \mathrm{He}^2\br{P,Q} ~\leq~ d_\mathrm{TV}\br{P,Q} ~\leq~ \mathrm{He}\br{P,Q} \sqrt{1 - \frac{\mathrm{He}^2\br{P,Q}}{4}} .
\end{equation}
We extend the above result to the $c$-primitive $f$-divergence as follows: 
\begin{equation}
\label{joint-range-hell-primc}
\bI_{f_c}\br{P,Q} ~\leq~ \br{c \wedge \bar{c}} \cdot \mathrm{He}\br{P,Q} \sqrt{1 - \frac{\mathrm{He}^2\br{P,Q}}{4}} .
\end{equation}
We use a mathematical software to plot (see Figure~\ref{joint-hell-fc}) the joint range between the $c$-primitive $f$-divergence and the Hellinger divergence which is given by the convex hull of
\[
J_2 := \bc{2 \br{1-\sqrt{p q} -\sqrt{\bar{p}\bar{q}}}, \frac{1}{2}\br{\abs{c p - \bar{c} q} + \abs{c \bar{p} - \bar{c} \bar{q}} - \abs{2 c - 1}} : p,q \in \bs{0,1}} .
\]
Then using this joint range, we verify that the bound given in \eqref{joint-range-hell-primc} is indeed true. We also note that the bound in \eqref{joint-range-hell-primc} is not tight but sufficient for our purposes (for analysing the hardness of the cost-sensitive classification problem in Chapter~\ref{cha:markov}). 

\begin{figure}
	\centering
	\includegraphics[height=0.7\linewidth]{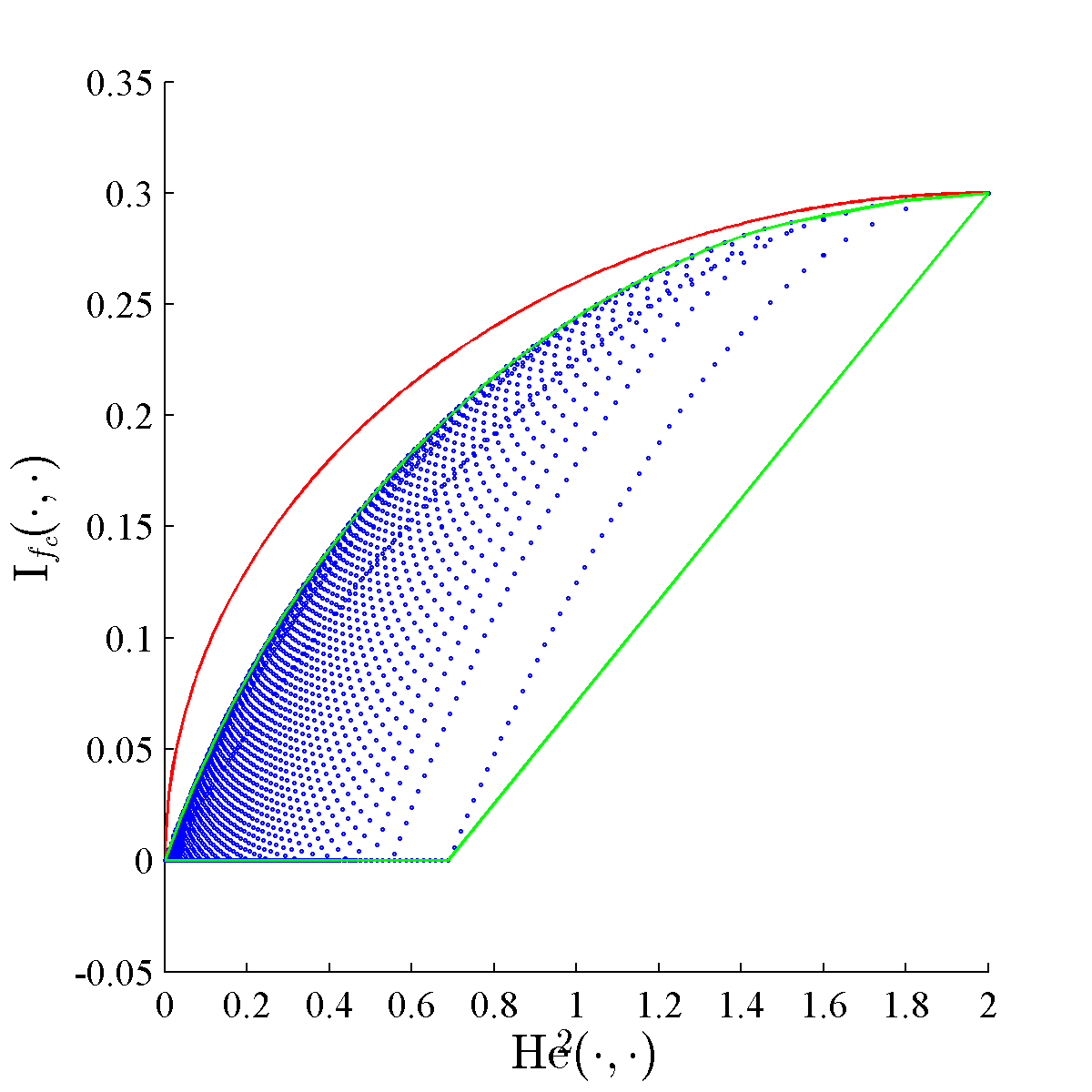}
	\caption[Joint range of Hellinger distance and a $c$-primitive $f$-divergence]{Joint range ($J_2$) of Hellinger distance and a $c$-primitive $f$-divergence (c=0.7) (\textcolor{blue}{$\cdots$}), convex hull of $J_2$ (\textcolor{green}{---}), and a parametric curve $\br{c \wedge \bar{c}} \cdot \mathrm{He} \br{P,Q} \sqrt{1 - \frac{\mathrm{He}^2 \br{P,Q}}{4}}$ (\textcolor{red}{---}). \label{joint-hell-fc}}
\end{figure}

\paragraph*{Sub-additive $f$-Divergences: }
Some $f$-divergences satisfy the sub-additivity property, which will be useful in analyzing the hardness of learning problems with repeated experiments (samples). The following lemma shows that both total variation and squared Hellinger divergences satisfy this property.  
\begin{lemma}
	\label{sub-add-f-div}
	For all collections of distributions $P_i , Q_i \in \cP \br{\cO_i}$, $i \in [k]$
	\[
	d_{\mathrm{TV}} \br{\bigotimes_{i=1}^k P_i , \bigotimes_{i=1}^k Q_i} ~\leq~ \sum_{i=1}^{k}{d_{\mathrm{TV}} \br{P_i ,Q_i}} ,
	\]	
	and 
	\[
	\mathrm{He}^2 \br{\bigotimes_{i=1}^k P_i , \bigotimes_{i=1}^k Q_i} ~\leq~ \sum_{i=1}^{k}{\mathrm{He}^2 \br{P_i ,Q_i}} .
	\]
\end{lemma}
\begin{proof}
	Firstly $\br{P,Q} \mapsto d_{\mathrm{TV}} \br{P,Q}$ is a metric. Thus 
	\begin{align*}
	& d_{\mathrm{TV}}\br{\bigotimes_{i=1}^k P_i , \bigotimes_{i=1}^k Q_i} \\
	~=~& d_{\mathrm{TV}} \br{P_1 \otimes \br{\bigotimes_{i=2}^k P_i} , Q_1 \otimes \br{\bigotimes_{i=2}^k Q_i}} \\ 
	~\leq~& d_{\mathrm{TV}} \br{P_1 \otimes \br{\bigotimes_{i=2}^k P_i} , Q_1 \otimes \br{\bigotimes_{i=2}^k P_i}} + d_{\mathrm{TV}} \br{Q_1 \otimes \br{\bigotimes_{i=2}^k P_i} , Q_1 \otimes \br{\bigotimes_{i=2}^k Q_i}} \\ 
	~=~& d_{\mathrm{TV}} \br{P_1 , Q_1} + d_{\mathrm{TV}} \br{\bigotimes_{i=2}^k P_i , \bigotimes_{i=2}^k Q_i} ,
	\end{align*}
	where the second line follows by definition, the third follows from the triangle inequality and the forth is easily verified from the definition of $d_\mathrm{TV}\br{\cdot,\cdot}$. To complete the proof proceed inductively. 
	
	Let $\mu$ be a product measure on $\cO_1 \times \cO_2$, written as $\mu = \mu_1 \otimes \mu_2$, where $\mu_i := \mu \circ \pi_i$ denotes the image measure of the projection $\pi_i:\bR^2 \ni \br{x_1,x_2} \mapsto \pi_i\br{x_1,x_2} = x_i$ w.r.t.\ $\mu$. Also let $P = P_1 \otimes P_2$, and $Q = Q_1 \otimes Q_2$. Define $p := \frac{dP}{d\mu}$, $q := \frac{dQ}{d\mu}$, $p_1 := \frac{dP_1}{d\mu_1}$, $p_2 := \frac{dP_2}{d\mu_2}$, $q_1 := \frac{dQ_1}{d\mu_1}$, and $q_2 := \frac{dQ_2}{d\mu_2}$. Then, by Tonelli's theorem, 
	\begin{align*}
	1 - \frac{1}{2} \mathrm{He}^2 \br{P,Q} ~=~& \int{\sqrt{pq} d\mu} \\
	~=~& \int{\sqrt{p_1 q_1} d\mu_1} \cdot \int{\sqrt{p_2 q_2} d\mu_2} \\
	~=~& \br{1 - \frac{1}{2} \mathrm{He}^2 \br{P_1,Q_1}} \cdot \br{1 - \frac{1}{2} \mathrm{He}^2 \br{P_2,Q_2}} .
	\end{align*}
	Thus we have 
	\begin{align*}
	\mathrm{He}^2 \br{P,Q} ~=~& 2 - 2 \br{1 - \frac{1}{2} \mathrm{He}^2 \br{P_1,Q_1}} \cdot \br{1 - \frac{1}{2} \mathrm{He}^2 \br{P_2,Q_2}} \\
	~=~& \mathrm{He}^2 \br{P_1,Q_1} + \mathrm{He}^2 \br{P_2,Q_2} - \frac{1}{2} \mathrm{He}^2 \br{P_1,Q_1} \mathrm{He}^2 \br{P_2,Q_2} \\
	~\leq~& \mathrm{He}^2 \br{P_1,Q_1} + \mathrm{He}^2 \br{P_2,Q_2} .
	\end{align*}
	To complete the proof proceed the above process iteratively. 
\end{proof}

\chapter{Asymmetric Learning Problems}
\label{cha:markov}

The central problem of this chapter is the cost-sensitive binary classification problem, where different costs are associated with different types of mistakes. Several important machine learning applications such as medical decision making, targeted marketing, and intrusion detection can be formalized as cost-sensitive classification setup (\cite{abe2004iterative}). 

The chapter proceeds as follows. In section~\ref{sec:preliminary} we show that the abstract language of transitions introduced in chapter~\ref{cha:decision}, is general enough to capture many of the existing practical problems in statistics and machine learning including the cost-sensitive classification problem. Then in section~\ref{sec:hard-cost-sense} we study the hardness of the cost-sensitive classification problem by extending the standard minimax lower bound of balanced binary classification problem (due to \cite{massart2006risk}) to cost-sensitive classification problem.

In section~\ref{sec:const-learn} we study the hardness of the constrained learning problem (specifically constrained cost-sensitive classification), which naturally leads us to a detailed investigation of strong data processing inequalities. After reviewing the known results in strong data processing inequalities, we make some novel progress in the direction of strong data processing inequalities for binary symmetric channels. We also extend the well-known contraction coefficient theorem (\cite{cohen1993relative}) for total variational divergence to $c$-primitive $f$-divergences.

Finally in section~\ref{sec:cost-privacy} we study the local privacy requirement as a form of constraint on learning problem. We review the decision theoretic reduction of the local privacy requirement, and based on that we propose a prioritized (cost-sensitive) privacy definition. 

\section{Preliminaries and Background} 
\label{sec:preliminary}

\paragraph*{General Learning Task:} 

Consider the \emph{General Learning Task} represented by the following transition diagram:
\begin{equation}
\label{rep-experiment-general-task}
\learningtask
\end{equation}
where $\Theta$, $\cO$, and $\cA$ are \emph{parameter}, \emph{observation}, and \emph{action} spaces respectively. The transitions $\varepsilon_n$ and $A$ denote \emph{repeated experiment of $\varepsilon:\Theta \rightsquigarrow \cO$} and \emph{algorithm} respectively. Note that the repeated experiment $\varepsilon_n$ induces the class of probability measures given by $\cP_{\varepsilon_n}\br{\cO^n} := \bc{\varepsilon_n \br{\theta} := \varepsilon \br{\theta}^n : \theta \in \Theta}$ (see Section~\ref{subsec:rep-exp}). We recall the following objects introduced in Chapter~\ref{cha:decision} :
\begin{align*}
\text{Loss} \quad & \ell:\Theta \times \cA \to \bR \\
\text{Regret} \quad & \Delta \ell \br{\theta,a} := \ell \br{\theta,a} - \inf_{a' \in \cA}{\ell \br{\theta,a'}} \\
\text{Full Risk} \quad & R_{\ell} \br{\varepsilon_n,\theta,A} := \Ee{\textsf{O}_1^n \sim \varepsilon_n \br{\theta}}{\Ee{a \sim A(\textsf{O}_1^n)}{\ell \br{\theta,a}}} \\
& R_{\Delta \ell} \br{\varepsilon_n,\theta,A} := \Ee{\textsf{O}_1^n \sim \varepsilon_n \br{\theta}}{\Ee{a \sim A(\textsf{O}_1^n)}{\Delta \ell \br{\theta,a}}} \\
\text{Full Minimax Risk} \quad & \underline{R}_{\ell}^\star \br{\varepsilon_n} := \inf_{A}{\sup_{\theta \in \Theta}{R_{\ell} \br{\varepsilon_n,\theta,A}}} \\
& \underline{R}_{\Delta \ell}^\star \br{\varepsilon_n} := \inf_{A}{\sup_{\theta \in \Theta}{R_{\Delta \ell} \br{\varepsilon_n,\theta,A}}} .
\end{align*}
One needs to carefully distinguish between the risk (and related notions) in terms of loss and regret based on the subscript (see Remark~\ref{loss-regret-based-risk}). The general learning task is compactly denoted by the tuple $\br{\ell,\varepsilon_n}$. 

In order to demonstrate the generality of the language of transitions, below we discuss some specific instantiations (supervised learning, multi-class probability estimation, binary classification, and parameter estimation) of this general learning task.

\paragraph*{Supervised Learning Problem:}

Let $\cX \times \cY$ be a measurable space, and let $D$ be an unknown joint probability measure on $\cX \times \cY$. The set $\cX$ is called the \textit{instance space}, the set $\cY$ the \textit{outcome space}. Let $S = \bc{\br{\textsf{X}_i , \textsf{Y}_i}}_{i=1}^m \in \br{\cX \times \cY}^m$ be a finite training sample, where each pair $\br{\textsf{X}_i , \textsf{Y}_i}$ is generated independently according to the unknown probability measure $D$. Then the goal of a learning algorithm is to find a function $f:\cX \rightarrow \cY$ which given a new instance $x \in \cX$, predicts its label to be $\hat{y} = f\br{x}$.

Here we rely on the fundamental assumption that both training and future (test) data are generated by the same fixed underlying probability measure $D$, which, although unknown, allows us to infer from training data to future data and therefore to generalize.

In order to measure the performance of a learning algorithm, we define an \textit{error function} $d:\cY \times \cY \rightarrow \bR_+$,
where $d\br{y,\hat{y}}$ quantifies the discrepancy between the predicted value $\hat{y}$ and the actual value $y$. The performance of any function $f:\cX \rightarrow \cY$ is then measured in terms of its \textit{generalization error}, which is defined as the expected error:
\begin{equation}
\label{generalization-error-eq}
\mathrm{er}_d \br{f,D} ~:=~ \Ee{\br{\textsf{X},\textsf{Y}} \sim D}{d\br{\textsf{Y},f\br{\textsf{X}}}} ,
\end{equation}
where the expectation is taken with respect to the probability measure $D$ on the data $\br{\textsf{X},\textsf{Y}}$. The best estimate $f_D^\star \in \cY^{\cX}$ is therefore the one for which the generalization error is as small as possible, that is,
\begin{equation}
\label{optimal-hypothesis-all}
f_D^\star ~:=~ \argmin_{f \in \cY^{\cX}}{\mathrm{er}_d \br{f,D}} .
\end{equation}
The function $f_D^\star$ is called the target hypothesis.

In order to avoid functions which over-fit the training sample and do not generalize well on the test data, one usually imposes constraints on the function $f$. One way to impose constraints is by restricting the possible choices of functions to a fixed class of functions from which the learning algorithm chooses its hypothesis. This function class is called the \textit{hypothesis class}. Given a fixed hypothesis class $\cF \subseteq \cY^\cX$, the goal of a learning algorithm is thus to choose the hypothesis function $f^\star$ in $\cF$ which has the smallest generalization error on data drawn according to the underlying probability measure $D$,
\begin{equation}
\label{optimal-hypothesis-cF}
f_{D,\cF}^\star ~:=~ \argmin_{f \in \cF}{\mathrm{er}_d \br{f,D}} .
\end{equation}
We will assume in the following that such an $f_{D,\cF}^\star$ exists.

The supervised learning problem can be derived from the general learning task \eqref{rep-experiment-general-task} with the following instantiation: 
\begin{itemize}
	\item the observation space is $\cO = \cX \times \cY$, where $\cX \subseteq \bR^d$,
	\item the action space is $\cA = \cF \subseteq \cY^\cX$, 
	\item the learning algorithm is $A = \hat{f}$, and
	\item the loss function is 
	\[
	\ell_d : \Theta \times \cF \ni \br{\theta,f} \mapsto \ell_d \br{\theta,f} := \mathrm{er}_d \br{f,\varepsilon\br{\theta}} \in \bR_+ ,
	\] 
	where $\varepsilon\br{\theta}$ is the probability measure associated with the parameter $\theta \in \Theta$. One needs to carefully distinguish between the error function $d: \cY \times \cY \rightarrow \bR$ which acts on the observation space, and the loss function $\ell_d: \Theta \times \cA \rightarrow \bR$ which acts on the parameter and decision spaces.
\end{itemize}
Then the transition diagram for this supervised learning problem $\br{\ell_d , \varepsilon_n}$ is
\begin{equation}
\label{sup-learning-transition}
\supervisedlearningtask .
\end{equation}

\paragraph*{Binary Classification:}

When $\cY = \bc{-1,1}$, the supervised learning task \eqref{sup-learning-transition} is called binary classification, which is a central problem in machine learning (\cite{devroye2013probabilistic}). A common error function for binary classification is simply the zero-one error defined by $d_{0-1} \br{y,\hat{y}} = \ind{\hat{y} \neq y}$. In this case the generalization error of a classifier $f:\cX \rightarrow \bc{-1,1}$ w.r.t.\ a probability measure $D$ is simply the probability that it predicts the wrong label on a randomly drawn example:
\[
\mathrm{er}_{d_{0-1}} \br{f,D} ~:=~ \Ee{\br{\textsf{X},\textsf{Y}} \sim D}{d_{0-1}\br{\textsf{Y},f\br{\textsf{X}}}} ~=~ \Pp{\br{\textsf{X},\textsf{Y}}\sim D}{f\br{\textsf{X}} \neq \textsf{Y}} .
\]
The optimal error over all possible classifiers $f:\cX \rightarrow \bc{-1,1}$ for a given probability measure $D$ is called the \textit{Bayes error} (minimum generalization error) associated with $D$:
\begin{equation}
\label{bayes-error-eq}
\underline{\mathrm{er}}_{d_{0-1}} \br{D} ~:=~  \inf_{f \in \bc{-1,1}^\cX}{\mathrm{er}_{d_{0-1}} \br{f,D}} .
\end{equation}
It is easily verified that, if $\eta_D\br{x}$ is defined as the conditional probability (under $D$) of a positive label given
$x$, $\eta_D\br{x} = \bP_D\bs{\textsf{Y}=1 \mid \textsf{X}=x}$, then the classifier $f_D^\star:\cX \rightarrow \bc{-1,1}$ given by
\[
f_D^\star \br{x} = \begin{cases}
1 & \text{if } \eta_D\br{x} \geq 1/2 \\
-1 & \text{otherwise}
\end{cases}
\]
achieves the Bayes error. Such a classifier is termed a \textit{Bayes classifier}. In general, $\eta_D$ is unknown so the above classifier cannot be constructed directly. 

By defining $\ell_{d_{0-1}} : \Theta \times \cF \ni \br{\theta , f} \mapsto \ell_{d_{0-1}} \br{\theta , f} := \mathrm{er}_{d_{0-1}} \br{f,\varepsilon\br{\theta}} \in \bR_+$, the binary classification problem $\br{\ell_{d_{0-1}} , \varepsilon_n}$ can be represented by the following transition diagram:
\begin{equation}
\label{sup-classification-transition}
\supervisedbinaryclassificationtask .
\end{equation}
Note that the repeated experiment $\varepsilon_n$ above induces the class of probability measures given by $\cP_{\varepsilon_n} \br{\br{\cX \times \bc{-1,1}}^n} := \bc{\varepsilon_n\br{\theta} := \varepsilon\br{\theta}^n \in \cP \br{\cX \times \bc{-1,1}}^n : \theta \in \Theta}$. Using the Bayes rule, the distribution $\bP_{\varepsilon\br{\theta}}$ can be decomposed as follows: 
\[ 
\bP_{\varepsilon\br{\theta}} \bs{\textsf{X}=x,\textsf{Y}=1} = \bP_{\varepsilon\br{\theta}} \bs{\textsf{X}=x} \cdot \bP_{\varepsilon\br{\theta}} \bs{\textsf{Y}=1 \mid \textsf{X}=x} = M_{\varepsilon\br{\theta}} \br{x} \cdot \eta_{\varepsilon\br{\theta}} \br{x} , 
\] 
where $M_{\varepsilon\br{\theta}} \br{x} := \bP_{\varepsilon\br{\theta}} \bs{\textsf{X}=x}$ and $\eta_{\varepsilon\br{\theta}} \br{x} := \bP_{\varepsilon\br{\theta}} \bs{\textsf{Y}=1 \mid \textsf{X}=x}$. For simplicity we will write $\bP_{\varepsilon\br{\theta}}$, $M_{\varepsilon\br{\theta}}$, $\eta_{\varepsilon\br{\theta}}$, and $f_{\varepsilon\br{\theta}}^\star$ as $\bP_\theta$, $M_{\theta}$, $\eta_{\theta}$, and $f_{\theta}^\star$ respectively. 

\paragraph*{Cost-sensitive Binary Classification:}
Suppose we are given gene expression profiles for some number of patients, together with labels for these patients indicating whether or not they had a certain form of a disease. We want to design a learning algorithm which automatically recognizes the diseased patient based on the gene expression profile of a patient. In this case, there are different costs associated with different types of mistakes (the health risk for a false label ``no'' is much higher than for a false ``yes''), and the \textit{cost-sensitive error function}  (for $c \in (0,1)$) can be used to capture this:
\[
d_c : \cY \times \cY \ni \br{y,\hat{y}} \mapsto d_c \br{y,\hat{y}} := \ind{\hat{y} \neq y} \cdot \bc{\bar{c} \cdot \ind{y=1} + c \cdot \ind{y=-1}} ,
\]
where $\bar{c}:= 1-c$. Then the performance measure (loss function) associated with the above cost-sensitive error function is given by 
\[
\ell_{d_{c}} : \Theta \times \cF \ni \br{\theta , f} \mapsto \ell_{d_{c}} \br{\theta , f} := \mathrm{er}_{d_{c}} \br{f,\varepsilon\br{\theta}} \in \bR_+ , 
\] 
where
\begin{align*}
\mathrm{er}_{d_{c}} \br{f,\varepsilon\br{\theta}} ~:=~& \Ee{\br{\textsf{X},\textsf{Y}} \sim \varepsilon\br{\theta}}{d_{c}\br{\textsf{Y},f\br{\textsf{X}}}} \\
~=~& \Ee{\br{\textsf{X},\textsf{Y}} \sim \varepsilon\br{\theta}}{\ind{f\br{\textsf{X}} \neq \textsf{Y}} \cdot \bc{\bar{c} \cdot \ind{\textsf{Y}=1} + c \cdot \ind{\textsf{Y}=-1}}} .
\end{align*}

For any $\eta:\cX \rightarrow \bs{0,1}$, and $f:\cX \rightarrow \bc{-1,1}$, define the conditional generalization error (given $x \in \cX$) as 
\begin{align*}
\mathrm{er}_{d_{c}} \br{f,\eta ; x} ~:=~& \Ee{\textsf{Y} \sim \eta\br{x}}{d_{c}\br{\textsf{Y},f\br{x}}} \\
~=~& \bar{c} \cdot \eta\br{x} \cdot \ind{f\br{x} \neq 1} + c \cdot \bar{\eta\br{x}} \cdot \ind{f\br{x} \neq -1} , 
\end{align*}
where $\bar{\eta\br{x}} := 1- \eta\br{x}$. Then $\mathrm{er}_{d_{c}} \br{f,\eta ; x}$ is minimized by 
\begin{align*}
f^\star\br{x} ~:=~& \argmin_{f \in \bc{-1,1}^\cX} {\Ee{\textsf{Y} \sim \eta (x)}{d_{c}\br{\textsf{Y},f\br{x}}}} \\
~=~& \mathrm{sign}\br{\bar{c} \cdot \eta\br{x} - c \cdot \bar{\eta\br{x}}} \\
~=~& \mathrm{sign}\br{\eta\br{x} - c} ,
\end{align*}
since $\mathrm{er}_{d_{c}} \br{f^\star,\eta ; x} = \bar{c} \cdot \eta\br{x} \wedge c \cdot \bar{\eta\br{x}}$. In order to find the optimal classifier for each $\theta \in \Theta$ (associated joint probability measure $\varepsilon\br{\theta}$ on $\cX \times \bc{-1,1}$) w.r.t.\ the cost-sensitive loss function, we note that
\begin{align*}
\inf_{f \in \bc{-1,1}^\cX} \ell_{d_{c}} \br{\theta,f} ~=~&
\inf_{f \in \bc{-1,1}^\cX} \mathrm{er}_{d_{c}} \br{f,\varepsilon\br{\theta}} \\ 
~=~& \inf_{f \in \bc{-1,1}^\cX} \Ee{\textsf{X} \sim M_\theta}{\Ee{\textsf{Y} \sim \eta_\theta \br{\textsf{X}}}{d_{c}\br{\textsf{Y},f\br{\textsf{X}}}}} \\
~=~& \Ee{\textsf{X} \sim M_\theta}{\inf_{f \in \bc{-1,1}^\cX}{\Ee{\textsf{Y} \sim \eta_\theta \br{\textsf{X}}}{d_{c}\br{\textsf{Y},f\br{\textsf{X}}}}}} \\
~=~& \ell_{d_{c}} \br{\theta,f_\theta^\star}, 
\end{align*}
where $M_{\theta} \br{x} := \bP_{\varepsilon\br{\theta}} \bs{\textsf{X}=x}$, $\eta_{\theta} \br{x} := \bP_{\varepsilon\br{\theta}} \bs{\textsf{Y}=1 \mid \textsf{X}=x}$, and $f_\theta^\star$ is given by
\begin{equation}
\label{bayes-classifier-lc}
f_\theta^\star\br{x} := \begin{cases}
1, & \text{if } \eta_\theta(x) \geq c \\
-1, & \text{otherwise}
\end{cases} .
\end{equation}

We instantiate the following objects related to the cost-sensitive classification problem
\begin{align*}
\text{Regret} \quad & \Delta \ell_{d_{c}} \br{\theta,f} := \ell_{d_{c}} \br{\theta,f} - \ell_{d_{c}} \br{\theta,f_\theta^\star} \\
\text{Full Risk} \quad & \cR_{\Delta \ell_{d_{c}}} \br{\varepsilon_n,\theta,\hat{f}} := \Ee{\bc{\br{\textsf{X}_i,\textsf{Y}_i}}_{i=1}^n \sim \varepsilon_n\br{\theta}}{\Ee{f \sim \hat{f}\br{\bc{\br{\textsf{X}_i,\textsf{Y}_i}}_{i=1}^n}}{\Delta \ell_{d_c} \br{\theta,f}}} \\
\text{Full Minimax Risk} \quad & \underline{\cR}_{\Delta \ell_{d_{c}}}^\star \br{\varepsilon_n} := \inf_{\hat{f}}{\sup_{\theta \in \Theta}{\cR_{\Delta \ell_{d_{c}}} \br{\varepsilon_n,\theta,\hat{f}}}} .
\end{align*}
The following lemma from \cite{scott2012calibrated} will be used later.
\begin{lemma}[\cite{scott2012calibrated}]
	\label{cost-sensitive-lemma}
	Consider the binary classification problem \eqref{sup-classification-transition}. For any $f \in \cF$ and $c \in (0,1)$,
	\[
	\Delta \ell_{d_{c}} \br{\theta,f} ~=~ \frac{1}{2} \cdot \Ee{\textsf{X}\sim M_\theta}{\abs{\eta_\theta \br{\textsf{X}} - c} \cdot \abs{f\br{\textsf{X}} - f_\theta^\star \br{\textsf{X}}}} ,
	\]
	where $f_\theta^\star$ is given by \eqref{bayes-classifier-lc}.
\end{lemma}
\begin{proof}
	Consider a fixed $x \in \cX$. Recall that 
	\[
	f_\theta^\star\br{x} = \argmin_{f \in \bc{-1,1}^\cX} {\Ee{\textsf{Y} \sim \eta_\theta(x)}{d_{c}\br{\textsf{Y},f\br{x}}}} = \mathrm{sign}\br{\eta_\theta \br{x} - c} .
	\]
	Therefore $\inf_{f \in \bc{-1,1}^\cX}{\Ee{\textsf{Y} \sim \eta_\theta(x)}{d_{c}\br{\textsf{Y},f\br{x}}}} = \Ee{\textsf{Y} \sim \eta_\theta(x)}{d_{c}\br{\textsf{Y},f_\theta^\star\br{x}}}$.
	This implies
	\begin{align*}
	& \Ee{\textsf{Y} \sim \eta_\theta(x)}{d_{c}\br{\textsf{Y},f\br{x}}} - \Ee{\textsf{Y} \sim \eta_\theta(x)}{d_{c}\br{\textsf{Y},f_\theta^\star\br{x}}} \\
	~=~& \bar{c} \, \eta_\theta \br{x} \ind{f(x) \neq 1} + c \, \bar{\eta_\theta \br{x}} \ind{f(x) \neq -1} \\
	& - \bc{\bar{c} \, \eta_\theta \br{x} \ind{f_\theta^\star(x) \neq 1} + c \, \bar{\eta_\theta \br{x}} \ind{f_\theta^\star(x) \neq -1}} \\
	~=~& \ind{f(x) \neq f_\theta^\star(x)} \card{\eta_\theta \br{x} - c} \\
	~=~& \frac{1}{2} \cdot \card{f(x) - f_\theta^\star(x)} \cdot \card{\eta_\theta \br{x} - c} .
	\end{align*}
	Then the proof is completed by noting that
	\begin{align*}
	\ell_{d_{c}} \br{\theta,f} - \ell_{d_{c}} \br{\theta,f_\theta^\star} ~=~& \Ee{\textsf{X} \sim M_\theta}{\Ee{\textsf{Y} \sim \eta_\theta\br{\textsf{X}}}{d_{c}\br{\textsf{Y},f\br{\textsf{X}}}} - \Ee{\textsf{Y} \sim \eta_\theta\br{\textsf{X}}}{d_{c}\br{\textsf{Y},f_\theta^\star\br{\textsf{X}}}}} \\
	~=~& \frac{1}{2} \, \Ee{\textsf{X} \sim M_\theta}{\card{f\br{\textsf{X}} - f_\theta^\star\br{\textsf{X}}} \cdot \card{\eta_\theta \br{\textsf{X}} - c}} .
	\end{align*}
\end{proof}

\paragraph*{Parameter Estimation Problem:}

The main goal of a parameter problem is to accurately \textit{reconstruct the parameters} of the original distribution from which the data is generated, using the loss function of the type $\rho:\Theta\times\Theta\rightarrow\bR$. This problem is represented by the following transition diagram (with $\cA = \Theta$, and $A = \hat{\theta}$):
\begin{equation}
\label{parameter-estimation-task-transition}
\parametertask .
\end{equation}

Let $\theta: \cP\br{\cO} \rightarrow \Theta$ denote a function defined on $\cP\br{\cO}$, that is, a mapping $P \mapsto \theta\br{P}$. The goal of the algorithm $\hat{\theta}$ is to estimate the parameter $\theta \br{P}$ based on observations $\textsf{O}_1^n$ drawn from the (unknown) distribution $P$. In certain cases, the parameter $\theta \br{P}$ uniquely determines the underlying distribution; for example, in the case of mean ($\theta$) estimation problem from the normal distribution family $\cP \br{\cO} = \bc{\cN\br{\theta,\Sigma}:\theta \in \bR^d}$ with known covariance matrix $\Sigma$, the parameter mapping $\theta \br{P} = \Ee{\textsf{O} \sim P}{\textsf{O}}$ uniquely determines distributions in $\cP\br{\cO}$. In other scenarios, however, $\theta$ does not uniquely determine the distribution (confer \cite{duchilecnotes} for general treatment with this broader viewpoint of estimating functions of distributions). In this chapter we consider the one-to-one function $P \mapsto \theta\br{P}$.

Observe that the class of probability measures induced by the repeated experiment $\varepsilon_n$ is written as $\cP_{\varepsilon_n} \br{\cO^n} := \bc{\varepsilon_n \br{\theta} := \varepsilon \br{\theta}^n \in \cP\br{\cO}^n : \theta \in \Theta}$. Let $\rho: \Theta \times \Theta \rightarrow \bR$ be a pseudo metric (that is, it satisfies symmetry and the triangle inequality) on $\Theta$. Then the minimax risk of this problem is defined as 
\begin{equation}
\label{minimax-risk-parameter}
\underline{\cR}_\rho^\star \br{\varepsilon_n} ~:=~ \inf_{\hat{\theta}} \sup_{\theta \in \Theta} \Ee{\textsf{O}_1^n \sim \varepsilon_n \br{\theta}}{\Ee{\tilde{\theta} \sim \hat{\theta}\br{\textsf{O}_1^n}}{\rho\br{\theta,\tilde{\theta}}}} .
\end{equation}

\paragraph*{Hardness of a Problem via minimax lower bounds:}

Understanding the hardness or fundamental limits of a learning problem is important for practice for the following reasons:
\begin{itemize}
	\item They give an estimate on the number of samples required for a good performance of a learning algorithm.
	\item They give an intuition about the quantities and structural properties which are essential for a learning process and therefore about which problems are inherently easier than others.
	\item They quantify the influence of parameters and indicate what prior knowledge is relevant in a learning setting and therefore they guide the analysis, design, and improvement of learning algorithms. 
\end{itemize}
Note that the ``hardness'' here corresponds to lower bounds on sample complexity (and not computational complexity). We demonstrate the hardness of a learning problem \eqref{rep-experiment-general-task} (and the instantiations of it) by obtaining lower bounds for the minimax risk $\underline{R}_\ell^\star \br{\varepsilon_n}$ of it.

In section~\ref{sec:hard-cost-sense} we review and extend techniques due to \cite{le2012asymptotic} and \cite{assouad1983deux} for obtaining minimax lower bounds for learning problems. Both techniques proceed by reducing the learning problem to an easier hypothesis testing problem  \citep{Tsybakov2009nonparbook,yang1999information,yu1997assouad}, then proving a lower bound on the probability of error in testing problems.

Le Cam's method, in its simplest form, provides lower bounds on the error in simple binary hypothesis testing problems, by using the connection between hypothesis testing and total variation distance.

Consider the parameter estimation problem \eqref{parameter-estimation-task-transition} with $\Theta = \bc{-1,1}^m$ for some $m$, where the objective is to determine every bit of the underlying unknown parameter $\theta \in \bc{-1, 1}^m$. In that setting, a key result known as \textit{Assouad's lemma} says that the difficulty of estimating the entire bit string $\theta$ is related to the difficulty of estimating each bit of $\theta$ separately, assuming all other bits are already known.

\section{Hardness of the Cost-sensitive Classification Problem}
\label{sec:hard-cost-sense}

In this section we follow the presentation of \cite{raginskylecnotes}. Before studying the hardness of the cost-sensitive classification, we study the hardness of the auxiliary problem of parameter estimation \eqref{parameter-estimation-task-transition}.

\subsection{Minimax Lower Bounds for Parameter Estimation Problem}

We derive the cost-dependent lower bound for $\underline{\cR}_\rho^\star \br{\varepsilon_n}$ (defined in \eqref{minimax-risk-parameter}) by extending the standard Le Cam and Assouad's techniques. We start with the two point method introduced by Lucien Le Cam for obtaining minimax lower bounds. 
\begin{proposition}
	\label{lecam-method}
	For any $c \in \br{0,1}$, the minimax risk $\underline{\cR}_\rho^\star \br{\varepsilon_n}$ (given by \eqref{minimax-risk-parameter}) of the parameter estimation problem \eqref{parameter-estimation-task-transition} with (pseudo metric) loss function $\rho:\Theta \times \Theta \rightarrow \bR$ is bounded from below as follows:
	\[
	\underline{\cR}_\rho^\star \br{\varepsilon_n} ~\geq~ \sup_{\theta \neq \theta'}{\bc{\rho\br{\theta,\theta'} \cdot \br{c \wedge \bar{c} - \bI_{f_c}\br{\varepsilon_n \br{\theta},\varepsilon_n \br{\theta'}}}}} ,
	\]
	where $f_c$ is given by \eqref{primitive-fdiv-tent-func}.
\end{proposition}
\begin{proof}
	Let $c \in \br{0,1}$ be arbitrary but fixed. Consider any two fixed parameters $\theta,\theta' \in \Theta$ s.t. $\theta \neq \theta'$ and an arbitrary estimator $\hat{\theta}:\cO^n \rightsquigarrow \Theta$. Let $P_\theta^n := \varepsilon_n \br{\theta}$, and $P_{\theta'}^n := \varepsilon_n \br{\theta'}$ (associated probability densities can be written as $dP_\theta^n$ and $dP_{\theta'}^n$). For an arbitrary (but fixed) set of observations $o_1^n \in \cO^n$, when $\bar{c} \cdot dP_{\theta'}^n\br{o_1^n} \geq c \cdot dP_\theta^n\br{o_1^n}$, we have
	\begin{align}
	& c \cdot dP_\theta^n\br{o_1^n} \Ee{\tilde{\theta} \sim \hat{\theta}(o_1^n)}{\rho(\theta,\tilde{\theta})} + \bar{c} \cdot dP_{\theta'}^n\br{o_1^n} \Ee{\tilde{\theta} \sim \hat{\theta}(o_1^n)}{\rho(\theta',\tilde{\theta})} \nonumber \\
	~=~& c \cdot dP_\theta^n\br{o_1^n} \Ee{\tilde{\theta} \sim \hat{\theta}(o_1^n)}{\rho(\theta,\tilde{\theta}) + \rho(\theta',\tilde{\theta})} + \br{\bar{c} \cdot dP_{\theta'}^n\br{o_1^n} - c \cdot dP_\theta^n\br{o_1^n}} \Ee{\tilde{\theta} \sim \hat{\theta}(o_1^n)}{\rho(\theta',\tilde{\theta})} \nonumber \\ 
	\stackrel{(i)}{~\geq~}& c \cdot dP_\theta^n\br{o_1^n} \Ee{\tilde{\theta} \sim \hat{\theta}(o_1^n)}{\rho(\theta,\tilde{\theta}) + \rho(\theta',\tilde{\theta})} \nonumber \\
	\stackrel{(ii)}{~\geq~}& c \cdot dP_\theta^n\br{o_1^n} \rho\br{\theta,\theta'}, \label{int-case-01}
	\end{align}
	where $(i)$ is due to $\bar{c} \cdot dP_{\theta'}^n\br{o_1^n} \geq c \cdot dP_\theta^n\br{o_1^n}$, and $(ii)$ is due to the triangle inequality. Similarly, for the case where $\bar{c} \cdot dP_{\theta'}^n\br{o_1^n} \leq c \cdot dP_\theta^n\br{o_1^n}$, we get
	\begin{equation}
	\label{int-case-02}
	c \cdot dP_\theta^n\br{o_1^n} \Ee{\tilde{\theta} \sim \hat{\theta}(o_1^n)}{\rho(\theta,\tilde{\theta})} + \bar{c} \cdot dP_{\theta'}^n\br{o_1^n} \Ee{\tilde{\theta} \sim \hat{\theta}(o_1^n)}{\rho(\theta',\tilde{\theta})} ~\geq~ \bar{c} \cdot dP_{\theta'}^n\br{o_1^n} \rho\br{\theta,\theta'}.
	\end{equation}
	By combining \eqref{int-case-01} and \eqref{int-case-02}, and summing over all $o_1^n \in \cO^n$, we get, for any two $\theta, \theta' \in \Theta$ and any estimator $\hat{\theta}$, 
	\begin{align}
	& c \cdot \Ee{\textsf{O}_1^n \sim P_\theta^n}{\Ee{\tilde{\theta} \sim \hat{\theta}\br{\textsf{O}_1^n}}{\rho(\theta,\tilde{\theta})}} + \bar{c} \cdot \Ee{\textsf{O}_1^n \sim P_{\theta'}^n}{\Ee{\tilde{\theta} \sim \hat{\theta}\br{\textsf{O}_1^n}}{\rho(\theta',\tilde{\theta})}} \label{tempo-112} \\
	~\geq~& \rho\br{\theta,\theta'} \cdot \int{c dP_\theta^n \wedge \bar{c} dP_{\theta'}^n} \nonumber \\
	~=~& \rho\br{\theta,\theta'} \cdot \br{c \wedge \bar{c} - \bI_{f_c}\br{P_\theta^n,P_{\theta'}^n}}, \label{sum-exp-link}
	\end{align}
	where the last equality follows from the definition of $c$-primitive $f$-divergences \eqref{eq-primitive-f-div-2}. By taking the supremum of both sides over the choices of $\theta, \theta'$ (since then the two terms in \eqref{tempo-112} collapse to one), we have
	\[
	\sup_{\theta \in \Theta}{\Ee{\textsf{O}_1^n \sim P_\theta^n}{\Ee{\tilde{\theta} \sim \hat{\theta}\br{\textsf{O}_1^n}}{\rho(\theta,\tilde{\theta})}}} ~\geq~ \sup_{\theta \neq \theta'}{\bc{\rho\br{\theta,\theta'} \cdot \br{c \wedge \bar{c} - \bI_{f_c}\br{\varepsilon_n(\theta),\varepsilon_n(\theta')}}}}.
	\]
	The proof is completed by taking the infimum of both sides over $\hat{\theta}$. 
\end{proof}
By setting $c = \frac{1}{2}$ in Proposition~\ref{lecam-method}, we recover Le Cam's (\cite{le2012asymptotic}) minimax lower bound for parameter estimation problem \eqref{parameter-estimation-task-transition}: 
\[
\underline{\cR}_\rho^\star \br{\varepsilon_n} ~\geq~ \frac{1}{2} \sup_{\theta \neq \theta'}{\bc{\rho\br{\theta,\theta'} \cdot \br{1 - \frac{1}{2} d_{\mathrm{TV}}\br{\varepsilon_n \br{\theta},\varepsilon_n \br{\theta'}}}}} ,
\]
since $\bI_{f_{\frac{1}{2}}} \br{P,Q} = \frac{1}{4} d_{\mathrm{TV}}\br{P,Q}$ (from \eqref{eq-primitive-f-div-3} with $c=\frac{1}{2}$). Now we provide an auxiliary result which will be useful in deriving the cost-dependent minimax lower bounds via Assouad's lemma (\cite{assouad1983deux}). 
\begin{corollary}
	\label{auxi-minimax-coro}
	Let $\pi$ be any prior distribution on $\Theta$, and let $\mu$ be any joint probability distribution of a random pair $(\theta,\theta')\in\Theta \times \Theta$, such that the marginal distributions of both $\theta$ and $\theta'$ are equal to $\pi$. Then for any $c \in \br{0,1}$, the minimax risk $\underline{\cR}_\rho^\star \br{\varepsilon_n}$ (given by \eqref{minimax-risk-parameter}) of the parameter estimation problem \eqref{parameter-estimation-task-transition} is bounded from below as follows:
	\[
	\underline{\cR}_\rho^\star \br{\varepsilon_n} ~\geq~ \Ee{(\theta,\theta') \sim \mu}{\rho\br{\theta,\theta'} \cdot \br{c \wedge \bar{c} - \bI_{f_c}\br{\varepsilon_n \br{\theta},\varepsilon_n \br{\theta'}}}}
	\]
\end{corollary} 
\begin{proof}
	First observe that for any prior $\pi$ 
	\[
	\underline{\cR}_\rho^\star \br{\varepsilon_n} ~\geq~ \inf_{\hat{\theta}} \Ee{\theta \sim \pi}{\Ee{\textsf{O}_1^n \sim \varepsilon_n\br{\theta}}{\Ee{\tilde{\theta} \sim \hat{\theta}\br{\textsf{O}_1^n}}{\rho(\theta,\tilde{\theta})}}} ,
	\]
	since the minimax risk can be lower bounded by the Bayesian risk (see Theorem~\ref{minimax-bayesian-comparison-theorem}). Then by taking expectation of both sides of \eqref{sum-exp-link} w.r.t $\mu$ and using the fact that, under $\mu$, both $\theta$ and $\theta'$ have the same distribution $\pi$, the proof is completed. 
\end{proof}

Using the above corollary and extending the standard Assouad's lemma, we derive the cost-dependent minimax lower bound for the parameter estimation problem \eqref{parameter-estimation-task-transition}. 
\begin{theorem}
	\label{assouad-method}
	Let $d \in \bN$, $\Theta=\bc{-1,1}^d$ and $\rho = \rho_{\mathrm{Ha}}$, where the Hamming distance $\rho_{\mathrm{Ha}}$ is given by \eqref{eq-hamming-distance}. Then for any $c \in \br{0,1}$, the minimax risk of the parameter estimation problem \eqref{parameter-estimation-task-transition} satisfies
	\[
	\underline{\cR}_{\rho_{\mathrm{Ha}}}^\star \br{\varepsilon_n} ~\geq~ d \br{c \wedge \bar{c} - \underset{\theta,\theta': \rho_{\mathrm{Ha}}\br{\theta,\theta'}=1}{\max} \bI_{f_c}\br{\varepsilon_n \br{\theta},\varepsilon_n \br{\theta'}}} .
	\]
\end{theorem}
\begin{proof}
	Recall that $\rho_{\mathrm{Ha}} (\theta,\theta') = \sum_{i=1}^{d}{\rho_i (\theta,\theta')}$, where $\rho_i (\theta,\theta') := \ind{\theta_i \neq \theta'_i}$, and each $\rho_i$ is a pseudo metric. Let $\pi (\theta) = \frac{1}{2^d}, \forall{\theta \in \bc{-1,1}^d}$. Also for each $i \in \bs{d}$, let $\mu_i$ be the distribution in $\Theta \times \Theta$ such that any random pair $(\theta,\theta') \in \Theta \times \Theta$ drawn according to $\mu_i$ satisfies
	\begin{enumerate}
		\item $\theta \sim \pi$ 
		\item $\rho_i (\theta,\theta') = 1$, and $\rho_{\mathrm{Ha}} (\theta,\theta') = 1$ ($\theta$ and $\theta'$ differ only in the $i$-th coordinate). 
	\end{enumerate}
	Then the marginal distribution of $\theta'$ under $\mu_i$ is
	\[
	\sum_{\theta \in \bc{-1,1}^d}^{}{\mu_i (\theta,\theta')} = \frac{1}{2^d} \sum_{\theta \in \bc{-1,1}^d}^{}{\ind{\theta_i \neq \theta'_i \text{ and } \theta_j = \theta'_j , j \neq i}} = \frac{1}{2^d} = \pi (\theta') ,
	\]
	since by construction of $\mu$, $\rho_{\mathrm{Ha}} \br{\theta , \theta'} = 1$ and for each $\theta'$ there is only one $\theta$ that differs from it in a single coordinate. Now consider 
	\begin{align*}
	\underline{\cR}_{\rho_{\mathrm{Ha}}}^\star \br{\varepsilon_n} \stackrel{(i)}{~\geq~}& \inf_{\hat{\theta}} \Ee{\theta \sim \pi}{\Ee{\textsf{O}_1^n \sim \varepsilon_n\br{\theta}}{\Ee{\tilde{\theta} \sim \hat{\theta}\br{\textsf{O}_1^n}}{\rho_{\mathrm{Ha}} \br{\theta,\tilde{\theta}}}}} \\
	\stackrel{(ii)}{~=~}& \inf_{\hat{\theta}} \sum_{i=1}^{d}{\Ee{\theta \sim \pi}{\Ee{\textsf{O}_1^n \sim \varepsilon_n\br{\theta}}{\Ee{\tilde{\theta} \sim \hat{\theta}\br{\textsf{O}_1^n}}{\rho_i \br{\theta,\tilde{\theta}}}}}} \\
	~\geq~& \sum_{i=1}^{d}{\inf_{\hat{\theta}} \Ee{\theta \sim \pi}{\Ee{\textsf{O}_1^n \sim \varepsilon_n\br{\theta}}{\Ee{\tilde{\theta} \sim \hat{\theta}\br{\textsf{O}_1^n}}{\rho_i \br{\theta,\tilde{\theta}}}}}} \\
	\stackrel{(iii)}{~\geq~}& \sum_{i=1}^{d}{\Ee{(\theta,\theta') \sim \mu_i}{\rho_i\br{\theta,\theta'} \cdot \br{c \wedge \bar{c} - \bI_{f_c}\br{\varepsilon_n \br{\theta},\varepsilon_n \br{\theta'}}}}} \\
	\stackrel{(iv)}{~=~}& \sum_{i=1}^{d}{\Ee{(\theta,\theta') \sim \mu_i}{\br{c \wedge \bar{c} - \bI_{f_c}\br{\varepsilon_n \br{\theta},\varepsilon_n \br{\theta'}}}}} \\
	~\geq~& \sum_{i=1}^{d}{\underset{\theta,\theta': \rho_{\mathrm{Ha}}\br{\theta,\theta'}=1}{\min}{\br{c \wedge \bar{c} - \bI_{f_c}\br{\varepsilon_n \br{\theta},\varepsilon_n \br{\theta'}}}}} \\
	~=~& d \br{c \wedge \bar{c} - \underset{\theta,\theta': \rho_{\mathrm{Ha}}\br{\theta,\theta'}=1}{\max} \bI_{f_c}\br{\varepsilon_n \br{\theta},\varepsilon_n \br{\theta'}}},
	\end{align*}
	where $(i)$ is due to the fact that the minimax risk is lower bounded by the Bayesian risk (see Theorem~\ref{minimax-bayesian-comparison-theorem}), $(ii)$ is due to $\rho_{\mathrm{Ha}} (\theta,\theta') = \sum_{i=1}^{d}{\rho_i (\theta,\theta')}$, $(iii)$ is by Corollary~\ref{auxi-minimax-coro}, and $(iv)$ is by the fact that $\rho_i (\theta,\theta')=1$ under $\mu_i$ for every $i$.
\end{proof}

If we re-normalize $\bI_{f_c}\br{\cdot ,\cdot}$ by $c \wedge \bar{c}$, and define $\bI_{f_c}^* \br{\cdot ,\cdot} := \frac{1}{c \wedge \bar{c}}\bI_{f_c} \br{\cdot ,\cdot}$, then the minimax lower bound in Theorem~\ref{assouad-method} can be written as follows:
\[
\underline{\cR}_{\rho_{\mathrm{Ha}}}^\star \br{\varepsilon_n} ~\geq~ d \cdot \br{c \wedge \bar{c}} \bs{1 - \underset{\theta,\theta': \rho_{\mathrm{Ha}}\br{\theta,\theta'}=1}{\max} \bI_{f_c}^*\br{\varepsilon_n \br{\theta},\varepsilon_n \br{\theta'}}} .
\]
Also note that, by setting $c = \frac{1}{2}$ in Theorem~\ref{assouad-method}, we recover the standard Assouad's lemma (\cite{assouad1983deux}):
\[
\underline{\cR}_{\rho_{\mathrm{Ha}}}^\star \br{\varepsilon_n} ~\geq~ \frac{d}{2} \br{1 - \frac{1}{2} \underset{\theta,\theta': \rho_{\mathrm{Ha}}\br{\theta,\theta'}=1}{\max} d_{\mathrm{TV}}\br{\varepsilon_n \br{\theta},\varepsilon_n \br{\theta'}}} .
\]

We use the following two properties of the Hellinger distance $\mathrm{He}^2 \br{P,Q}$ (shown in Chapter~\ref{cha:decision}) to derive a more practically useful version of Assouad's lemma: 
\begin{itemize}
	\item $\bI_{f_c}\br{P,Q} ~\leq~ \br{c \wedge \bar{c}} \cdot \mathrm{He} \br{P,Q}$, for all distributions $P, Q \in \cP\br{\cO}$ (refer \eqref{joint-range-hell-primc})	
	\item $\mathrm{He}^2 \br{\bigotimes_{i=1}^k P_i , \bigotimes_{i=1}^k Q_i} ~\leq~ \sum_{i=1}^{k}{\mathrm{He}^2 \br{P_i ,Q_i}}$, for all distributions $P_i, Q_i \in \cP\br{\cO_i}$, $i \in [k]$	
\end{itemize}
Armed with these facts, we prove the following version of Assouad's lemma: 
\begin{corollary}
	\label{assouad-corro}
	Let $\cO$ be some set and $c \in [0,1]$. Define 
	\[
	\cP_\varepsilon \br{\cO} := \bc{\varepsilon\br{\theta} \in \cP \br{\cO} : \theta \in \bc{-1,1}^d}
	\]
	be a class of probability measures induced by the transition $\varepsilon : \bc{-1,1}^d \rightsquigarrow \cO$. Suppose that there exists some function $\alpha: \bs{0,1} \rightarrow \bR_+$, such that 
	\[
	\mathrm{He}^2 \br{\varepsilon\br{\theta},\varepsilon\br{\theta'}} \leq \alpha\br{c} , \quad \text{ if } \rho_{\mathrm{Ha}}\br{\theta,\theta'} = 1 ,
	\] 
	i.e. the two probability distributions $\varepsilon\br{\theta}$ and $\varepsilon\br{\theta'}$ (associated with the two parameters $\theta$ and $\theta'$ which differ  only in one coordinate) are sufficiently close w.r.t.\ Hellinger distance. Then the minimax risk of the parameter estimation problem \eqref{parameter-estimation-task-transition} with parameter space $\Theta = \bc{-1,1}^d$ and the loss function $\rho = \rho_{\mathrm{Ha}}$ is bounded below by
	\begin{equation}
	\label{assouad-practical-eq}
	\underline{\cR}_{\rho_{\mathrm{Ha}}}^\star \br{\varepsilon_n} ~\geq~ d \cdot (c \wedge \bar{c}) \cdot \br{1 - \sqrt{\alpha\br{c} n}} .
	\end{equation}
\end{corollary}
\begin{proof}
	For any two $\theta,\theta' \in \Theta$ with $\rho_{\mathrm{Ha}}\br{\theta,\theta'} = 1$, we have
	\begin{align*}
	\bI_{f_c}\br{\varepsilon_n\br{\theta},\varepsilon_n\br{\theta'}} &~\leq~ \br{c \wedge \bar{c}} \cdot \mathrm{He} \br{\varepsilon_n\br{\theta},\varepsilon_n\br{\theta'}} \\
	&~\leq~ \br{c \wedge \bar{c}} \cdot \sqrt{\sum_{i=1}^{n}{\mathrm{He}^2 \br{\varepsilon\br{\theta},\varepsilon\br{\theta'}}}} \\
	&~\leq~ \br{c \wedge \bar{c}} \cdot \sqrt{\alpha \br{c} n} \\
	\end{align*}
	Substituting this bound into Theorem~\ref{assouad-method} completes the proof.
\end{proof}

The number of training samples $n$ appear in the minimax lower bound \eqref{assouad-practical-eq}. Thus the hardness of the problem can be expressed as a function of the sample size along with other problem specific parameters.

\subsection{Minimax Lower Bounds for Cost-sensitive Classification Problem}
A natural question to ask regarding cost-sensitive classification problem is how does the hardness of the problem depend upon the cost parameter $c \in \bs{0,1}$. Let $\cF \subseteq \bc{-1,1}^\cX$ be the action space and $h \in \bs{0,c \wedge \bar{c}}$ be the margin parameter whose interpretation is explained below. Then we choose a parameter space $\Theta_{h,\cF}$ (thus the experiment $\varepsilon_{h,\cF}: \Theta_{h,\cF} \rightsquigarrow \br{\cX \times \bc{-1,1}}$) such that:
\begin{enumerate}
	\item $\forall{\theta \in \Theta_{h,\cF}}$, $f_\theta^\star \in \cF$, where $f_\theta^\star$ is given by \eqref{bayes-classifier-lc}. That is we restrict the parameter space s.t. the Bayes classifier associated with each choice of parameter lies within the predetermined function class $\cF$. 
	\item \begin{equation}
	\label{noise-condition}
	\abs{\eta_\theta \br{\textsf{X}} - c} \geq h \text{ a.s. } \forall{\theta \in \Theta_{h,\cF}} .
	\end{equation}
	This condition is a generalized notion of \textit{Massart noise condition} with margin $h \in \bs{0 , c \wedge \bar{c}}$ (\cite{massart2006risk}). The motivation for this condition is well established by \cite{massart2006risk}. They have argued that under certain ``margin'' type conditions (\citep{vapnik74theory,tsybakov2004optimal}) like this, it is possible to design learning algorithms for the binary classification problem, with better rates compared to the case where no such condition is satisfied. 
\end{enumerate}

Thus we consider the problem represented by following transition diagram
\begin{equation}
\label{specific-learning-task-eq}
\specificlearningtask ,
\end{equation}
and the minimax risk (in terms of regret) of it given by 
\begin{equation}
\label{cost-minmax-risk-eq}
\underline{\cR}_{\Delta \ell_{d_{c}}}^\star \br{\varepsilon_n} := \inf_{\hat{f}}{\sup_{\theta \in \Theta_{h,\cF}}{\Ee{\bc{\br{\textsf{X}_i,\textsf{Y}_i}}_{i=1}^n \sim \varepsilon_n \br{\theta}}{\Ee{f \sim \hat{f}\br{\bc{\br{\textsf{X}_i,\textsf{Y}_i}}_{i=1}^n}}{\Delta \ell_{d_c} \br{\theta,f}}}}} .
\end{equation}
The following is a generalization of the result proved in \cite[Theorem~4]{massart2006risk} for $c = \frac{1}{2}$.
\begin{theorem}
	\label{main-theo-cost-classi}
	Let $\cF$ be a VC class of binary-valued functions on $\cX$ with VC dimension (refer section~\ref{vc-dimension-sec}) $V \geq 2$. Then for any $n \geq V$ and any $h \in [0,c \wedge \bar{c}]$, the minimax risk \eqref{cost-minmax-risk-eq} of the cost-sensitive binary classification problem \eqref{specific-learning-task-eq} is lower bounded as follows: 
	\[
	\underline{\cR}_{\Delta \ell_{d_c}}^\star \br{\varepsilon_n} ~\geq~ K \cdot (c \wedge \bar{c}) \cdot \min\br{\sqrt{\frac{(c \wedge \bar{c}) V}{n}},(c \wedge \bar{c}) \cdot \frac{V}{nh}}
	\]
	where $K > 0$ is some absolute constant. 
\end{theorem}
\begin{proof}
	Instantiate $\Theta = \cA = B := \bc{-1,1}^{V-1}$, $\cO = \cX \times \bc{-1,1}$, and $A=\hat{b}$ in the general learning task \eqref{rep-experiment-general-task}. Then the resulting parameter estimation problem can be represented by the following transition diagram:
	\begin{center}
		\delparametertask .
	\end{center}
	Let $\cP_{\varepsilon_n} \br{\cO^n} := \bc{\varepsilon_n \br{b} := \varepsilon \br{b}^n \in \cP \br{\cO}^n:b\in B}$ be the class of probability measures induced by the experiment $\varepsilon_n$. Then the minimax risk of this problem w.r.t.\ Hamming distance $\rho_{\mathrm{Ha}}$ is given by 
	\[
	\underline{\cR}_{\rho_{\mathrm{Ha}}}^\star \br{\varepsilon_n} ~=~ \inf_{\hat{b}} \max_{b \in B} \Ee{\textsf{O}_1^n \sim \varepsilon_n \br{b}}{\Ee{b' \sim \hat{b}\br{\textsf{O}_1^n}}{\rho_{\mathrm{Ha}}\br{b,b'}}} .
	\]
	
	Observe that $\bP_{\varepsilon \br{b}}\bs{\textsf{X}=x,\textsf{Y}=y} = \bP_{\varepsilon \br{b}}\bs{\textsf{X}=x} \cdot \bP_{\varepsilon \br{b}}\bs{\textsf{Y}=y | \textsf{X}=x}$ for $b \in B$ (by Bayes rule). For simplicity, we will write $\bP_{\varepsilon \br{b}} \bs{\cdot}$ as $\bP_b \bs{\cdot}$. Now we will construct these distributions.
	
	\textbf{Construction of marginal distribution }$\bP_b\bs{\textsf{X}=x}, x \in \cX$: Since $\cF$ is a \textit{VC class} with \textit{VC dimension} $V$, $\exists \bc{x_1,...,x_V} \subset \cX$ that is shattered, i.e. $\text{for any } \beta \in \bc{-1,1}^V, \exists f \in \cF \text{ s.t. } f(x_i)=\beta_i, \forall{i \in \bs{V}}$. Given $p \in \bs{0,1/(V-1)}$, for each $b \in B$, let
	\begin{equation}
	\label{pb-marginal-eq}
	\bP_b \bs{\textsf{X}=x} = \begin{cases}
	p, & \text{if } x=x_i \text{ for some } i \in \bs{V-1} \\
	1-(V-1)p, & \text{if } x=x_V \\
	0, & \text{otherwise}
	\end{cases}
	\end{equation}
	A particular value for $p$ will be chosen later. 
	
	\textbf{Construction of conditional distribution }$\bP_b\bs{\textsf{Y}=y|\textsf{X}=x}, y \in \bc{-1,1}, x \in \cX$: For each $b \in B$, let
	\begin{equation}
	\label{cost-regression}
	\eta_b \br{x} := \bP_b \bs{\textsf{Y}=1 | \textsf{X}=x} = \begin{cases}
	c-h, & \text{if } x=x_i \text{ for some } i \in \bs{V-1}, \text{ and } b_i = -1 \\
	c+h, & \text{if } x=x_i \text{ for some } i \in \bs{V-1}, \text{ and } b_i = 1 \\
	0, & \text{otherwise} .
	\end{cases}
	\end{equation}
	Then the corresponding Bayes classifier can be given as follows:
	\begin{equation}
	\label{bayes-opt-func-proof}
	f_b^\star(x) = \begin{cases}
	-1, & \text{if } x=x_i \text{ for some } i \in \bs{V-1}, \text{ and } b_i = -1 \\
	1, & \text{if } x=x_i \text{ for some } i \in \bs{V-1}, \text{ and } b_i = 1 \\
	-1, & \text{otherwise}
	\end{cases}
	\end{equation}
	Now we show that $\bc{\varepsilon\br{b} : b \in B} \subseteq \bc{\varepsilon\br{\theta} : \theta \in \Theta_{h,\cF}}$. First of all, from \eqref{cost-regression} we see that $\card{\eta_b \br{x} - c} \geq h$ for all $x$ (indeed, $\card{\eta_b \br{x} - c} = h$ when $x \in \bc{x_1,...,x_{V-1}}$, and $\card{\eta_b \br{x} - c} = c$ otherwise). Second, because $\bc{x_1,...,x_V}$ is shattered by $\cF$, there exists at least one $f \in \cF$, such that $f_b^\star (x) = f (x)$ for all $x \in \bc{x_1,...,x_V}$. Thus, we get $B \subset \Theta_{h,\cF}$.
	
	\textbf{Reduction to Parameter Estimation Problem: }
	We start with the following observation 
	\[
	\underline{\cR}_{\Delta \ell_{d_c}}^\star \br{\varepsilon_n} ~\geq~ \inf_{\hat{f}}{\max_{b \in B}{\Ee{\bc{\br{\textsf{X}_i,\textsf{Y}_i}}_{i=1}^n \sim \varepsilon_n \br{b}}{\Ee{f \sim \hat{f}\br{\bc{\br{\textsf{X}_i,\textsf{Y}_i}}_{i=1}^n}}{\ell_{d_c} \br{b,f} - \ell_{d_c} \br{b,f_b^\star}}}}} ,
	\]
	since $B \subset \Theta_{h,\cF}$. Define $M_\theta \br{x} := \bP_{\theta} \bs{\textsf{X} = x}$, and $\eta_\theta\br{x} := \bP_\theta \bs{\textsf{Y}=1 \mid \textsf{X}=x}$, for $x \in \cX$. By Lemma~\ref{cost-sensitive-lemma}, for any classifier $f:\cX \rightarrow \bc{-1,1}$ and any $\theta \in \Theta_{h,\cF}$, we have 
	\[
	\ell_{d_c} \br{\theta,f} - \ell_{d_c} \br{\theta,f_\theta^\star} ~=~ \frac{1}{2} \cdot \Ee{\textsf{X} \sim M_\theta}{\abs{\eta_\theta \br{\textsf{X}} - c} \cdot \abs{f\br{\textsf{X}} - f_\theta^\star \br{\textsf{X}}}} .
	\]
	If $\theta \in \Theta_{h,\cF}$, then using the above equation and the margin condition \eqref{noise-condition} we get
	\[
	\ell_{d_c} \br{\theta,f} - \ell_{d_c} \br{\theta,f_\theta^\star} ~\geq~ \frac{h}{2} \cdot \Ee{\textsf{X} \sim M_\theta}{\abs{f\br{\textsf{X}} - f_\theta^\star \br{\textsf{X}}}} = \frac{h}{2} \cdot \norm{f-f_\theta^\star}_{L_1 \br{M_\theta}},
	\]
	where $\norm{f}_{L_1 \br{M_\theta}}$ is given by \eqref{eq-lp-norm-func} with $p=1$ and $\mu = M_\theta$. Since there is no confusion, we can simply drop $M_\theta$ and write the $L_1$ norm as $\norm{\cdot}_{L_1}$.  Hence we have
	\begin{align*}
	& \inf_{\hat{f}}{\max_{b \in B}{\Ee{\bc{\br{\textsf{X}_i,\textsf{Y}_i}}_{i=1}^n \sim \varepsilon_n \br{b}}{\Ee{f \sim \hat{f}\br{\bc{\br{\textsf{X}_i,\textsf{Y}_i}}_{i=1}^n}}{\ell_{d_c} \br{b,f} - \ell_{d_c} \br{b,f_b^\star}}}}} \\
	~\geq~& \frac{h}{2} \cdot \inf_{\hat{f}}{\max_{b \in B}{\Ee{\bc{\br{\textsf{X}_i,\textsf{Y}_i}}_{i=1}^n \sim \varepsilon_n \br{b}}{\Ee{f \sim \hat{f}\br{\bc{\br{\textsf{X}_i,\textsf{Y}_i}}_{i=1}^n}}{\norm{f-f_b^\star}_{L_1}}}}} .
	\end{align*}
	Define 
	\[
	b_f ~:=~ \argmin_{b \in B}{\norm{f-f_b^\star}_{L_1}}.
	\]
	Then for any $b \in B$,
	\[
	\norm{f_{b_f}^\star-f_b^\star}_{L_1} \leq \norm{f_{b_f}^\star-f}_{L_1} + \norm{f-f_b^\star}_{L_1} \leq 2 \norm{f-f_b^\star}_{L_1} ,
	\]
	where the first inequality is due to the triangle inequality and the second follows from the definitions of $b_f$ and $f_{\theta}^\star$. Thus we have 
	\begin{align*}
	& \inf_{\hat{f}}{\max_{b \in B}{\Ee{\bc{\br{\textsf{X}_i,\textsf{Y}_i}}_{i=1}^n \sim \varepsilon_n \br{b}}{\Ee{f \sim \hat{f}\br{\bc{\br{\textsf{X}_i,\textsf{Y}_i}}_{i=1}^n}}{\ell_{d_c} \br{b,f} - \ell_{d_c} \br{b,f_b^\star}}}}} \\
	~\geq~& \frac{h}{4} \cdot \inf_{\hat{f}}{\max_{b \in B}{\Ee{\bc{\br{\textsf{X}_i,\textsf{Y}_i}}_{i=1}^n \sim \varepsilon_n \br{b}}{\Ee{f \sim \hat{f}\br{\bc{\br{\textsf{X}_i,\textsf{Y}_i}}_{i=1}^n}}{\norm{f_{b_f}^\star-f_b^\star}_{L_1}}}}} \\
	~=~& \frac{h}{4} \cdot \inf_{\hat{b}}{\max_{b \in B}{\Ee{\bc{\br{\textsf{X}_i,\textsf{Y}_i}}_{i=1}^n \sim \varepsilon_n \br{b}}{\Ee{b' \sim \hat{b}\br{\bc{\br{\textsf{X}_i,\textsf{Y}_i}}_{i=1}^n}}{\norm{f_{b'}^\star-f_b^\star}_{L_1}}}}} .
	\end{align*}
	For any two $b,b' \in B$, we have 
	\begin{align*}
	\norm{f_{b'}^\star-f_{b}^\star}_{L_1} &= \int_{\cX}^{ }{\abs{f_{b'}^\star(x)-f_{b}^\star(x)}\bP_b\bs{\textsf{X}=x} dx} \\
	&= p \sum_{i=1}^{V-1}{\abs{f_{b'}^\star(x_i)-f_{b}^\star(x_i)}} \\
	&= p \sum_{i=1}^{V-1}{\abs{b'_i-b_i}} \\
	&= 2 p \cdot \rho_{\mathrm{Ha}} \br{b,b'} ,
	\end{align*}
	where the second and third equalities are from \eqref{pb-marginal-eq} and \eqref{bayes-opt-func-proof}. Finally we get 
	\begin{align}
	& \inf_{\hat{f}}{\max_{b \in B}{\Ee{\bc{\br{\textsf{X}_i,\textsf{Y}_i}}_{i=1}^n \sim \varepsilon_n \br{b}}{\Ee{f \sim \hat{f}\br{\bc{\br{\textsf{X}_i,\textsf{Y}_i}}_{i=1}^n}}{\ell_{d_c} \br{b,f} - \ell_{d_c} \br{b,f_b^\star}}}}} \nonumber \\
	~\geq~& \frac{ph}{2} \cdot \inf_{\hat{b}}{\max_{b \in B}{\Ee{\bc{\br{\textsf{X}_i,\textsf{Y}_i}}_{i=1}^n \sim \varepsilon_n \br{b}}{\Ee{b' \sim \hat{b}\br{\bc{\br{\textsf{X}_i,\textsf{Y}_i}}_{i=1}^n}}{\rho_{\mathrm{Ha}} \br{b,b'}}}}} \nonumber \\
	~=~& \frac{ph}{2} \cdot \underline{\cR}_{\rho_{\mathrm{Ha}}}^\star \br{\varepsilon_n} . \label{lower-bound-temp-eq}
	\end{align}
	
	\textbf{Applying Assouad's Lemma: }
	For any two $b,b' \in B$ we have
	\begin{align*}
	\mathrm{He}^2 \br{\varepsilon\br{b},\varepsilon\br{b'}} &~=~ \sum_{i=1}^{V} \sum_{y \in \bc{-1,1}}^{}{\br{\sqrt{\bP_b\bs{\textsf{X}=x_i,\textsf{Y}=y}}-\sqrt{\bP_{b'}\bs{\textsf{X}=x_i,\textsf{Y}=y}}}^2} \\
	&~=~ p \sum_{i=1}^{V-1} \sum_{y \in \bc{-1,1}}^{}{\br{\sqrt{\bP_b \bs{\textsf{Y}=y|\textsf{X}=x_i}}-\sqrt{\bP_{b'}\bs{\textsf{Y}=y|\textsf{X}=x_i}}}^2} \\
	&~=~ p \sum_{i=1}^{V-1} {\ind{b_i \neq b'_i} \bc{\br{\sqrt{c-h} - \sqrt{c+h}}^2 + \br{\sqrt{\bar{c-h}} - \sqrt{\bar{c+h}}}^2}} \\
	&~=~ 2 p \br{1-\sqrt{c^2-h^2}-\sqrt{\bar{c}^2-h^2}} \rho_{\mathrm{Ha}}\br{b,b'} ,
	\end{align*}
	where the second and third equalities are from \eqref{pb-marginal-eq} and \eqref{cost-regression}. Thus the condition of the Corollary~\ref{assouad-corro} is satisfied with 
	\begin{equation*}
	2 p \br{1-\sqrt{c^2-h^2}-\sqrt{\bar{c}^2-h^2}} ~\leq~ 4 p \frac{h^2}{c \wedge \bar{c}} ~=:~ \alpha\br{c} ,
	\end{equation*}
	where the inequality is from Lemma~\ref{aux-bound-lemma} (see below). Therefore we get
	\begin{align*}
	& \inf_{\hat{f}}{\max_{b \in B}{\Ee{\bc{\br{\textsf{X}_i,\textsf{Y}_i}}_{i=1}^n \sim \varepsilon_n \br{b}}{\Ee{f \sim \hat{f}\br{\bc{\br{\textsf{X}_i,\textsf{Y}_i}}_{i=1}^n}}{\ell_{d_c} \br{b,f} - \ell_{d_c} \br{b,f_b^\star}}}}} \\
	~\geq~& \frac{p h (V-1)}{2} \br{c \wedge \bar{c} - c \wedge \bar{c} \cdot \sqrt{4 p \frac{h^2}{c \wedge \bar{c}} n}} \\
	~=~& \frac{p h (V-1)}{2} \br{c \wedge \bar{c} - 2 h \sqrt{c \wedge \bar{c} \cdot p n}} ,
	\end{align*}
	where the first inequality is due to \eqref{assouad-practical-eq} and \eqref{lower-bound-temp-eq}. If we let $p=\frac{c \wedge \bar{c}}{9nh^2}$, then the term in the parentheses will be equal to $\frac{c \wedge \bar{c}}{3}$, and
	\[
	\inf_{\hat{f}}{\max_{b \in B}{\Ee{\bc{\br{\textsf{X}_i,\textsf{Y}_i}}_{i=1}^n \sim \varepsilon_n \br{b}}{\Ee{f \sim \hat{f}\br{\bc{\br{\textsf{X}_i,\textsf{Y}_i}}_{i=1}^n}}{\ell_{d_c} \br{b,f} - \ell_{d_c} \br{b,f_b^\star}}}}} ~\geq~ \frac{(c \wedge \bar{c})^2 (V-1)}{54nh} ,
	\]
	assuming that the condition $p \leq 1/(V-1)$ holds. This will be the case if $h \geq \sqrt{\frac{c \wedge \bar{c} (V-1)}{9n}}$. Therefore 
	\begin{equation}
	\label{cost-case-1}
	\underline{\cR}_{\Delta \ell_{d_c}}^\star \br{\varepsilon_n} ~\geq~ \frac{(c \wedge \bar{c})^2 (V-1)}{54nh} , \quad \text{ if } h \geq \sqrt{\frac{c \wedge \bar{c} (V-1)}{9n}}.
	\end{equation}
	If $h \leq \sqrt{\frac{c \wedge \bar{c} (V-1)}{9n}}$, we can use the above construction with $\tilde{h} = \sqrt{\frac{c \wedge \bar{c} (V-1)}{9n}}$. Then, because $\Theta_{\tilde{h},\cF} \subseteq \Theta_{h,\cF}$ whenever $\tilde{h} \geq h$, we see that 
	\begin{align}
	& \inf_{\hat{f}}{\sup_{\theta \in \Theta_{h,\cF}}{\Ee{\bc{\br{\textsf{X}_i,\textsf{Y}_i}}_{i=1}^n \sim \varepsilon_n \br{\theta}}{\Ee{f \sim \hat{f}\br{\bc{\br{\textsf{X}_i,\textsf{Y}_i}}_{i=1}^n}}{\ell_{d_c} \br{\theta,f} - \ell_{d_c} \br{\theta,f_\theta^\star}}}}} \nonumber \\
	~\geq~& \inf_{\hat{f}}{\sup_{\theta \in \Theta_{\tilde{h},\cF}}{\Ee{\bc{\br{\textsf{X}_i,\textsf{Y}_i}}_{i=1}^n \sim \varepsilon_n \br{\theta}}{\Ee{f \sim \hat{f}\br{\bc{\br{\textsf{X}_i,\textsf{Y}_i}}_{i=1}^n}}{\ell_{d_c} \br{\theta,f} - \ell_{d_c} \br{\theta,f_\theta^\star}}}}} \nonumber \\
	~\geq~& \frac{(c \wedge \bar{c})^2 (V-1)}{54n\tilde{h}} \nonumber \\
	~=~& \frac{(c \wedge \bar{c})^{\frac{3}{2}}}{18} \sqrt{\frac{V-1}{n}}, \quad \text{ if } h \leq \sqrt{\frac{c \wedge \bar{c} (V-1)}{9n}}. \label{cost-case-2}
	\end{align}
	Observe that $\frac{(c \wedge \bar{c})^2 (V-1)}{54nh} \leq \frac{(c \wedge \bar{c})^{\frac{3}{2}}}{18} \sqrt{\frac{V-1}{n}}$ if $h \geq \sqrt{\frac{c \wedge \bar{c} (V-1)}{9n}}$, and $\frac{(c \wedge \bar{c})^2 (V-1)}{54nh} > \frac{(c \wedge \bar{c})^{\frac{3}{2}}}{18} \sqrt{\frac{V-1}{n}}$ otherwise. Then combining \eqref{cost-case-1} and \eqref{cost-case-2} completes the proof.
\end{proof}

\begin{lemma}
	\label{aux-bound-lemma}
	For $h \in \bs{0,c \wedge \bar{c}}$, we have $1-\sqrt{c^2-h^2}-\sqrt{\bar{c}^2-h^2} ~\leq~ 2 \frac{h^2}{c \wedge \bar{c}}$.
\end{lemma}
\begin{proof}
	Let $A = 1-\sqrt{c^2-h^2}-\sqrt{\bar{c}^2-h^2}$. Take series expansion of $A$ w.r.t.\ $h$ to get 
	\begin{align*}
	A ~=~& \frac{1}{2} \br{\frac{1}{c} + \frac{1}{\bar{c}}} h^2 + \frac{1}{8} \br{\frac{1}{c^3} + \frac{1}{\bar{c}^3}} h^4 + \frac{1}{16} \br{\frac{1}{c^5} + \frac{1}{\bar{c}^5}} h^6 + \frac{5}{128} \br{\frac{1}{c^7} + \frac{1}{\bar{c}^7}} h^8 \\
	& + \frac{7}{256} \br{\frac{1}{c^9} + \frac{1}{\bar{c}^9}} h^{10} + \frac{21}{1024} \br{\frac{1}{c^{11}} + \frac{1}{\bar{c}^{11}}} h^{12} + \frac{33}{2048} \br{\frac{1}{c^{13}} + \frac{1}{\bar{c}^{13}}} h^{14} + \cdots .
	\end{align*}
	Now $\frac{1}{2} \br{\frac{1}{c} + \frac{1}{\bar{c}}} \leq \frac{1}{c} \vee \frac{1}{\bar{c}} = \frac{1}{c \wedge \bar{c}}$ (since average is less than maximum). Thus
	\[
	A \leq h \bs{\frac{h}{c \wedge \bar{c}} + \frac{1}{4} \frac{h^3}{\br{c \wedge \bar{c}}^3} + \frac{1}{8} \frac{h^5}{\br{c \wedge \bar{c}}^5} + \frac{5}{64} \frac{h^7}{\br{c \wedge \bar{c}}^7} + \frac{7}{128} \frac{h^9}{\br{c \wedge \bar{c}}^9} + \frac{21}{512} \frac{h^{11}}{\br{c \wedge \bar{c}}^{11}} + \cdots}.
	\]
	Now we have 
	\begin{align*}
	h \leq c \wedge \bar{c} \implies& \frac{h}{c \wedge \bar{c}} \leq 1 \\
	\implies& \br{\frac{h}{c \wedge \bar{c}}}^\alpha \leq \frac{h}{c \wedge \bar{c}} , \forall{\alpha \geq 1} .
	\end{align*}
	Thus 
	\begin{align*}
	A \leq& h \bs{\frac{h}{c \wedge \bar{c}} + \frac{h}{c \wedge \bar{c}}\bc{\frac{1}{4} + \frac{1}{8} + \frac{5}{64} + \frac{7}{128} + \frac{21}{512} + \cdots}} \\
	\leq& h \bs{\frac{2 h}{c \wedge \bar{c}}} = \frac{2 h^2}{c \wedge \bar{c}} , 
	\end{align*}
	where the second inequality follows from the fact that (can be shown with the aid of computer or using the properties of gamma function)
	\[
	\frac{1}{4} + \frac{1}{8} + \frac{5}{64} + \frac{7}{128} + \frac{21}{512} + \cdots \leq 1 .
	\]
\end{proof}

When $h = 0$ (or being too small), we get a minimax lower bound of order $\br{c \wedge \bar{c}}^{\frac{3}{2}} \cdot \sqrt{\frac{V}{n}}$, and when $h = c \wedge \bar{c}$, we obtain a bound of the order $\br{c \wedge \bar{c}} \cdot \frac{V}{n}$.

When $c=1/2$ in Theorem~\ref{main-theo-cost-classi}, we recover the standard minimax lower bounds for balanced binary classification (with zero-one loss function) presented in \cite[Theorem~4]{massart2006risk}:
\[
\underline{\cR}_{\Delta \ell^{0-1}}^\star \br{\varepsilon_n} ~\geq~ K' \cdot \min\br{\sqrt{\frac{V}{n}},\frac{V}{nh}} ,
\]
for some constant $K' > 0$. 

It would be interesting to study the hardness of the following classification problem settings which are closely related to the binary cost-sensitive classification problem that we considered in this thesis:
\begin{enumerate}
	\item Cost-sensitive classification with example dependent costs (\citep{zadrozny2001learning,zadrozny2003cost}).
	\item Binary classification problem w.r.t.\ generalized performance measures (\cite{koyejo2014consistent}) such as arithmetic, geometric and harmonic means of the true positive and true negative rates. These measures are more appropriate for imbalanced classification problem (\citep{cardie1997improving,elkan2001foundations}) than the usual classification accuracy. 
\end{enumerate}

\section{Constrained Learning Problem}
\label{sec:const-learn}
In the normal theoretical analysis of the learning problem \eqref{rep-experiment-general-task}, it is assumed that the decision maker has access to clean data (in $\cO^n$), that their observations are from the pattern they are expected to predict. In the real world, this is usually not the case, data is normally corrupted or needs to be corrupted to meet privacy requirements. We can formalize these constraints (noisy data and privacy requirements) in the language of transitions by introducing the channel $T:\cO \rightsquigarrow \hat{\cO}$, where $\hat{\cO}$ is the new observation space for the decision maker. The \emph{Constrained learning task} (denoted by $\br{\ell,\varepsilon_n,T_{1:n}}$, where $\varepsilon_n$ is the repeated experiment \eqref{repeated-exp-eq}, and $T_{1:n}$ is the parallelized transition \eqref{paralle-same-channel-eq}), can be represented by the following transition diagram: 
\begin{equation}
\label{constrained-task-eq}
\constrainedlearningtask .
\end{equation}

For convenience we define the corrupted experiment $\tilde{\varepsilon} := T \circ \varepsilon$, and denote the repeated corrupted experiment by $\tilde{\varepsilon}_n$. We study the hardness of this constrained learning problem \eqref{constrained-task-eq} by producing the minimax lower bounds of it. From the \emph{Weak Data Processing Inequality} ($\mathbb{I}_{f}\br{\tilde{\varepsilon}_n} \leq \mathbb{I}_{f}\br{\varepsilon_n}$), we have
\[
\underline{\cR}_\ell^\star \br{\tilde{\varepsilon}_n} ~\geq~ \underline{\cR}_\ell^\star \br{\varepsilon_n} ,
\]
i.e. the constrained problem is harder than the original unconstrained problem. Then the minimax lower bounds for the unconstrained problem \eqref{rep-experiment-general-task} are applicable to this constrained task as well. However, this provides us with no means to compare various choices of channel $T$ for a given problem. 

For some $T$, the weak data processing theorem can be strengthened, in the sense that one can find a constant called the \textit{contraction coefficient} $\eta_f \br{T} < 1$ (formally defined in Definition~\ref{sdpi-defn}) such that $\mathbb{I}_{f}\br{\tilde{\varepsilon}_n} \leq \eta_f \br{T} \mathbb{I}_{f}\br{\varepsilon_n}$, for all experiments $\varepsilon: \Theta \rightsquigarrow \cO$. The strengthened inequality is called the \emph{Strong Data Processing Inequality}. The contraction coefficient $\eta_f \br{T}$  provides a means to measure the amount of corruption present in $T$. For example if $T$ is constant and maps all input distributions to the same output distribution, then $\eta_f \br{T} = 0$. If $T$ is an invertible function, then $\eta_f \br{T} = 1$. Together with the minimax lower bounds of unconstrained problem, this strong data processing inequality leads to meaningful lower bounds that allow the comparison of different corrupted experiments. In what follows we will present some new results on strong data processing inequalities. Specifically, we will derive an explicit closed form for the contraction coefficient of any channel w.r.t.\ $c$-primitive $f$-divergence, and obtain efficiently computable lower bounds for the contraction coefficients of binary symmetric channels w.r.t.\ symmetric $f$-divergences. 

\subsection{Strong Data Processing Inequalities}
\label{sec:sdpi}

Consider a channel $T:\cX \rightsquigarrow \cY$, with finite input ($\cX$) and output ($\cY$) spaces (also let $\abs{\cX} \geq 2$ and $\abs{\cY} \geq 2$). Observe that $T$ can also be viewed as an experiment with a different parameter space $\cX$. The classes of probability measures generated by the experiment $\varepsilon$ and the channel $T$ are given by $\cP_\varepsilon \br{\cO} := \bc{\varepsilon\br{\theta} \in \cP\br{\cO} : \theta \in \Theta}$ and $\cP_T \br{\cY} := \bc{T\br{x} \in \cP\br{\cY} : x \in \cX}$ respectively.

Strong data processing inequalities has become an intensive research area in the information theory community recently (see \cite{Raginsky2014} and references therein). Early work includes \cite{ahlswede1976spreading} and \cite{dobrushin1956central} (see \cite{cohen1993relative} for further history).

Below we formally define the contraction coefficient $\eta_f \br{T}$ of the channel $T$ w.r.t.\ $f$-divergence. Recall that $\cP_* \br{\cX}$ means the set of all strictly positive distributions over $\cX$.
\begin{definition}
	\label{sdpi-defn}
	Given a transition $T: \cX \rightsquigarrow \cY$ and $\mu \in \cP_* \br{\cX}$, we define 
	\begin{align*}
	\eta_f (\mu , T) ~:=~& \sup_{\nu \neq \mu , \nu \in \cP \br{\cX}}{\frac{\bI_f \br{T \circ \nu , T \circ \mu}}{\bI_f \br{\nu , \mu}}} \\
	\eta_f (T) ~:=~& \sup_{\mu \in \cP_* \br{\cX}}{\eta_f (\mu , T)} .
	\end{align*}
	If $\eta_f (\mu , T) < 1$, we say that $T$ satisfies the strong data processing inequality (SDPI) at $\mu$ w.r.t.\ $f$-divergence i.e. 
	\[
	\bI_f \br{T \circ \nu , T \circ \mu} \leq \eta_f (\mu , T) \bI_f \br{\nu , \mu}
	\]
	for all $\nu \in \cP\br{\cX}$. Moreover if $\eta_f (T) < 1$ we say that $T$ satisfies the SDPI w.r.t.\ $f$-divergence i.e. 
	\[
	\bI_f \br{T \circ \nu , T \circ \mu} \leq \eta_f (T) \bI_f \br{\nu , \mu}
	\]
	for all $\nu \in \cP\br{\cX}$ and $\mu \in \cP_*\br{\cX}$.
\end{definition}

\cite{cohen1993relative} showed that the contraction coefficient $\eta_f (T)$ of any channel $T$ with respect to any $f$-divergence is universally upper-bounded by the so-called \textit{Dobrushin's coefficient} of $T$ \citep{dobrushin1956central,dobrushin1956central2}. Dobrushin's coefficient is extensively studied in the context of Markov chains (see \cite{paz1971introduction} and \cite{isaacson1976markov} for detailed discussions).
\begin{theorem}
	\label{contract-tv-theorem}
	Define the Dobrushin's coefficient \citep{dobrushin1956central,dobrushin1956central2} of a channel $T \in \cM\br{\cX , \cY}$ as
	\begin{equation}
	\label{dob-def}
	\vartheta \br{T} := \frac{1}{2} \max_{x,x' \in \cX}{d_{\mathrm{TV}}\br{T\br{x} , T\br{x'}}} .
	\end{equation}
	Then the contraction coefficient of the channel $T$ w.r.t.\ any $f$-divergence is bounded above as follows
	\begin{equation}
	\label{universal-dob}
	\eta_f \br{T} ~\leq~ \eta_{\mathrm{TV}} \br{T} ~=~ \vartheta \br{T} ,
	\end{equation}
	where $\eta_{\mathrm{TV}} \br{T}$ is the contraction coefficient of the channel $T$ w.r.t.\ total variation divergence (i.e. $\mathrm{TV} \br{x} = \abs{x-1}$).
\end{theorem}

We are interested in the question of how loose the bound in Theorem~\ref{contract-tv-theorem} can be. For this we study the contraction coefficients w.r.t.\ $c$-primitive $f$-divergences \eqref{eq-primitive-f-div-4}. Recall that any $f$-divergence can be written as a weighted integral of $c$-primitive $f$-divergences as given in Theorem~\ref{fdiv-weight-theorem}.

As a starting point, we define the \textit{generalized Dobrushin's coefficient} of a channel $T:\cX \rightsquigarrow \cY$ as follows 
\begin{equation}
\label{general-dob-coeff-defn}
\vartheta_c (T) ~:=~ \frac{1}{c \wedge \bar{c}} \cdot \max_{x,x' \in \cX}{\bI_{f_c} \br{T(x),T(x')}} ,
\end{equation}
where $c \in \br{0,1}$ and $\vartheta_0 (T) = \vartheta_1 (T) := 0$. Note that when $c=\frac{1}{2}$, this gives the standard Dobrushin's coefficient. From the definition \eqref{general-dob-coeff-defn} we can note the following properties of $\vartheta_c (T)$:
\begin{enumerate}
	\item For a given $T$, $\vartheta_c (T)$ is symmetric in $c$ about $\frac{1}{2}$ i.e. $\vartheta_c = \vartheta_{\bar{c}}$. 
	\item For the fully-informative channel we have $\vartheta_{c} (T_{id}) = 1$ (can be easily shown from the definition), and for the non-informative channel $\vartheta_{c} (T_{\bullet \cX}) = 0$. Thus for any channel $T$ we have $0 \leq \vartheta_{c} (T) \leq 1$.
	\item For two channels $T_1$ and $T_2$, it is possible that $\vartheta_{c} (T_1) > \vartheta_{c} (T_2)$ and $\vartheta_{c'} (T_1) < \vartheta_{c'} (T_2)$ i.e. optimal channel choice based on contraction coefficient will depend on $c$ (see Figure~\ref{gen-dob-fig}).
\end{enumerate}

We have plotted the generalized Dobrushin's coefficient for binary channels $T:\bs{2} \rightsquigarrow \bs{2}$, and ternary channels $T:\bs{3} \rightsquigarrow \bs{3}$. Based on those observations (see Figure~\ref{gen-dob-fig}), we conjecture the following properties of $\vartheta_c (T)$:
\begin{enumerate}
	\item For a given $T$, and for  any $0 < c < c' < 0.5$, $\vartheta_c (T) \leq \vartheta_{c'} (T)$. 
	\item The maximum point of the curve $\vartheta_c (T)$ always occurs at $c = \frac{1}{2}$ (this matches the fact that $\eta_{\mathrm{TV}} (T)$ upper bounds the contraction coefficient $\eta_{f} (T)$ of any $f$-divergence).  
\end{enumerate}

\begin{figure}
	\centering
	\begin{tikzpicture}
	\begin{axis}[domain=0:1, xlabel=$c$, ylabel=$\vartheta_{c} (T)$, axis x line=center, axis y line=center, ymin=0, ymax=1.25, xmin=0, xmax=1.15, xtick={0,1}, ytick={1}]
	\addplot[color=blue] {0} node[pos=.6,above] {$T=T_{\bullet \cX}$};
	\addplot[color=red] {1} node[pos=.8,above] {$T=T_{id}$};
	\addplot[color=gray] {max(abs(0.2*x-0.8*(1-x))+abs(0.8*x-0.2*(1-x))-abs(2*x-1),abs(0.8*x-0.2*(1-x))+abs(0.2*x-0.8*(1-x))-abs(2*x-1))/min(2*x,2*(1-x)))} node[pos=.5,left] {$T=T_1$};
	\addplot[color=purple] {max(abs(0.1*x-0.6*(1-x))+abs(0.9*x-0.4*(1-x))-abs(2*x-1),abs(0.6*x-0.1*(1-x))+abs(0.4*x-0.9*(1-x))-abs(2*x-1))/min(2*x,2*(1-x)))} node[pos=.8,right] {$T=T_2$};
	\end{axis}
	\end{tikzpicture}
	\caption[Generalized Dobrushin's coefficient of channels]{Generalized Dobrushin's coefficient of two arbitrary channels $T_1 = \begin{bmatrix}
		0.8 & 0.2 \\
		0.2 & 0.8 
		\end{bmatrix}$ (\textcolor{gray}{---}) and $T_2 = \begin{bmatrix}
		0.1 & 0.6 \\
		0.9 & 0.4 
		\end{bmatrix}$ (\textcolor{purple}{---}), fully-informative channel $T_{id}$ (\textcolor{red}{---}), and non-informative channel $T_{\bullet \cX}$ (\textcolor{blue}{---}). \label{gen-dob-fig}}
\end{figure}
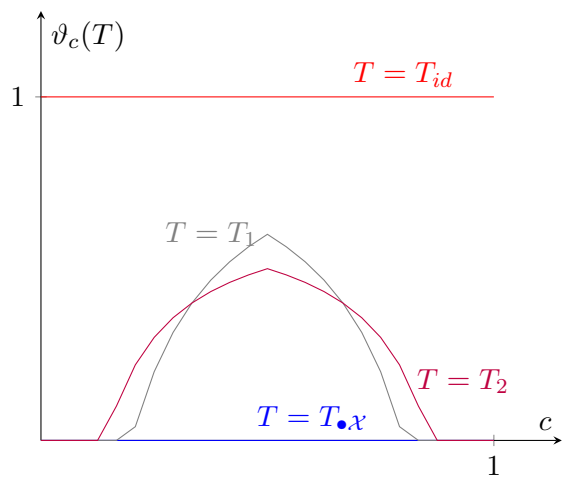


We now relate the contraction coefficient $\eta_{f_c} (T)$ w.r.t.\ $c$-primitive $f$-divergence to the generalized Dobrushin's coefficient $\vartheta_c (T)$ defined above.
\begin{theorem}
	\label{contract-fc-theorem}
	For any $T \in \cM \br{\cX , \cY}$, and $c \in \bs{0,1}$ we have 
	\[
	\eta_{f_c} \br{T} ~=~ \vartheta_c \br{T} 
	\]
	where $\eta_{f_c} \br{T}$ is the contraction coefficient of the channel $T$ w.r.t.\ $c$-primitive $f$-divergence.
\end{theorem}
\begin{proof}
	We follow the proof of Theorem~3.1 in \citep{Raginsky2014}.
	
	We start with the following generalization of \textit{strong Markov contraction lemma} from \cite{cohen1993relative}: for any signed measure $\tilde{\nu}$ on $\cX$, any $c \in \bs{0,1}$ and any Markov kernel $T \in \cM\br{\cX , \cY}$, we have
	\begin{equation}
	\label{contract-lemma}
	\norm{T \circ \tilde{\nu}}_{\mathrm{TV}} \leq \vartheta_c \br{T} \norm{\tilde{\nu}}_{\mathrm{TV}} + \br{1 - \vartheta_c \br{T}} \card{\tilde{\nu}\br{\cX}} ,
	\end{equation}
	where the total variation norm $\norm{\tilde{\nu}}_{\mathrm{TV}}$ of signed measure $\tilde{\nu}$ is given by $\norm{\tilde{\nu}}_{\mathrm{TV}} := \sum_{x \in \cX}{\abs{\tilde{\nu}\br{x}}}$. This can be shown by simply following through the steps of the proof of Lemma 3.2 in \cite{cohen1993relative}. Let $\tilde{\nu} = c \nu - \bar{c} \mu$, where $\nu$ and $\mu$ are probability measures on $\cX$. Then $T \circ\tilde{\nu} = c T \circ \nu - \bar{c} T \circ \mu$ and $\card{\tilde{\nu}\br{\cX}} = \card{2 c - 1}$. By using \eqref{contract-lemma} and the definition of $c$-primitive $f$-divergence \eqref{eq-primitive-f-div-4}, we get 
	\begin{align*}
	& d_{\mathrm{TV}}\br{c T \circ \nu , \bar{c} T \circ \mu} ~\leq~ \vartheta_c \br{T} d_{\mathrm{TV}}\br{c \nu , \bar{c} \mu} + \br{1 - \vartheta_c \br{T}} \card{2 c - 1} \\
	\implies& d_{\mathrm{TV}}\br{c T \circ \nu , \bar{c} T \circ \mu} - \card{2 c - 1} ~\leq~ \vartheta_c \br{T} \bc{d_{\mathrm{TV}}\br{c \nu , \bar{c} \mu} - \card{2 c - 1}} \\
	\implies& \bI_{f_c} \br{T \circ \nu , T \circ \mu} ~\leq~ \vartheta_c \br{T} \cdot \bI_{f_c} \br{\nu , \mu} .
	\end{align*}
	Now it remains to show that this bound is achieved for some probability measures $\mu$ and $\nu$.
	
	To that end, let us first assume that $\card{\cX} > 2$. Let $x_0 , x_1 \in \cX$ achieve the maximum in \eqref{general-dob-coeff-defn}, pick some $\epsilon_1 , \epsilon_2 , \epsilon \in (0, 1)$ such that $\epsilon_1 \neq \epsilon_2 , \epsilon_1 + \epsilon < 1, \epsilon_2 + \epsilon < 1$, and consider the following probability distributions:
	\begin{itemize}
		\item $\nu$ that puts the mass $1 - \epsilon_1 - \epsilon$ on $x_0$, $\epsilon_1$ on $x_1$, and distributes the remaining mass of $\epsilon$ evenly among the set $\cX \setminus \bc{x_0, x_1}$;
		\item $\mu$ that puts the mass $1 - \epsilon_2 - \epsilon$ on $x_0$, $\epsilon_2$ on $x_1$, and distributes the remaining mass of $\epsilon$ evenly among the set $\cX \setminus \bc{x_0, x_1}$.
	\end{itemize}
	Then a simple calculation (using a mathematical software) gives
	\begin{align*}
	\bI_{f_c}\br{\nu , \mu} =& \frac{1}{2} \bc{\card{c \epsilon_1 - \bar{c}\epsilon_2} + \abs{c \bar{\br{\epsilon_1 + \epsilon}} - \bar{c} \bar{\br{\epsilon_2+\epsilon}}} + \br{\card{\cX} - 2} \cdot \card{c \epsilon - \bar{c}\epsilon} - \card{2 c - 1}} \\
	\bI_{f_c}\br{T \circ \nu , T \circ \mu} =& \vartheta_c \br{T} \cdot \bI_{f_c}\br{\nu , \mu} .
	\end{align*}
	For $\card{\cX} = 2$, the idea is the same, except that there is no need for the extra slack $\epsilon$.
\end{proof}
From the above theorem, for any $T \in \cM \br{\cX , \cY}$ and for all $P, Q \in \cP \br{\cX}$, we have 
\[
\bI_{f_c} \br{T \circ P , T \circ Q} ~\leq~ \vartheta_c \br{T} \bI_{f_c} \br{P,Q} .
\]
Thus for any general $f$-divergence (using the weighted integral representation) we get
\begin{align}
\bI_{f} \br{T \circ P , T \circ Q} ~=~& \int_{0}^{1}{\gamma_f\br{c} \bI_{f_c} \br{T \circ P , T \circ Q} dc} \nonumber \\
~\leq~& \int_{0}^{1}{\vartheta_c \br{T} \gamma_f\br{c} \bI_{f_c} \br{P,Q} dc} \label{temp-sdpi-eqq} \\
~\leq~& \vartheta_{\frac{1}{2}} \br{T} \cdot \bI_f \br{P,Q}  ,\nonumber
\end{align}
where the last inequality is due to the fact that $\vartheta_c \br{T} \leq \vartheta_{\frac{1}{2}} \br{T}$. Even though this doesn't fully answer the question of how loose the universal bound \eqref{universal-dob} can be, \eqref{temp-sdpi-eqq} along with the Figure~\ref{gen-dob-fig}, sheds some light in that direction. 

\begin{remark}
	Let $f$ be a convex function with $f\br{1} = 0$, which can be written in the following form 
	\[
	f(u) = \alpha u + \beta u^2 + \int_{0}^{\infty}{\br{\frac{tu}{1+t^2} - \frac{u}{u+t}} v(dt)}, 
	\]
	where $\alpha \in \mathbb{R}, \beta \geq 0$, and $v$ is a non-negative measure on $\bR_+$ such that $\int_{0}^{\infty}{\frac{1}{1+t^2}v(dt)} < \infty$.	Note that for such function $f$, we have
	\begin{equation*}
	f''(u) ~=~ 2\beta + 2 \int_{0}^{\infty}{\frac{t}{(u+t)^3} v(dt)}
	\end{equation*}
	and thus 
	\begin{equation}
	\label{operator-cvx-weight}
	\gamma_f (c) ~=~ \frac{2}{c^3} f''\br{\frac{\bar{c}}{c}} = \frac{2}{c^3} \br{\beta+\int_{0}^{\infty}{\frac{t}{\br{\frac{\bar{c}}{c}+t}^3} v(dt)}} .
	\end{equation}
	\cite{Raginsky2014} has shown that for this class of functions, $\eta_{f} (T) = S(T)^2$, where
	\[
	S(T) ~:=~ \sup_{\mu}{\sup_{f,g}{\Ee{(\textsf{X},\textsf{Y}) \sim \mu \otimes T}{f(\textsf{X}) g(\textsf{Y})}}} ,
	\]
	for  $\E{f(\textsf{X})} = \E{g(\textsf{Y})} = 0$ and $\E{f(\textsf{X})^2} = \E{g(\textsf{Y})^2} = 1$. Still it is hard to compute $S(T)$ for any non-trivial channels. 
	
	Following divergences are generated from the functions that satisfy the above mentioned conditions:
	\begin{itemize}
		\item $\mathrm{KL}$-divergence satisfies \eqref{operator-cvx-weight} with $\gamma_{\mathrm{KL}}\br{c} = \frac{1}{c^2 \bar{c}}$, $\beta = 0$, and $v\br{dt} = dt$.
		\item $\chi^2$-divergence satisfies \eqref{operator-cvx-weight} with $\gamma_{\chi^2}\br{c} = \frac{1}{c^3}$, $\beta = 1$, and $v\br{dt} = 0$.
		\item squared Hellinger divergence satisfies \eqref{operator-cvx-weight} with $\gamma_{\mathrm{He}^2}\br{c} = \frac{1}{2 \br{c \bar{c}}^{\frac{3}{2}}}$, $\beta = 0$, and $v\br{dt} = \frac{2}{\pi \sqrt{t}} dt$.
	\end{itemize}
\end{remark}

\subsection{Binary Symmetric Channels}
\label{subsec:bsc}
Despite the extensive research in the strong data processing inequalities, an efficiently computable closed form for the contraction coefficient of a channel w.r.t.\ most of the $f$-divergences are not known. Indeed it is not understood at least for the simplest case of symmetric channels. Here we point to a technical report \cite{makur2016comparison}, which attempts to find simpler criteria for a given channel $T$ being dominated by a symmetric channel $W$ (in the sense that $\bI_{\mathrm{KL}}\br{W \circ \mu , W \circ \nu} \geq \bI_{\mathrm{KL}}\br{T \circ \mu , T \circ \nu}$, $\forall{\mu , \nu \in \cP\br{\cX}}$). This suggests that obtaining an efficiently computable closed form for the contraction coefficient of symmetric channels might guide us in upper bounding the contraction coefficient of general channels. 

In any case, since $\eta_f \br{T}$ is not known in a computable form for any non-trivial channels, we will now consider the \textit{binary symmetric channel} (BSC) $T:\bs{2} \rightsquigarrow \bs{2}$. Let $\textsf{X}$ and $\textsf{Y}$ be the input and output random variables of the channel respectively. This BSC can be written in a matrix form as follows
\[
T = \begin{bmatrix}
p & 1-p \\
1-p & p 
\end{bmatrix} , \quad p \in \bs{0,1}
\] 
where the rows represent the outputs of the channel, and the columns represent the inputs. The $(i,j)$-th entry of the matrix represents $\bP\bs{\textsf{Y}=i \mid \textsf{X}=j}$.  

To better understand the insights of the contraction coefficients of a BSC w.r.t.\ $f$-divergences, we consider the restrictive setting of symmetric $f$-divergences with symmetric experiments. First we define the following classes:
\begin{align*}
\cF_{\mathrm{symm}} ~:=~& \bc{f:\br{0,\infty} \rightarrow \bR \, : f \text{ is convex}, f(1)=0, \text{ and } f(x)=f^\diamond (x), \forall{x}} , \\
\cP_{\mathrm{symm}} ~:=~& \bc{P : P = \begin{bmatrix}
	p & 1-p \\
	1-p & p 
	\end{bmatrix} \text{ for some } p \in \br{0,1}} .
\end{align*}
We note the following properties of the function class $\cF_{\mathrm{symm}}$: 
\begin{itemize}
	\item $\bI_f \br{P,Q} = \bI_f \br{Q,P}$ for all distributions $P$ and $Q$ and $f \in \cF_{\mathrm{symm}}$.
	\item For any $f,g \in \cF_{\mathrm{symm}}$ and $\alpha , \beta \in \bR$, we have $\alpha f + \beta g \in \cF_{\mathrm{symm}}$.
	\item Defining a function $f \in \cF_{\mathrm{symm}}$ only in the domain $(0,1]$ is sufficient, as we can extend it for $x \in [1,\infty)$ by using the symmetric property $f\br{x}=f^\diamond (x)=x\cdot f\br{\frac{1}{x}}$. 
\end{itemize}
Define $f_{\mathrm{tv}} \br{x} := \abs{x-1}$, $f_{\mathrm{tri}} \br{x} := \frac{(x-1)^2}{x+1}$, $f_{\mathrm{He}} \br{x} := \br{\sqrt{x}-1}^2$, $f_{\mathrm{tvtri}} \br{x} := \frac{f_{\mathrm{tv}} \br{x} + f_{\mathrm{tri}} \br{x}}{2}$, $f_{\mathrm{tvHe}} \br{x} := \frac{f_{\mathrm{tv}} \br{x} + f_{\mathrm{He}} \br{x}}{2}$, and $f_{\mathrm{triHe}} \br{x} := \frac{f_{\mathrm{tri}} \br{x} + f_{\mathrm{He}} \br{x}}{2}$. Here $f_{\mathrm{tv}}, f_{\mathrm{tri}}$, and $f_{\mathrm{He}}$ are associated with total variation, triangular discrimination, and squared Hellinger divergences respectively (\cite{reid2011information}). One can easily verify that $f_{\mathrm{tv}}$, $f_{\mathrm{tri}}$, $f_{\mathrm{He}}$, $f_{\mathrm{tvtri}}$, $f_{\mathrm{tvHe}}$, $f_{\mathrm{triHe}} \in \cF_{\mathrm{symm}}$.

Note that the composition of the channel $T = \begin{bmatrix}
t & 1-t \\
1-t & t 
\end{bmatrix} \in \cP_{\mathrm{symm}}$ and the symmetric experiment $E = \begin{bmatrix}
e & 1-e \\
1-e & e 
\end{bmatrix} \in \cP_{\mathrm{symm}}$ can be written as follows
\[
T \circ E = \begin{bmatrix}
te+(1-t)(1-e) & t(1-e)+(1-t)e \\
t(1-e)+(1-t)e & te+(1-t)(1-e) 
\end{bmatrix} = \begin{bmatrix}
s(t,e) & 1-s(t,e) \\
1-s(t,e) & s(t,e) 
\end{bmatrix} ,
\]
where $s(t,e) := t e + (1-t)(1-e)$. The auxiliary notion that we are mainly interested here is: 
\begin{equation}
\label{aux-obj-def}
\eta_f^{\mathrm{symm}} \br{T} ~:=~ \sup_{E \in \cP_{\mathrm{symm}}}{\frac{\bI_f \br{T \circ E}}{\bI_f \br{E}}} \quad \text{where } f \in \cF_{\mathrm{symm}}, \text{ and } T \in \cP_{\mathrm{symm}} .
\end{equation}
Observe that $\eta_f^{\mathrm{symm}} \br{T} \leq \eta_f \br{T}$ for any BSC $T$ (where $\eta_f \br{T}$ is the contraction coefficient of $T$ w.r.t.\ $f$-divergence). The above identity is simplified in the following Lemma, using the symmetric nature of the objects involved:
\begin{lemma}
	\label{symm-contract-first-indentity}
	For any channel $T = \begin{bmatrix}
	t & 1-t \\
	1-t & t 
	\end{bmatrix} \in \cP_{\mathrm{symm}}$ and $f \in \cF_{\mathrm{symm}}$, we have
	\begin{equation}
	\label{aux-obj-form-1}
	\eta_f^{\mathrm{symm}} \br{T} ~=~ \sup_{e \in (0,1)}{\frac{s(t,e) \cdot f \br{\frac{1-s(t,e)}{s(t,e)}}}{e \cdot f \br{\frac{1-e}{e}}}} , 
	\end{equation}
	where $s(t,e) := t e + (1-t)(1-e)$.
\end{lemma}
\begin{proof}
For any binary symmetric channel $A = \begin{bmatrix}
a & 1-a \\
1-a & a 
\end{bmatrix}$ (with $a \in \br{0,1}$) and $f \in \cF_{\mathrm{symm}}$, we have
\begin{align*}
\bI_f \br{A} ~=~& \bI_f \br{ \begin{bmatrix}
	a  \\
	1-a  
	\end{bmatrix} ,  \begin{bmatrix}
	1-a \\
	a 
	\end{bmatrix} } \\
~=~& \int{f\br{\frac{dP}{dQ}} dQ} \\
~=~& f\br{\frac{a}{1-a}} \cdot \br{1-a} + f\br{\frac{1-a}{a}} \cdot a \\
~=~& \br{1-a} \cdot \frac{a}{1-a} \cdot f\br{\frac{1-a}{a}} + a \cdot f\br{\frac{1-a}{a}} \\
~=~& 2 a \cdot f\br{\frac{1-a}{a}} ,
\end{align*}
where the fourth equality is due to $f \br{t} = f^\diamond \br{t}$. Thus for any symmetric channel $T \in \cP_{\mathrm{symm}}$ and symmetric experiment $E \in \cP_{\mathrm{symm}}$, we get
\begin{align*}
\bI_f \br{E} ~=~& 2 e \cdot f\br{\frac{1-e}{e}} \\
\bI_f \br{T \circ E} ~=~& 2 s \cdot f\br{\frac{1-s}{s}} ,
\end{align*}
where $s= t e + (1-t)(1-e)$. Thus the proof is completed by taking the ratio between the above two divergences.
\end{proof}

We can better understand $\eta_f^{\mathrm{symm}} \br{T}$ by defining the following functions (for $e,t \in (0,1)$)
\begin{align}
F_f(e) ~:=~& e \cdot f \br{\frac{1-e}{e}} \label{Fe-def} \\
g_f(t,e) ~:=~& \frac{F_f\br{s(t,e)}}{F_f(e)} = \frac{s(t,e) \cdot f \br{\frac{1-s(t,e)}{s(t,e)}}}{e \cdot f \br{\frac{1-e}{e}}} . \label{gte-def}
\end{align}
Therefore $\eta_f^{\mathrm{symm}} \br{T}$ can be compactly written as follows (from Lemma~\ref{symm-contract-first-indentity})
\begin{equation}
\label{aux-obj-form-2}
\eta_f^{\mathrm{symm}} \br{T} ~=~ \sup_{e \in \br{0,1}}{g_f (t,e)} . 
\end{equation}
We attempt to characterize $\eta_f^{\mathrm{symm}} \br{T}$, by studying the behavior of $g_f (t,e)$. First we note the symmetric nature of $F_f\br{e}$ in the following lemma:
\begin{lemma}
Let $f \in \cF_{\mathrm{symm}}$. Then $F_f$ defined in \eqref{Fe-def} is convex, non-negative, and symmetric about $\frac{1}{2}$ with $F_f \br{\frac{1}{2}} = 0$.
\end{lemma}
\begin{proof}
	First we show that $f(x) \geq 0, \forall{x \in \br{0,\infty}}$ by using the facts that $f(1)=0$ and $f(x) = f^\diamond (x) = x f\br{\frac{1}{x}}$. Observe that showing $f(x) \geq 0, \forall{x \in \br{0,1}}$ is sufficient. Suppose that $\exists x \in (0,1) \text{ s.t. } f(x) < 0$. Then for $x' = \frac{1}{x} \in (1,\infty)$, we have $f(x') = f^\diamond (x') = x' f\br{\frac{1}{x'}} < 0$. But $f(1) = 0$ and $f$ is convex. This is a contradiction. Thus $f(x) \geq 0, \forall{x \in \br{0,1}}$.
	 
	Consider 
	\[
	F_f\br{\frac{1}{2} + \epsilon} = \br{\frac{1}{2} + \epsilon} \cdot f \br{\frac{1- \br{\frac{1}{2} + \epsilon}}{\frac{1}{2} + \epsilon}} = \br{\frac{1}{2} + \epsilon} \cdot f \br{\frac{\frac{1}{2} - \epsilon}{\frac{1}{2} + \epsilon}} 
	\]
	and 
	\[
	F_f\br{\frac{1}{2} - \epsilon} = \br{\frac{1}{2} - \epsilon} \cdot f \br{\frac{1- \br{\frac{1}{2} - \epsilon}}{\frac{1}{2} - \epsilon}} = \br{\frac{1}{2} - \epsilon} \cdot f \br{\frac{\frac{1}{2} + \epsilon}{\frac{1}{2} - \epsilon}} .
	\]
	Then using the property $f(x) = x f\br{\frac{1}{x}}$, one can easily see that $F_f\br{\frac{1}{2} + \epsilon} = F_f\br{\frac{1}{2} - \epsilon}$. Thus $F_f(x)$ is even symmetric about $\frac{1}{2}$.
	
	\begin{itemize}
		\item $F_f \br{\frac{1}{2}} = 0$ since $f(1) = 0$.
		\item $F_f(x) \geq 0 , \forall{x \in \br{0,1}}$ since $f(x) \geq 0 , \forall{x \in (0,\infty)}$.
		\item $F_f(x)$ is convex because it is a perspective transform of $f$ which is convex.
	\end{itemize}
\end{proof}

Note that for any $f,g \in \cF_{\mathrm{symm}}$ and $\alpha , \beta \in \bR$, we have $F_{\alpha f + \beta g} (x) = \alpha F_f (x) + \beta F_g (x)$. By using the symmetric nature of $F_f\br{\cdot}$, we can further simplify the identity $\eta_f^{\mathrm{symm}} \br{T}$. Since $\eta_f^{\mathrm{symm}}\br{\begin{bmatrix}
	t & 1-t \\
	1-t & t 
	\end{bmatrix}} = \eta_f^{\mathrm{symm}}\br{\begin{bmatrix}
	1-t & t \\
	t & 1-t 
	\end{bmatrix}}$, hereafter we assume $t \geq \frac{1}{2}$ without loss of generality. 
\begin{lemma}
For a given fixed channel $T = \begin{bmatrix}
t & 1-t \\
1-t & t 
\end{bmatrix} \in \cP_{\mathrm{symm}}$ with $t \geq \frac{1}{2}$ (wlog) and $f \in \cF_{\mathrm{symm}}$, define 
\begin{equation}
\phi_f^t \br{\epsilon} ~:=~ \frac{F_f\br{\frac{1}{2} + c_t \epsilon}}{F_f \br{\frac{1}{2} + \epsilon}} ,
\end{equation}
where $c_t = 2 t - 1 \in [0,1)$ and $\epsilon \in \br{-\frac{1}{2} , \frac{1}{2}}$. Then we have
\begin{equation}
\eta_f^{\mathrm{symm}} \br{T} ~=~ \sup_{\epsilon \in [0,1/2)}{\phi_f^t\br{\epsilon}} .
\end{equation}
\end{lemma}
\begin{proof}
Let $e = \frac{1}{2} + \epsilon$ where $\epsilon \in \br{-\frac{1}{2} , \frac{1}{2}}$. Then we have
\begin{align*}
s(t,e) ~=~& \frac{1}{2} + \epsilon (2 t - 1) = \frac{1}{2} + c_t \epsilon , \text{ where } c_t = 2 t - 1 \in [0,1) \\
g_f(t,e) ~=~& \frac{F_f\br{s(t,e)}}{F_f (e)} = \frac{F_f\br{\frac{1}{2} + c_t \epsilon}}{F_f \br{\frac{1}{2} + \epsilon}} = \phi_f^t \br{\epsilon} . 
\end{align*}
Observe that $\phi_f^t \br{\epsilon}$ is symmetric about $0$ since $F_f\br{\cdot}$ is symmetric about $\frac{1}{2}$. Then by using \eqref{aux-obj-form-2} we get
\[
\eta_f^{\mathrm{symm}} \br{T} = \sup_{e \in (0,1)}{g_f(t,e)} = \sup_{\epsilon \in \br{-1/2,1/2}}{\phi_f^t\br{\epsilon}} = \sup_{\epsilon \in [0,1/2)}{\phi_f^t\br{\epsilon}} .
\]
\end{proof}
Let $L_f \br{\epsilon} := F_f\br{1/2 + \epsilon}$. Then for fixed $c_t \in [0,1)$ we have
\[
\phi_f^t \br{\epsilon} = \frac{L_f \br{c_t \epsilon}}{L_f \br{\epsilon}} , \text{ where } \epsilon \in [0,1/2) . 
\]
Note that $L_f \br{0} = 0 , L_f \br{\cdot} \geq 0$ and $L_f$ is convex (for $f \in \cF_{\mathrm{symm}}$). Since we want to study the behavior of $\phi_f^t \br{\epsilon}$, we consider the derivative of it
\[
\br{\phi_f^t}' \br{\epsilon} = \frac{\partial}{\partial \epsilon} \phi_f^t \br{\epsilon} = \frac{c_t L_f' \br{c_t \epsilon} L_f \br{\epsilon} - L_f \br{c_t \epsilon} L_f' \br{\epsilon}}{{L_f \br{\epsilon}}^2} . 
\] 
Based on this we can observe two important behavior patterns of $\phi_f^t \br{\epsilon}$ :
\begin{enumerate}
	\item If $\br{\phi_f^t}' \br{\epsilon} \leq 0$, $\forall{\epsilon \in (0,1/2)}$, then $\phi_f^t \br{\epsilon}$ is maximized at $\epsilon \rightarrow 0$, minimized at $\epsilon \rightarrow 1/2$. That is
	\[
	\lim_{\epsilon \rightarrow 1/2} {\phi_f^t\br{\epsilon}} \leq \phi_f^t \br{\epsilon} \leq \lim_{\epsilon \rightarrow 0} {\phi_f^t\br{\epsilon}} 
	\]
	which is equivalent to 
	\begin{equation}
	\label{gte-cond-1}
	\lim_{e \rightarrow 1} {g_f (t,e)} \leq g_f (t,e) \leq \lim_{e \rightarrow 1/2} {g_f (t,e)}.
	\end{equation}
	\item If $\br{\phi_f^t}' \br{\epsilon} = 0$, $\forall{\epsilon \geq 0}$, then $\phi_f^t \br{\epsilon}$ is equal for all $\epsilon \in (0,1/2)$. That is
	\[
	\lim_{\epsilon \rightarrow 1/2} {\phi_f^t\br{\epsilon}} = \phi_f^t \br{\epsilon} = \lim_{\epsilon \rightarrow 0} {\phi_f^t\br{\epsilon}} 
	\]
	which is equivalent to 
	\begin{equation}
	\label{gte-cond-2}
	\lim_{e \rightarrow 1} {g_f (t,e)} = g_f (t,e) = \lim_{e \rightarrow 1/2} {g_f (t,e)}.
	\end{equation}
\end{enumerate}
For the above two cases we have 
\[
\eta_f^{\mathrm{symm}} \br{T} = \sup_{e \in (0,1)}{g_f(t,e)} = \lim_{e \rightarrow 1/2} {g_f (t,e)} .
\]
Note that $g_f (t,1/2)$ is not well defined. But for the second case above, where $\br{\phi_f^t}' \br{\epsilon} = 0$, $\forall{\epsilon \in (0,1/2)}$, we can obtain an efficiently computable closed form for $\eta_f^{\mathrm{symm}} \br{T}$. The following proposition characterizes the subclass of $\cF_{\mathrm{symm}}$ which satisfies this condition.
\begin{proposition}
Define $h_\alpha \br{x} := \frac{\abs{1-x}^\alpha}{(1+x)^{\alpha - 1}}$, for $x \in (0,1]$ and $\alpha \in \bR$. Then 
\begin{multline*}
\cF_{\mathrm{symm}}^* := \left\{ f:(0,\infty) \rightarrow \bR : \forall{x \in (0,1]}, f(x)=K \cdot h_{\alpha_f} (x) \text{ for some } K>0, \alpha_f \geq 1 , \right. \\ 
\left. \text{ and } \forall{x \in [1,\infty)}, f(x)=f^\diamond (x) \right\} \subseteq \cF_{\mathrm{symm}} .
\end{multline*}
For any $T = \begin{bmatrix}
t & 1-t \\
1-t & t 
\end{bmatrix} \in \cP_{\mathrm{symm}}$ (with $t \geq \frac{1}{2}$) and $f \in \cF_{\mathrm{symm}}^*$, we get
\[
\eta_f^{\mathrm{symm}} \br{T} ~=~ \lim_{e \rightarrow 1} {g_f (t,e)} ~=~ (2 t - 1)^{\alpha_f} .
\]
\end{proposition}
\begin{proof}
For any $f \in \cF_{\mathrm{symm}}^*$, we have $f(1)=0$, $f(x)=f^\diamond\br{x}$, and $f$ is convex (since $h_\alpha$ is convex for $\alpha \geq 1$). Thus $\cF_{\mathrm{symm}}^* \subseteq \cF_{\mathrm{symm}}$.

If 
\[
\frac{c_t L_f' \br{c_t \epsilon}}{L_f \br{c_t \epsilon}} = \frac{L_f' \br{\epsilon}}{L_f \br{\epsilon}} ,
\]
then $\br{\phi_f^t}' \br{\epsilon} = 0$, $\forall{\epsilon \in (0,1/2)}$ (thus $\eta_f^{\mathrm{symm}} \br{T} = \lim_{e \rightarrow 1} {g_f (t,e)}$). By letting $\psi = \log{L_f}$, the above condition can be written as follows,
\[
c_t \psi' \br{c_t \epsilon} = \psi' \br{\epsilon}
\]
that is we require $\psi'$ to be $(-1)$-homogeneous. For a function $\psi' \br{x} = \alpha x^{-1}$ which is $(-1)$-homogeneous, we have (for some constant $K > 0$)
\begin{align*}
& \psi' \br{x} ~=~ \alpha \frac{1}{x} , \, x \geq 0 \, (\text{to enforce symmetry}) \\
\impliedby & \psi \br{x} ~=~ \alpha \log{x} + \log{K} ~=~ \log{L_f\br{x}} \\
\impliedby & L_f \br{x} ~=~ K x^\alpha ~=~ F_f \br{1/2 + x} \\
\impliedby & F_f \br{y} ~=~ K \br{y - 1/2}^\alpha ~=~ y f\br{\frac{1-y}{y}} , \, \text{ where } y = 1/2 + x \geq 1/2 \\ 
\impliedby & f \br{z} ~=~ K \cdot \frac{(1-z)^\alpha}{(1+z)^{\alpha - 1}} , \, \text{ where } z = \frac{1-y}{y} \leq 1 . \\ 
\end{align*}
That is for any $f \in \cF_{\mathrm{symm}}^*$, we have $\br{\phi_f^t}' \br{\epsilon} = 0$, $\forall{\epsilon \in (0,1/2)}$. Thus for any $f \in \cF_{\mathrm{symm}}^*$, we get
\[
\eta_f^{\mathrm{symm}} \br{T} = \lim_{e \rightarrow 1} {g_f (t,e)} = \frac{F_f \br{t}}{\lim_{e \rightarrow 1} {F_f \br{e}}} = \frac{t \cdot h_{\alpha_f} \br{\frac{1-t}{t}}}{\lim_{x \rightarrow 0} {h_{\alpha_f} \br{x}}} = \frac{(2 t - 1)^{\alpha_f}}{1} .
\]
\end{proof}
Note that $f_{\mathrm{tv}}, f_{\mathrm{tri}} \in \cF_{\mathrm{symm}}^*$ with $\alpha_{f_{\mathrm{tv}}} = 1$ and $\alpha_{f_{\mathrm{tri}}} = 2$ (recall that $f_{\mathrm{tv}} \br{t} = \abs{t-1}$, and $f_{\mathrm{tri}} \br{t} = \frac{(t-1)^2}{t+1}$). Thus from the above proposition and \eqref{gte-cond-2}, for any $T = \begin{bmatrix}
t & 1-t \\
1-t & t 
\end{bmatrix} \in \cP_{\mathrm{symm}}$ (with $t \geq 1/2$), we have 
\[
\eta_{f_{\mathrm{tv}}}^{\mathrm{symm}} \br{T} ~=~ \lim_{e \rightarrow 1/2} {g_{f_{\mathrm{tv}}} (t,e)} ~=~ 2 t - 1
\]
and
\[
\eta_{f_{\mathrm{tri}}}^{\mathrm{symm}} \br{T} ~=~ \lim_{e \rightarrow 1/2} {g_{f_{\mathrm{tri}}} (t,e)} ~=~ \br{2 t - 1}^2 .
\]
${g_{f_{\mathrm{tv}}} (t,e)}$ and ${g_{f_{\mathrm{tri}}} (t,e)}$ are shown in Figures~\ref{fig:fvar} and \ref{fig:ftri} respectively. 

For $f_{\mathrm{He}} \br{t} = \br{\sqrt{t}-1}^2$, we have 
\begin{align*}
F_{f_{\mathrm{He}}} \br{e} ~=~& 1 - 2 \sqrt{e \cdot (1-e)} , \quad e \in (0,1) \\
L_{f_{\mathrm{He}}} \br{\epsilon} ~=~& 1 - \sqrt{1 - 4 \epsilon^2} , \quad \epsilon \in [0,1/2) \\
\phi_{f_{\mathrm{He}}}^t \br{\epsilon} ~=~& \frac{1 - \sqrt{1 - 4 c_t^2 \epsilon^2}}{1 - \sqrt{1 - 4 \epsilon^2}} , \quad c_t = 2t-1 \in (0,1) \\
\lim_{\epsilon \rightarrow 0} {\phi_{f_{\mathrm{He}}}^t\br{\epsilon}} ~=~& c_t^2 \\
\lim_{\epsilon \rightarrow 1/2} {\phi_{f_{\mathrm{He}}}^t\br{\epsilon}} ~=~& 1 - \sqrt{1-c_t^2} .
\end{align*}
By using simple calculations, one can easily verify that (see Figure~\ref{fig:fhellinger})
\[
\lim_{\epsilon \rightarrow 1/2} {\phi_{f_{\mathrm{He}}}^t\br{\epsilon}} \leq \phi_{f_{\mathrm{He}}}^t \br{\epsilon} \leq \lim_{\epsilon \rightarrow 0} {\phi_{f_{\mathrm{He}}}^t\br{\epsilon}} = (2t-1)^2 =  \eta_{f_{\mathrm{He}}}^{\mathrm{symm}} \br{T} .
\]
We observed that $f_{\mathrm{tvtri}}$, $f_{\mathrm{tvHe}}$, and $f_{\mathrm{triHe}}$ also satisfy \eqref{gte-cond-1} (see Figures~\ref{fig:gftvtri}, \ref{fig:gftvh}, and \ref{fig:gftrih}): 
\begin{align*}
& \lim_{e \rightarrow 1} {g_{f_{\mathrm{tvtri}}} (t,e)} \leq g_{f_{\mathrm{tvtri}}} (t,e) \leq \lim_{e \rightarrow 1/2} {g_{f_{\mathrm{tvtri}}} (t,e)} = \eta_{f_{\mathrm{tvtri}}}^{\mathrm{symm}} \br{T} = 2t-1 \\
& \lim_{e \rightarrow 1} {g_{f_{\mathrm{tvHe}}} (t,e)} \leq g_{f_{\mathrm{tvHe}}} (t,e) \leq \lim_{e \rightarrow 1/2} {g_{f_{\mathrm{tvHe}}} (t,e)} = \eta_{f_{\mathrm{tvHe}}}^{\mathrm{symm}} \br{T} = 2t-1 \\
& \lim_{e \rightarrow 1} {g_{f_{\mathrm{triHe}}} (t,e)} \leq g_{f_{\mathrm{triHe}}} (t,e) \leq \lim_{e \rightarrow 1/2} {g_{f_{\mathrm{triHe}}} (t,e)} = \eta_{f_{\mathrm{triHe}}}^{\mathrm{symm}} \br{T} = (2t-1)^2 .
\end{align*}

\begin{figure}
	\centering
	\includegraphics[width=0.6\linewidth,height=0.6\linewidth]{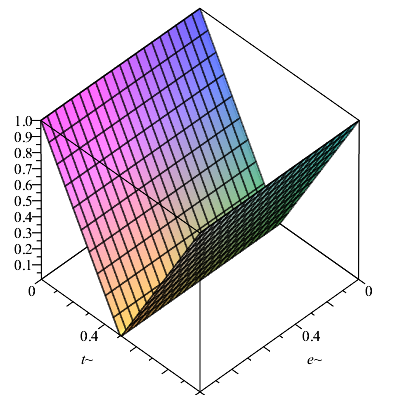}
	\caption[behavior of a binary symmetric channel w.r.t.\ total variation divergence]{$g_{f_{\mathrm{tv}}} (t,e)$ of a binary symmetric channel w.r.t.\ total variation divergence.}
	\label{fig:fvar}
\end{figure}

\begin{figure}
	\centering
	\includegraphics[width=0.6\linewidth,height=0.6\linewidth]{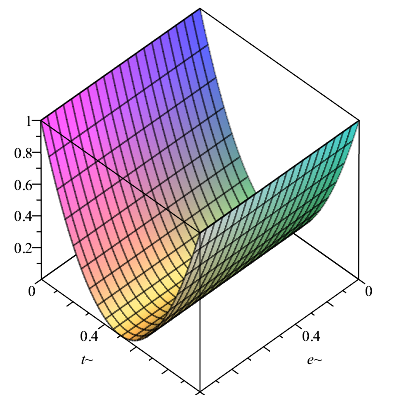}
	\caption[behavior of a binary symmetric channel w.r.t.\ triangular discrimination divergence]{$g_{f_{\mathrm{tri}}} (t,e)$ of a binary symmetric channel w.r.t.\ triangular discrimination divergence.}
	\label{fig:ftri}
\end{figure}

\begin{figure}
	\centering
	\includegraphics[width=0.6\linewidth,height=0.6\linewidth]{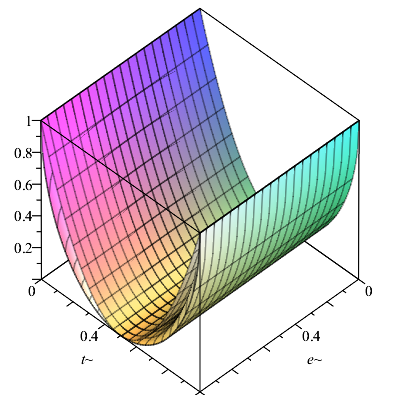}
	\caption[behavior of a binary symmetric channel w.r.t.\ symmetric squared Hellinger divergence]{$g_{f_{\mathrm{He}}} (t,e)$ of a binary symmetric channel w.r.t.\ symmetric squared Hellinger divergence (sandwiched according to \eqref{gte-cond-1}).}
	\label{fig:fhellinger}
\end{figure}

\begin{figure}
	\centering
	\includegraphics[width=0.6\linewidth,height=0.6\linewidth]{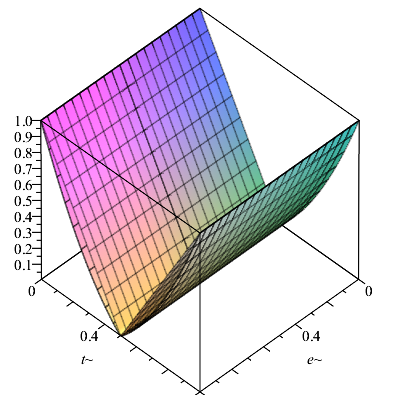}
	\caption[behavior of a binary symmetric channel w.r.t.\ $\bI_{f_{\mathrm{tvtri}}}$]{$g_{f_{\mathrm{tvtri}}} (t,e)$ of a binary symmetric channel w.r.t.\ $\bI_{f_{\mathrm{tvtri}}}$ (sandwiched according to \eqref{gte-cond-1}).}
	\label{fig:gftvtri}
\end{figure}

\begin{figure}
	\centering
	\includegraphics[width=0.6\linewidth,height=0.6\linewidth]{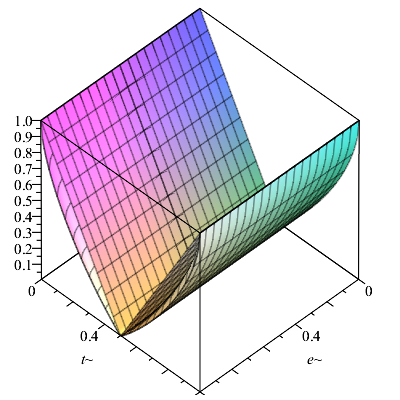}
	\caption[behavior of a binary symmetric channel w.r.t.\ $\bI_{f_{\mathrm{tvHe}}}$]{$g_{f_{\mathrm{tvHe}}} (t,e)$ of a binary symmetric channel w.r.t.\ $\bI_{f_{\mathrm{tvHe}}}$ (sandwiched according to \eqref{gte-cond-1}).}
	\label{fig:gftvh}
\end{figure}

\begin{figure}
	\centering
	\includegraphics[width=0.6\linewidth,height=0.6\linewidth]{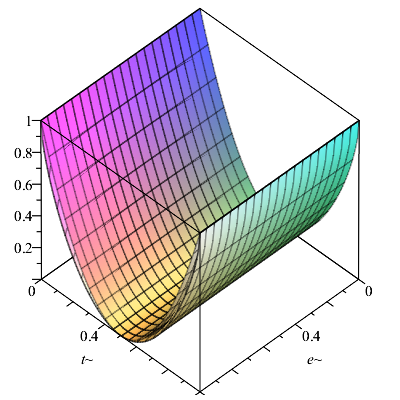}
	\caption[behavior of a binary symmetric channel w.r.t.\ $\bI_{f_{\mathrm{triHe}}}$]{$g_{f_{\mathrm{triHe}}} (t,e)$ of a binary symmetric channel w.r.t.\ $\bI_{f_{\mathrm{triHe}}}$ (sandwiched according to \eqref{gte-cond-1}).}
	\label{fig:gftrih}
\end{figure}

Thus for all binary symmetric channels, and certain subset of symmetric $f$-divergences, we are able to obtain a lower bound (of the form $\eta_f^{\mathrm{symm}} \br{T} \leq \eta_f \br{T}$) on the contraction coefficients. At this stage, we point out the following possible extensions for the above exercise (some of them will follow through the above approach to certain level): 
\begin{itemize}
	\item relax the symmetric experiments restriction in $\eta_f^{\mathrm{symm}} \br{T}$, to obtain $\eta_f \br{T}$ for $f \in \cF_{\mathrm{symm}}$ and $T \in \cP_{\mathrm{symm}}$. 
	\item extend the study of $\eta_f^{\mathrm{symm}} \br{T}$ to $k$-ary symmetric channels and experiments with $k > 2$.
	\item extend the study of $\eta_f^{\mathrm{symm}} \br{T}$ to non-symmetric $f$.
\end{itemize}

\subsection{Hardness of Constrained Learning Problem}

We now use the strong data processing inequalities and minimax lower bound techniques to analyse the hardness of the constrained learning problem that we introduced in the beginning of this section. First we generalize Le Cam's (Proposition~\ref{lecam-method}) and Assouad's (Theorem~\ref{assouad-method} and Corollary~\ref{assouad-corro}) results for the constrained parameter estimation problem (\eqref{constrained-task-eq} with $\Theta = \cA$, $A = \hat{\theta}$, and $\ell = \rho$). Most of these generalizations follows directly from the original versions, thus don't require any proof.
\begin{proposition}
	\label{lecam-method-corrupted}
	For any $c \in \br{0,1}$, the minimax risk of the constrained parameter estimation problem (\eqref{constrained-task-eq} with $\Theta = \cA$, $A = \hat{\theta}$, and $\ell = \rho$) is lower bounded as
	\[
	\underline{R}_\rho^\star \br{\tilde{\varepsilon}_n} ~\geq~ \sup_{\theta \neq \theta'}{\bc{\rho\br{\theta,\theta'} \cdot \bs{\frac{1}{2} - n \br{\frac{1}{2} - c \wedge \bar{c}} - \vartheta_c\br{T} n \cdot \bI_{f_c}\br{\varepsilon\br{\theta},\varepsilon\br{\theta'}}}}} .
	\]
\end{proposition}
\begin{proof}
	From Proposition~\ref{lecam-method}, we have that
	\[
	\underline{R}_\rho^\star \br{\tilde{\varepsilon}_n} ~\geq~ \sup_{\theta \neq \theta'}{\bc{\rho\br{\theta,\theta'} \cdot \br{c \wedge \bar{c} - \cdot \bI_{f_c}\br{\tilde{\varepsilon}_n\br{\theta},\tilde{\varepsilon}_n\br{\theta'}}}}} .
	\]
	Further from Lemma~\ref{sub-add-f-div}, we have that
	\[
	d_{\mathrm{TV}} \br{c \tilde{\varepsilon}_n\br{\theta} , \bar{c} \tilde{\varepsilon}_n\br{\theta'}} ~\leq~ n d_{\mathrm{TV}} \br{c \tilde{\varepsilon}\br{\theta} , \bar{c} \tilde{\varepsilon}\br{\theta'}} .
	\]
	Thus we have 
	\begin{align*}
	\bI_{f_c}\br{\tilde{\varepsilon}_n\br{\theta},\tilde{\varepsilon}_n\br{\theta'}} ~\leq~& n \bI_{f_c}\br{\tilde{\varepsilon}\br{\theta},\tilde{\varepsilon}\br{\theta'}} + \br{c \wedge \bar{c} - \frac{1}{2}} \cdot (1-n) \\
	~\leq~& \vartheta_c\br{T} n \bI_{f_c}\br{\tilde{\varepsilon}\br{\theta},\tilde{\varepsilon}\br{\theta'}} + \br{c \wedge \bar{c} - \frac{1}{2}} \cdot (1-n) .
	\end{align*}
\end{proof}

\begin{theorem}
	\label{assouad-method-corrupted}
	Let $\Theta=\bc{-1,1}^d$ and $\rho = \rho_{\mathrm{Ha}}$ (defined in \eqref{eq-hamming-distance}). Then for any $c \in \br{0,1}$, the minimax risk of the constrained parameter estimation problem (\eqref{constrained-task-eq} with $\Theta = \cA$, $A = \hat{\theta}$, and $\ell = \rho$) is lower bounded as  
	\[
	\underline{R}_{\rho_{\mathrm{Ha}}}^\star \br{\tilde{\varepsilon}_n} ~\geq~ d \br{c \wedge \bar{c} - \underset{\theta,\theta': \rho_{\mathrm{Ha}}\br{\theta,\theta'}=1}{\max} \bI_{f_c}\br{\tilde{\varepsilon}_n\br{\theta},\tilde{\varepsilon}_n\br{\theta'}}}
	\]
\end{theorem}

\begin{corollary}
	\label{assouad-corro-corrupted}
	Let $\cO$ be some set and $c \in \br{0,1}$. Define 
	\[
	\cP_\varepsilon \br{\cO} := \bc{\varepsilon\br{\theta} \in \cP \br{\cO} : \theta \in \bc{-1,1}^d}
	\]
	be a class of probability measures induced by the transition $\varepsilon : \bc{-1,1}^d \rightsquigarrow \cO$. Suppose that there exists some cost-dependent constant $\alpha\br{c} > 0$, such that 
	\[
	\mathrm{He}^2 \br{\varepsilon\br{\theta},\varepsilon\br{\theta'}} \leq \alpha\br{c} , \quad \text{ if } \rho_{\mathrm{Ha}}\br{\theta,\theta'} = 1 .
	\] 
	The minimax risk of the constrained parameter estimation problem (\eqref{constrained-task-eq} with $\Theta = \cA$, $A = \hat{\theta}$, and $\ell = \rho_{\mathrm{Ha}}$) is lower bounded as
	\begin{equation}
	\label{corr-assouad-practical-eq}
	\underline{\cR}_{\rho_{\mathrm{Ha}}}^\star \br{\tilde{\varepsilon}_n} ~\geq~ d \cdot (c \wedge \bar{c}) \cdot \br{1 - \sqrt{\alpha\br{c} \eta_{\mathrm{He}^2}\br{T} n}} ,
	\end{equation}
	where $T$ is as per \eqref{constrained-task-eq} and $\eta_{\mathrm{He}^2}\br{T}$ is the contraction coefficient of $T$ w.r.t.\ squared Hellinger distance.
\end{corollary}
\begin{proof}
	For any two $\theta,\theta' \in \Theta$ with $\rho_{\mathrm{Ha}}\br{\theta,\theta'} = 1$, we have
	\begin{align*}
	\bI_{f_c}\br{\tilde{\varepsilon}_n\br{\theta},\tilde{\varepsilon}_n\br{\theta'}} &~\leq~ \br{c \wedge \bar{c}} \cdot \mathrm{He} \br{\tilde{\varepsilon}_n\br{\theta},\tilde{\varepsilon}_n\br{\theta'}} \\
	&~\leq~ \br{c \wedge \bar{c}} \cdot \sqrt{\sum_{i=1}^{n}{\mathrm{He}^2 \br{\tilde{\varepsilon}\br{\theta},\tilde{\varepsilon}\br{\theta'}}}} \\
	&~\leq~ \br{c \wedge \bar{c}} \cdot \sqrt{\sum_{i=1}^{n}{\eta_{\mathrm{He}^2}\br{T} \mathrm{He}^2 \br{\varepsilon\br{\theta},\varepsilon\br{\theta'}}}} \\
	&~\leq~ \br{c \wedge \bar{c}} \cdot \sqrt{\alpha \br{c} \eta_{\mathrm{He}^2}\br{T} n} \\
	\end{align*}
\end{proof}

Consider the corrupted cost-sensitive binary classification problem represented by the following transition diagram:
\begin{equation}
\label{corr-specific-learning-task-eq}
\corrspecificlearningtask ,
\end{equation}
and the minimax risk of it given by 
\begin{equation}
\label{corr-cost-minmax-risk-eq}
\underline{\cR}_{\Delta \ell_{d_{c}}}^\star \br{\tilde{\varepsilon}_n} := \inf_{\hat{f}}{\sup_{\theta \in \Theta_{h,\cF}}{\Ee{\bc{\br{\textsf{X}_i,\textsf{Y}_i}}_{i=1}^n \sim \tilde{\varepsilon}_n \br{\theta}}{\Ee{f \sim \hat{f}\br{\bc{\br{\textsf{X}_i,\textsf{Y}_i}}_{i=1}^n}}{\Delta \ell_{d_c} \br{\theta,f}}}}} .
\end{equation}

\begin{theorem}
	\label{corr-main-theo-cost-classi}
	Let $\cF$ be a VC class of binary-valued functions on $\cX$ with VC dimension $V \geq 2$. Then for any $n \geq V$ and any $h \in [0,c \wedge \bar{c}]$, the minimax risk \eqref{corr-cost-minmax-risk-eq} of the corrupted cost-sensitive binary classification \eqref{corr-specific-learning-task-eq} is lower bounded as follows: 
	\[
	\underline{\cR}_{\Delta \ell_{d_c}}^\star \br{\tilde{\varepsilon}_n} ~\geq~ K \cdot (c \wedge \bar{c}) \cdot \min\br{\sqrt{\frac{(c \wedge \bar{c}) V}{\eta_{\mathrm{He}^2}\br{T} n}},(c \wedge \bar{c}) \cdot \frac{V}{\eta_{\mathrm{He}^2}\br{T} nh}}
	\]
	where $K > 0$ is some absolute constant. 
\end{theorem}

The number of samples that appear in the minimax lower bound of the original learning problem is scaled by the contraction coefficient $\eta_f (T)$ in the case of corrupted learning problem. Hence the rate is unaffected, only the constants. However, a penalty of factor $\eta_f (T)$ is unavoidable no matter what learning algorithm is used, suggesting that $\eta_f (T)$ is a valid way of measuring the amount of corruption.

\section{Cost-sensitive Privacy Notions}
\label{sec:cost-privacy}
Suppose a trustworthy data curator gathers sensitive data from a large number of data providers, with the goal of learning statistical facts about the underlying population. A data analyst makes a statistical query on the sensitive dataset from the data curator. Thus the main challenge for the data curator is to send back a randomized response such that the utility of the task of the data analyst is increased while maintaining the privacy of the data providers. This requires a formal definition of privacy, and differential privacy has been put forth as such formalization (\cite{dwork2006calibrating}). Differential privacy requires that the data analyst knows no more about any individual in the sensitive dataset after the analysis is completed, than she knew before the analysis was begun. That is the impact on the data provider is the same independent of whether or not he was in the analysis. It is possible to reduce the problem of enforcing differential privacy to a statistical decision problem (\cite{wasserman2010statistical}). We exploit this observation and extend it further (see section~\ref{non-homo-privacy}).

In a more restrictive requirement than the differential privacy, called ``local privacy'' (\citep{duchi2013local,warner1965randomized}), the data providers don't even trust the data curator collecting the data. When the sensitive data to be protected is other than the value of a single individual, it is common to consider different definitions for privacy requirements.

A \textit{privacy mechanism} is an algorithm that takes as input a database, a universe $\cV$ of data types (of the database), and optionally a set of queries, and produce a randomized response. The privacy mechanism can be represented by a transition $T: \cO \rightsquigarrow \hat{\cO}$. Below we represent some of the privacy enforced settings via transition diagrams:
\begin{itemize}
	\item We need to protect an \textit{abstract set of secrets} $\cX$ (for example geographical locations of army base points) from the data analyst, who wants to learn some summary statistic about the probability distribution which generated the secrets i.e. something about the actual parameter $\theta$ from the parameter space $\Theta$. This setting is represented by the following transition diagram
	\begin{equation}
	\label{localprivacyrequirement}
	\localprivacytask , 
	\end{equation}
	where $T$ is the privacy mechanism and $\cZ$ is the new outcome space observed by the data analyst. When the outcome space $\cZ = \cX$, the resulting transition diagram is
	\begin{equation}
	\label{localprivacyrequirementspecial}
	\simplelocalprivacytask .
	\end{equation}
	\item We need to protect the \textit{entries of the database} by releasing a \textit{sanitized database} (this approach is also referred to as non-interactive method). Let the database universe be $\cV$. Then by repeatedly applying the privacy mechanism $T$ in the transition diagram~\eqref{localprivacyrequirementspecial} with $\cX = \cV$, over all the entries of the database, we get
	\begin{equation}
	\label{localprivacyrequirementspecialccc}
	\simplelocalprivacytaskccc .
	\end{equation}
	\item In comparison to the above non-interactive approach, it is possible to protect the entries of the database by corrupting the response for the database query appropriately. This interactive (query dependent) method can be represented by the following transition digram
	\begin{equation}
	\label{datbaseprivacyrequirement}
	\databaseprivacytask , 
	\end{equation}
	where $\cX = \cV^n$ is the database (with universe $\cV$) to be protected, $f$ is a query on the database, and $H$ is the privacy mechanism over the query outcome space $\cY$. We need to enforce restrictions on the composite mechanism $T = H \circ f$ in order to protect the elements of $\cX$. By appropriate tailoring, these restrictions can be reduced to the restrictions on the mechanism $H$ depending on $f$. 
\end{itemize}
Based on the discussions above, without loss of generality, we only consider the privacy definitions for the setting represented by the transition diagram~\eqref{localprivacyrequirement}, with finite $\cX$ and $\cZ$. 

\subsection{Symmetric Local Privacy}
First we briefly review the \textit{(symmetric) local privacy} notion which is well studied in the literature (\citep{dwork2008differential,duchi2013local}). Consider the setting represented by the transition diagram~\eqref{localprivacyrequirement} with finite $\cX$ and $\cZ$. The (symmetric) local privacy imposes \textit{indistinguishability} between pairs of secrets and protects all of them equally: 
\begin{definition}[\citep{duchi2013local}]
	\label{indistinguishability-localprivacy-defn}
	Given $\epsilon > 0$, let $\cM \br{\cX , \cZ ; \epsilon} \subseteq \cM \br{\cX , \cZ}$ denote the set of all $\epsilon$-locally private mechanisms where 
	\begin{equation}
	\label{indistinguishablelocalprivacydefinition}
	T \in \cM \br{\cX , \cZ ; \epsilon} ~\iff~ \frac{T(z \mid x_i)}{T(z \mid x_j)} \leq \eexp, \, \forall{x_i,x_j \in \cX, z \in \cZ}.
	\end{equation} 
\end{definition}

Below we provide a hypothesis testing based interpretation of the above definition, essentially noted by \cite{wasserman2010statistical}.

\paragraph*{Hypothesis Testing Interpretation:}

Based on the random outcome in $\cZ$ from the privacy mechanism $T$, we want determine whether it is generated by the secret $x_i$ or $x_j$. Let the labels $1$ and $0$ correspond to the probability measures $T\br{x_i}$ and $T\br{x_j}$ respectively. Consider a statistical test (recall from section~\ref{sec:stat-tests}) $r_{ij}: \cZ \rightarrow \bc{0,1}$. Then the false negative and false positive rates of this test are given by  
\begin{align}
\mathrm{FN}_{r_{ij}} ~:=~& \sum_{z \in \cZ}{T\br{z \mid x_i} \ind{r_{ij}\br{z}=0}} , \text{ and } \label{false-negative-def}\\
\mathrm{FP}_{r_{ij}} ~:=~& \sum_{z \in \cZ}{T\br{z \mid x_j} \ind{r_{ij}\br{z}=1}} , \label{false-positive-def}
\end{align}
respectively. The $\epsilon$-local privacy condition on a mechanism $T$ is equivalent to the following set of constraints on the false negative and false positive rates: 
\begin{theorem}[\citep{Kairouz2014}]
	\label{local-privacy-operator-thm}
	For any $\epsilon > 0$, a mechanism $T \in \cM\br{\cX , \cZ}$ is $\epsilon$-locally private if and only if the following conditions are satisfied for all $x_i,x_j \in \cX$, and all statistical tests $r_{ij}: \cZ \rightarrow \bc{0,1}$: 
	\begin{align}
	\mathrm{FN}_{r_{ij}} + \eexp \cdot \mathrm{FP}_{r_{ij}} ~\geq~& 1 , \label{hyp-symm-cond1}\\
	\eexp \cdot \mathrm{FN}_{r_{ij}} + \mathrm{FP}_{r_{ij}} ~\geq~& 1 . \label{hyp-symm-cond2}
	\end{align}
\end{theorem}
The above operational interpretation says that it is impossible to get both small false negative and false positive rates from data obtained via a $\epsilon$-locally private mechanism. The above characterization is graphically represented in Figure~\ref{loc-sym}, where the shaded region of the left side diagram (Figure~\ref{loc-sym-a}) can be mathematically written as follows
\begin{equation}
\cS\br{\epsilon} ~:=~ \bc{\br{\mathrm{FP},\mathrm{FN}}: \mathrm{FN} + \eexp \cdot \mathrm{FP} \geq 1, \text{ and } \eexp \cdot \mathrm{FN} + \mathrm{FP} \geq 1} .
\end{equation}
We define the \textit{privacy region} of a mechanism $T$ with respect to $x_i$ and $x_j$ as
\begin{equation}
\cS\br{T,x_i,x_j} ~:=~ \text{conv}\br{\bc{\br{\mathrm{FP}_{r_{ij}},\mathrm{FN}_{r_{ij}}}: \text{ for all } r_{ij}: \cZ \rightarrow \bc{0,1}}} ,
\end{equation}
where $\text{conv}\br{\cdot}$ is the convex hull of a set. The following corollary, which follows immediately from Theorem~\ref{local-privacy-operator-thm}, gives a necessary and sufficient condition for a mechanism to be $\epsilon$-locally private. 
\begin{corollary}
	A mechanism $T$ is $\epsilon$-locally private if and only if $\cS\br{T,x_i,x_j} \subseteq \cS\br{\epsilon}$ for all $x_i,x_j \in \cX$.
\end{corollary}

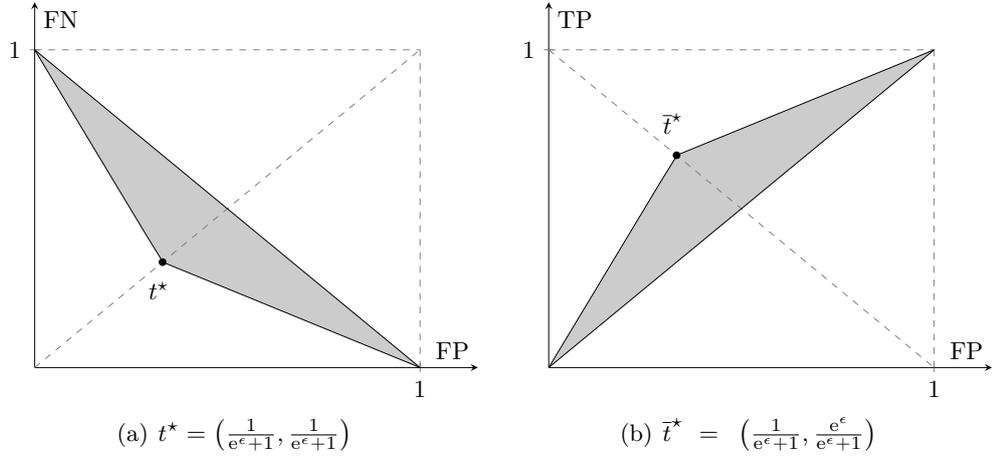
\begin{figure}
	\centering
	\subfigure[$t^\star = \br{\frac{1}{\mathrm{e}^{\epsilon}+1},\frac{1}{\mathrm{e}^{\epsilon}+1}}$\label{loc-sym-a}]{
		\resizebox{0.45\linewidth}{!}{
			\begin{tikzpicture}
			\begin{axis}[domain=0:1, xlabel=FP, ylabel=FN, axis x line=center, axis y line=center, ymin=0, ymax=1.15, xmin=0, xmax=1.15, xtick={0,1}, ytick={1}]
			\addplot[color=gray,dashed] (0,1)--(1,1);
			\addplot[color=gray,dashed] (1,0)--(1,1);
			\addplot[color=gray,dashed] {x} ;
			\addplot[name path=lbound,color=black] {max(1-exp(0.7*1)*x,exp(-0.7*1)*(1-x))};
			\addplot[name path=ubound,color=black] {1-x};
			\addplot[color=gray!40] fill between[of=lbound and ubound];
			\fill (0.3318,0.3318) circle (1.7pt);
			\node (tij) at (0.3218,0.2418) {$t^\star$} ;
			\end{axis}
			\end{tikzpicture}
		}}
		\subfigure[$\bar{t}^\star ~=~ \br{\frac{1}{\mathrm{e}^{\epsilon}+1},\frac{\mathrm{e}^{\epsilon}}{\mathrm{e}^{\epsilon}+1}}$\label{loc-sym-b}]{
			\resizebox{0.45\linewidth}{!}{
				\begin{tikzpicture}
				\begin{axis}[domain=0:1, xlabel=FP, ylabel=TP, axis x line=center, axis y line=center, ymin=0, ymax=1.15, xmin=0, xmax=1.15, xtick={0,1}, ytick={1}]
				\addplot[color=gray,dashed] (0,1)--(1,1);
				\addplot[color=gray,dashed] (1,0)--(1,1);
				\addplot[color=gray,dashed] {1-x} ;
				\addplot[name path=lbound,color=black] {min(exp(0.7*1)*x,1-exp(-0.7*1)*(1-x))};
				\addplot[name path=ubound,color=black] {x};
				\addplot[color=gray!40] fill between[of=lbound and ubound];
				\fill (0.3318,0.6682) circle (1.7pt);
				\node (tij) at (0.3218,0.7682) {$\bar{t}^\star$} ;
				\end{axis}
				\end{tikzpicture}
			}}
			\caption[Operational characteristic representation of $\epsilon$-local privacy mechanisms]{Operational characteristic representation of $\epsilon$-local privacy mechanisms (with $\epsilon=0.7$).\label{loc-sym}}
		\end{figure}

		\subsection{Non-homogeneous Local Privacy}
		\label{non-homo-privacy}
		Now if we want to protect some secrets more than others we need to break the inherent symmetry in the privacy definition of the previous section. Here we introduce an asymmetric privacy notion which is a simple extension of \cite{chatzikokolakis2013broadening}. We replace the undirected pairwise cost terms in the definition of \cite{chatzikokolakis2013broadening} with directed cost terms in order to enforce asymmetry. 
		\begin{definition}
			Define $n:=\abs{\cX}$. Given $\mathbb{R}_+^{n \times n}$ matrix $C$ (with $(i,j)^{\text{th}}$ entry given by the `directed' cost $C_{ij} \in \bs{0,1}$), let $\cM \br{\cX , \cZ ; C} \subseteq \cM \br{\cX , \cZ}$ denote the set of all $C$-locally private mechanisms where
			\[
			T \in \cM \br{\cX , \cZ ; C} ~\iff~ \frac{T(z \mid x_i)}{T(z \mid x_j)} \leq \mathrm{e}^{C_{ij}}, \, \forall{x_i,x_j \in \cX, z \in \cZ}.
			\] 
		\end{definition}
		When $C$ is a symmetric matrix with $0$'s as the diagonal entries and $\epsilon$'s as the off-diagonal entries, we recover the usual $\epsilon$-local privacy requirement. 
		
		Suppose we want to prioritize only $x_{i^\star}$'s privacy and treat others equally. In this case we can choose $C$ be a symmetric matrix with $0$'s as the diagonal entries, $(c \cdot \epsilon)$'s (where $c \in \bs{0,1}$) in the $i^\star$-th row (except the diagonal term), $(\bar{c} \cdot \epsilon)$'s in the $i^\star$-th column (except the diagonal term), and $(0.5 \cdot \epsilon)$'s in other places:
		\[
		\begin{bmatrix}
		0 & 0.5 & \dots & \bar{c} & \dots & 0.5 \\
		0.5 & 0 & \dots & \bar{c} & \dots & 0.5 \\
		\vdots & \vdots &  & \vdots &  & \vdots \\
		c & c & \dots & 0 & \dots & c \\
		\vdots & \vdots &  & \vdots &  & \vdots \\
		0.5 & 0.5 & \dots & \bar{c} & \dots & 0 
		\end{bmatrix} \cdot \epsilon .
		\] 
		

		\paragraph*{Hypothesis Testing Interpretation:}
		We extend the hypothesis testing based interpretation of the $\epsilon$-local privacy definition, to this general case. Then the $C$-local privacy condition on a mechanism $T$ is equivalent to the following set of constraints on the false negative and false positive rates: 
		\begin{theorem}
			\label{general-local-privacy-operator-thm}
			For any $C \in \mathbb{R}_+^{n \times n}$, a mechanism $T \in \cM\br{\cX , \cZ}$ is $C$-locally private if and only if the following conditions are satisfied for all $x_i,x_j \in \cX$, and all statistical tests $r_{ij}:\cZ \rightarrow \bc{0,1}$:
			\begin{align}
			\mathrm{FN}_{r_{ij}} + \mathrm{e}^{C_{ij}} \cdot \mathrm{FP}_{r_{ij}} ~\geq~& 1 , \label{gen-hyp-symm-cond1}\\
			\mathrm{e}^{C_{ji}} \cdot \mathrm{FN}_{r_{ij}} + \mathrm{FP}_{r_{ij}} ~\geq~& 1 . \label{gen-hyp-symm-cond2}
			\end{align}
		\end{theorem}
		\begin{proof}
			From the definition of $C$-local privacy, for any statistical test $r_{ij}:\cZ \rightarrow \bc{0,1}$, we have
			\begin{align*}
			T\br{z \mid x_i} ~\leq~& \mathrm{e}^{C_{ij}} \cdot T\br{z \mid x_j} \\
			\implies \sum_{z \in \cZ}{T\br{z \mid x_i} \ind{r_{ij}=1}} ~\leq~& \mathrm{e}^{C_{ij}} \cdot \sum_{z \in \cZ}{T\br{z \mid x_j} \ind{r_{ij}=1}} \\
			\implies 1 - \text{FN}_{r_{ij}} ~\leq~& \mathrm{e}^{C_{ij}} \cdot \text{FP}_{r_{ij}} ,
			\end{align*}
			and 
			\begin{align*}
			T\br{z \mid x_j} ~\leq~& \mathrm{e}^{C_{ji}} \cdot T\br{z \mid x_i} \\
			\implies \sum_{z \in \cZ}{T\br{z \mid x_j} \ind{r_{ij}=0}} ~\leq~& \mathrm{e}^{C_{ij}} \cdot \sum_{z \in \cZ}{T\br{z \mid x_i} \ind{r_{ij}=0}} \\
			\implies 1 - \text{FP}_{r_{ij}} ~\leq~& \mathrm{e}^{C_{ji}} \cdot \text{FN}_{r_{ij}} .
			\end{align*}
		\end{proof}
		The above characterization is graphically represented in Figure~\ref{gen-loc-sym}, where the shaded region of the left side diagram (Figure~\ref{gen-loc-sym-a}) can be mathematically written as follows
		\begin{equation}
		\cS\br{C,x_i,x_j} ~:=~ \bc{\br{\mathrm{FP},\mathrm{FN}}: \mathrm{FN} + \mathrm{e}^{C_{ij}} \cdot \mathrm{FP} \geq 1, \text{ and } \mathrm{e}^{C_{ji}} \cdot \mathrm{FN} + \mathrm{FP} \geq 1} .
		\end{equation}
		
		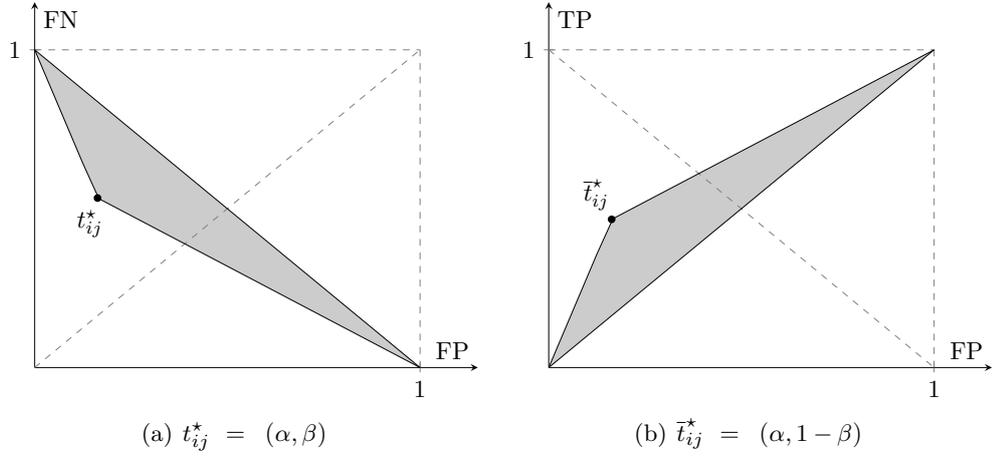
\begin{figure}
			\centering
			\subfigure[$t_{ij}^\star ~=~ \br{\alpha,\beta}$\label{gen-loc-sym-a}]{
				\resizebox{0.45\linewidth}{!}{
					\begin{tikzpicture}
					\begin{axis}[domain=0:1, xlabel=FP, ylabel=FN, axis x line=center, axis y line=center, ymin=0, ymax=1.15, xmin=0, xmax=1.15, xtick={0,1}, ytick={1}]
					\addplot[color=gray,dashed] (0,1)--(1,1);
					\addplot[color=gray,dashed] (1,0)--(1,1);
					\addplot[color=gray,dashed] {x} ;
					\addplot[name path=lbound,color=black] {max(1-exp(1.5*0.7)*x,exp(-1.5*0.3)*(1-x))};
					\addplot[name path=ubound,color=black] {1-x};
					\addplot[color=gray!40] fill between[of=lbound and ubound];
					\fill (0.1632,0.5335) circle (1.7pt);
					\node (tij) at (0.1432,0.4535) {$t_{ij}^\star$} ;
					\end{axis}
					\end{tikzpicture}
				}}
				\subfigure[$\bar{t}_{ij}^\star ~=~ \br{\alpha,1-\beta}$\label{gen-loc-sym-b}]{
					\resizebox{0.45\linewidth}{!}{
						\begin{tikzpicture}
						\begin{axis}[domain=0:1, xlabel=FP, ylabel=TP, axis x line=center, axis y line=center, ymin=0, ymax=1.15, xmin=0, xmax=1.15, xtick={0,1}, ytick={1}]
						\addplot[color=gray,dashed] (0,1)--(1,1);
						\addplot[color=gray,dashed] (1,0)--(1,1);
						\addplot[color=gray,dashed] {1-x} ;
						\addplot[name path=lbound,color=black] {min(exp(1.5*0.7)*x,1-exp(-1.5*0.3)*(1-x))};
						\addplot[name path=ubound,color=black] {x};
						\addplot[color=gray!40] fill between[of=lbound and ubound];
						\fill (0.1632,0.4665) circle (1.7pt);
						\node (tij) at (0.1232,0.5435) {$\bar{t}_{ij}^\star$} ;
						\end{axis}
						\end{tikzpicture}
					}}
					\caption[Operational characteristic representation of $C$-local privacy mechanisms]{Operational characteristic representation of $C$-local privacy mechanisms (with $C_{ij}=1.05$ and $C_{ji}=0.45$). Note that $C_{ij} > C_{ji}$. Observe that $\alpha = \frac{\mathrm{e}^{C_{ji}}-1}{\mathrm{e}^{\br{C_{ij} + C_{ji}}}-1}$ and $\beta = \frac{\mathrm{e}^{C_{ij}}-1}{\mathrm{e}^{\br{C_{ij} + C_{ji}}}-1}$.\label{gen-loc-sym}}
				\end{figure}
				
				Note that unlike Figure~\ref{loc-sym} which holds $\forall{i \neq j}$, here in general we get a different picture for each choice of $i$ and $j$. The following corollary, which follows immediately from Theorem~\ref{general-local-privacy-operator-thm}, gives a necessary and sufficient condition on the privacy region for $C$-local privacy. 
				\begin{corollary}
					A mechanism $T$ is $C$-locally private if and only if $\cS\br{T,x_i,x_j} \subseteq \cS\br{C,x_i,x_j}$ for all $x_i,x_j \in \cX$.
				\end{corollary}
				
				To facilitate the mechanism design, we define the following 
				\begin{equation} 
				\cS\br{C,x_i} ~:=~ \bigcap_{x_j \in \cX \setminus x_i}{\cS\br{C,x_i,x_j}} , 
				\end{equation}
				using which we can design $T\br{\cdot \mid x_i}$. This is illustrated in Figure~\ref{gen-mech}. 
				
				\begin{figure}
					\centering
					\subfigure[\label{gen-mech-a}]{
						\resizebox{0.45\linewidth}{!}{
							\begin{tikzpicture}
							\begin{axis}[domain=0:1, xlabel=FP, ylabel=FN, axis x line=center, axis y line=center, ymin=0, ymax=1.15, xmin=0, xmax=1.15, xtick={0,1}, ytick={1}]
							\addplot[color=gray,dashed] (0,1)--(1,1);
							\addplot[color=gray,dashed] (1,0)--(1,1);
							\addplot[color=gray,dashed] {x} ;
							\addplot[name path=lbound,color=blue] {max(1-exp(0.95*0.8)*x,exp(-0.95*0.2)*(1-x))};
							\addplot[name path=ubound,color=black] {1-x};
							\addplot[color=blue!20,opacity=0.4] fill between[of=lbound and ubound];
							\fill[color=blue] (0.1320,1-0.2822) circle (1.7pt);
							\node (tij) at (0.0720,1-0.2622) {$t_{ij}^\star$} ;
							
							\addplot[name path=lbound2,color=red] {max(1-exp(0.95*0.6)*x,exp(-0.95*0.4)*(1-x))};
							\addplot[color=red!20,opacity=0.4] fill between[of=lbound2 and ubound];
							\fill[color=red] (0.2915,1-0.5155) circle (1.7pt);
							\node (tik) at (0.2815,1-0.5755) {$t_{ik}^\star$} ;
							
							\fill[black, name intersections={of=lbound and lbound2}]
							(intersection-2) circle (1.7pt) node[above] {$t_{i}^\star$};
							\end{axis}
							\end{tikzpicture}
						}}
						\subfigure[\label{gen-mech-b}]{
							\resizebox{0.45\linewidth}{!}{
								\begin{tikzpicture}
								\begin{axis}[domain=0:1, xlabel=FP, ylabel=TP, axis x line=center, axis y line=center, ymin=0, ymax=1.15, xmin=0, xmax=1.15, xtick={0,1}, ytick={1}]
								\addplot[color=gray,dashed] (0,1)--(1,1);
								\addplot[color=gray,dashed] (1,0)--(1,1);
								\addplot[color=gray,dashed] {1-x} ;
								\addplot[name path=lbound,color=blue] {min(exp(0.95*0.8)*x,1-exp(-0.95*0.2)*(1-x))};
								\addplot[name path=ubound,color=black] {x};
								\addplot[color=blue!20,opacity=0.4] fill between[of=lbound and ubound];
								\fill[color=blue] (0.1320,0.2822) circle (1.7pt);
								\node (tij) at (0.0720,0.2622) {$\bar{t}_{ij}^\star$} ;
								
								\addplot[name path=lbound2,color=red] {min(exp(0.95*0.6)*x,1-exp(-0.95*0.4)*(1-x))};
								\addplot[color=red!20,opacity=0.4] fill between[of=lbound2 and ubound];
								\fill[color=red] (0.2915,0.5155) circle (1.7pt);
								\node (tik) at (0.2815,0.5755) {$\bar{t}_{ik}^\star$} ;
								
								\fill[black, name intersections={of=lbound and lbound2}]
								(intersection-2) circle (1.7pt) node[above] {$\bar{t}_{i}^\star$};
								\end{axis}
								\end{tikzpicture}
							}}
							\caption[Feasible region for $T\br{\cdot \mid x_i}$ under $C$-local privacy.]{Feasible region for $T\br{\cdot \mid x_i}$ under $C$-local privacy. The point $t_i^\star$ will impact the optimal privacy mechanism's $T\br{\cdot \mid x_i}$.\label{gen-mech}}
						\end{figure}
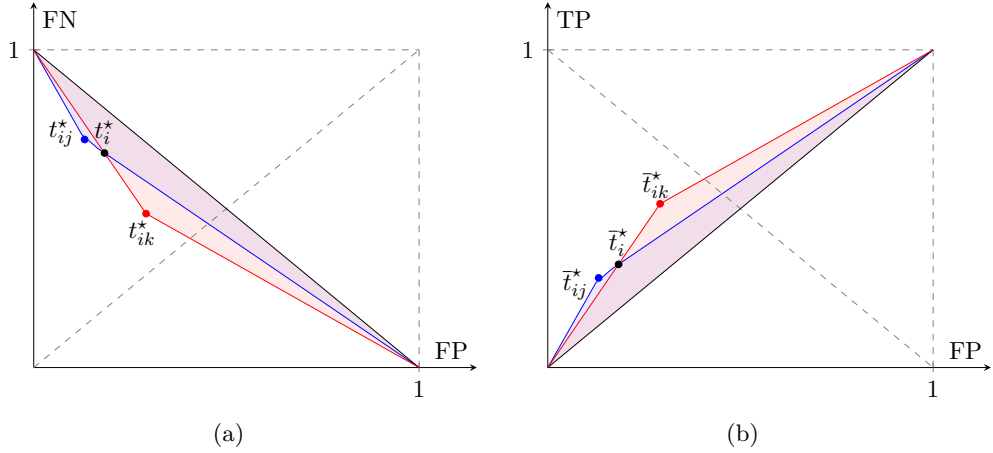

								Figure~\ref{comparison-figure} is plotted for two different $C$: one with $C_{ij}=1.05$ and $C_{ji}=0.45$, and the other with $C_{ij}=C_{ji}=0.75$. This diagram shows how much we lose by prioritizing someone's privacy than others. It can be observed that the permissible ROC region for the privacy mechanism gets shrunk (compared to the equal privacy case) when we enforce prioritized privacy for someone.   
								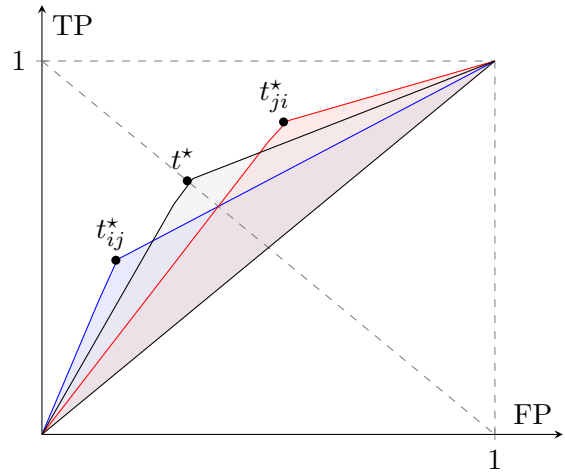
\begin{figure}
									\centering
									\begin{tikzpicture}
									\begin{axis}[domain=0:1, xlabel=FP, ylabel=TP, axis x line=center, axis y line=center, ymin=0, ymax=1.15, xmin=0, xmax=1.15, xtick={0,1}, ytick={1}]
									\addplot[color=gray,dashed] (0,1)--(1,1);
									\addplot[color=gray,dashed] (1,0)--(1,1);
									\addplot[color=gray,dashed] {1-x} ;
									
									\addplot[name path=ubound,color=black] {x};
									
									\addplot[name path=lbound1,color=blue] {min(exp(1.5*0.7)*x,1-exp(-1.5*0.3)*(1-x))};
									\addplot[color=blue!20,opacity=0.4] fill between[of=lbound1 and ubound];
									\fill (0.1632,0.4665) circle (1.7pt);
									\node (tij) at (0.1532,0.5365) {$t_{ij}^\star$};
									
									\addplot[name path=lbound2,color=red] {min(exp(1.5*0.3)*x,1-exp(-1.5*0.7)*(1-x))};
									\addplot[color=red!20,opacity=0.4] fill between[of=lbound2 and ubound];
									\fill (0.5335,0.8368) circle (1.7pt);
									\node (tji) at (0.5135,0.9068) {$t_{ji}^\star$};
									
									\addplot[name path=lbound3,color=black] {min(exp(1.5*0.5)*x,1-exp(-1.5*0.5)*(1-x))};
									\addplot[color=gray!20,opacity=0.4] fill between[of=lbound3 and ubound];
									\fill (0.3208,0.6792) circle (1.7pt);
									\node (tii) at (0.3108,0.7392) {$t^\star$};
									\end{axis}
									\end{tikzpicture}
									\caption[Comparison between $\epsilon$-local privacy and $C$-local privacy.]{Comparison between $\epsilon$-local privacy (with $\epsilon = 0.75$) and $C$-local privacy (with $C_{ij} = 1.05$, and $C_{ji} = 0.45$).\label{comparison-figure}}
								\end{figure}

\section{Conclusion}
\label{sec:concpart2}

The cost-sensitive classification problem plays a crucial role in mission critical machine learning applications. We have studied the hardness of this problem and emphasized the impact of cost terms on the hardness. 

Strong data processing inequalities (SDPI) are very useful in analysing the hardness of constrained learning problems. Despite extensive investigation, the geometric insights of the SDPI are not fully understood. This chapter provides some direction. To this end, we have derived an explicit form for the contraction coefficient of any channel w.r.t.\ $c$-primitive $f$-divergence, and we have obtained efficiently computable lower bound for the contraction coefficient of any binary symmetric channel w.r.t.\ any symmetric $f$-divergence. 

We pose the following open problems as future directions: 
\begin{itemize}
	\item There are some divergences other than $f$-divergences which satisfy the weak data processing inequality, such as Neyman-Pearson $\alpha$-divergences (\citep{polyanskiy2010arimoto,raginsky2011shannon}). Thus it would be interesting to study strong data processing inequalities w.r.t.\ those divergences as well.  
	\item Recently people have attempted to relate several types of channel ordering to the strong data processing inequalities (\citep{makur2016comparison,polyanskiy2015strong}). It would be interesting to study the relationship between the statistical deficiency based channel ordering (\cite{raginsky2011shannon}) and the strong data processing inequalities.
	\item Wider exploration of asymmetric privacy notions.
\end{itemize}

\section{Appendix}

\subsection{VC Dimension}
\label{vc-dimension-sec}
A measure of complexity in learning theory should reflect which learning problems are inherently easier than others. The standard approach in statistical theory is to define the complexity of the learning problem through some notion of ``richness'', ``size'', ``capacity'' of the hypothesis class. 

The complexity measure proposed in \cite{vapnik1971uniform}, the \textit{Vapnik-Chervonenkis (VC) dimension} is a \textit{combinatorial} measure of the richness of classes of binary-valued functions when evaluated on samples. VC-dimension is independent of the underlying probability measure and of the particular sample, and hence is worst-case estimate with regard to these quantities.

We use the notation $x_1^m$ for a sequence $\br{x_1,\dots ,x_m} \in \cX^m$, and for a class of binary-valued functions $\cF \subseteq \bc{-1,1}^{\cX}$, we denote by $\cF_{|x_1^m}$
the \textit{restriction} of $\cF$ to $x_1^m$:
\[
\cF_{|x_1^m} = \bc{\br{f\br{x_1},\dots ,f\br{x_m}} \mid f \in \cF} .
\]
Define the \textit{$m$-th shatter coefficient} of $\cF$ as follows:
\[
\bS_m \br{\cF} := \max_{x_1^m \in \cX^m}{\abs{\cF_{|x_1^m}}} .
\]

\begin{definition}
	Let $\cF \subseteq \bc{-1,1}^{\cX}$ and let $x_1^m = \br{x_1,\dots ,x_m} \in \cX^m$. We say $x_1^m$ is shattered by $\cF$ if $\abs{\cF_{|x_1^m}} = 2^m$; i.e. if  $\forall{\vb \in \bc{-1,1}^m}, \exists{f_{\vb} \in \cF} \text{ s.t. } \br{f_{\vb}\br{x_1},\dots ,f_{\vb}\br{x_m}} = \vb$. The \textit{Vapnik-Chervonenkis (VC) dimension} of $\cF$, denoted by $\text{VCdim}\br{\cF}$, is the cardinality of the largest set of points in $\cX$ that can be shattered by $\cF$:
	\[
	\text{VCdim}\br{\cF} = \max{\bc{m \in \bN \mid \bS_m \br{\cF} = 2^m}} .
	\]
	If $\cF$ shatters arbitrarily large sets of points in $\cX$, then $\text{VCdim}\br{\cF} = \infty$. If $\text{VCdim}\br{\cF} < \infty$, we say that $\cF$ is a VC class.
\end{definition}

\chapter{Exp-concavity of Proper Composite Losses}
\label{cha:mixability}


Loss functions are the means by which the quality of a prediction in learning problem is evaluated. A composite loss (the composition of a class probability estimation (CPE) loss with an invertible link function which is essentially just a re-parameterization) is proper if its risk is minimized when predicting the true underlying class probability (a formal definition is given later). In \cite{vernet2011composite}, there is an argument that shows that there is no point in using losses that are neither proper nor proper composite as they are inadmissible. Flexibility in the choice of loss function is important to tailor the solution to a learning problem (\cite{buja2005loss}, \cite{hand1994deconstructing}, \cite{hand2003local}), and it could be attained by characterizing the set of loss functions using natural parameterizations.

The goal of the learner in a \textit{game of prediction with expert advice} (which is formally described in section \ref{sec:games}) is to predict as well as the best expert in the given pool of experts. The regret bound of the learner depends on the merging scheme used to merge the experts' predictions and the type of loss function used to measure the performance. It has already been shown that constant regret bounds are achievable for mixable losses when the Aggregating Algorithm is the merging scheme (\cite{vovk1995game}), and for exp-concave losses when the Weighted Average Algorithm is the merging scheme (\cite{kivinen1999averaging}). We can see that the exp-concavity trivially implies mixability. Even though the converse implication is not true in general, under some re-parameterization we can make it possible. This chapter discusses general conditions on proper losses under which they can be transformed to an exp-concave loss through a suitable link function. In the binary case, these conditions give two concrete formulas (Proposition \ref{geoprop} and Corollary \ref{specialcoro}) for link functions that can transform $\beta$-mixable proper losses into $\beta$-exp-concave, proper, composite losses. The explicit form of the link function given in Proposition \ref{geoprop} is derived using the same geometric construction used in \cite{van2012exp}. 

Further we extend the work by \cite{vernet2011composite}, to provide a complete characterization of the exp-concavity of the proper composite multi-class losses in terms of the Bayes risk associated with the underlying proper loss, and the link function. The mixability of proper losses (mixability of a proper composite loss is equivalent to the mixability of its generating proper loss) is studied in \cite{van2012mixability}. Using these characterizations (for the binary case), in Corollary \ref{specialcoro} we derive an \textit{exp-concavifying link} function that can also transform any $\beta$-mixable proper loss into a $\beta$-exp-concave composite loss. Since for the multi-class losses these conditions do not hold in general, we propose a geometric approximation approach (Proposition \ref{exp_concave_approx}) which takes a parameter $\epsilon$ and transforms the mixable loss function appropriately on a subset $S_\epsilon$ of the prediction space. When the prediction space is $\Delta^n$, any prediction belongs to the subset $S_\epsilon$ for sufficiently small $\epsilon$. In the conclusion we provide a way to use the Weighted Average Algorithm with learning rate $\beta$ for proper $\beta$-mixable but non-exp-concave loss functions to achieve $O(1)$ regret bound.

The exp-concave losses achieve $O(\log{T})$ regret bound in online convex optimization algorithms, which is a more general setting of online learning problems. Thus the exp-concavity characterization of composite losses could be helpful in constructing exp-concave losses for online learning problems.

The chapter is organized as follows. In Section \ref{sec:preli} we formally introduce the loss function, several loss types, conditional risk, proper composite losses and a game of prediction with expert advice. In Section \ref{sec:expsec} we consider our main problem --- whether one can always find a link function to transform $\beta$-mixable losses into $\beta$-exp-concave losses. Section \ref{sec:conc} concludes with a brief discussion. The impact of the choice of substitution function on the regret of the learner is explored via experiments in Appendix \ref{sec:subfunc}. In Appendix \ref{sec:probability}, we discuss the mixability conditions of \textit{probability games} with continuous outcome space. Detailed proofs are in Appendix \ref{sec:proof}.

\section{Preliminaries and Background}
\label{sec:preli}

This section provides the necessary background on loss functions, conditional risks, and the sequential prediction problem.

\subsection{Notation}

We use the following notation throughout this chapter. A superscript prime, $A'$ denotes transpose of the matrix or vector $A$, except when applied to a real-valued function where it denotes derivative ($f'$). We denote the matrix multiplication of compatible matrices $A$ and $B$ by $A \cdot B$, so the inner product of two vectors $x,y \in \mathbb{R}^n$ is $x' \cdot y$. Let $[n]:=\{1,...,n\}$, $\mathbb{R}_{+} := [0,\infty)$ and the $n$-simplex $\Delta^n:=\{(p_1,...,p_n)': 0 \leq p_i \leq 1, \forall{i \in [n]}, \mathrm{and} \sum_{i \in [n]} p_i = 1\}$. If $x$ is a $n$-vector, $A=\mathrm{diag}(x)$ is the $n \times n$ matrix with entries $A_{i,i}=x_i$ , $i \in [n]$ and $A_{i,j}=0$ for $i \neq j$. If $A-B$ is positive definite (resp. semi-definite), then we write $A \succ B$ (resp. $A \succcurlyeq B$). We use $e_i^n$ to denote the $i$th $n$-dimensional unit vector, $e_i^n=(0,...,0,1,0,...0)'$ when $i \in [n]$, and define $e_i^n=0_n$ when $i > n$. The $n$-vector $\vone_n:=(1,...,1)'$. We write $\llbracket P \rrbracket=1$ if $P$ is true and $\llbracket P \rrbracket=0$ otherwise. Given a set $S$ and a weight vector $w$, the \textit{convex combination} of the elements of the set w.r.t the weight vector is denoted by $\mathrm{co}_w S$, and the \textit{convex hull} of the set which is the set of all possible convex combinations of the elements of the set is denoted by $\mathrm{co} S$ (\cite{rockafellar1970convex}). If $S,T \subset \RR^n$, then the \textit{Minkowski sum} $S \varoplus T := \{ s+t: s  \in S, t \in T \}$. $\mathcal{Y}^{\mathcal{X}}$ represents the set of all functions $f:\mathcal{X}\rightarrow\mathcal{Y}$. We say $f: C \subset \bR^n \rightarrow \bR^n$ is \textit{monotone} (resp. \textit{strictly monotone}) on $C$ when for all $x$ and $y$ in $C$,
\[
(f(x)-f(y))' \cdot (x-y) \geq 0 ~\text{ resp. }~ (f(x)-f(y))' \cdot (x-y) > 0 ;
\]
confer \cite{hiriart2013convex}. Other notation (the Kronecker product $\otimes$, the Jacobian $\textsf{D}$, and the Hessian $\textsf{H}$) is defined in Appendix A of \cite{van2012mixability}. $\textsf{D}f(v)$ and $\textsf{H}f(v)$ denote the Jacobian and Hessian of $f(v)$ w.r.t.\ $v$ respectively. When it is not clear from the context, we will explicitly mention the variable; for example $\textsf{D}_{\tilde{v}}f(v)$ where $v=h(\tilde{v})$.

\subsection{Loss Functions}

For a prediction problem with an instance space $\mathcal{X}$, outcome space $\mathcal{Y}$ and prediction space $\mathcal{V}$, a loss function $\ell:\mathcal{Y} \times \mathcal{V} \rightarrow \mathbb{R}_{+}$ (bivariate function representation) can be defined to assign a penalty $\ell(y,v)$ for predicting $v \in \mathcal{V}$ when the actual outcome is $y \in \mathcal{Y}$. When the outcome space $\mathcal{Y}=[n], \, n \geq 2$, the loss function $\ell$ is called a \textit{multi-class loss} and it can be expressed in terms of its partial losses $\ell_{i}:=\ell(i,\cdot)$ for any outcome $i \in [n]$, as 
\[
\ell(y,v) = \sum_{i \in [n]} \llbracket y=i \rrbracket \ell_{i}(v) .
\] 
The vector representation of the multi-class loss is given by $\ell:\mathcal{V}\rightarrow \mathbb{R}_+^n$, which assigns a vector $\ell(v)=(\ell_1(v),...,\ell_n(v))'$ to each prediction $v \in \mathcal{V}$. A loss is differentiable if all of its partial losses are differentiable. In this thesis, we will use the bivariate function representation ($\ell(y,v)$) to denote a general loss function and the vector representation for multi-class loss functions.

The \textit{super-prediction set} of a binary loss $\ell$ is defined as
\begin{equation*}
S_{\ell} := \{ x \in \mathbb{R}^n : \exists v \in \mathcal{V}, x \geq \ell(v) \},
\end{equation*}
where inequality is component-wise. For any dimension $n$ and $\beta \geq 0$, the $\beta$-exponential operator $E_{\beta}:[0,\infty]^n \rightarrow [0,1]^n$ is defined by 
\[
E_{\beta}(x):=(e^{-\beta x_1},...,e^{-\beta x_n})'.
\] 
For $\beta > 0$ it is clearly invertible with inverse 
\[
E_{\beta}^{-1}(z)=-\beta^{-1}(\ln{z_1},...,\ln{z_n})'.
\] 
The $\beta$-exponential transformation of the super-prediction set is given by
\begin{equation*}
E_{\beta}(S_{\ell}) := \{ (e^{-\beta x_1},...,e^{-\beta x_n})' \in \mathbb{R}^n : (x_1,...,x_n)' \in S_{\ell} \}, \quad \beta > 0.
\end{equation*}

A multi-class loss $\ell$ is 
\begin{itemize}
	\item \textit{convex} if $f(v)=\ell_{y}(v)$ is convex in $v$ for all $y \in [n]$, 
	\item $\alpha$\textit{-exp-concave} (for $\alpha > 0$) if $f(v)=e^{-\alpha \ell_{y}(v)}$ is concave in $v$ for all $y \in [n]$ (\cite{cesa2006prediction}), 
	\item \textit{weakly mixable} if the super-prediction set $S_{\ell}$ is convex (\cite{kalnishkan2005weak}), and
	\item $\beta$\textit{-mixable} (for $\beta > 0$) if the set $E_{\beta}(S_{\ell})$ is convex (\cite{vovk2009prediction,vovk1995game}).
\end{itemize}
The \textit{mixability constant} $\beta_{\ell}$ of a loss $\ell$ is the largest $\beta$ such that $\ell$ is $\beta$-mixable; i.e. 
\[
\beta_{\ell} := \sup{\{ \beta > 0 : \ell \, \mbox{is} \, \beta\mbox{-mixable}\}}.
\] 
If the loss function $\ell$ is $\alpha$-exp-concave (resp. $\beta$-mixable) then it is $\alpha'$-exp-concave for any $0 < \alpha' \leq \alpha$ (resp. $\beta'$-mixable for any $0 < \beta' \leq \beta$), and its $\lambda$-scaled version ($\lambda \ell$) for some $\lambda > 0$ is $\frac{\alpha}{\lambda}$-exp-concave (resp. $\frac{\beta}{\lambda}$-mixable). If the loss $\ell$ is $\alpha$-exp-concave, then it is convex and $\alpha$-mixable (\cite{cesa2006prediction}). 

For a multi-class loss $\ell$, if the prediction space $\mathcal{V}=\Delta^n$ then it is said to be \textit{multi-class probability estimation (CPE) loss}, where the predicted values are directly interpreted as probability estimates: $\ell:\Delta^n \rightarrow \mathbb{R}_+^n$. We will say a multi-CPE loss is \textit{fair} whenever $\ell_{i}(e_i^n)=0$, for all $i \in [n]$. That is, there is no loss incurred for perfect prediction. Examples of multi-CPE losses include 
\begin{enumerate}
	\item the \textit{square loss} $\ell_{i}^{\mathrm{sq}}(q):=\sum_{j \in [n]}(\llbracket i=j \rrbracket - q_j)^2$,
	\item the \textit{log loss} $\ell_{i}^{\mathrm{log}}(q):=-\log{q_i}$,
	\item the \textit{absolute loss} $\ell_{i}^{\mathrm{abs}}(q):=\sum_{j \in [n]}|\llbracket i=j \rrbracket - q_j|$, and
	\item the \textit{0-1 loss} $\ell_{i}^{\mathrm{01}}(q):=\llbracket i \notin \arg \max_{j \in [n]} q_j \rrbracket$.
\end{enumerate}

\subsection{Conditional and Full Risks}

Let $\textsf{X}$ and $\textsf{Y}$ be random variables taking values in the instance space $\mathcal{X}$ and the outcome space $\mathcal{Y}=[n]$ respectively. Let $D$ be the joint distribution of $(\textsf{X},\textsf{Y})$ and for $x \in \mathcal{X}$, denote the conditional distribution by $p(x)=(p_1(x),...,p_n(x))'$ where $p_i(x):=P(\textsf{Y}=i|\textsf{X}=x), \, \forall{i \in [n]}$, and the marginal distribution by $M(x):=P(\textsf{X}=x)$. For any multi-CPE loss $\ell$, the \textit{conditional Bayes risk} is defined as
\begin{equation}
\label{condrisk}
L_{\ell} : \Delta^n \times \Delta^n \ni (p,q) \mapsto L_{\ell}(p,q)=\mathbb{E}_{\textsf{Y}\sim p}[\ell_{\textsf{Y}}(q)]=p' \cdot \ell(q)=\sum_{i \in [n]}p_i \ell_i(q) \in \mathbb{R}_+ \, ,
\end{equation}
where $\textsf{Y}\sim p$ represents a Multinomial distribution with parameter $p \in \Delta^n$. The \textit{full Bayes risk} of the estimator function $q:\mathcal{X}\rightarrow\Delta^n$ is defined as
\begin{equation*}
\hat{\cR}_{\ell}(M,p,q):=\mathbb{E}_{(\textsf{X},\textsf{Y})\sim D}[\ell_{\textsf{Y}}(q(\textsf{X}))]=\mathbb{E}_{\textsf{X}\sim M}[L_{\ell}(p(\textsf{X}),q(\textsf{X}))].
\end{equation*}
Furthermore the \textit{minimum full Bayes risk} is defined as
\begin{equation*}
\underline{\hat{\cR}}_{\ell}(M,p):=\inf_{q \in (\Delta^n)^{\mathcal{X}}}{\hat{\cR}_{\ell}(M,p,q)}=\mathbb{E}_{\textsf{X}\sim M}[\Lubar_{\ell}(p(\textsf{X}))],
\end{equation*}
where $\Lubar_{\ell}(p)=\inf_{q \in \Delta^n}{L_{\ell}(p,q)}$ is the \textit{minimum conditional Bayes risk} and is always concave (\cite{gneiting2007strictly}). If $\ell$ is fair, $\Lubar_{\ell}(e_i^n)=\ell_{i}(e_i^n)=0$. One can understand the effect of choice of loss in terms of the conditional perspective (\cite{reid2011information}), which allows one to ignore the marginal distribution $M$ of $\textsf{X}$ which is typically unknown.

\subsection{Proper and Composite Losses}
\label{sec:proper}

A multi-CPE loss $\ell:\Delta^n \rightarrow \mathbb{R}_+^n$ is said to be \textit{proper} if for all $p \in \Delta^n$, $\Lubar_{\ell}(p) = L_{\ell}(p,p)=p' \cdot \ell(p)$ (\cite{buja2005loss}, \cite{gneiting2007strictly}), and \textit{strictly proper} if $\Lubar_{\ell}(p) < L_{\ell}(p,q)$ for all $p,q \in \Delta^n$ and $p \neq q$. It is easy to see that the log loss, square loss, and 0-1 loss are proper while absolute loss is not. Furthermore, both log loss and square loss are strictly proper while 0-1 loss is proper but not strictly proper.

Given a proper loss $\ell: \Delta^n \to \RR_+^n$ with differentiable Bayes conditional risk $\Lubar_\ell: \Delta^n \mapsto \RR_+$, in order to be able to calculate derivatives easily, following \cite{van2012mixability} we define
\begin{align}
\Deltatil^n &:= \bc{(p_1, \ldots, p_{n-1})' : p_i \ge 0, \ \sum_{i=1}^{n-1} p_i \le 1} \\
\Pi_\Delta &: \RR_+^n \ni p = (p_1, \ldots, p_n)' \mapsto \ptil = (p_1, \ldots, p_{n-1})' \in
\RR_+^{n-1} \\
\Pi_\Delta^{-1} &: \Deltatil^n \ni \ptil = (\ptil_1, \ldots, \ptil_{n-1})' \mapsto p = (\ptil_1, \ldots, \ptil_{n-1}, 1 - \sum\nolimits_{i=1}^{n-1} \ptil_i)' \in \Delta^n \\
\Lubartil_\ell &: \Deltatil^n \ni \ptil \mapsto \Lubar_\ell(\Pi_\Delta^{-1} (\ptil)) \in \RR_+ \\
\elltil &: \Deltatil^n \ni \ptil \mapsto \Pi_\Delta (\ell (\Pi_\Delta^{-1}(\ptil))) \in \RR_+^{n-1}.
\end{align}
Let $\psitil: \Deltatil^n \to \mathcal{V} \subseteq \RR_+^{n-1}$ be continuous and strictly monotone (hence invertible) for some convex set $\mathcal{V}$.
It induces $\psi: \Delta^n \to \mathcal{V}$ via
\begin{align}
\label{eq:link_def}
\psi := \psitil \circ \Pi_\Delta.
\end{align}
Clearly $\psi$ is continuous and invertible with $\psi^{-1} = \Pi_\Delta^{-1} \circ \psitil^{-1}$. We can now extend the notion of properness to the prediction space $\mathcal{V}$ from $\Delta^n$ using this link function. Given a proper loss $\ell: \Delta^n \to \RR_+^n$, a \textit{proper composite loss} $\ell^{\psi}:\mathcal{V} \rightarrow \mathbb{R}_{+}^n$ for multi-class probability estimation is defined as $\ell^{\psi}:=\ell \circ \psi^{-1} = \ell \circ \Pi_\Delta^{-1} \circ \psitil^{-1}$. We can easily see that the conditional Bayes risks of the composite loss $\ell^{\psi}$ and the underlying proper loss $\ell$ are equal ($\Lubar_{\ell}=\Lubar_{\ell^{\psi}}$). Every continuous proper loss has a convex super-prediction set (\cite{vernet2011composite}). Thus they are weakly mixable. Since by applying a link function the super-prediction set won't change (as it is just a re-parameterization), all proper composite losses are also weakly mixable.

\subsection{Game of Prediction with Expert Advice}
\label{sec:games}
Let $\mathcal{Y}$ be the outcome space, $\mathcal{V}$ be the prediction space, and $\ell:\mathcal{Y} \times \mathcal{V} \rightarrow \mathbb{R}_{+}$ be the loss function, then a \textit{game of prediction with expert advice} represented by the tuple ($\mathcal{Y}, \mathcal{V}, \ell$) can be described as follows: for each trial $t=1,...,T$,
\begin{itemize}
	\item{$N$ experts make their prediction $v_t^1,...,v_t^N \in \mathcal{V}$}
	\item{the learner makes his own decision $v_t \in \mathcal{V}$}
	\item{the environment reveals the actual outcome $y_t \in \mathcal{Y}$}
\end{itemize}
Let $S=(y_1,...,y_T)$ be the outcome sequence in $T$ trials. Then the \textit{cumulative loss} of the learner over $S$ is given by $L_{S,\ell}:=\sum_{t=1}^T \ell(y_t,v_t)$, of the \textit{i}-th expert is given by $L_{S,\ell}^i:=\sum_{t=1}^T \ell(y_t,v_t^i)$, and the \textit{regret} of the learner is given by $R_{S,\ell}:=L_{S,\ell}-\min_i L_{S,\ell}^i$. The goal of the learner is to predict as well as the best expert; to which end the learner tries to minimize the regret.

When using the exponential weights algorithm (which is an important family of algorithms in game of prediction with expert advice), at the end of each trial, the weight of each expert is updated as $w_{t+1}^i \varpropto w_t^i \cdot e^{-\eta \ell(y_t,v_t^i)}$ for all $i \in [N]$, where $\eta$ is the learning rate and $w_t^i$ is the weight of the $i^{\mathrm{th}}$ expert at time $t$ (the weight vector of experts at time $t$ is denoted by $w_t=(w_t^1,...,w_t^N)'$). Then based on the weights of experts, their predictions are merged using different merging schemes to make the learner's prediction. The Aggregating Algorithm and the Weighted Average Algorithm are two important algorithms in the family of exponential weights algorithm.

Consider multi-class games with outcome space $\mathcal{Y}=[n]$. In the Aggregating Algorithm with learning rate $\beta$, first the loss vectors of the experts and their weights are used to make a \textit{generalized prediction} $g_t=(g_t(1),...,g_t(n))'$ which is given by
\begin{align*}
g_t ~:=~& E_{\beta}^{-1}\left( \mathrm{co}_{w_t}\{E_{\beta}\left( (\ell_1(v_t^i),...,\ell_n(v_t^i))' \right) \}_{i \in [N]} \right) \\
~=~& E_{\beta}^{-1}\left(\sum_{i \in [N]}w_t^i (e^{-\beta \ell_1(v_t^i)},...,e^{-\beta \ell_n(v_t^i)})' \right).
\end{align*}
Then this generalized prediction is mapped into a permitted prediction $v_t$ via a \textit{substitution function} such that $(\ell_1(v_t),...,\ell_n(v_t))' \leq c(\beta) (g_t(1),...,g_t(n))'$, where the inequality is element-wise and the constant $c(\beta)$ depends on the learning rate. If $\ell$ is $\beta$-mixable, then $E_{\beta}(S_{\ell})$ is convex, so $\mathrm{co}\{E_{\beta}\left( \ell(v) \right) : v \in \mathcal{V} \} \subseteq E_{\beta}(S_{\ell})$, and we can always choose a substitution function with $c(\beta)=1$. Consequently for $\beta$-mixable losses, the learner of the Aggregating Algorithm is guaranteed to have regret bounded by $\frac{\log{N}}{\beta}$ (\cite{vovk1995game}).

In the Weighted Average Algorithm with learning rate $\alpha$, the experts' predictions are simply merged according to their weights to make the learner's prediction $v_t=\mathrm{co}_{w_t}\{v_t^i\}_{i \in [N]}$, and this algorithm is guaranteed to have a $\frac{\log{N}}{\alpha}$ regret bound for $\alpha$-exp-concave losses (\cite{kivinen1999averaging}). In either case it is preferred to have bigger values for the constants $\beta$ and $\alpha$ to have better regret bounds. Since an $\alpha$-exp-concave loss is $\beta$-mixable for some $\beta \geq \alpha$, the regret bound of the Weighted Average Algorithm is worse than that of the Aggregating Algorithm by a small constant factor. In \cite{vovk2001competitive}, it is noted that $(\ell_1(\mathrm{co}_{w_t}\{v_t^i\}_i),...,\ell_n(\mathrm{co}_{w_t}\{v_t^i\}_i))' \leq g_t = E_{\beta}^{-1}\left( \mathrm{co}_{w_t} \{ E_{\beta} \left( (\ell_1(v_t^i),...,\ell_n(v_t^i))' \right) \}_i \right)$ is always guaranteed only for $\beta$-exp-concave losses. Thus for $\alpha$-exp-concave losses, the Weighted Average Algorithm is equivalent to the Aggregating Algorithm with the weighted average of the experts' predictions as its substitution function and $\alpha$ as the learning rate for both algorithms.

Even though the choice of substitution function will not have any impact on the regret bound and the weight update mechanism of the Aggregating Algorithm, it will certainly have impact on the actual regret of the learner over a given sequence of outcomes. According to the results given in Appendix \ref{sec:subfunc} (where we have empirically compared some substitution functions), this impact on the actual regret varies depending on the outcome sequence, and in general the regret values for practical substitution functions don't differ much --- thus we can stick with a computationally efficient substitution function.

\section{Exp-Concavity of Proper Composite Losses}
\label{sec:expsec}
Exp-concavity of a loss is desirable for better (logarithmic) regret bounds in online convex optimization algorithms, and for efficient implementation of exponential weights algorithms. In this section we will consider whether one can always find a link function that can transform a $\beta$-mixable proper loss into $\beta$-exp-concave composite loss --- first by using the geometry of the set $E_\beta(S_\ell)$ (Section \ref{sec:geometry}), and then by using the characterization of the composite loss in terms of the associated Bayes risk (Sections \ref{sec:calculus}, and \ref{sec:link}).

\subsection{Geometric approach}
\label{sec:geometry}
In this section we will use the same construction used by \cite{van2012exp} to derive an explicit closed form of a link function that could re-parameterize any $\beta$-mixable loss into a $\beta$-exp-concave loss, under certain conditions which are explained below. Given a multi-class loss $\ell:\mathcal{V}\rightarrow \RR_+^n$, define
\begin{eqnarray}
\ell(\mathcal{V}) &:=& \{ \ell(v) : v \in \mathcal{V} \}, \\
\mathcal{B}_{\beta} &:=& \mathrm{co} E_\beta (\ell(\mathcal{V})).
\end{eqnarray}
For any $g \in \mathcal{B}_{\beta}$ let $c(g):=\sup{\{ c \geq 0 : (g+c\vone_n) \in \mathcal{B}_{\beta}\}}$. Then the ``north-east'' boundary of the set $\mathcal{B}_{\beta}$ is given by $\partial_{\vone_n} \mathcal{B}_{\beta} := \{ g + c(g) : g \in \mathcal{B}_{\beta} \}$. The following proposition is the main result of this section.
\begin{proposition}
	\label{geoprop}
	Assume $\ell$ is strictly proper and it satisfies the condition : $\partial_{\vone_n} \mathcal{B}_{\beta} \subseteq E_\beta (\ell(\mathcal{V}))$ for some $\beta > 0$. Define $\psi(p) := J E_\beta(\ell(p))$ for all $p \in \Delta^n$, where $J=[I_{n-1},-\vone_{n-1}]$. Then $\psi$ is invertible, and $\ell \circ \psi^{-1}$ is $\beta$-exp-concave over $\psi(\Delta^n)$, which is a convex set.
\end{proposition}
The condition stated in the above proposition is satisfied by any $\beta$-mixable proper loss in the binary case ($n=2$), but it is not guaranteed in the multi-class case where $n > 2$. In the binary case the link function can be given as $\psi(p)=e^{-\beta \ell_1(p)} - e^{-\beta \ell_2(p)}$ for all $p \in \Delta^2$.

Unfortunately, the condition that
$\partial_{\vone_n} \mathcal{B}_{\beta} \subseteq E_\beta (\ell(\mathcal{V}))$
is generally not satisfied;
an example based on squared loss ($\beta= 1$ and $n=3$ classes) is shown in Figure \ref{fig:hole_square},
where for $A$ and $B$ in $E_\beta (\ell(\mathcal{V}))$, the mid-point $C$ can travel along the ray of direction $\vone_3$ without hitting any point in the exp-prediction set $E_\beta(\ell(\mathcal{V}))$.
Therefore we resort to \emph{approximating} a given $\beta$-mixable loss by a sequence of $\beta$-exp-concave losses parameterised by positive constant $\epsilon$,
while the approximation approaches the original loss in some appropriate sense as $\epsilon$ tends to 0.
Without loss of generality, we assume $\cV = \Delta^n$.

Inspired by Proposition~\ref{geoprop}, a natural idea to construct the approximation is by adding ``faces'' to the exp-prediction set such that all rays in the $\vone_n$ direction will be blocked.
Technically, it turns out more convenient to add faces that block rays in (almost) all directions of positive orthant.
See Figure \ref{fig:extend_3d} for an illustration.
In particular, we extend the ``rim'' of the exp-prediction set by hyperplanes that are $\epsilon$ close to  axis-parallel.
The key challenge underlying this idea is to design an appropriate \emph{parameterisation} of the surrogate loss $\elltil_\epsilon$,
which not only produces such an extended exp-prediction set,
but also ensures that $\elltil_\epsilon(p) = \ell(p)$ for almost all $p \in \Delta^n$ as $\epsilon \downarrow 0$.

%
%

\begin{figure*}[t]
	\begin{minipage}[b]{0.31\textwidth}
		\includegraphics[width=1\textwidth,height=0.9\textwidth]{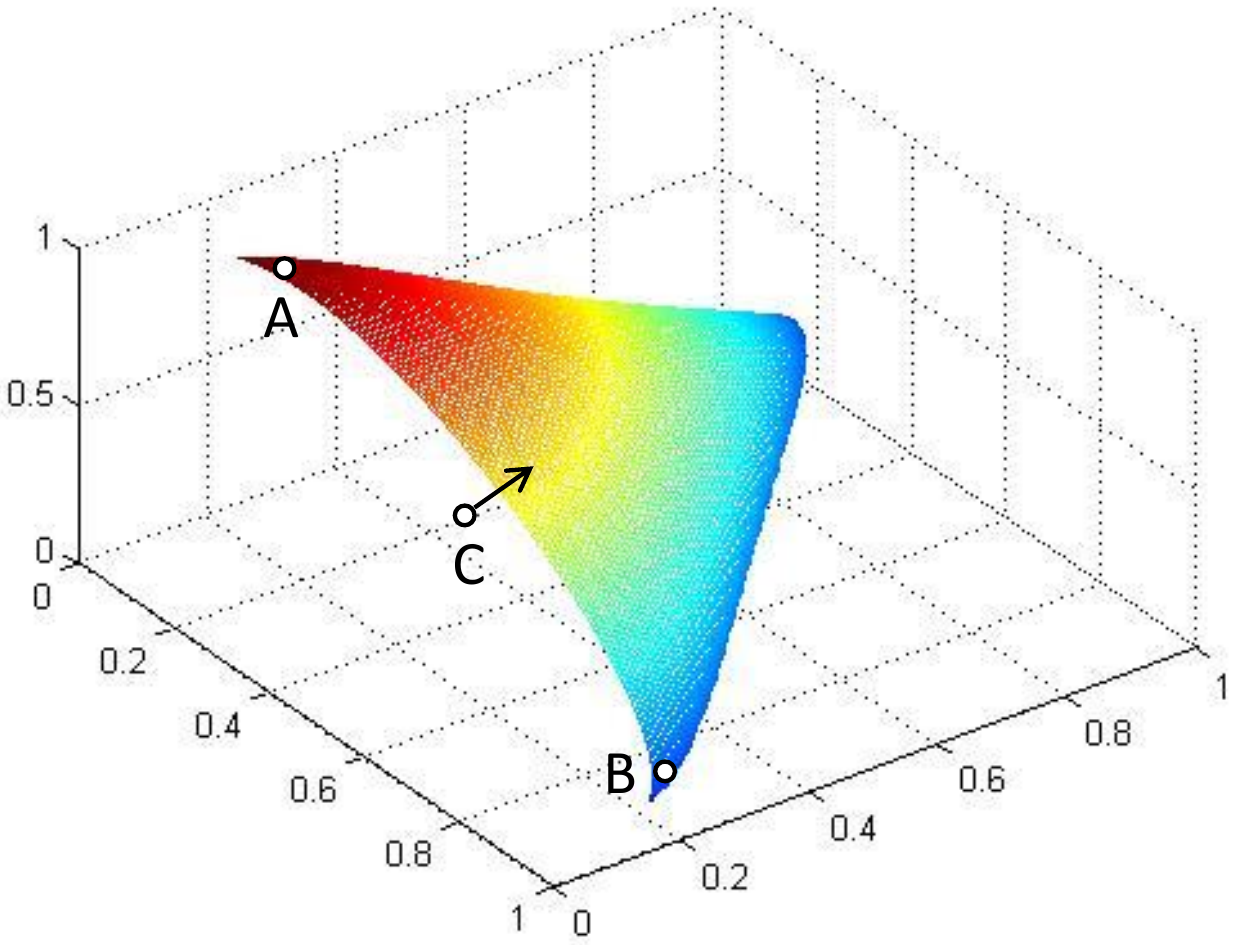}
		\caption{Ray ``escaping'' in $\vone_n$ direction. More evidence in Figure \ref{fig:proj_square} in Appendix \ref{sec:square_loss}.}
		\label{fig:hole_square}
	\end{minipage}
	~
	\begin{minipage}[b]{0.31\textwidth}
		\includegraphics[width=1\textwidth,height=0.9\textwidth]{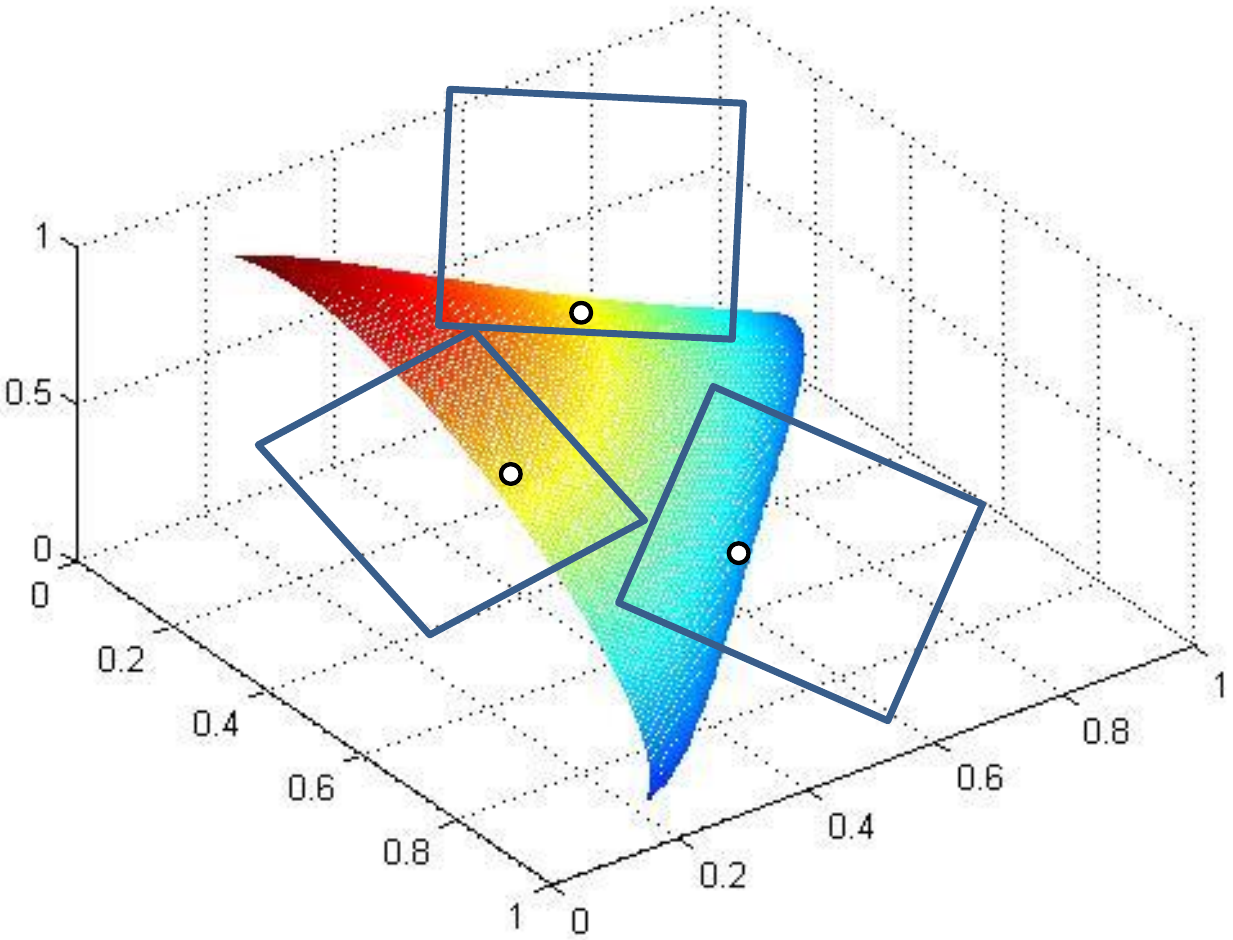}
		\caption{Adding ``faces'' to block rays in (almost) all positive directions.}
		\label{fig:extend_3d}
	\end{minipage}
	~
	\begin{minipage}[b]{0.34\textwidth}
		\centering
		\includegraphics[width=0.94\textwidth,height=0.84\textwidth]{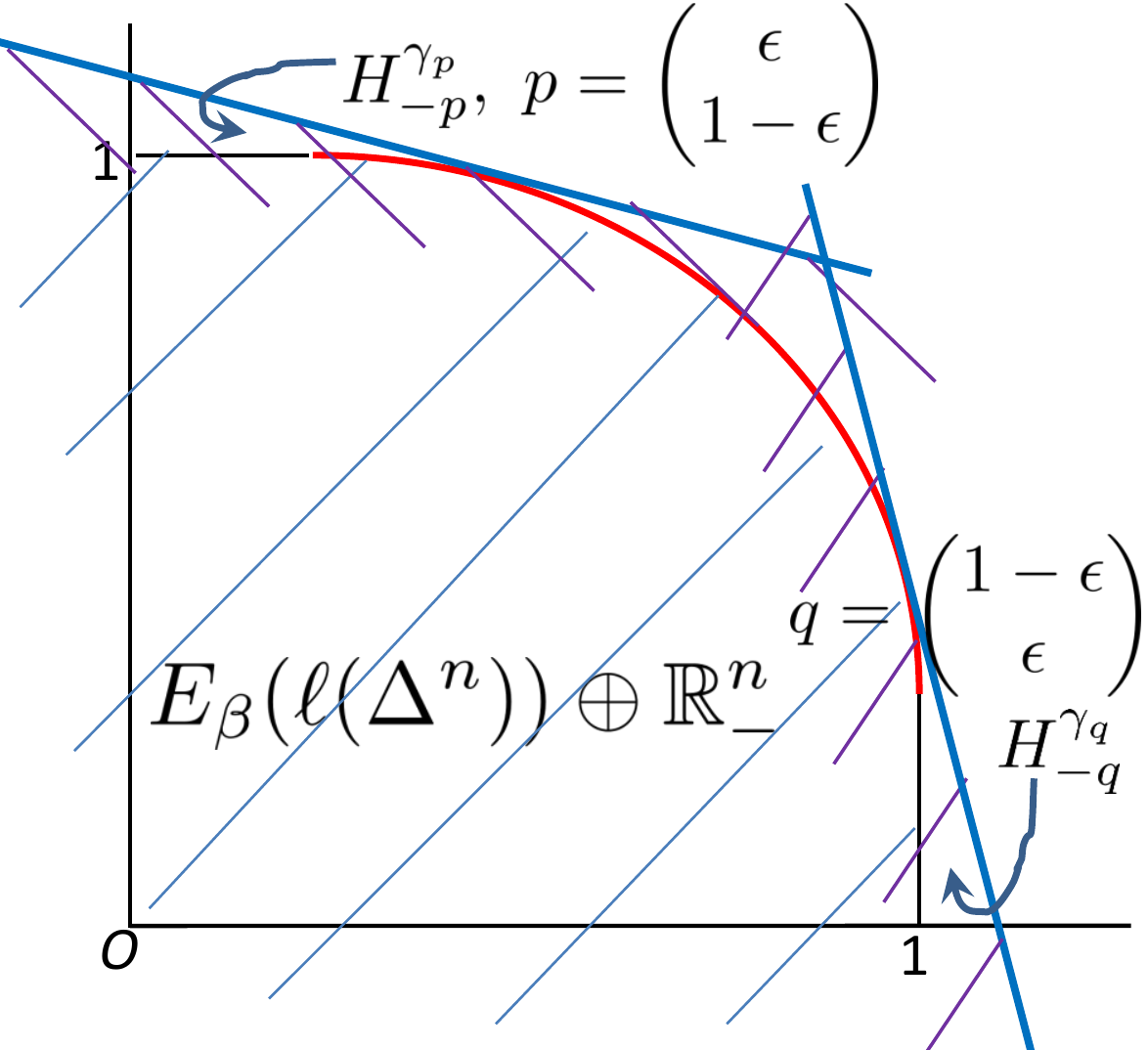}
		\caption{Sub-exp-prediction set extended by removing near axis-parallel supporting hyperplanes.}
		\label{fig:extend_2d}
	\end{minipage}
\end{figure*}

Given a $\beta$-mixable loss $\ell$, its sub-exp-prediction set defined as follows must be convex:
\begin{align}
T_\ell := E_\beta(\ell(\Delta^n))\varoplus\RR_-^n = \{ E_\beta(\ell(p)) - x : p \in \Delta^n, x \in \RR_+^n\}.
\end{align}
Note $T_\ell$ extends infinitely in any direction $p \in \RR_-^n$.
Therefore it can be written in terms of supporting hyperplanes as
\begin{align}
\label{eq:supp_T_ell}
T_\ell = \mathop{\bigcap}\limits_{p \in \Delta^n} H_{-p}^{\gamma_p}, \where \gamma_p = \min_{x \in T_\ell} x' \cdot (-p), \text{ and } H_{-p}^{\gamma_p} := \{x: x' \cdot (-p) \ge \gamma_p\}.
\end{align}
To extend the sub-exp-prediction set with ``faces'', we remove some hyperplanes involved in \eqref{eq:supp_T_ell} that correspond to the $\epsilon$ ``rim'' of the simplex (see Figure \ref{fig:extend_2d} for an illustration in 2-D)
\begin{align}
\label{eq:ext_T_ell}
T_\ell^\epsilon := \mathop{\bigcap}\limits_{p \in \Delta^n_\epsilon} H_{-p}^{\gamma_p}, \where \Delta^n_\epsilon := \{ p \in \Delta^n: \min_i p_i > \epsilon\}.
\end{align}

Since $\epsilon > 0$, for any $p \in \Delta^n_\epsilon$,
$E_\beta^{-1} (H_{-p}^{\gamma_p} \cap \RR_+^n)$ is exactly the super-prediction set of a log-loss with appropriate scaling and shifting (see proof in Appendix \ref{sec:proof}).
So it must be convex.
Therefore $E_\beta^{-1}(T_\ell^\epsilon \cap \RR_+^n) = \mathop{\bigcap}_{p \in \Delta^n_\epsilon} E_\beta^{-1} (H_{-p}^{\gamma_p} \cap \RR_+^n)$ must be convex, and its recession cone is clearly $\RR_+^n$.
This guarantees that the following loss is proper over $p \in \Delta^n$ \citep[][Proposition 2]{williamson2014geometry}:
\begin{align}
\label{def:elltil}
\elltil_\epsilon (p) = \argmin_{z \in E_\beta^{-1}(T_\ell^\epsilon \cap \RR_+^n)} p' \cdot z, 
\end{align}
where the argmin must be attained uniquely (Appendix \ref{sec:proof}). Our next proposition states that $\elltil_\epsilon$ meets all the requirements of approximation suggested above.

\begin{proposition}
	\label{exp_concave_approx}
	For any $\epsilon > 0$, $\tilde{\ell}_\epsilon$ satisfies the condition $\partial_{\vone_n} \cB_{\beta} \subseteq E_\beta (\ell(\cV))$. In addition, $\elltil_\epsilon = \ell$ over a subset $S_\epsilon \subseteq \Delta^n$,
	where for any $p$ in the relative interior of $\Delta^n$, 
	$p \in S_\epsilon$ for sufficiently small $\epsilon$ i.e. $\lim_{\epsilon\downarrow 0}\mathrm{vol}(\Delta^n\setminus S_\epsilon)=0$.
\end{proposition}

Note $\|\elltil_\epsilon(p)-\ell(p)\|$ is not bounded for $p\notin S_\epsilon$. While the result does not show that \emph{all} $\beta$-mixable losses can be made $\beta$-exp-concave, it is suggestive that such a result may be obtainable by a different argument.

\subsection{Calculus approach}
\label{sec:calculus}
Proper composite losses are defined by the proper loss $\ell$ and the link $\psi$. In this section we will characterize the exp-concave proper composite losses in terms of $(\textsf{H}\Lubartil_{\ell}(\tilde{p}),\textsf{D}\tilde{\psi}(\tilde{p}))$. The following proposition provides the identities of the first and second derivatives of the proper composite losses (\cite{vernet2011composite}).
\begin{proposition}
	\label{multiderivatives}
	For all $i \in [n]$, $\tilde{p} \in \mathring{\tilde{\Delta}}^n$ (the interior of ${\tilde{\Delta}}^n$), and $v=\tilde{\psi}(\tilde{p}) \in \mathcal{V} \subseteq \mathbb{R}_+^{n-1}$ (so $\tilde{p}={\tilde{\psi}}^{-1}(v)$),
	\begin{eqnarray}
	\textnormal{\textsf{D}}\ell_i^{\psi}(v) &=& - (e_i^{n-1}-\tilde{p})' \cdot k(\tilde{p}), \\
	\textnormal{\textsf{H}}\ell_i^{\psi}(v) &=& - \left( (e_i^{n-1}-\tilde{p})' \otimes I_{n-1} \right) \cdot \textnormal{\textsf{D}}_v[k(\tilde{p})] + k(\tilde{p})' \cdot [\textnormal{\textsf{D}}\tilde{\psi}(\tilde{p})]^{-1},
	\end{eqnarray}
	where
	\begin{equation}
	\label{kpeq}
	k(\tilde{p}):=-\textnormal{\textsf{H}}\Lubartil_\ell(\tilde{p}) \cdot [\textnormal{\textsf{D}}\tilde{\psi}(\tilde{p})]^{-1}.
	\end{equation}
\end{proposition}
The term $k(\tilde{p})$ can be interpreted as the curvature of the Bayes risk function $\Lubartil_\ell$ relative to the rate of change of the link function $\tilde{\psi}$. In the binary case where $n=2$, above proposition reduces to
\begin{eqnarray}
(\ell_1^{\psi})'(v) &=& -(1-\tilde{p}) k(\tilde{p}) \quad ; \quad (\ell_2^{\psi})'(v) = \tilde{p} k(\tilde{p}) , \label{binaryfirstder} \\
(\ell_1^{\psi})''(v) &=& \frac{-(1-\tilde{p}) k'(\tilde{p}) + k(\tilde{p})}{\tilde{\psi}'(\tilde{p})} , \label{binarysecder1}\\
(\ell_2^{\psi})''(v) &=& \frac{\tilde{p} k'(\tilde{p}) + k(\tilde{p})}{\tilde{\psi}'(\tilde{p})} , \label{binarysecder2}
\end{eqnarray}
where $k(\tilde{p})=\frac{-{\Lubartil_\ell}''(\tilde{p})}{{\tilde{\psi}}'(\tilde{p})} \geq 0$ and so $\frac{d}{dv}k(\tilde{p})=\frac{d}{d\tilde{p}}k(\tilde{p}) \cdot \frac{d}{dv}\tilde{p}=\frac{k'(\tilde{p})}{{\tilde{\psi}}'(\tilde{p})}$.


A loss $\ell:\Delta^n \rightarrow \mathbb{R}_+^n$ is $\alpha$-exp-concave (i.e. $\Delta^n \ni q \mapsto \ell_y(q)$ is $\alpha$-exp-concave for all $y \in [n]$) if and only if the map $\Delta^n \ni q \mapsto L_{\ell}(p,q)=p' \cdot \ell(q)$ is $\alpha$-exp-concave for all $p \in \Delta^n$. It can be easily shown that the maps $v \mapsto \ell_y^\psi (v)$ are $\alpha$-exp-concave if and only if $\textsf{H}\ell_y^\psi (v) \succcurlyeq \alpha \textsf{D}\ell_y^\psi(v)' \cdot \textsf{D}\ell_y^\psi(v)$. By applying Proposition \ref{multiderivatives} we obtain the following characterization of the $\alpha$-exp-concavity of the composite loss $\ell^\psi$.
\begin{proposition}
	\label{multiexpprop}
	A proper composite loss $\ell^\psi=\ell \circ \psi^{-1}$ is $\alpha$-exp-concave (with $\alpha > 0$ and $v=\tilde{\psi}(\tilde{p})$) if and only if for all $\tilde{p} \in \mathring{\tilde{\Delta}}^n$ and for all $i \in [n]$
	\begin{equation}
	\left( (e_i^{n-1}-\tilde{p})' \otimes I_{n-1} \right) \cdot \textnormal{\textsf{D}}_v[k(\tilde{p})] \preccurlyeq k(\tilde{p})' \cdot [\textnormal{\textsf{D}}\tilde{\psi}(\tilde{p})]^{-1} - \alpha k(\tilde{p})' \cdot (e_i^{n-1}-\tilde{p}) \cdot (e_i^{n-1}-\tilde{p})' \cdot k(\tilde{p}). \label{multiexpcondition}
	\end{equation}
\end{proposition}
Based on this characterization, we can determine which loss functions can be exp-concavified by a chosen link function and how much a link function can exp-concavify a given loss function. In the binary case ($n=2$), the above proposition reduces to the following.
\begin{proposition}
	\label{complexversion}
	Let $\tilde{\psi}:[0,1]\rightarrow \mathcal{V} \subseteq \mathbb{R}$ be an invertible link and $\ell:\Delta^2 \rightarrow \mathbb{R}_+^2$ be a strictly proper binary loss with weight function $w(\tilde{p}):=-\textnormal{\textsf{H}}\Lubartil_\ell(\tilde{p})=-{\Lubartil_\ell}''(\tilde{p})$. Then the binary composite loss $\ell^\psi := \ell \circ \Pi_\Delta^{-1} \circ \tilde{\psi}^{-1}$ is $\alpha$-exp-concave (with $\alpha > 0$) if and only if
	\begin{equation}
	\label{maineq}
	-\frac{1}{\tilde{p}} + \alpha w(\tilde{p}) \tilde{p} \enspace \leq \enspace \frac{w'(\tilde{p})}{w(\tilde{p})} - \frac{{\tilde{\psi}}''(\tilde{p})}{{\tilde{\psi}}'(\tilde{p})} \enspace \leq \enspace \frac{1}{1-\tilde{p}} - \alpha w(\tilde{p}) (1-\tilde{p}), \quad \forall \tilde{p} \in (0,1).
	\end{equation}
\end{proposition}

The following proposition gives an easier to check necessary condition for the binary proper losses that generate an $\alpha$-exp-concave (with $\alpha > 0$) binary composite loss given a particular link function. Since scaling a loss function will not affect what a sensible learning algorithm will do, it is possible to normalize the loss functions by normalizing their weight functions by setting $w(\frac{1}{2}) = 1$. By this normalization we are scaling the original loss function by $\frac{1}{w(\frac{1}{2})}$ and the super-prediction set is scaled by the same factor. If the original loss function is $\beta$-mixable (resp. $\alpha$-exp-concave), then the normalized loss function is $\beta w(\frac{1}{2})$-mixable (resp. $\alpha w(\frac{1}{2})$-exp-concave).
\begin{proposition}
	\label{simplerversion}
	Let $\tilde{\psi}:[0,1]\rightarrow \mathcal{V} \subseteq \mathbb{R}$ be an invertible link and $\ell:\Delta^2 \rightarrow \mathbb{R}_+^2$ be a strictly proper binary loss with weight function $w(\tilde{p}):=-\textnormal{\textsf{H}}\Lubartil_\ell(\tilde{p})=-{\Lubartil_\ell}''(\tilde{p})$ normalised such that $w(\frac{1}{2}) = 1$. Then the binary composite loss $\ell^\psi := \ell \circ \Pi_\Delta^{-1} \circ \tilde{\psi}^{-1}$ is $\alpha$-exp-concave (with $\alpha > 0$) only if
	\begin{equation}
	\label{compsimple}
	\frac{{\tilde{\psi}}'(\tilde{p})}{\tilde{p} (2{\tilde{\psi}}'(\frac{1}{2}) - \alpha ({\tilde{\psi}}(\tilde{p}) - {\tilde{\psi}}(\frac{1}{2})))} \lesseqgtr w(\tilde{p}) \lesseqgtr \frac{{\tilde{\psi}}'(\tilde{p})}{(1 - \tilde{p}) (2{\tilde{\psi}}'(\frac{1}{2}) + \alpha ({\tilde{\psi}}(\tilde{p}) - {\tilde{\psi}}(\frac{1}{2})))}, \quad \forall \tilde{p} \in (0,1),
	\end{equation}
	where $\lesseqgtr$ denotes $\leq$ for $\tilde{p} \geq \frac{1}{2}$ and denotes $\geq$ for $\tilde{p} \leq \frac{1}{2}$.
\end{proposition}

Proposition~\ref{complexversion} provides necessary \textit{and} sufficient conditions for the exp-concavity of binary composite losses, whereas Proposition~\ref{simplerversion} provides simple necessary \textit{but not} sufficient conditions. By setting $\alpha = 0$ in all the above results we have obtained for exp-concavity, we recover the convexity conditions for proper and composite losses which are already derived by \cite{reid2010composite} for the binary case and \cite{vernet2011composite} for multi-class.

\subsection{Link functions}
\label{sec:link}

\begin{figure}
  \centering
    \includegraphics[width=0.55\textwidth]{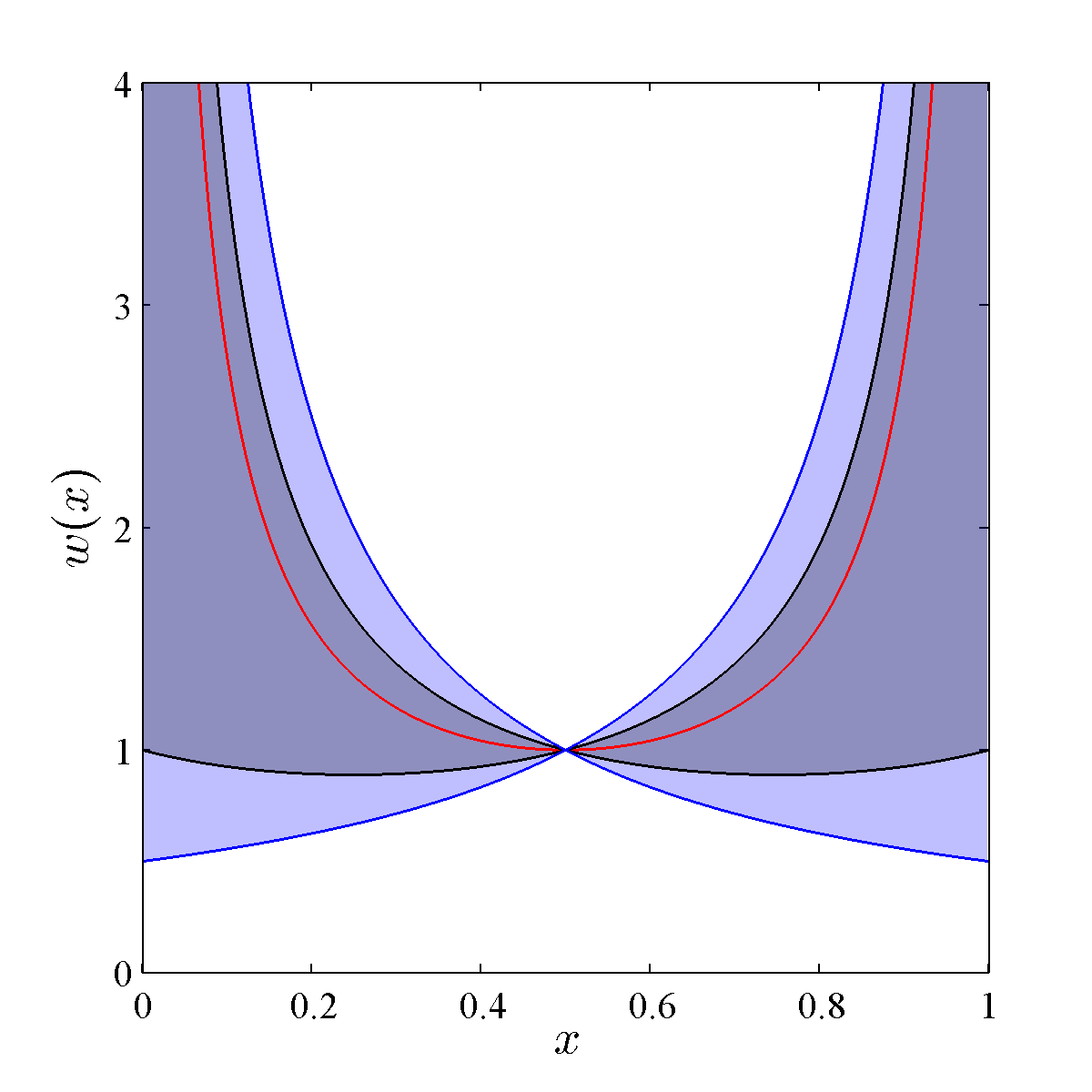}
    \caption[Necessary but not sufficient region of normalised weight functions to ensure $\alpha$-exp-concavity and convexity of proper losses]{Necessary but not sufficient region of normalised weight functions to ensure $\alpha$-exp-concavity and convexity of proper losses (\textcolor{red}{---} $\alpha = 4$; \textcolor{black}{---} $\alpha = 2$; \textcolor{blue}{---} convexity). \label{fig:identity}}
\end{figure}

\begin{figure}
  \centering
    \includegraphics[width=0.55\textwidth]{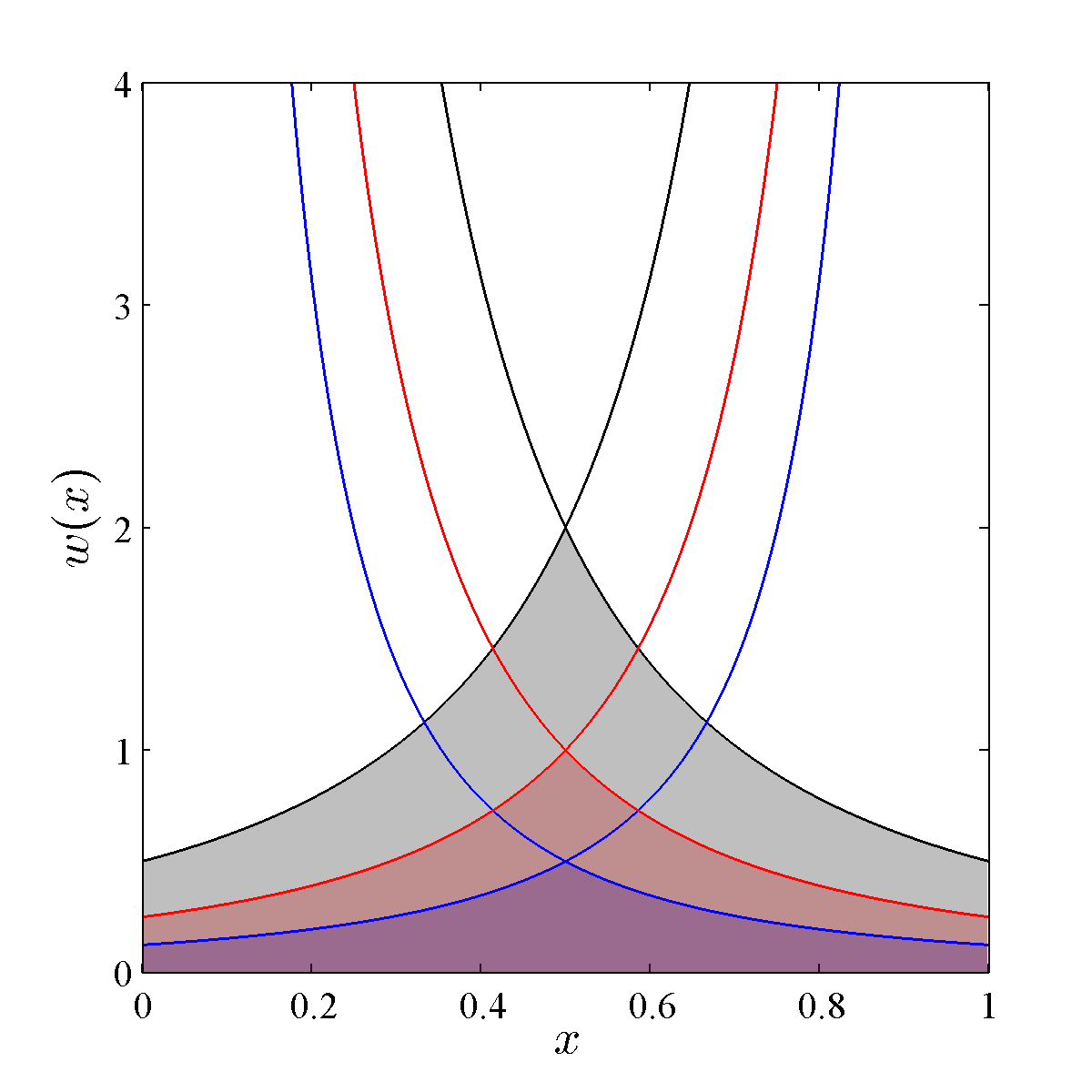}
    \caption[Necessary and sufficient region of unnormalised weight functions to ensure $\alpha$-exp-concavity of composite losses with canonical link]{Necessary and sufficient region of unnormalised weight functions to ensure $\alpha$-exp-concavity of composite losses with canonical link (\textcolor{black}{---} $\alpha = 2$; \textcolor{red}{---} $\alpha = 4$; \textcolor{blue}{---} $\alpha = 8$). \label{fig:can}}
\end{figure}
A proper loss can be exp-concavified ($\alpha > 0$) by some link function only if the loss is mixable ($\beta_{\ell} > 0$) and the maximum possible value for exp-concavity constant is the mixability constant of the loss (since the link function won't change the super-prediction set and an $\alpha$-exp-concave loss is always $\beta$-mixable for some $\beta \geq \alpha$).

By applying the \textit{identity link} ${\tilde{\psi}}(\tilde{p})=\tilde{p}$ in \eqref{maineq} we obtain the necessary and sufficient conditions for a binary proper loss to be $\alpha$-exp-concave (with $\alpha > 0$) as given by,
\begin{equation}
\label{identitynscond}
-\frac{1}{\tilde{p}} + \alpha w(\tilde{p}) \tilde{p} \enspace \leq \enspace \frac{w'(\tilde{p})}{w(\tilde{p})} \enspace \leq \enspace \frac{1}{1-\tilde{p}} - \alpha w(\tilde{p}) (1-\tilde{p}), \quad \forall \tilde{p} \in (0,1).
\end{equation}
By substituting ${\tilde{\psi}}(\tilde{p}) = \tilde{p}$ in \eqref{compsimple} we obtain the following necessary but not sufficient (simpler) constraints for a normalized binary proper loss to be $\alpha$-exp-concave
\begin{equation}
\label{propersimple}
\frac{1}{\tilde{p} (2 - \alpha (\tilde{p} - \frac{1}{2}))} \enspace \lesseqgtr \enspace w(\tilde{p}) \enspace \lesseqgtr \enspace \frac{1}{(1 - \tilde{p}) (2 + \alpha (\tilde{p} - \frac{1}{2}))}, \quad \forall \tilde{p} \in (0,1),
\end{equation}
which are illustrated as the shaded region in Figure~\ref{fig:identity} for different values of $\alpha$. Observe that normalized proper losses can be $\alpha$-exp-concave only for $0 < \alpha \leq 4$. When $\alpha = 4$, only the normalized weight function of log loss ($w_{\ell^{\mathrm{log}}}(\tilde{p})=\frac{1}{4\tilde{p}(1-\tilde{p})}$) will satisfy \eqref{propersimple}, and when $\alpha > 4$, the allowable (necessary) $w(\tilde{p})$ region to ensure $\alpha$-exp-concavity vanishes. Thus normalized log loss is the most exp-concave normalized proper loss. Observe (from Figure~\ref{fig:identity}) that normalized square loss ($w_{\ell^{\mathrm{sq}}}(\tilde{p})=1$) is at most 2-exp-concave. Further from \eqref{propersimple}, if $\alpha' > \alpha$, then the allowable $w(\tilde{p})$ region to ensure $\alpha'$-exp-concavity will be within the region for $\alpha$-exp-concavity, and also any allowable $w(\tilde{p})$ region to ensure $\alpha$-exp-concavity will be within the region for convexity, which is obtained by setting $\alpha = 0$ in \eqref{propersimple}. Here we recall the fact that, if the normalized loss function is $\alpha$-exp-concave, then the original loss function is $\frac{\alpha}{w(\frac{1}{2})}$-exp-concave. The following theorem provides \textit{sufficient} conditions for the exp-concavity of binary proper losses.

\begin{theorem}
	\label{propexpsuff}
	A binary proper loss $\ell:\Delta^2 \rightarrow \mathbb{R}_+^2$ with the weight function $w(\tilde{p})=-{\Lubartil_\ell}''(\tilde{p})$ normalized such that $w(\frac{1}{2})=1$ is $\alpha$-exp-concave (with $\alpha > 0$) if
	\begin{align*}
	w(\tilde{p}) &~=~ \frac{1}{\tilde{p} \left( 2 - \int_{\tilde{p}}^{1/2}a(t)dt \right)} ~\text{ for } a(\tilde{p}) \text{ s.t. } \\
	& \left[ \frac{\alpha (1-\tilde{p})}{\tilde{p}} - \frac{2}{\tilde{p}(1-\tilde{p})} \right] + \frac{1}{\tilde{p}(1-\tilde{p})} \int_{\tilde{p}}^{1/2}a(t)dt ~\leq~ a(\tilde{p}) ~\leq~ -\alpha, \, \forall{\tilde{p} \in (0,1/2]}, \\
	\intertext{and}
	w(\tilde{p}) &~=~ \frac{1}{(1-\tilde{p}) \left( 2 - \int_{\frac{1}{2}}^{\tilde{p}}b(t)dt \right)} ~\text{ for } b(\tilde{p}) \text{ s.t. } \\
	& \left[ \frac{\alpha \tilde{p}}{(1-\tilde{p})} - \frac{2}{\tilde{p}(1-\tilde{p})} \right] + \frac{1}{\tilde{p}(1-\tilde{p})} \int_{1/2}^{\tilde{p}}b(t)dt ~\leq~ b(\tilde{p}) ~\leq~ -\alpha, \, \forall{\tilde{p} \in [\textstyle\frac{1}{2},1)}.
	\end{align*}
\end{theorem}
For square loss we can find that $a(\tilde{p})=\frac{-1}{\tilde{p}^2}$ and $b(\tilde{p})=\frac{-1}{(1-\tilde{p})^2}$ will satisfy the above sufficient condition with $\alpha=4$ and for log loss $a(\tilde{p})=b(\tilde{p})=-4$ will satisfy the sufficient condition with $\alpha=4$. It is also easy to see that for symmetric losses $a(\tilde{p})$ and $b(\tilde{p})$ will be symmetric.

When the \textit{canonical link} function ${\tilde{\psi}}_{\ell}(\tilde{p}) := - \textsf{D}\Lubartil_\ell(\tilde{p})'$ is combined with a strictly proper loss to form $\ell^{\psi_\ell}$, since $\textsf{D}{\tilde{\psi}}_{\ell}(\tilde{p}) = - \textsf{H}\Lubartil_\ell(\tilde{p})$, the first and second derivatives of the composite loss become considerably simpler as follows
\begin{eqnarray}
\textsf{D}\ell_i^{\psi_\ell}(v) &=& - (e_i^{n-1}-\tilde{p})', \label{canlinkcond1} \\
\textsf{H}\ell_i^{\psi_\ell}(v) &=& - [\textsf{H}\Lubartil_\ell(\tilde{p})]^{-1}. \label{canlinkcond2}
\end{eqnarray}
Since a proper loss $\ell$ is $\beta$-mixable if and only if $\beta \textsf{H}\Lubartil_\ell(\tilde{p}) \succcurlyeq \textsf{H}\Lubartil_{\ell^{\mathrm{log}}}(\tilde{p})$ for all $\tilde{p} \in \mathring{\tilde{\Delta}}^n$ (\cite{van2012mixability}), by applying the canonical link any $\beta$-mixable proper loss will be transformed to $\alpha$-exp-concave proper composite loss (with $\beta \geq \alpha > 0$) but $\alpha = \beta$ is not guaranteed in general. In the binary case, since ${\tilde{\psi}}_{\ell}'(\tilde{p}) = - {\Lubartil_\ell}''(\tilde{p}) = w(\tilde{p})$, we get
\begin{equation}
\label{binarycanonical}
w(\tilde{p}) \leq \frac{1}{\alpha \tilde{p}^2} \quad \text{and} \quad w(\tilde{p}) \leq \frac{1}{\alpha (1-\tilde{p})^2}, \quad \forall \tilde{p} \in (0,1),
\end{equation}
as the necessary and sufficient conditions for $\ell^{\psi_{\ell}}$ to be $\alpha$-exp-concave. In this case when $\alpha \rightarrow \infty$ the allowed region vanishes (since for proper losses $w(\tilde{p}) \geq 0$). From Figure~\ref{fig:can} it can be seen that, if the normalized loss function satisfies
\begin{equation*}
w(\tilde{p}) \leq \frac{1}{4\tilde{p}^2} \quad \text{and} \quad w(\tilde{p}) \leq \frac{1}{4(1-\tilde{p})^2}, \quad \forall \tilde{p} \in (0,1),
\end{equation*}
then the composite loss obtained by applying the canonical link function on the unnormalized loss with weight function $w_{\mathrm{org}}(\tilde{p})$ is $\frac{4}{w_{\mathrm{org}}(\frac{1}{2})}$-exp-concave.

We now consider whether one can always find a link function that can transform a $\beta$-mixable proper loss into $\beta$-exp-concave composite loss. In the binary case, such a link function exists and is given in the following corollary.
\begin{corollary}
	\label{specialcoro}
	Let $w_{\ell}(\tilde{p})=-{\Lubartil_\ell}''(\tilde{p})$. The exp-concavifying link function $\tilde{\psi}_{\ell}^*$ defined via
	\begin{equation}
	\tilde{\psi}_{\ell}^*(\tilde{p}) = \frac{w_{\ell^{\mathrm{log}}}(\frac{1}{2})}{w_{\ell}(\frac{1}{2})} \int_{0}^{\tilde{p}}{\frac{w_{\ell}(v)}{w_{\ell^{\mathrm{log}}}(v)}dv}, \quad \forall{\tilde{p} \in [0,1]}
	\end{equation}
	(which is a valid strictly increasing link function) will always transform a $\beta$-mixable proper loss $\ell$ into $\beta$-exp-concave composite loss $\ell^{\psi_{\ell}^*}$, where $\ell^{\psi_{\ell}^*}_y(v)=\ell_y \circ \Pi_\Delta^{-1} \circ (\tilde{\psi}_{\ell}^*)^{-1}(v)$.
\end{corollary}
For log loss, the exp-concavifying link is equal to the identity link and the canonical link could be written as $\int_{0}^{\tilde{p}}{w_{\ell}(v)dv}$. If $\ell$ is a binary proper loss with weight function $w_\ell(\tilde{p})$, then we can define a new proper loss $\ell^{\mathrm{mix}}$ with weight function $w_{\ell^{\mathrm{mix}}}(\tilde{p})=\frac{w_\ell(\tilde{p})}{w_{\ell^{\mathrm{log}}}(\tilde{p})}$. Then applying the exp-concavifying link $\tilde{\psi}_{\ell}^*$ on the original loss $\ell$ is equivalent to applying the canonical link ${\tilde{\psi}}_{\ell}$ on the new loss $\ell^{\mathrm{mix}}$.

The links constructed by the geometric and calculus approaches can be completely different (see Appendix~\ref{sec:square_loss}, \ref{sec:boosting_loss}, and \ref{sec:log_loss}). The former can be further varied by replacing $\one_n$ with any direction in the positive orthant, and the latter can be arbitrarily rescaled. Furthermore, as both links satisfy \eqref{maineq} with $\alpha=\beta$, any appropriate interpolation also works.

\section{Conclusions}
\label{sec:conc}

If a loss is $\beta$-mixable, one can run the Aggregating Algorithm with learning rate $\beta$ and obtain a $\frac{\log N}{\beta}$ regret bound. Similarly a $\frac{\log N}{\alpha}$ regret bound can be attained by the Weighted Average Algorithm with learning rate $\alpha$, when the loss is $\alpha$-exp-concave. \cite{vovk2001competitive} observed that the weighted average of the expert predictions (\cite{kivinen1999averaging}) will be a \textit{perfect} (in the technical sense defined in \cite{vovk2001competitive}) substitution function for the Aggregating Algorithm if and only if the loss function is exp-concave. Thus if we have to use a proper, mixable but non-exp-concave loss function $\ell$ for a sequential prediction (online learning) problem, an $O(1)$ regret bound could be achieved by the following two approaches:
\begin{itemize}
	\item{Use the Aggregating Algorithm (\cite{vovk1995game}) with the \textit{inverse loss} $\ell^{-1}$ (\cite{williamson2014geometry}) as the universal substitution function.}
	\item{Apply the exp-concavifying link ($\tilde{\psi}_{\ell}^*$) on $\ell$, derive the $\beta_{\ell}$-exp-concave composite loss $\ell^{\psi_{\ell}^{*}}$. Then use the Weighted Average Algorithm (\cite{kivinen1999averaging}) with $\ell^{\psi_{\ell}^{*}}$ to obtain the learner's prediction in the transformed domain ($v_{\mathrm{avg}} \in \psi_{\ell}^{*}(\tilde{\Delta}^n)$). Finally output the inverse link value of this prediction ($(\psi_{\ell}^{*})^{-1}(v_{\mathrm{avg}})$).}
\end{itemize}
In either approach we are faced with a computational problem of evaluating an inverse function. But in the binary class case the inverse of a strictly monotone function can be efficiently evaluated using one sided bisection method (or lookup table). So in conclusion, the latter approach can be more convenient and efficient in computation than the former.

When $n=2$, we have shown that one can always transform a beta-mixable proper loss into beta-exp-concave proper composite loss using either geometric link function (Proposition~\ref{geoprop}) or calculus-based link function (Corollary~\ref{specialcoro}). When $n>2$, we observed that the square loss (which is a mixable proper loss) cannot be exp-concavified via the geometric link. And by the calculus approach, it is hard to obtain an explicit form for exp-concavifying link function when $n>2$. Thus when $n>2$, characterization of proper mixable losses that are exp-concavifiable still remains an open problem. For this, we may need to consider other possible ways to construct link functions \cite{williamson2014geometry}.

\section{Appendix}
\subsection{Substitution Functions}
\label{sec:subfunc}

\begin{figure}
	\centering
	\includegraphics[width=8cm]{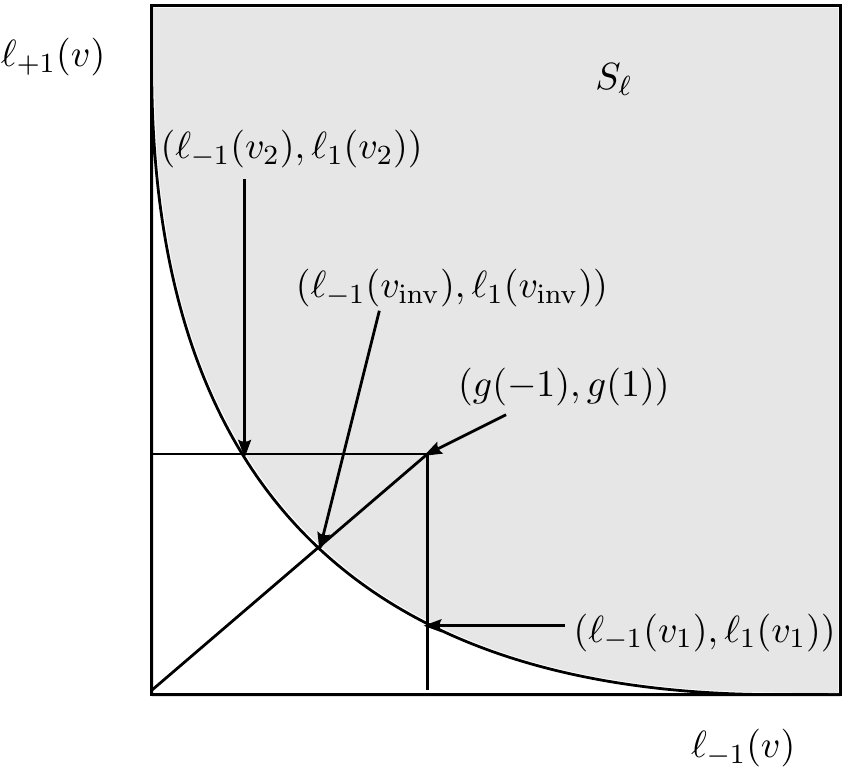}
	\caption[Super-prediction set ($S_\ell$) of a binary game]{Super-prediction set ($S_\ell$) of a binary game $(\{-1,1\},[0,1],\ell)$, where a generalized prediction is represented by $(g(-1),g(1))$. \label{fig:losscurve}}
\end{figure}

We consider the following choices of substitution functions --- \textit{Best look ahead}, \textit{Worst look ahead}, \textit{Inverse loss} (\cite{williamson2014geometry}) and \textit{Weighted average} (\cite{kivinen1999averaging}). The first two choices are just hypothetical ones as they look ahead the actual outcome first and then choose their predictions to incur low and high loss respectively, still they give us an idea about how well or worse the Aggregating Algorithm performs over a given outcome sequence. The weighted average is a computationally efficient and easily implementable substitution function, but it is applicable only for exp-concave losses.

For a binary game represented by $\left(\{-1,1\},[0,1],\ell \right)$ and shown in the Figure \ref{fig:losscurve},
\begin{itemize}
	\item{If the outcome $y=-1$, the Best look ahead and the Worst look ahead will choose the predictions $v_2$ and $v_1$ and incur losses $\ell_{-1}(v_2)$ and $\ell_{-1}(v_1)=g(-1)$ respectively; and if $y=1$, they will choose $v_1$ and $v_2$ and suffer losses $\ell_{1}(v_1)$ and $\ell_{1}(v_2)=g(1)$ respectively.}
	\item{The Inverse loss will choose the prediction $v_{\mathrm{inv}}$ such that $\frac{\ell_{1}(v_{\mathrm{inv}})}{\ell_{-1}(v_{\mathrm{inv}})}=\frac{g(1)}{g(-1)}$ and will incur a loss $\ell(y,v_{\mathrm{inv}})$, and the Weighted average will choose $v_{\mathrm{avg}}=\sum_i w^i v^i$ (where $w^i$ and $v^i$ are the weight and the prediction of the $i$-the expert respectively) and will incur a loss $\ell(y,v_{\mathrm{avg}})$.}
\end{itemize}
Further if the loss function $\ell$ is chosen to be the square loss (which is both 2-mixable and $\frac{1}{2}$-exp-concave), then we have $v_1=\sqrt{g(-1)}$ (since $\ell_{-1}(v_1)=v_1^2=g(-1)$), $v_2=1-\sqrt{g(1)}$ (since $\ell_{1}(v_2)=(1-v_2)^2=g(1)$), and $v_{\mathrm{inv}}=\frac{\sqrt{g(-1)}}{\sqrt{g(-1)}+\sqrt{g(1)}}$ (since $\frac{(1-v_{\mathrm{inv}})^2}{v_{\mathrm{inv}}^2}=\frac{g(1)}{g(-1)}$). Thus for a binary square loss game over an outcome sequence  $y_1,...,y_T$, the cumulative losses of the Aggregating Algorithm for different choices of substitution function are given as follows:
\begin{itemize}
	\item{Best look ahead: $\sum_1^T \left(1-\sqrt{g_t(-y_t)}\right)^2$}
	\item{Worst look ahead: $\sum_1^T g_t(y_t)$}
	\item{Inverse loss: $\sum_1^T \left(y_t-\frac{\sqrt{g_t(0)}}{\sqrt{g_t(0)}+\sqrt{g_t(1)}}\right)^2$}
	\item{Weighted average: $\sum_1^T \left(y_t-\sum_i w_i^t v_i^t\right)^2$}
\end{itemize}

Some experiments are conducted on a binary square loss game to compare these substitution functions. For this, binary outcome sequences of 100 elements are generated using the Bernoulli distribution with success probabilities 0.5, 0.7, 0.9, and 1.0 (these sequences are represented by $\{y_t\}_{p=0.5}$, $\{y_t\}_{p=0.7}$, $\{y_t\}_{p=0.9}$ and $\{y_t\}_{p=1.0}$ respectively). Furthermore the following expert settings are used:
\begin{itemize}
	\item{2 experts where one expert always make the prediction $v=0$, and the other one always makes the prediction $v=1$. This setting is represented by $\{E_t\}_{\text{set.}1}$.}
	\item{3 experts where two experts are as in the previous setting, and the other one is always accurate expert. This setting is represented by $\{E_t\}_{\text{set.}2}$.}
	\item{101 constant experts where the prediction values of the experts are from 0 to 1 with equal interval. This setting is represented by $\{E_t\}_{\text{set.}3}$.}
\end{itemize}
The results of these experiments are presented in the figures \ref{fig:p0d5},\ref{fig:p0d7},\ref{fig:p0d9}, and \ref{fig:p1d0}. From these figures, it can be seen that for the expert setting $\{E_t\}_{\text{set.}1}$, the difference between the regret values of the worst look ahead and the best look ahead substitution functions relative to the theoretical regret bound is very high, whereas that relative difference is very low for the expert setting $\{E_t\}_{\text{set.}3}$. Further the performance of the Aggregating Algorithm over a real dataset is shown in the Figure \ref{fig:footballdata}. From these results for both simulated dataset (for all three expert settings) and real dataset, observe that the difference between the regret values of the inverse loss and the weighted average substitution functions relative to the theoretical regret bound is very low.

\begin{figure}[htbp]
	\caption[Cumulative regret of the Aggregating Algorithm over the outcome sequence for different choices of substitution functions]{Cumulative regret of the Aggregating Algorithm over the outcome sequence $\{y_t\}_{p=0.5}$ for different choices of substitution functions (Best look ahead(\textcolor{blue}{---}), Worst look ahead(\textcolor{black}{---}), Inverse loss(\textcolor{green}{---}), and Weighted average(\textcolor{magenta}{---})) with learning rate $\eta$ and expert setting $\{E_t\}_i$ (theoretical regret bound is shown by \textcolor{red}{- - -}). \label{fig:p0d5}}
	{
		\subfigure[$\eta = 0.1, \{E_t\}_{\text{set.}1}$]{\label{fig:p0d5_n0d1_e1}
			\includegraphics[width=0.32\linewidth]{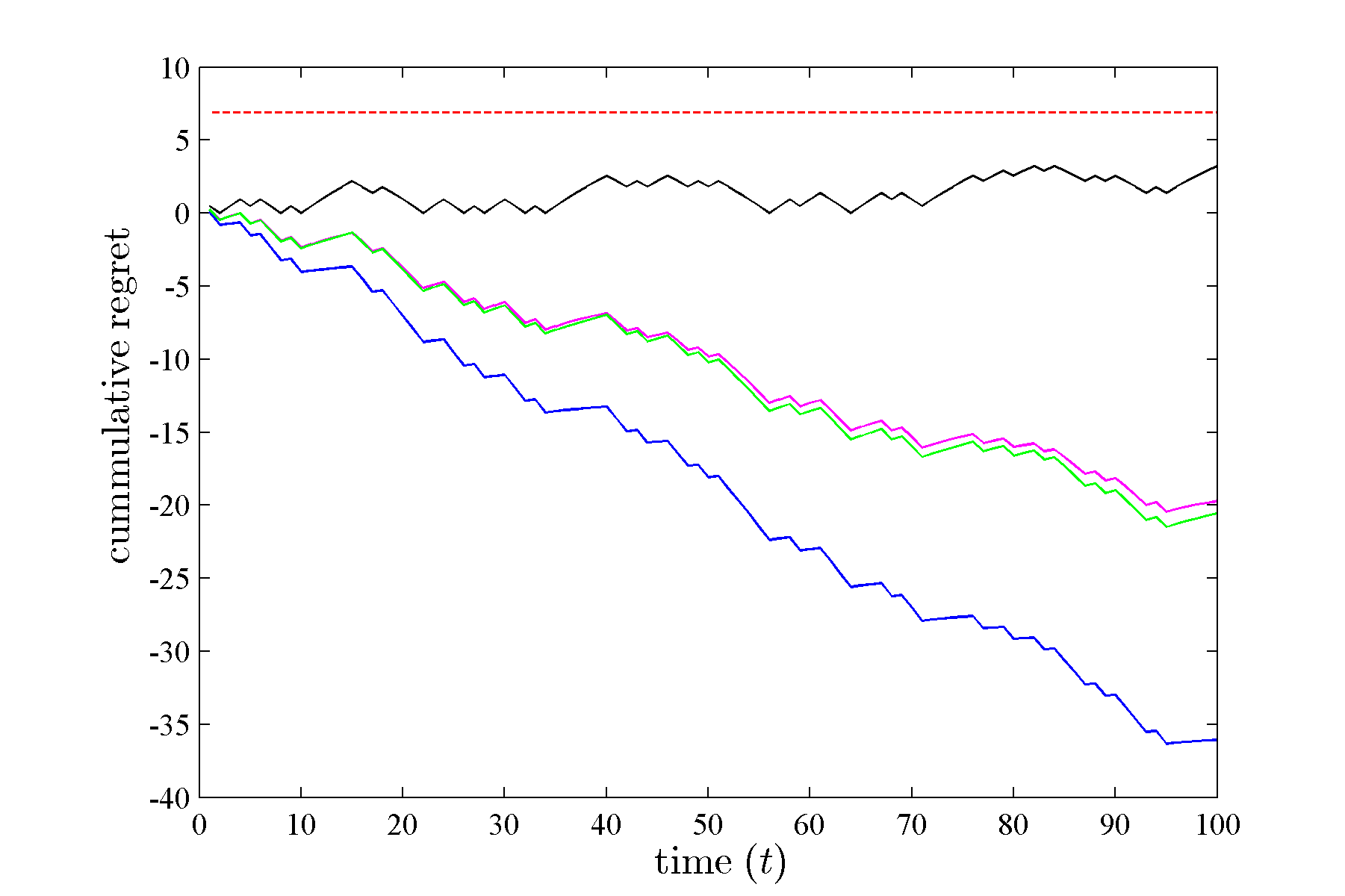}}
		\subfigure[$\eta = 0.3, \{E_t\}_{\text{set.}1}$]{\label{fig:p0d5_n0d3_e1}
			\includegraphics[width=0.32\linewidth]{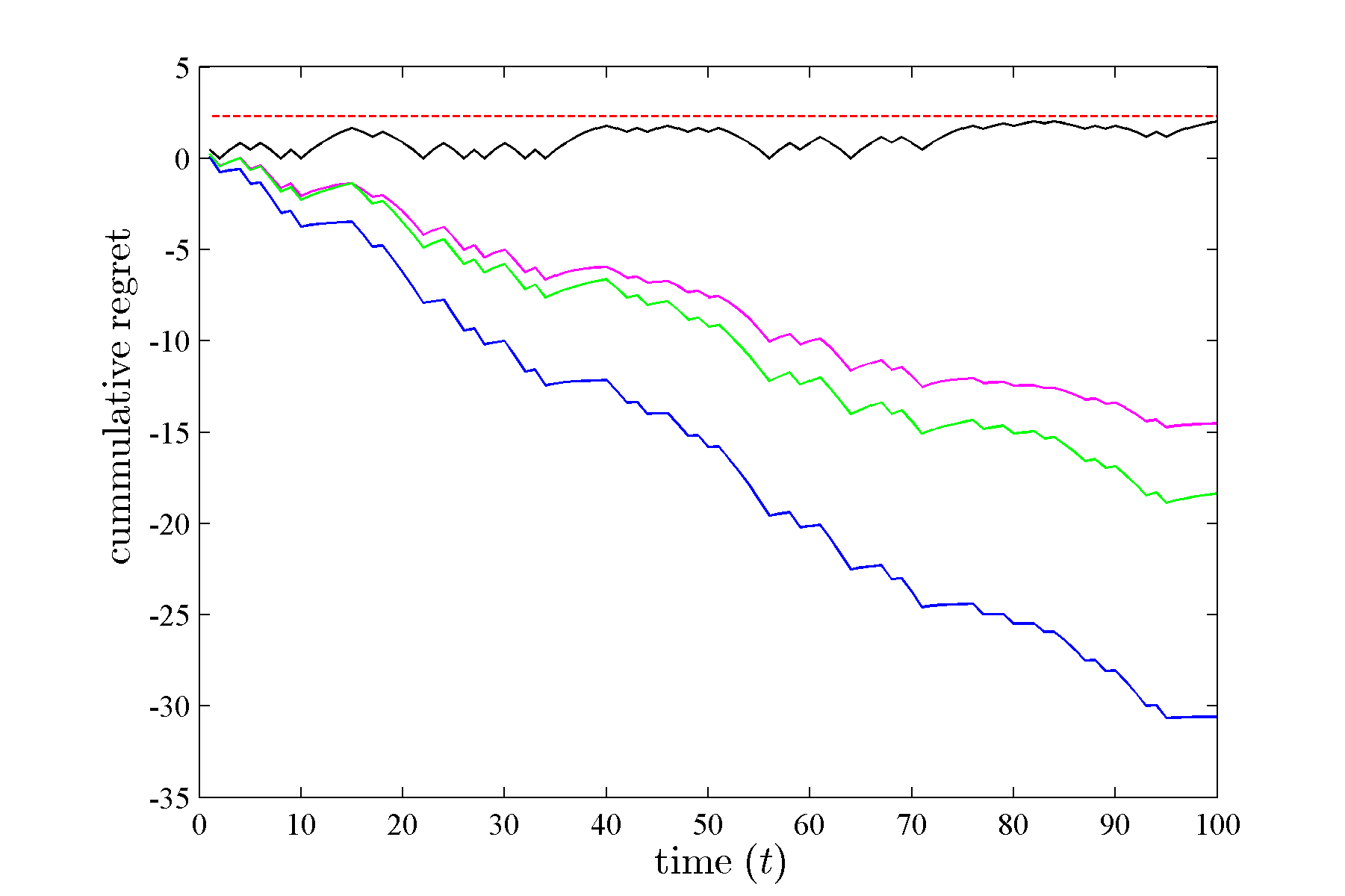}}
		\subfigure[$\eta = 0.5, \{E_t\}_{\text{set.}1}$]{\label{fig:p0d5_n0d5_e1}
			\includegraphics[width=0.32\linewidth]{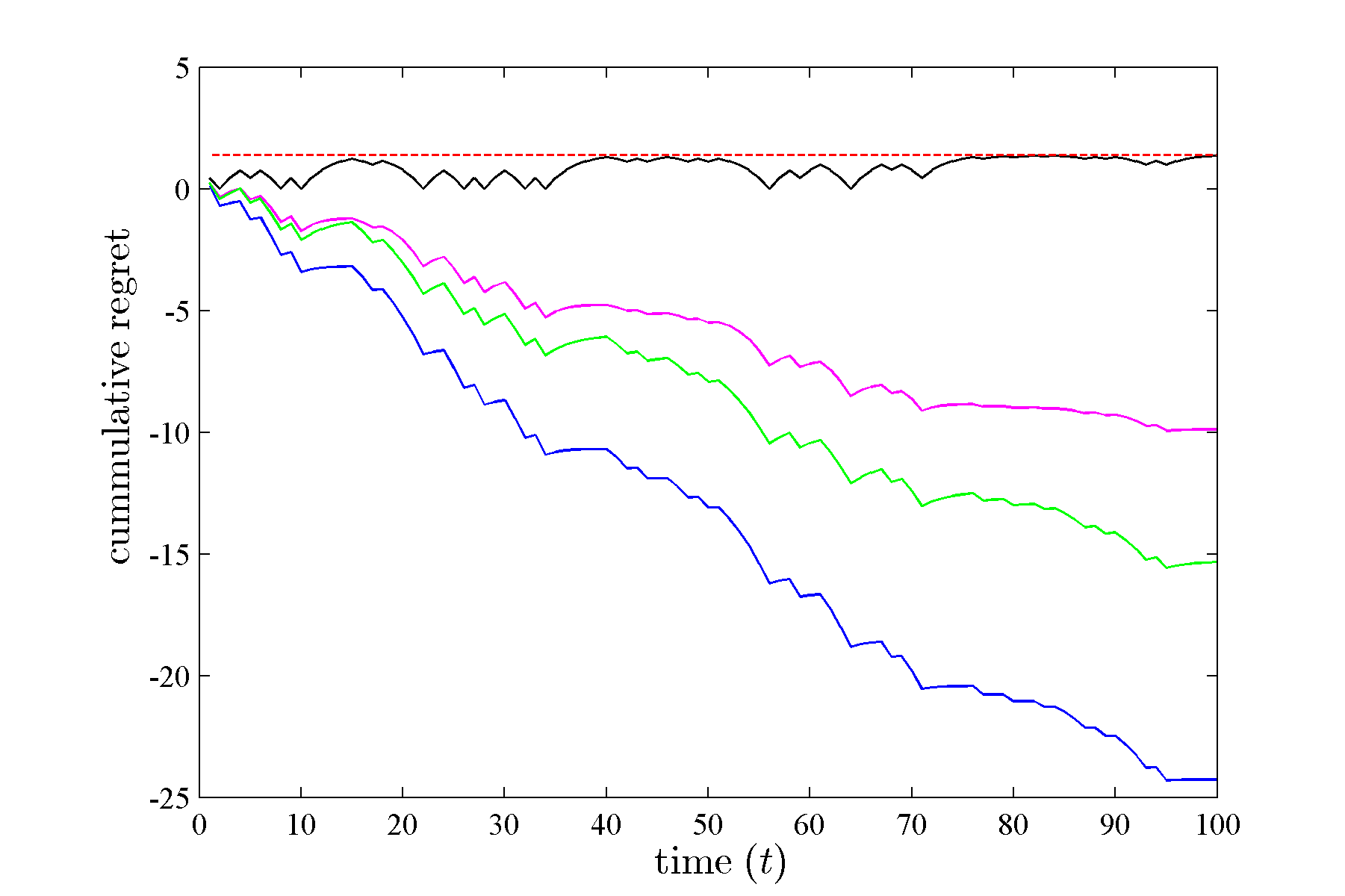}}
		
		\subfigure[$\eta = 0.1, \{E_t\}_{\text{set.}2}$]{\label{fig:p0d5_n0d1_e2}
			\includegraphics[width=0.32\linewidth]{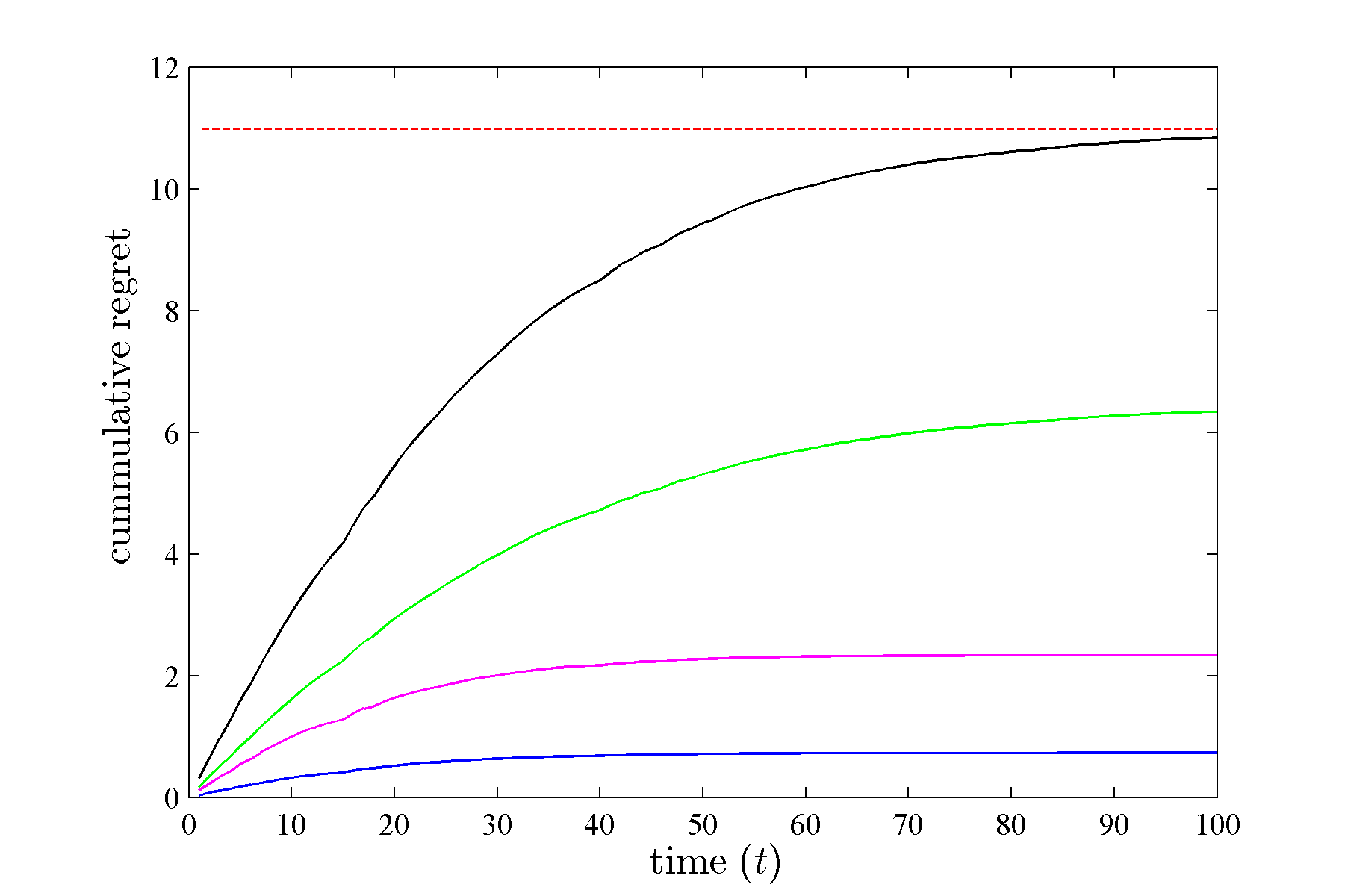}}
		\subfigure[$\eta = 0.3, \{E_t\}_{\text{set.}2}$]{\label{fig:p0d5_n0d3_e2}
			\includegraphics[width=0.32\linewidth]{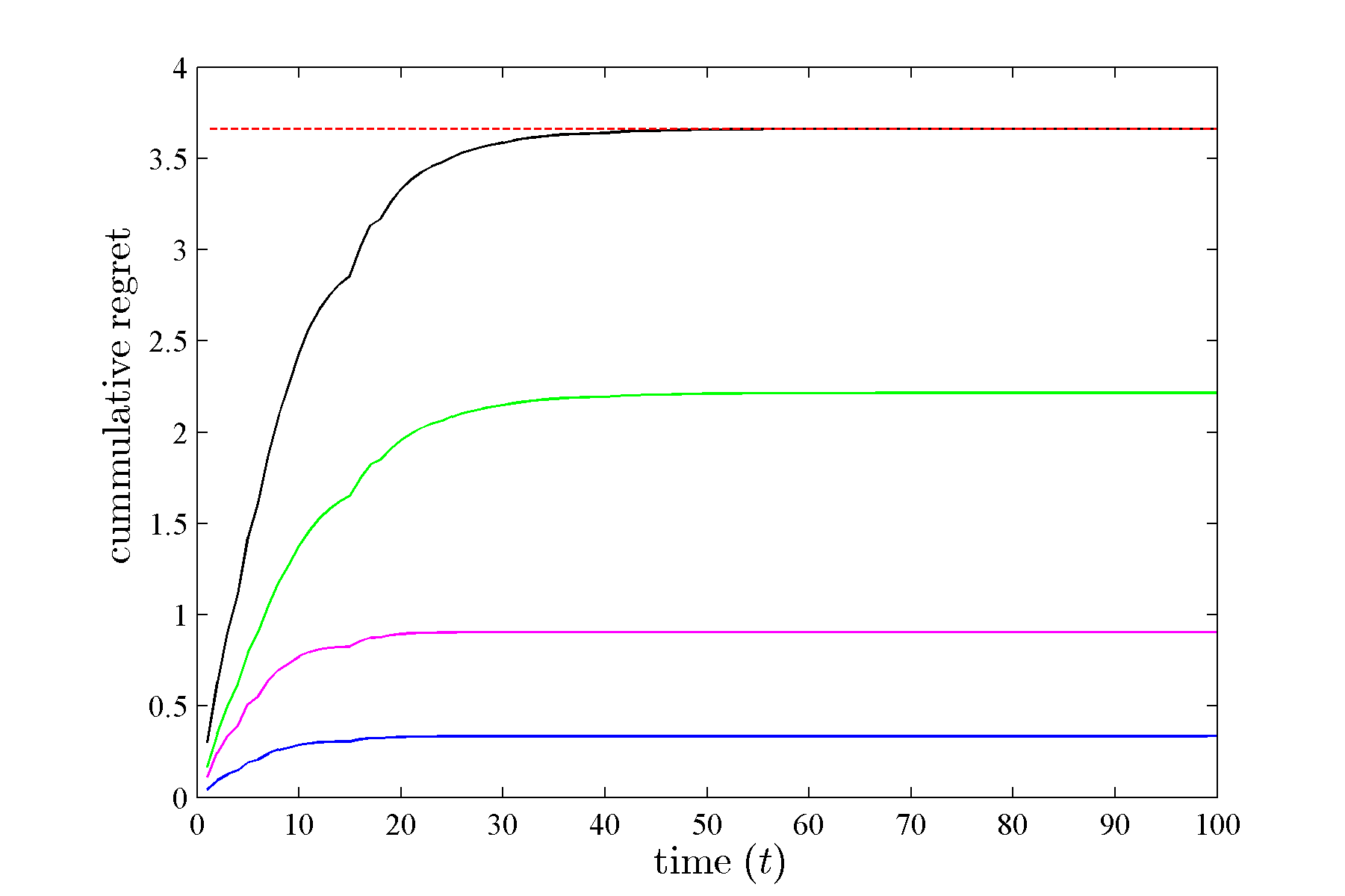}}
		\subfigure[$\eta = 0.5, \{E_t\}_{\text{set.}2}$]{\label{fig:p0d5_n0d5_e2}
			\includegraphics[width=0.32\linewidth]{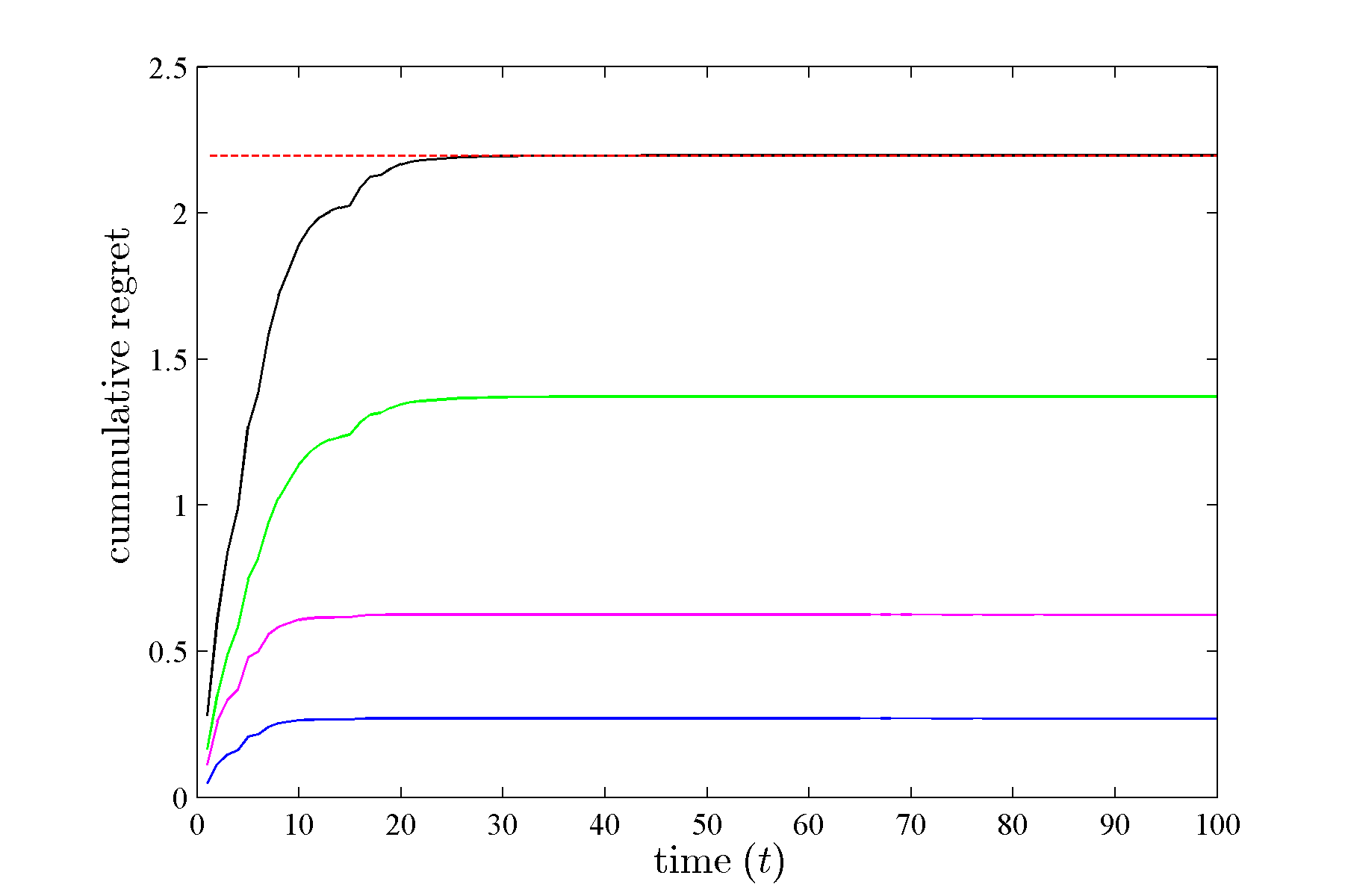}}
		
		\subfigure[$\eta = 0.1, \{E_t\}_{\text{set.}3}$]{\label{fig:p0d5_n0d1_e3}
			\includegraphics[width=0.32\linewidth]{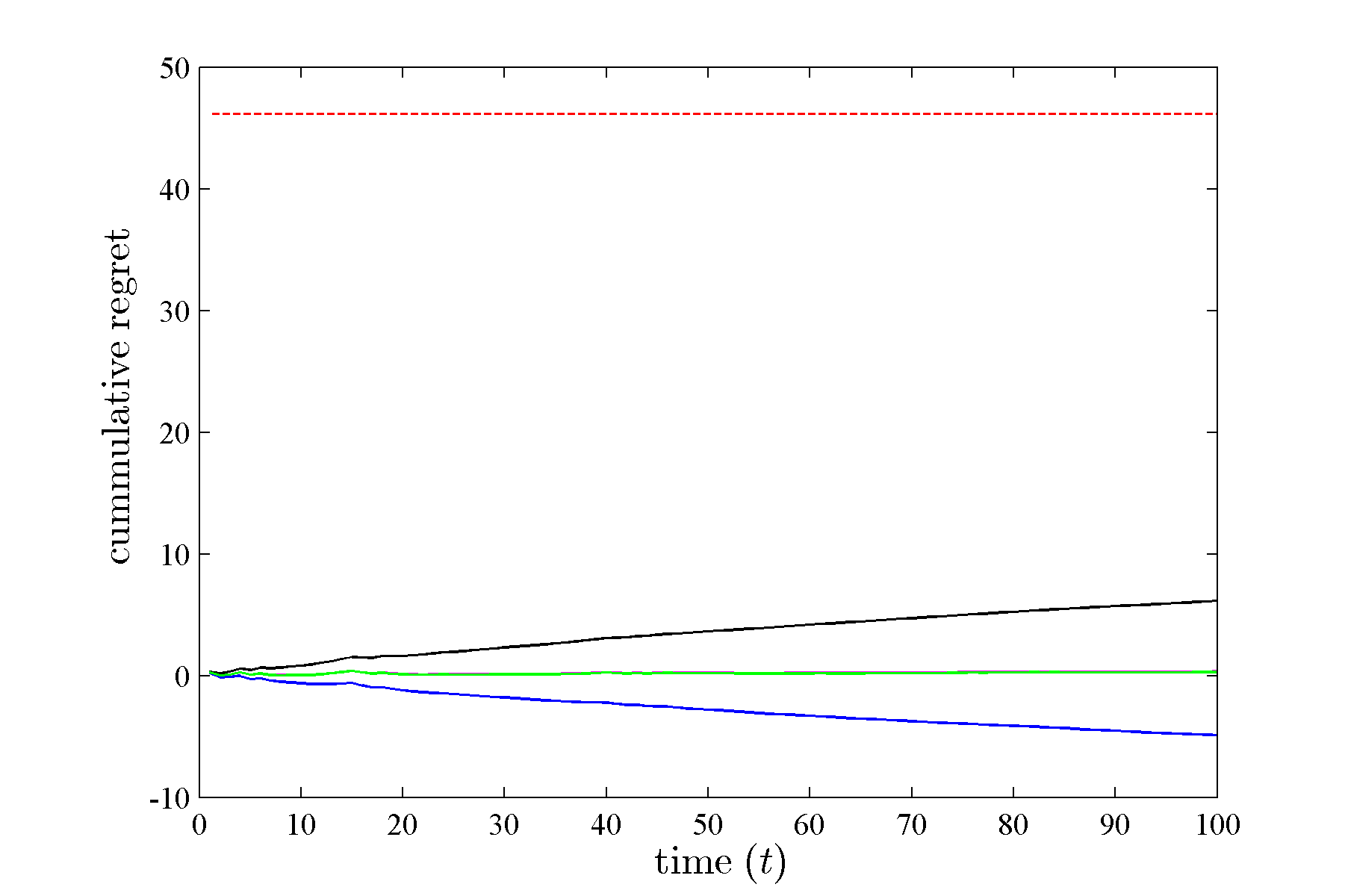}}
		\subfigure[$\eta = 0.3, \{E_t\}_{\text{set.}3}$]{\label{fig:p0d5_n0d3_e3}
			\includegraphics[width=0.32\linewidth]{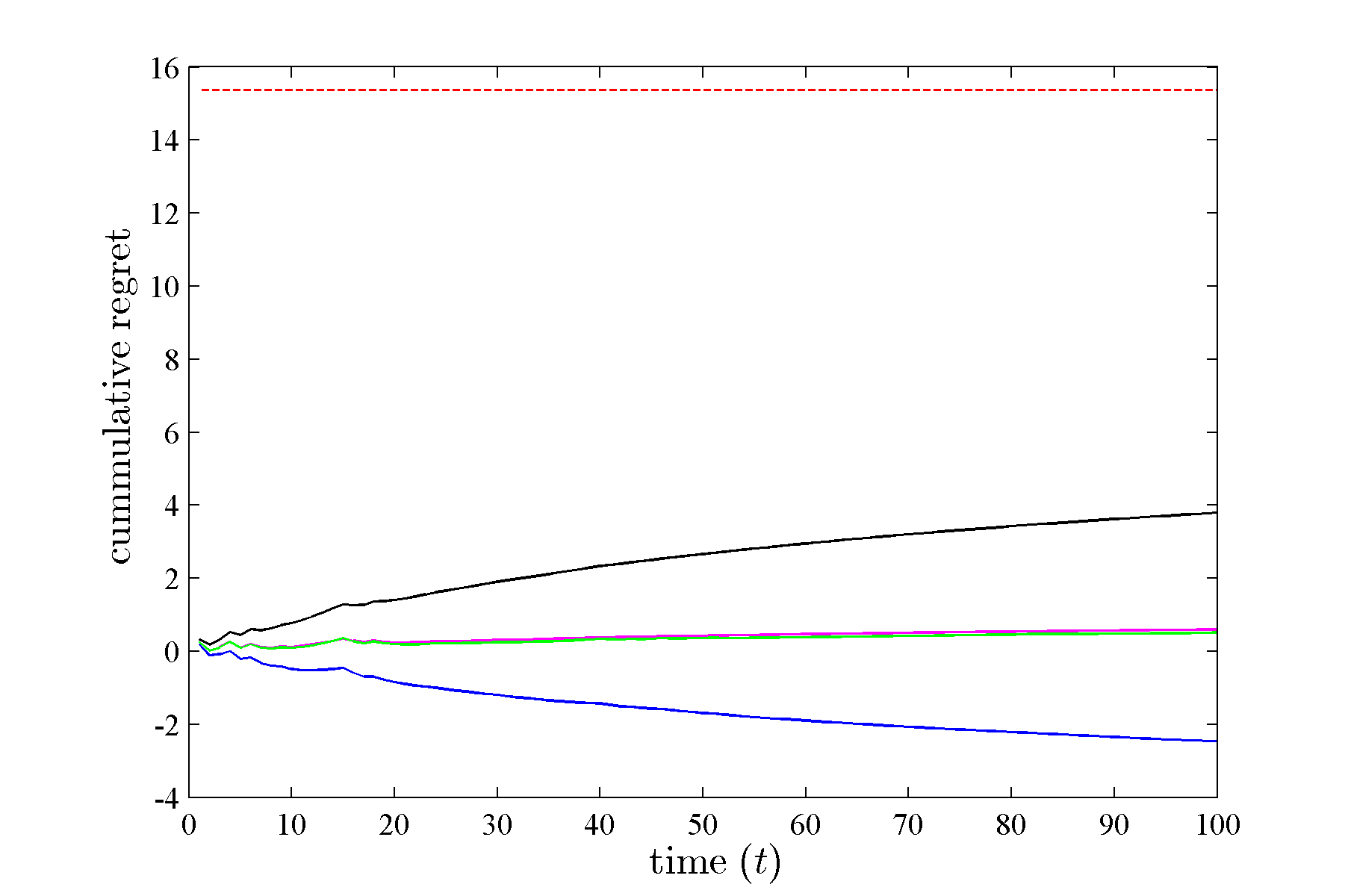}}
		\subfigure[$\eta = 0.5, \{E_t\}_{\text{set.}3}$]{\label{fig:p0d5_n0d5_e3}
			\includegraphics[width=0.32\linewidth]{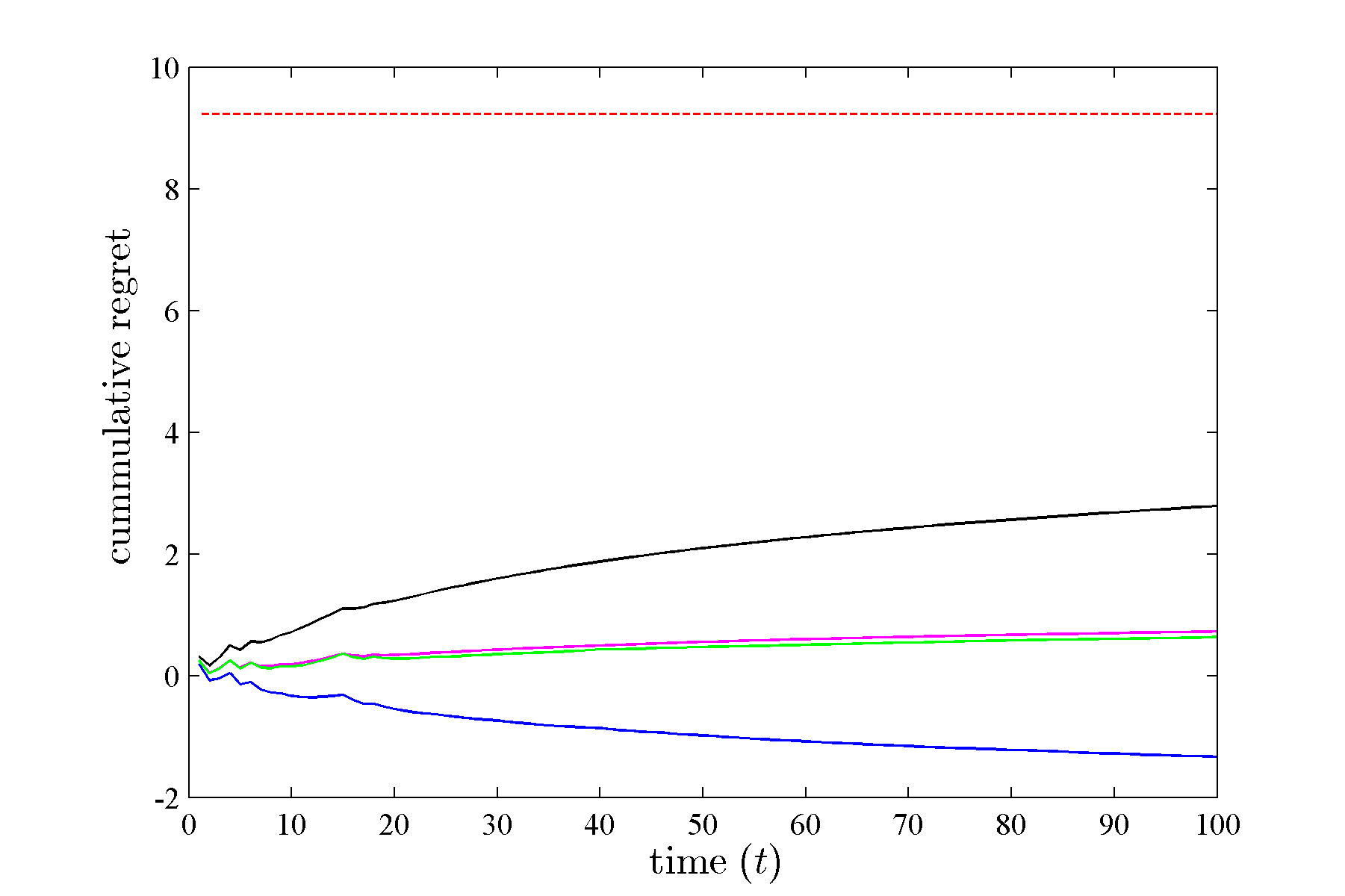}}
	}
\end{figure}

\begin{figure}[htbp]
	\caption[Cumulative regret of the Aggregating Algorithm over the outcome sequence for different choices of substitution functions]{Cumulative regret of the Aggregating Algorithm over the outcome sequence $\{y_t\}_{p=0.7}$ for different choices of substitution functions (Best look ahead(\textcolor{blue}{---}), Worst look ahead(\textcolor{black}{---}), Inverse loss(\textcolor{green}{---}), and Weighted average(\textcolor{magenta}{---})) with learning rate $\eta$ and expert setting $\{E_t\}_i$ (theoretical regret bound is shown by \textcolor{red}{- - -}). \label{fig:p0d7}}
	{
		\subfigure[$\eta = 0.1, \{E_t\}_{\text{set.}1}$]{\label{fig:p0d7_n0d1_e1}
			\includegraphics[width=0.32\linewidth]{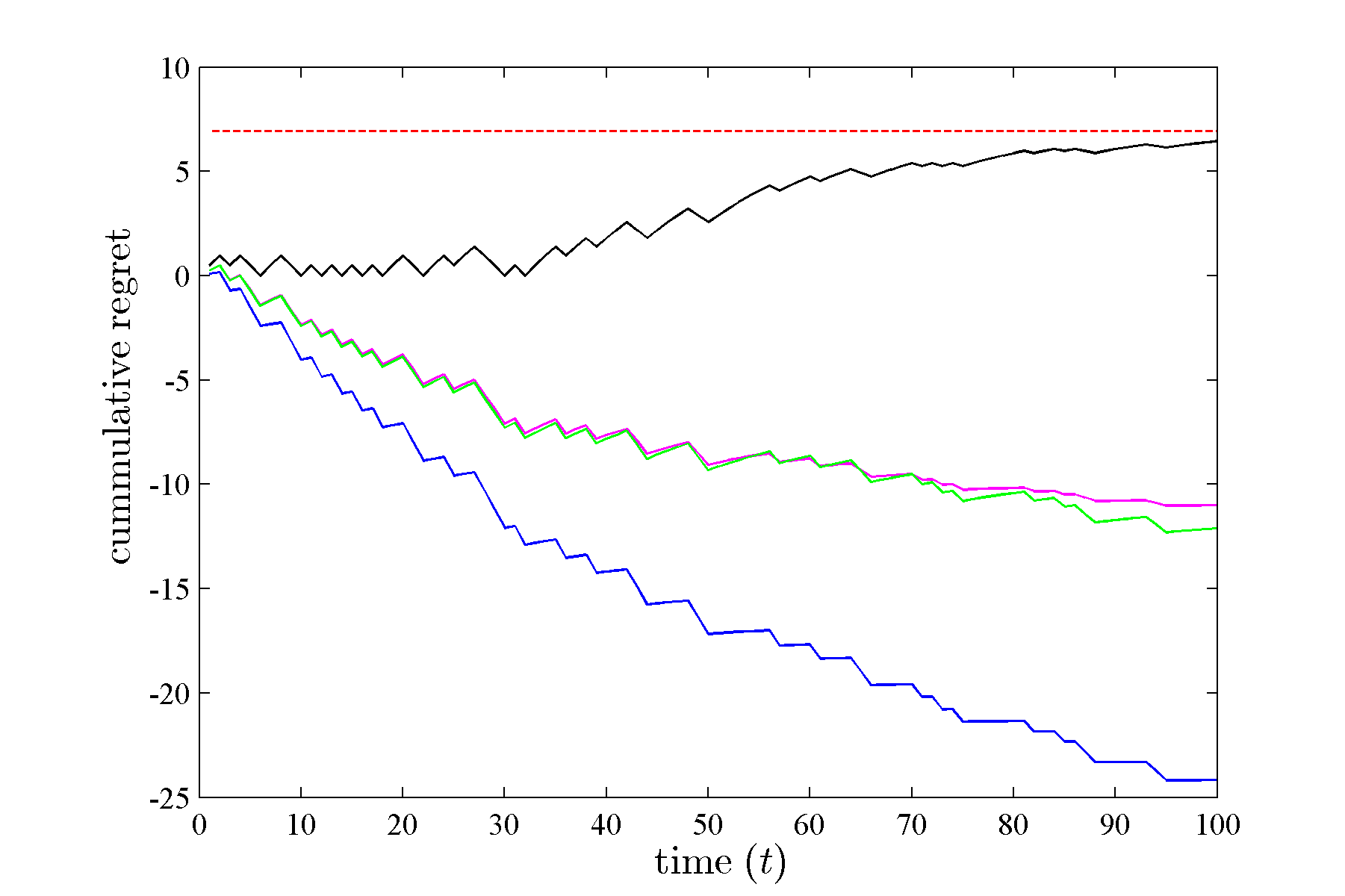}}
		\subfigure[$\eta = 0.3, \{E_t\}_{\text{set.}1}$]{\label{fig:p0d7_n0d3_e1}
			\includegraphics[width=0.32\linewidth]{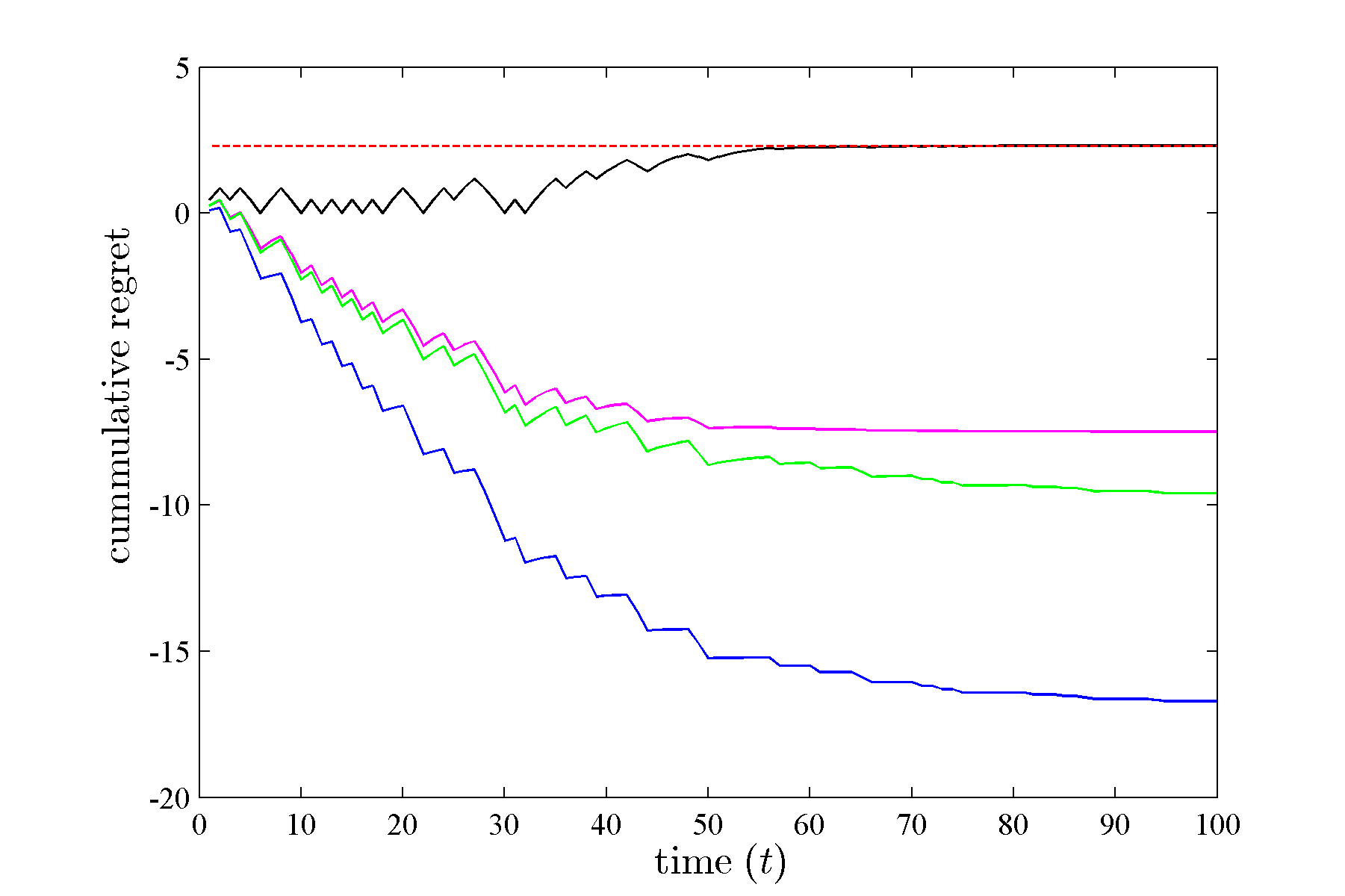}}
		\subfigure[$\eta = 0.5, \{E_t\}_{\text{set.}1}$]{\label{fig:p0d7_n0d5_e1}
			\includegraphics[width=0.32\linewidth]{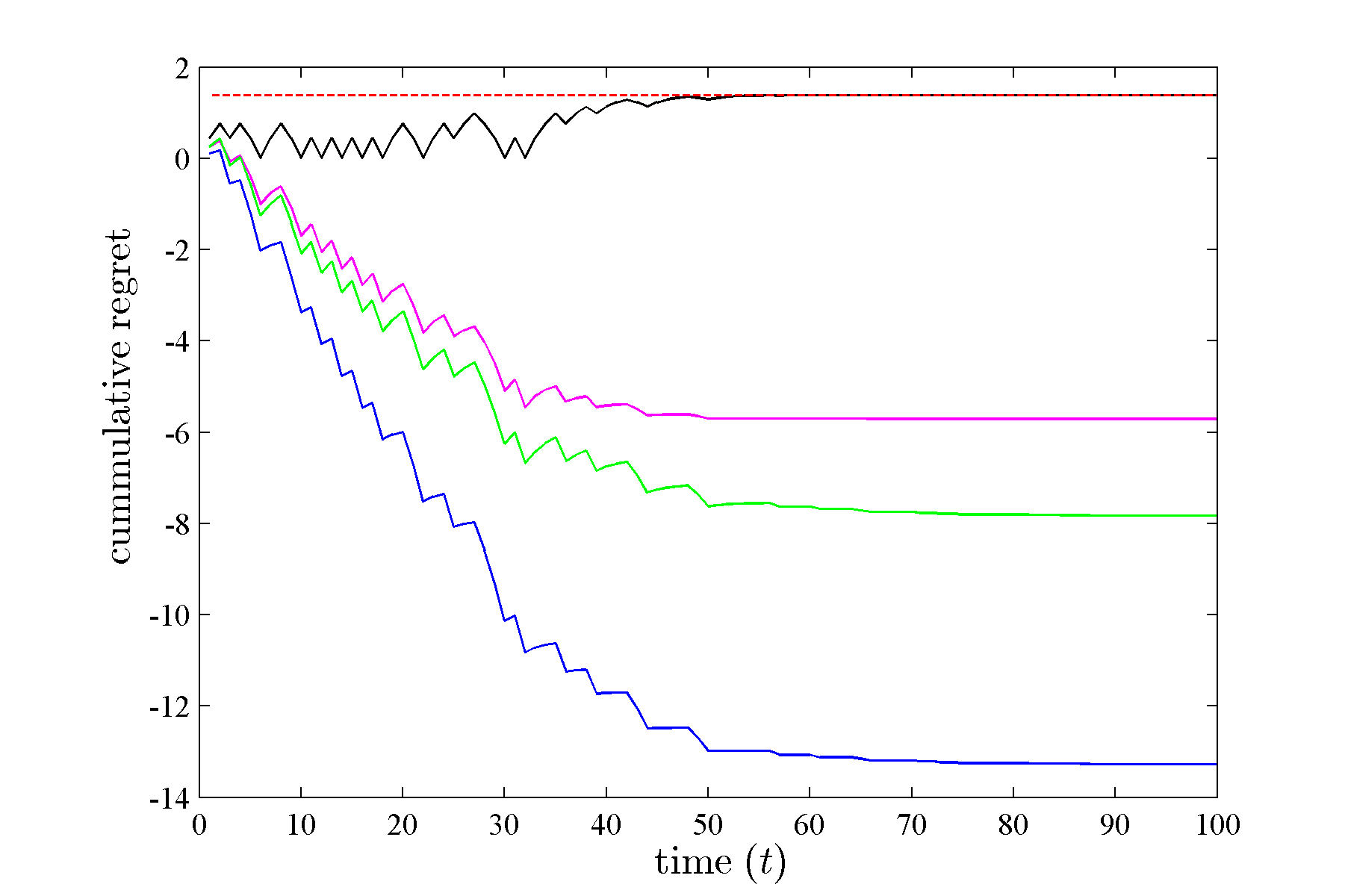}}
		
		\subfigure[$\eta = 0.1, \{E_t\}_{\text{set.}2}$]{\label{fig:p0d7_n0d1_e2}
			\includegraphics[width=0.32\linewidth]{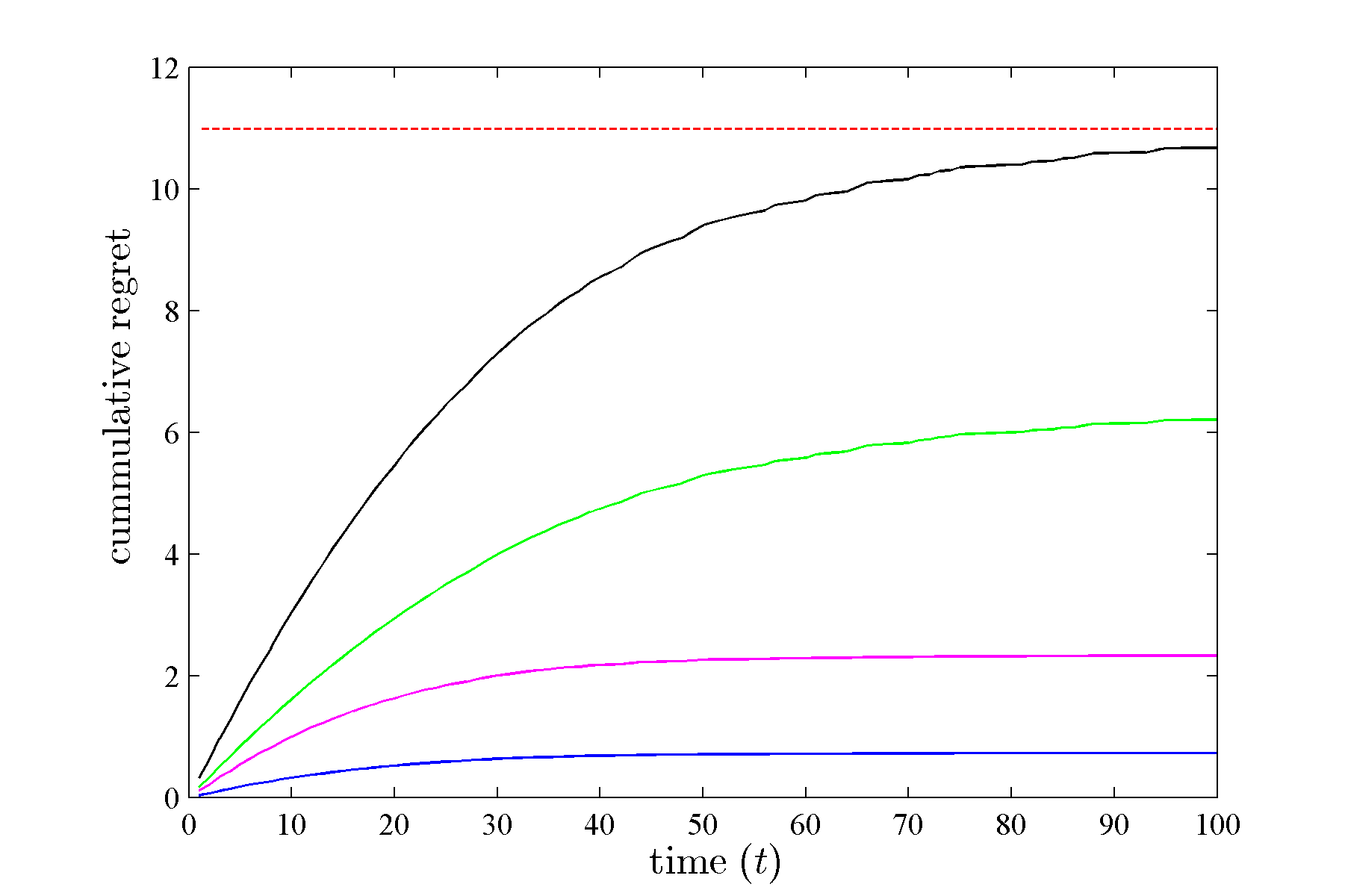}}
		\subfigure[$\eta = 0.3, \{E_t\}_{\text{set.}2}$]{\label{fig:p0d7_n0d3_e2}
			\includegraphics[width=0.32\linewidth]{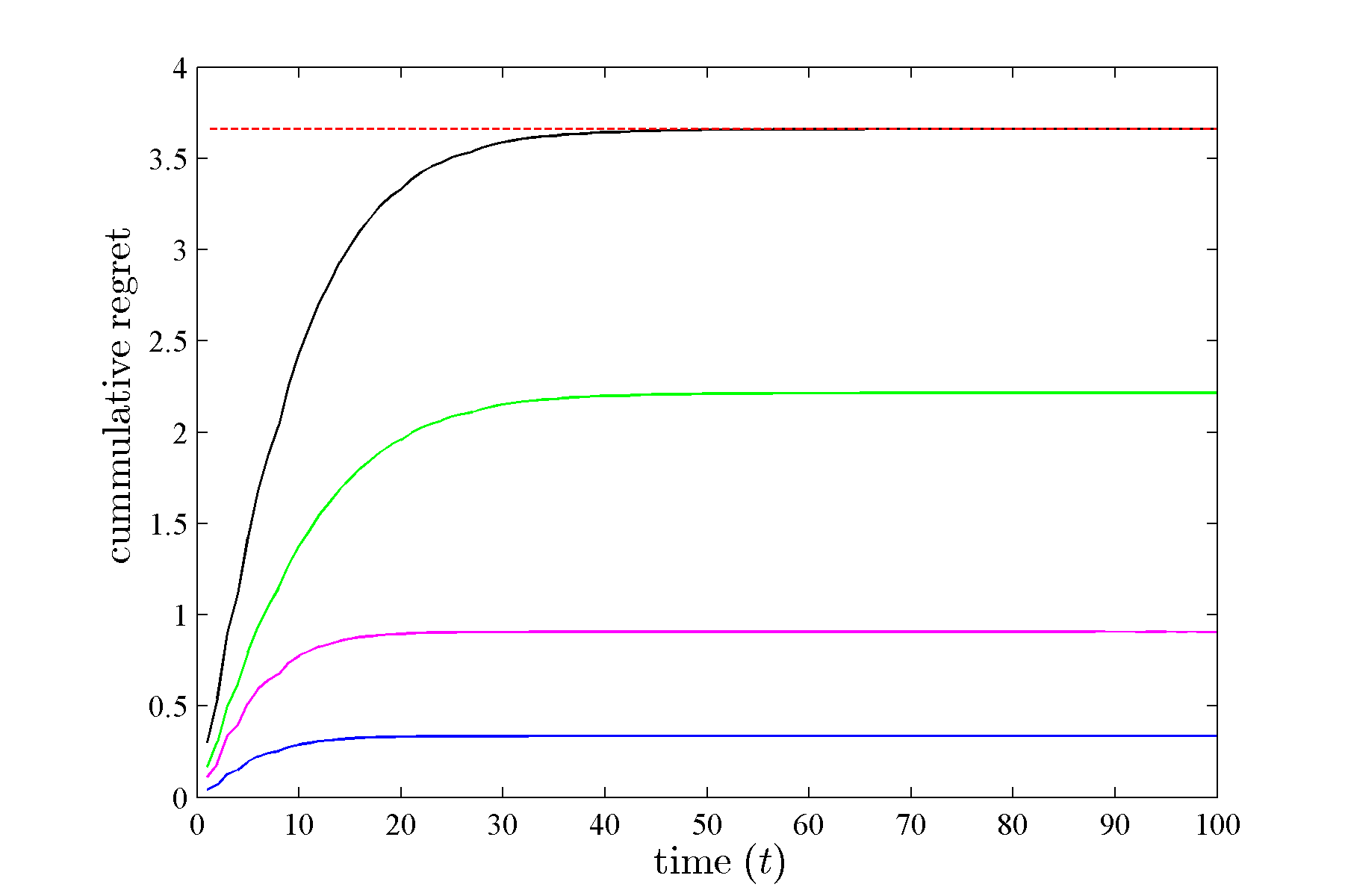}}
		\subfigure[$\eta = 0.5, \{E_t\}_{\text{set.}2}$]{\label{fig:p0d7_n0d5_e2}
			\includegraphics[width=0.32\linewidth]{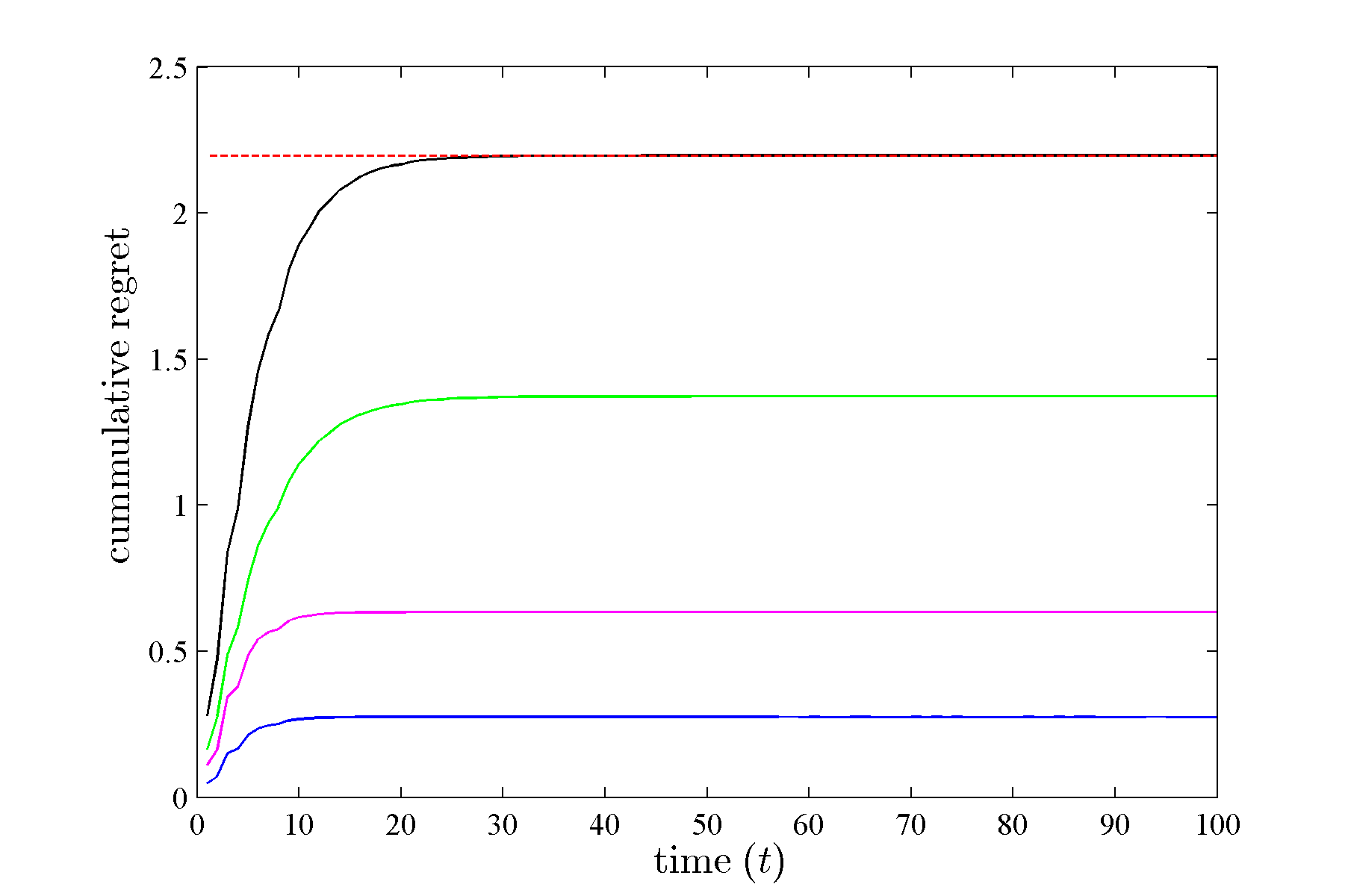}}
		
		\subfigure[$\eta = 0.1, \{E_t\}_{\text{set.}3}$]{\label{fig:p0d7_n0d1_e3}
			\includegraphics[width=0.32\linewidth]{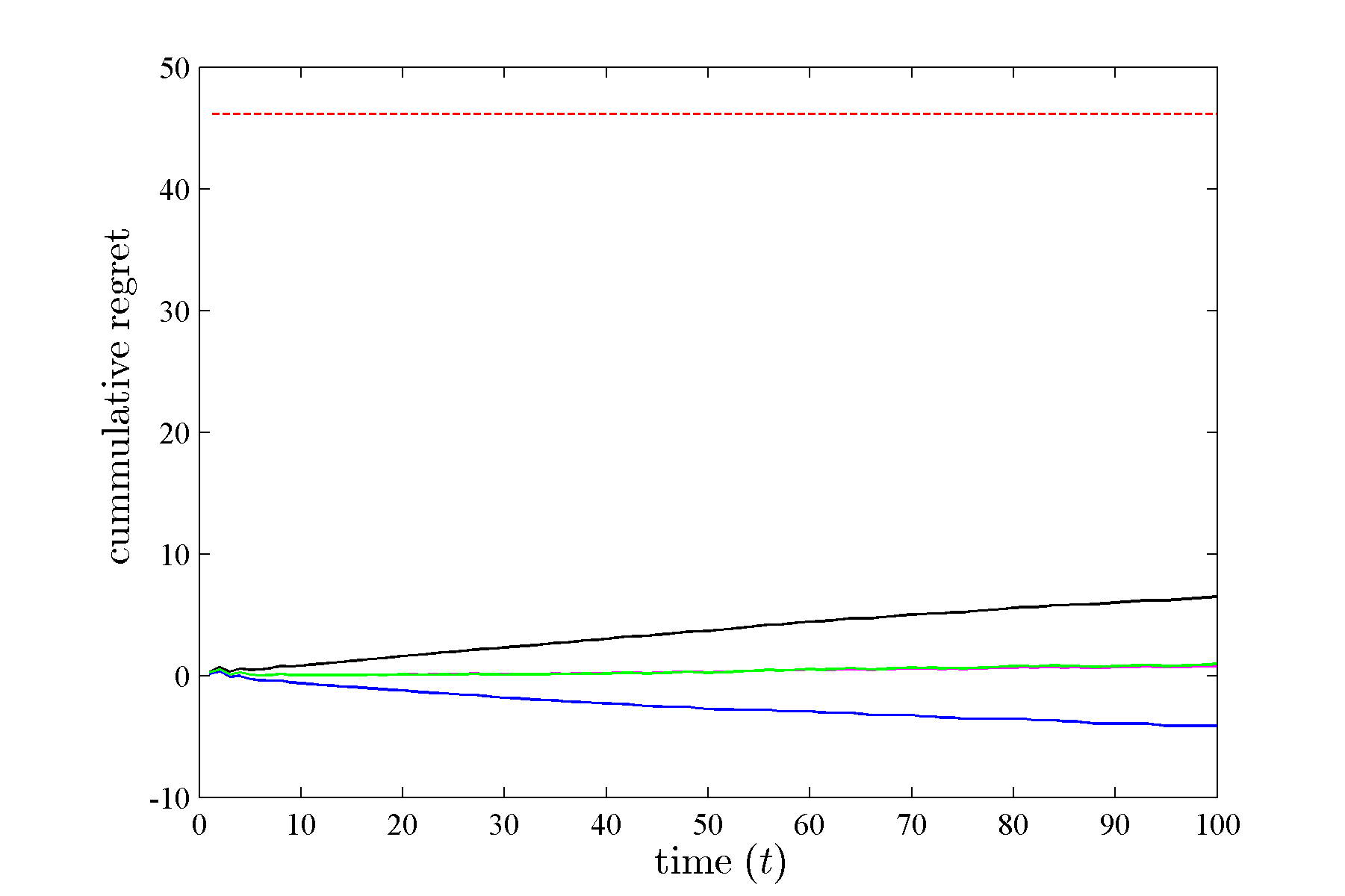}}
		\subfigure[$\eta = 0.3, \{E_t\}_{\text{set.}3}$]{\label{fig:p0d7_n0d3_e3}
			\includegraphics[width=0.32\linewidth]{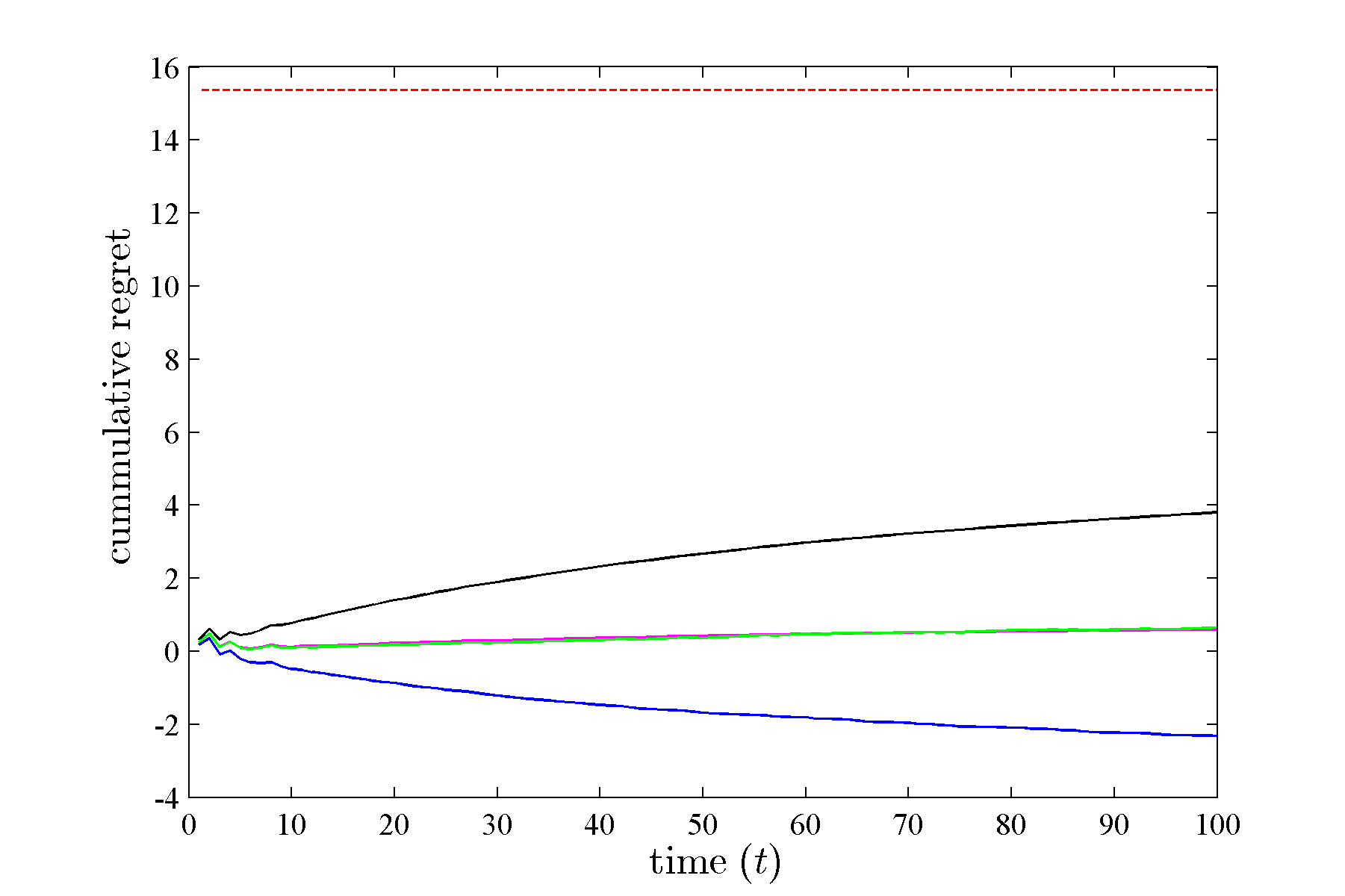}}
		\subfigure[$\eta = 0.5, \{E_t\}_{\text{set.}3}$]{\label{fig:p0d7_n0d5_e3}
			\includegraphics[width=0.32\linewidth]{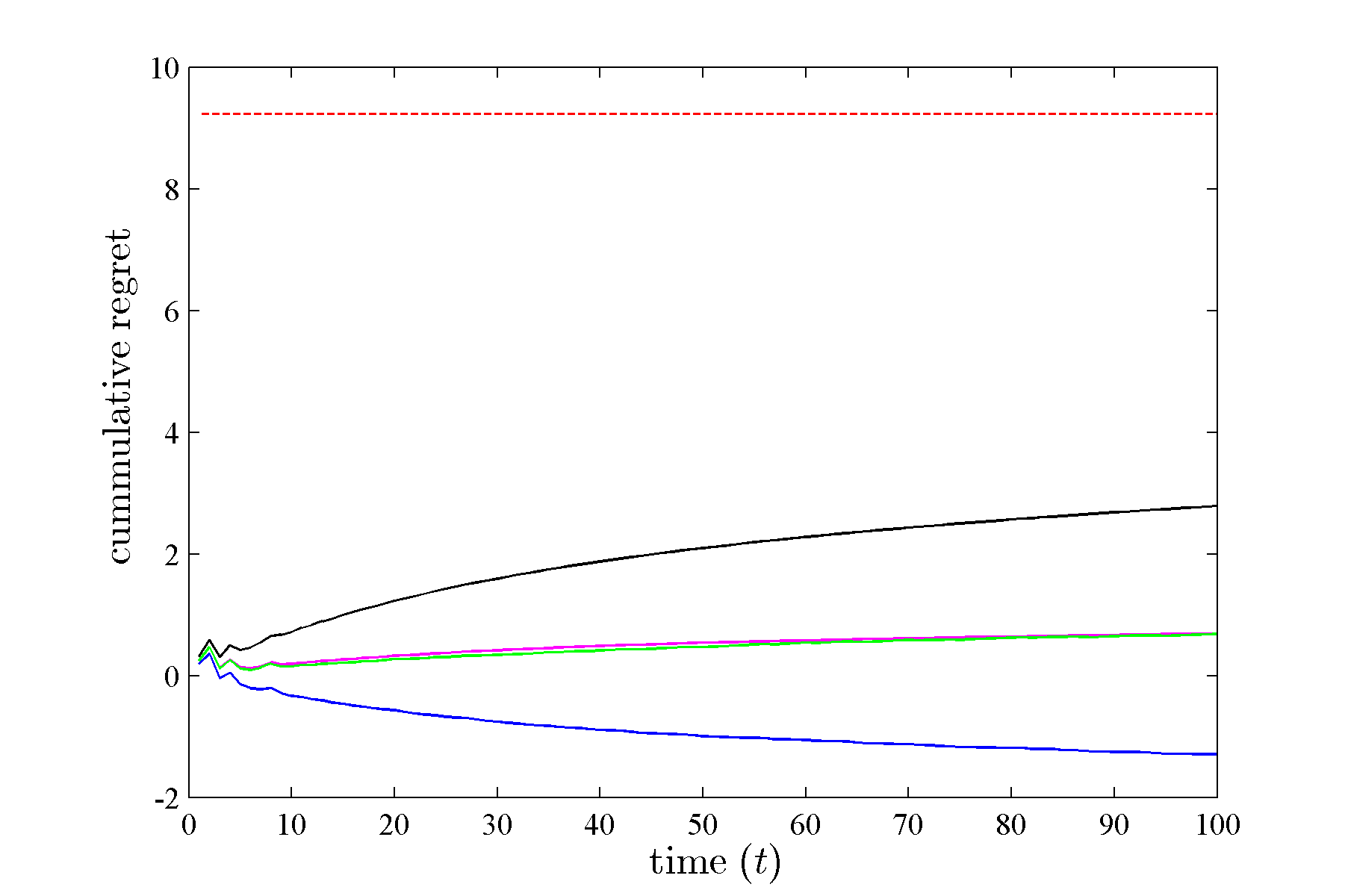}}
	}
\end{figure}

\begin{figure}[htbp]
	\caption[Cumulative regret of the Aggregating Algorithm over the outcome sequence for different choices of substitution functions]{Cumulative regret of the Aggregating Algorithm over the outcome sequence $\{y_t\}_{p=0.9}$ for different choices of substitution functions (Best look ahead(\textcolor{blue}{---}), Worst look ahead(\textcolor{black}{---}), Inverse loss(\textcolor{green}{---}), and Weighted average(\textcolor{magenta}{---})) with learning rate $\eta$ and expert setting $\{E_t\}_i$ (theoretical regret bound is shown by \textcolor{red}{- - -}). \label{fig:p0d9}}
	{
		\subfigure[$\eta = 0.1, \{E_t\}_{\text{set.}1}$]{\label{fig:p0d9_n0d1_e1}
			\includegraphics[width=0.32\linewidth]{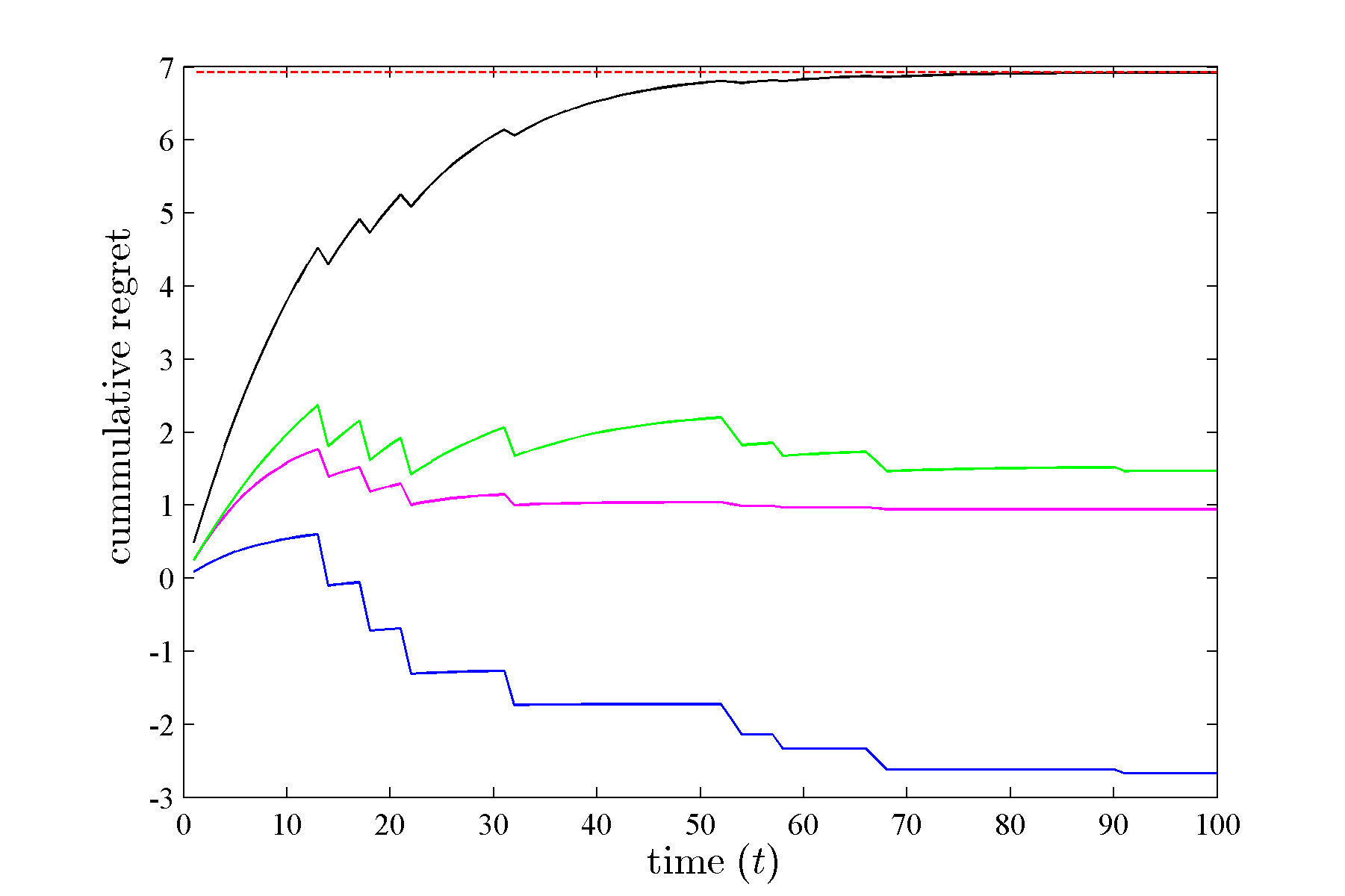}}
		\subfigure[$\eta = 0.3, \{E_t\}_{\text{set.}1}$]{\label{fig:p0d9_n0d3_e1}
			\includegraphics[width=0.32\linewidth]{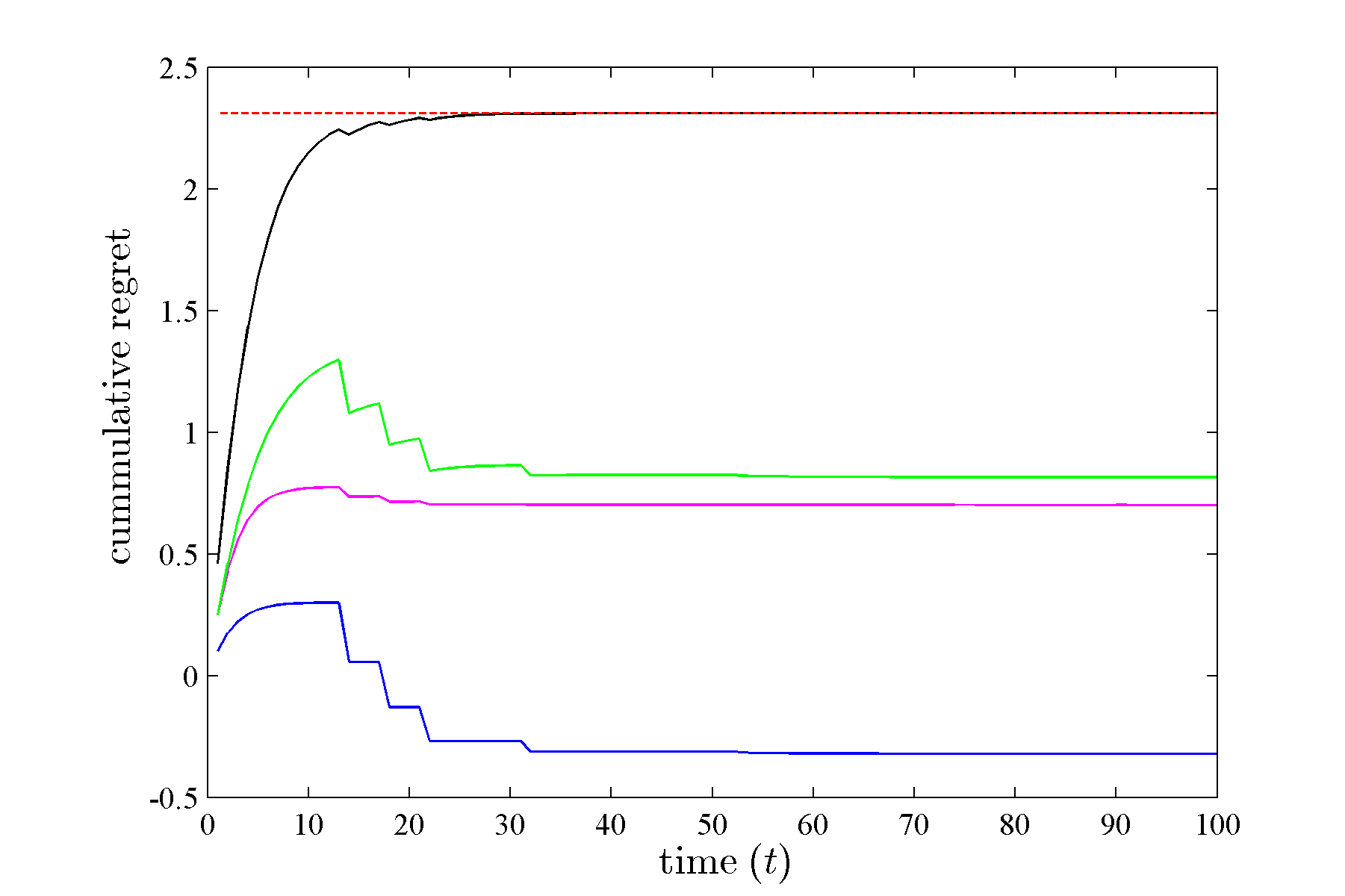}}
		\subfigure[$\eta = 0.5, \{E_t\}_{\text{set.}1}$]{\label{fig:p0d9_n0d5_e1}
			\includegraphics[width=0.32\linewidth]{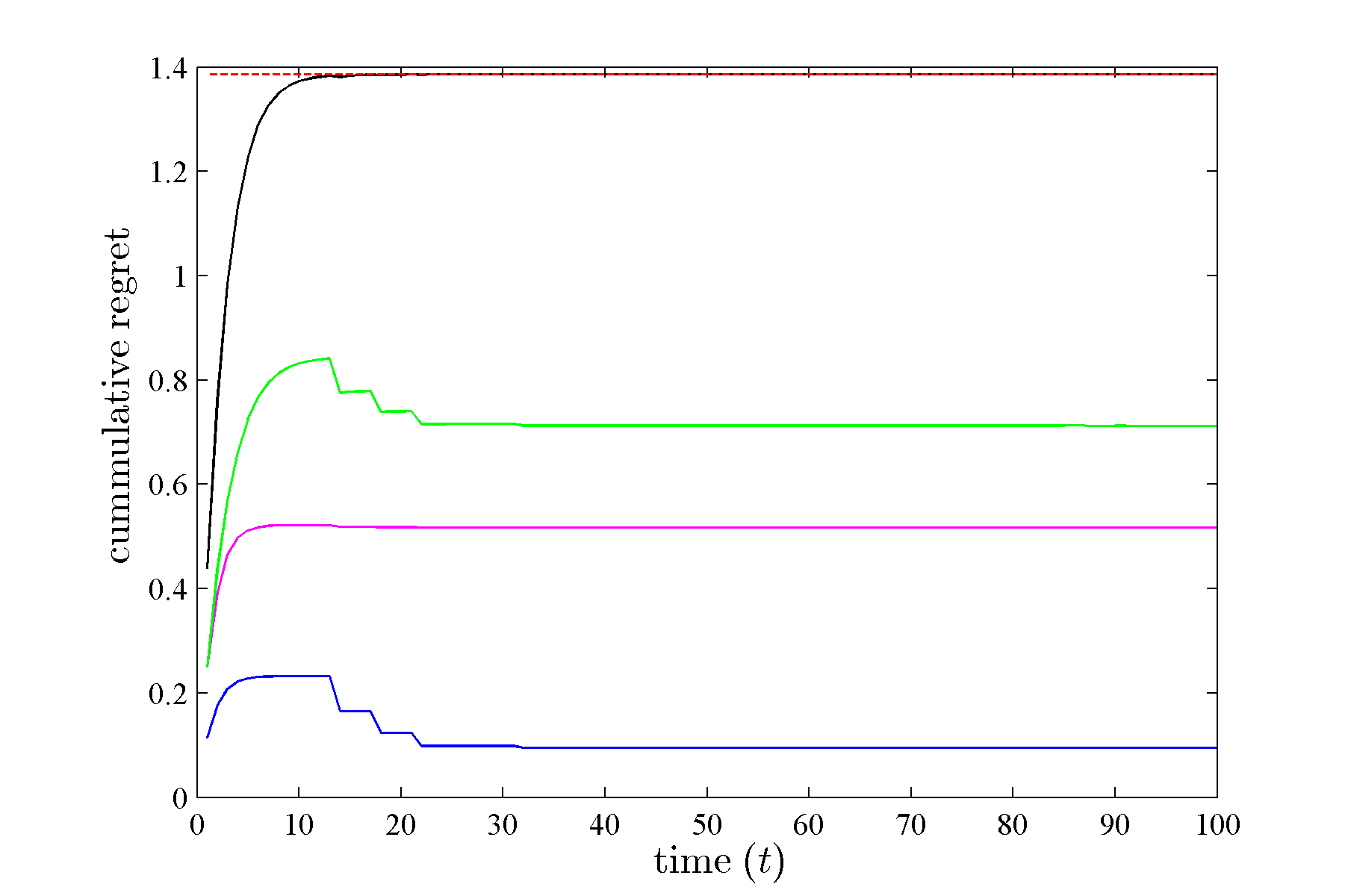}}
		
		\subfigure[$\eta = 0.1, \{E_t\}_{\text{set.}2}$]{\label{fig:p0d9_n0d1_e2}
			\includegraphics[width=0.32\linewidth]{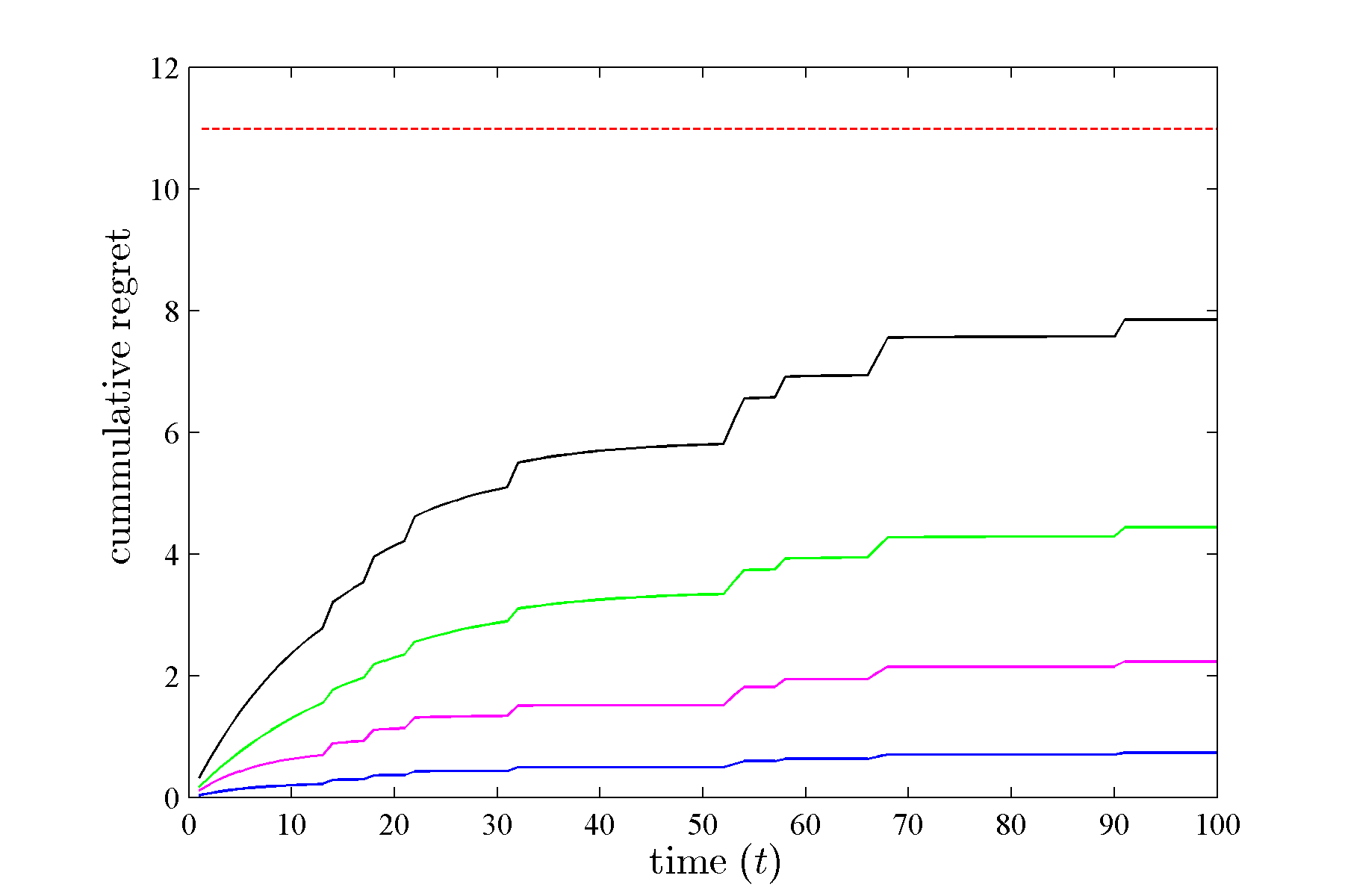}}
		\subfigure[$\eta = 0.3, \{E_t\}_{\text{set.}2}$]{\label{fig:p0d9_n0d3_e2}
			\includegraphics[width=0.32\linewidth]{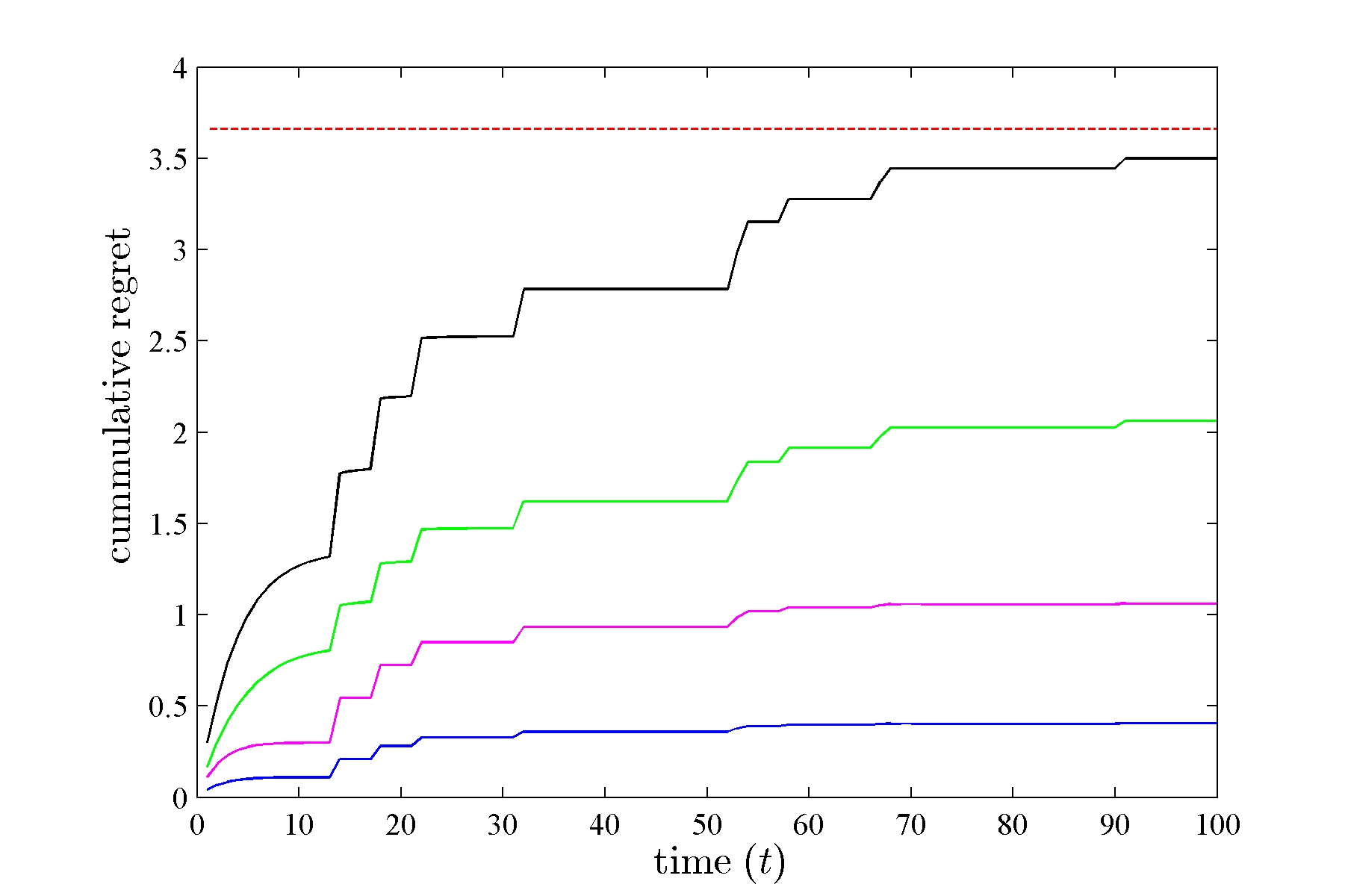}}
		\subfigure[$\eta = 0.5, \{E_t\}_{\text{set.}2}$]{\label{fig:p0d9_n0d5_e2}
			\includegraphics[width=0.32\linewidth]{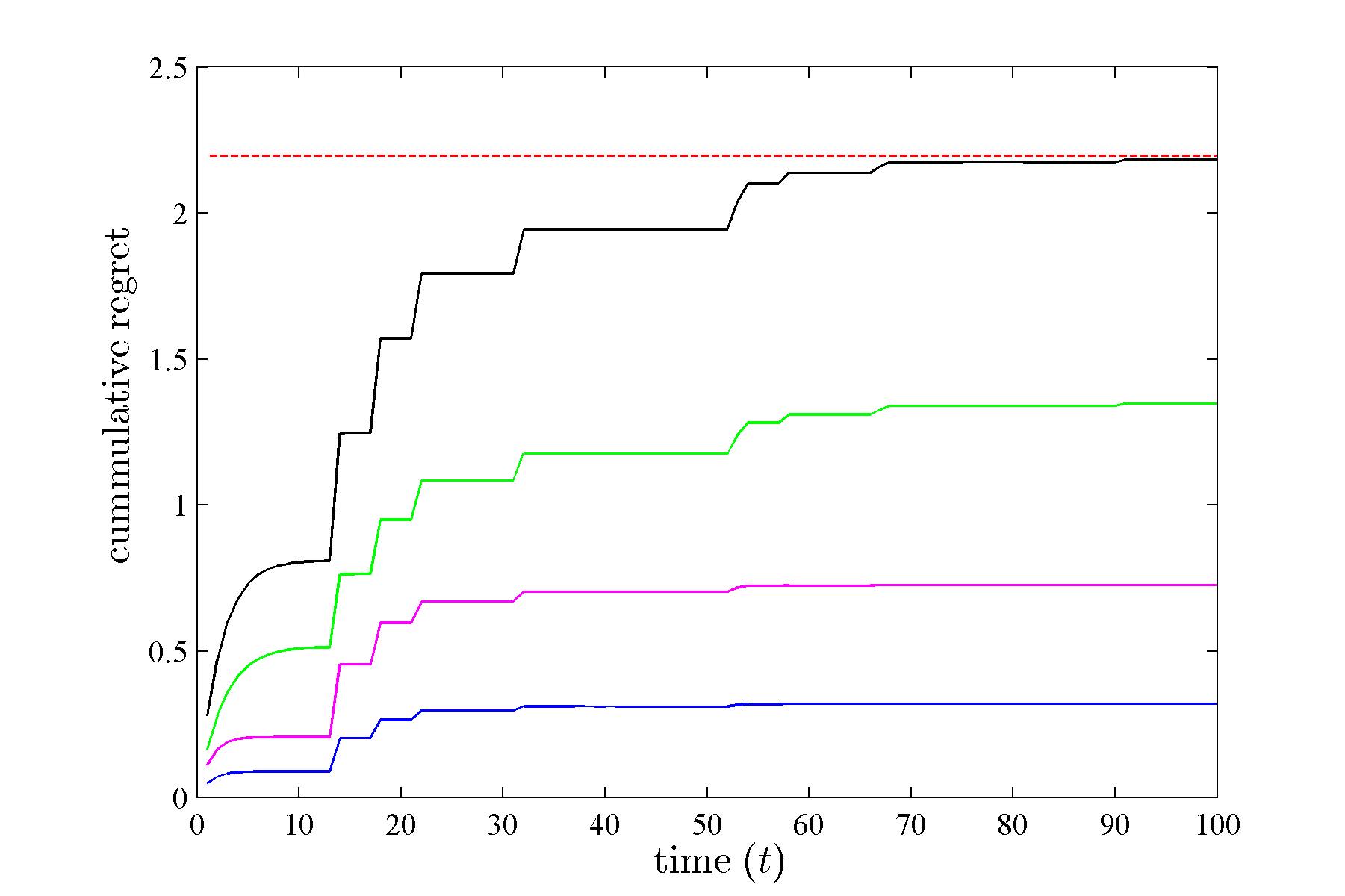}}
		
		\subfigure[$\eta = 0.1, \{E_t\}_{\text{set.}3}$]{\label{fig:p0d9_n0d1_e3}
			\includegraphics[width=0.32\linewidth]{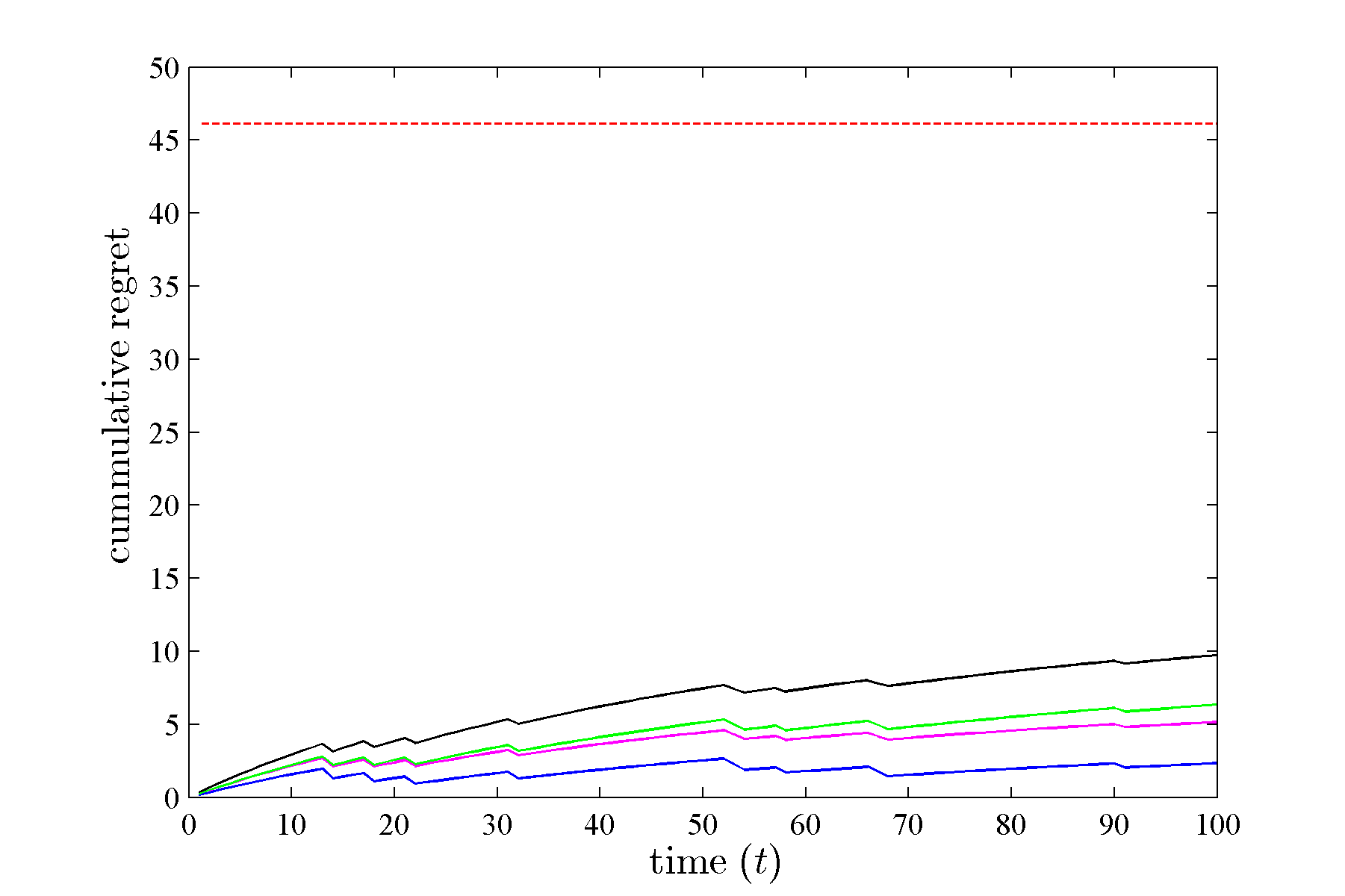}}
		\subfigure[$\eta = 0.3, \{E_t\}_{\text{set.}3}$]{\label{fig:p0d9_n0d3_e3}
			\includegraphics[width=0.32\linewidth]{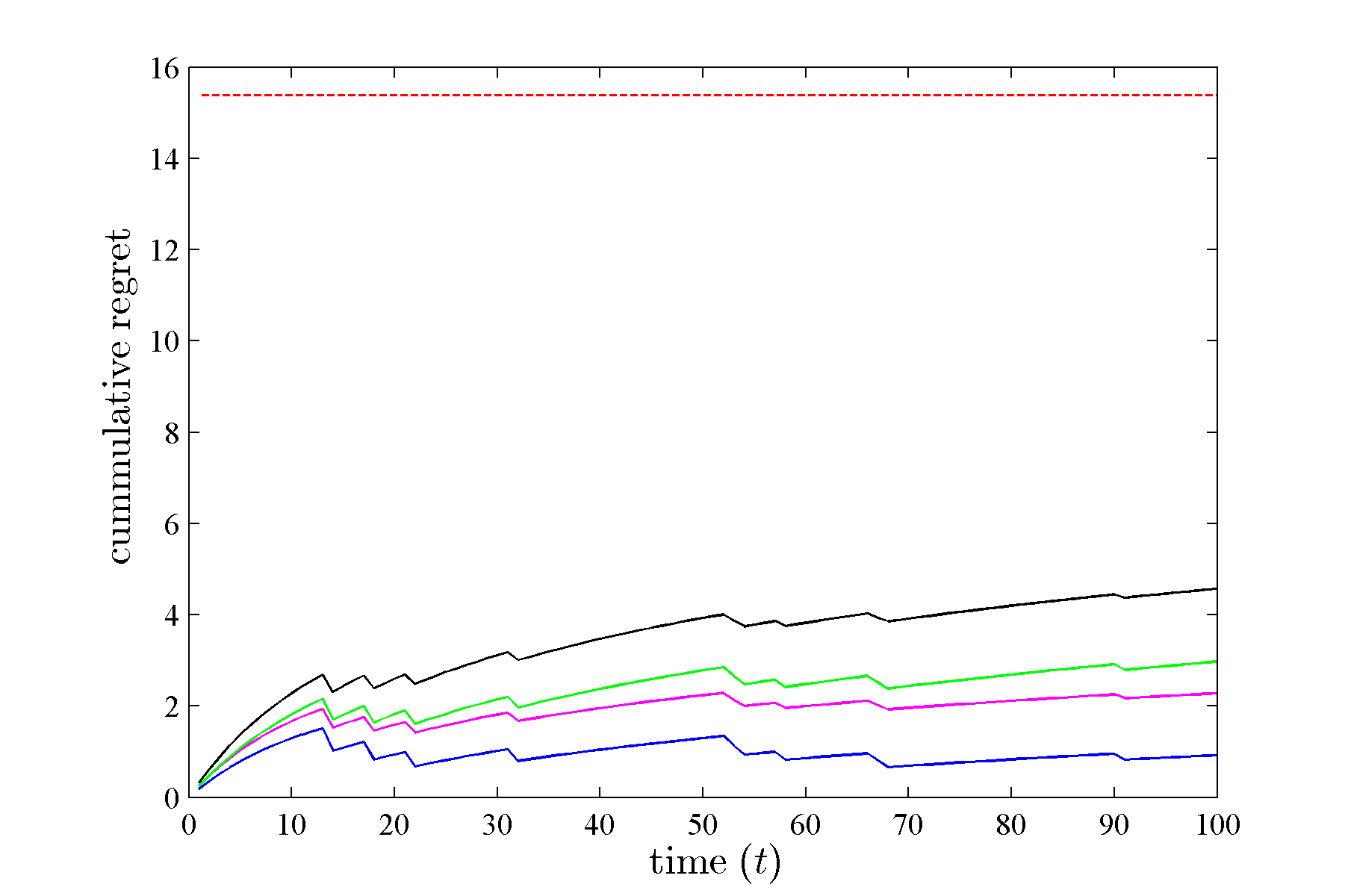}}
		\subfigure[$\eta = 0.5, \{E_t\}_{\text{set.}3}$]{\label{fig:p0d9_n0d5_e3}
			\includegraphics[width=0.32\linewidth]{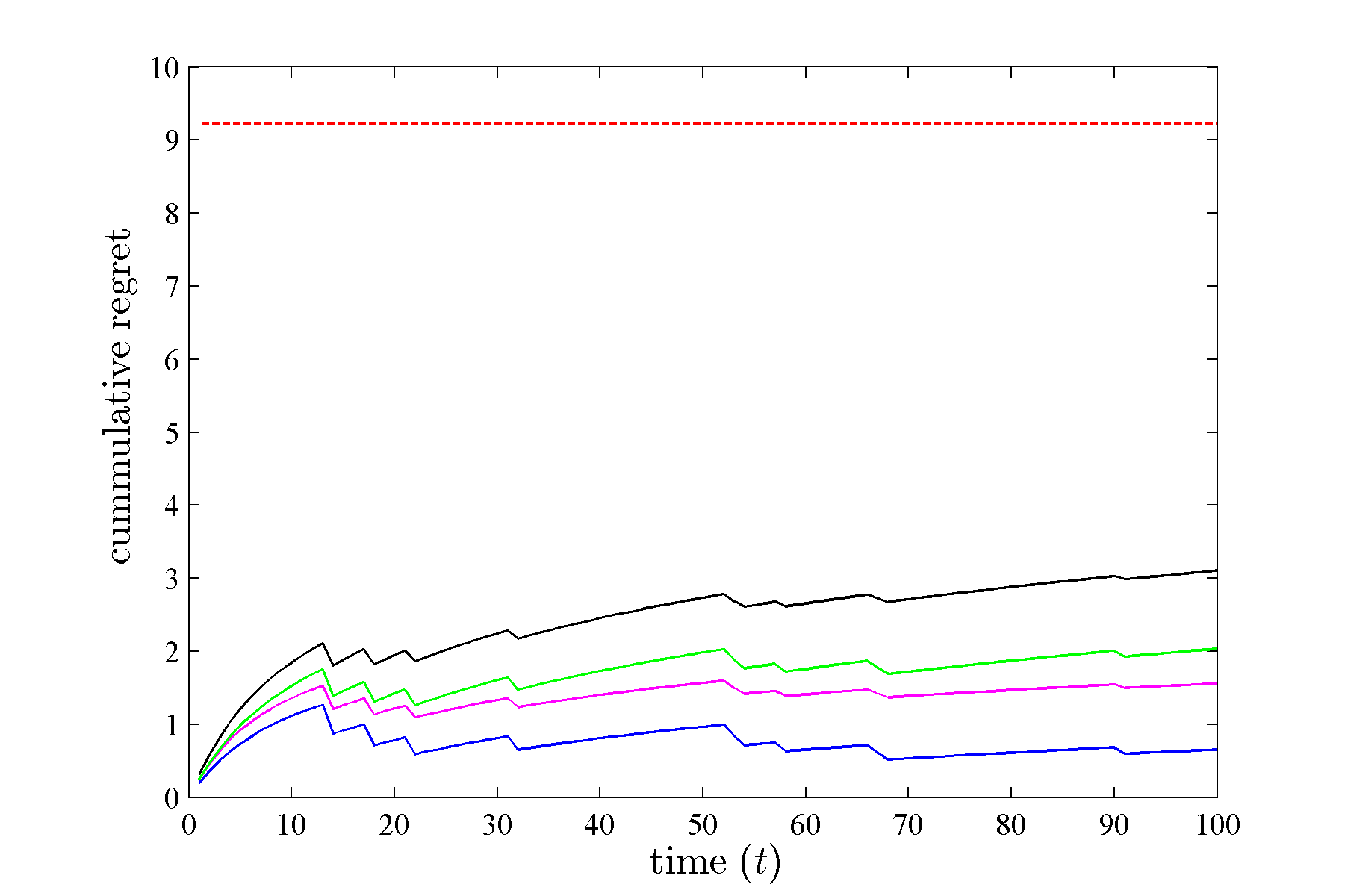}}
	}
\end{figure}

\begin{figure}[htbp]
	\caption[Cumulative regret of the Aggregating Algorithm over the outcome sequence for different choices of substitution functions]{Cumulative regret of the Aggregating Algorithm over the outcome sequence $\{y_t\}_{p=1.0}$ for different choices of substitution functions (Best look ahead(\textcolor{blue}{---}), Worst look ahead(\textcolor{black}{---}), Inverse loss(\textcolor{green}{---}), and Weighted average(\textcolor{magenta}{---})) with learning rate $\eta$ and expert setting $\{E_t\}_i$ (theoretical regret bound is shown by \textcolor{red}{- - -}).\label{fig:p1d0}}
	{
		\subfigure[$\eta = 0.1, \{E_t\}_{\text{set.}1}$]{\label{fig:p1d0_n1d0_e1a}
			\includegraphics[width=0.32\linewidth]{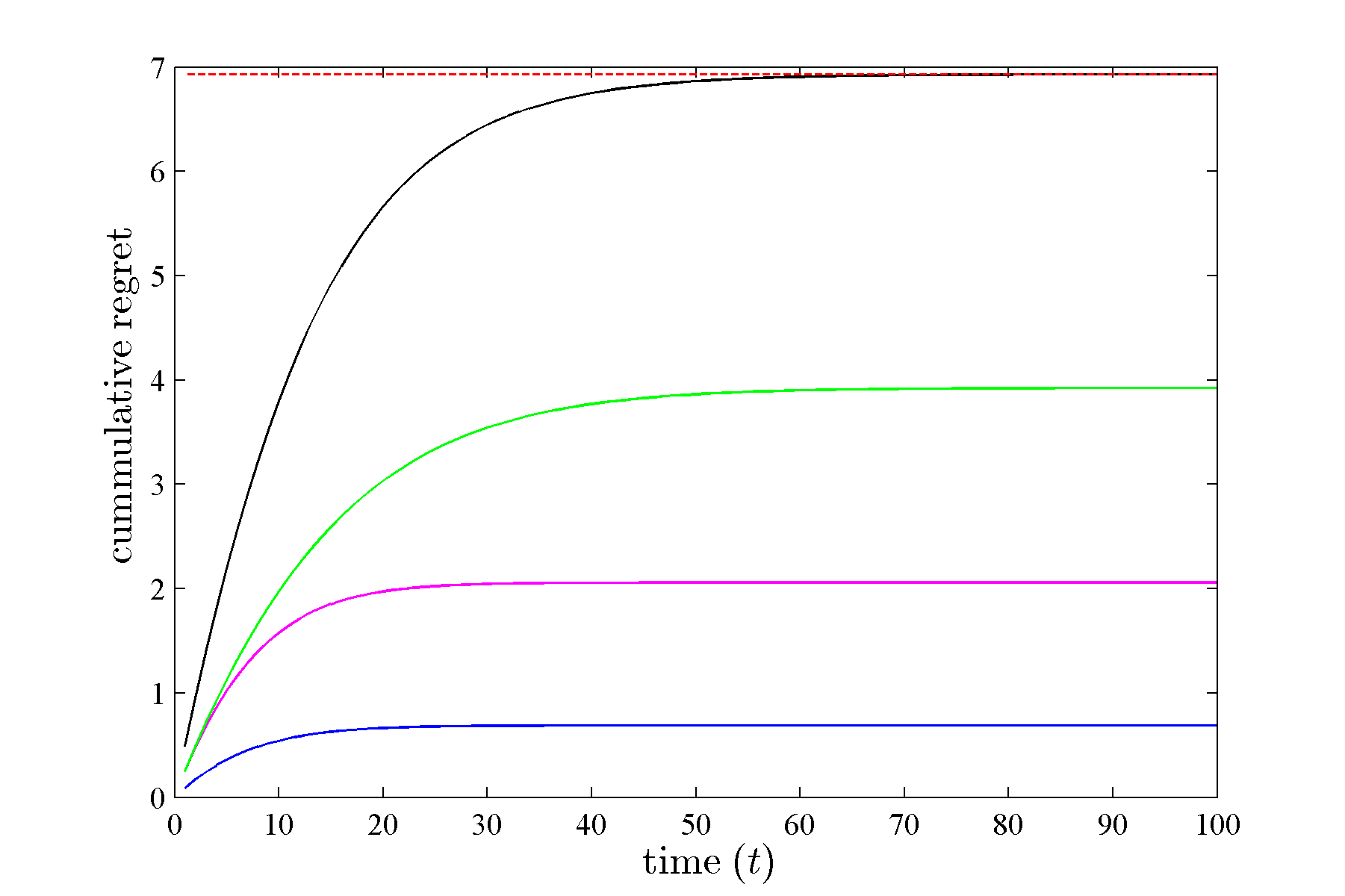}}
		\subfigure[$\eta = 0.3, \{E_t\}_{\text{set.}1}$]{\label{fig:p1d0_n1d0_e1b}
			\includegraphics[width=0.32\linewidth]{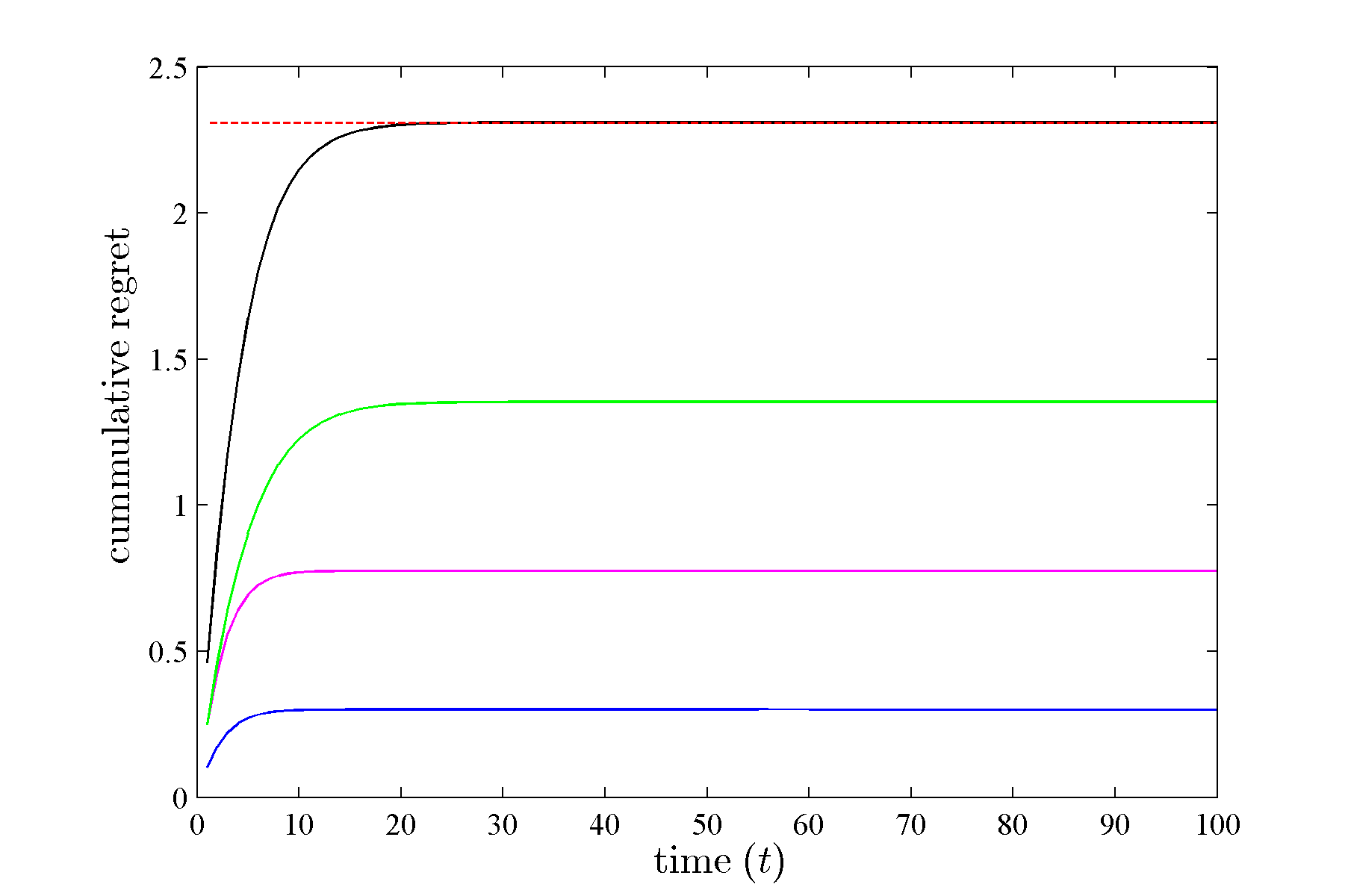}}
		\subfigure[$\eta = 0.5, \{E_t\}_{\text{set.}1}$]{\label{fig:p1d0_n1d0_e1c}
			\includegraphics[width=0.32\linewidth]{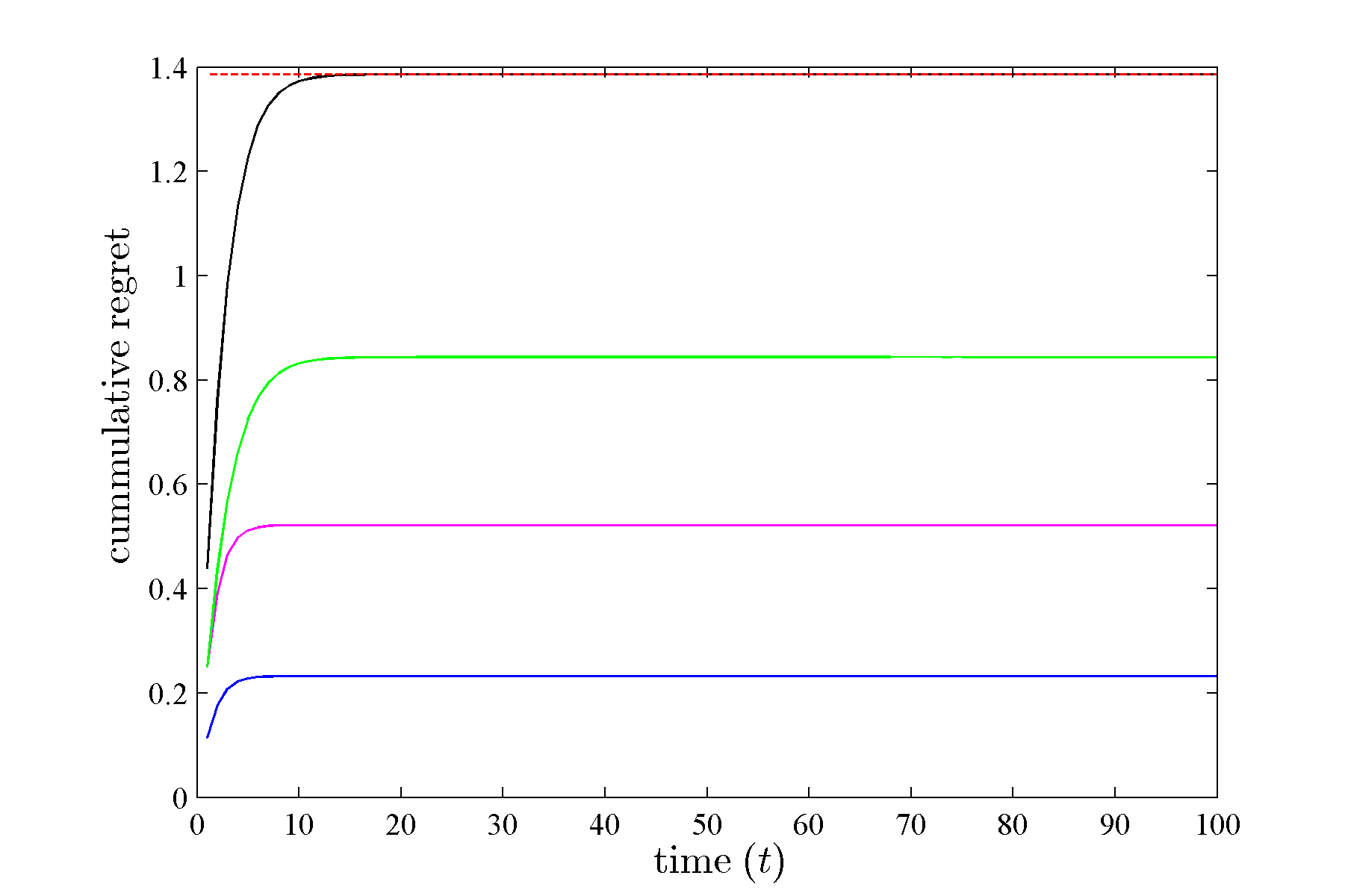}}
		
		\subfigure[$\eta = 0.1, \{E_t\}_{\text{set.}2}$]{\label{fig:p1d0_n1d0_e2a}
			\includegraphics[width=0.32\linewidth]{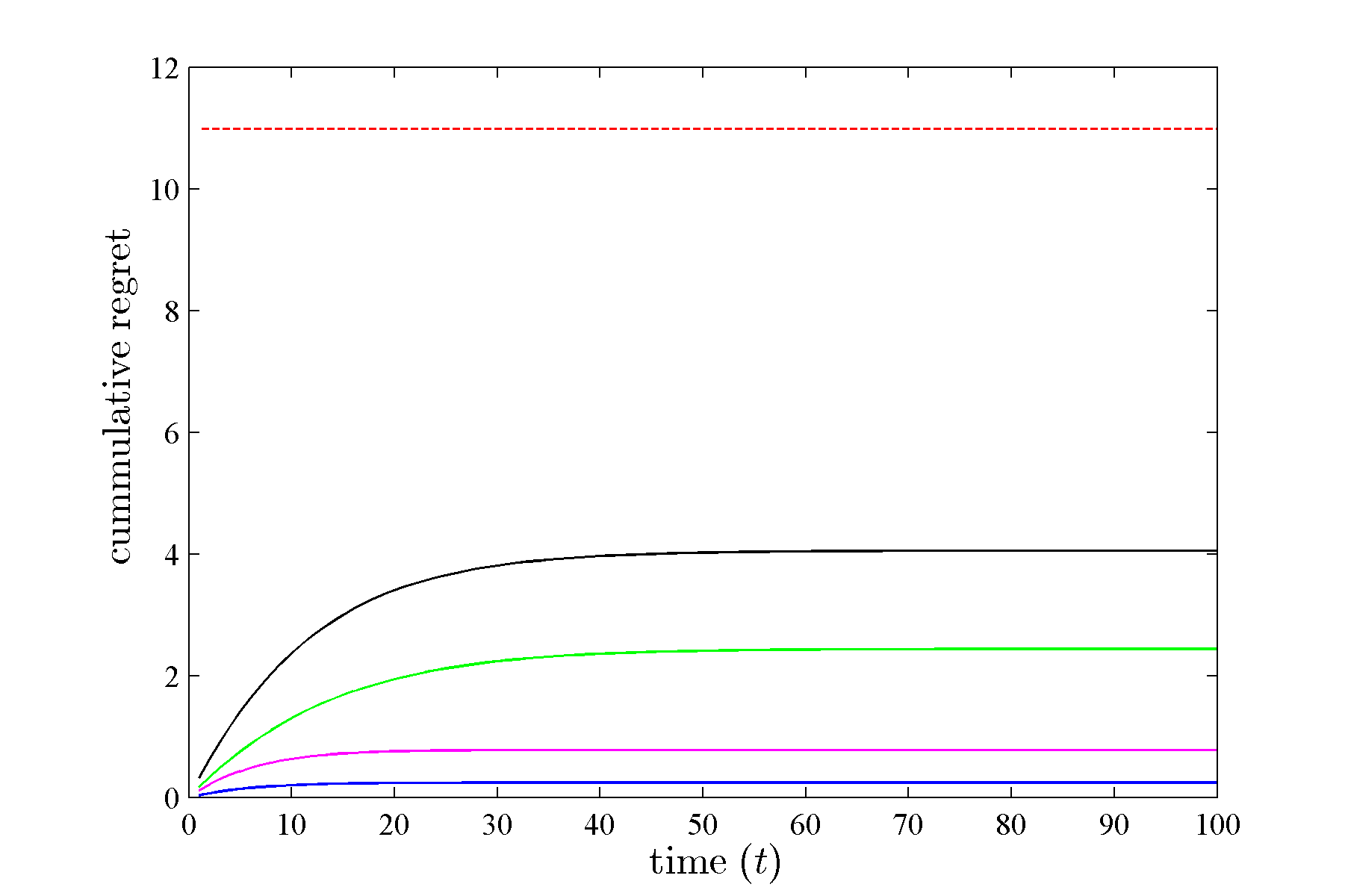}}
		\subfigure[$\eta = 0.3, \{E_t\}_{\text{set.}2}$]{\label{fig:p1d0_n1d0_e2b}
			\includegraphics[width=0.32\linewidth]{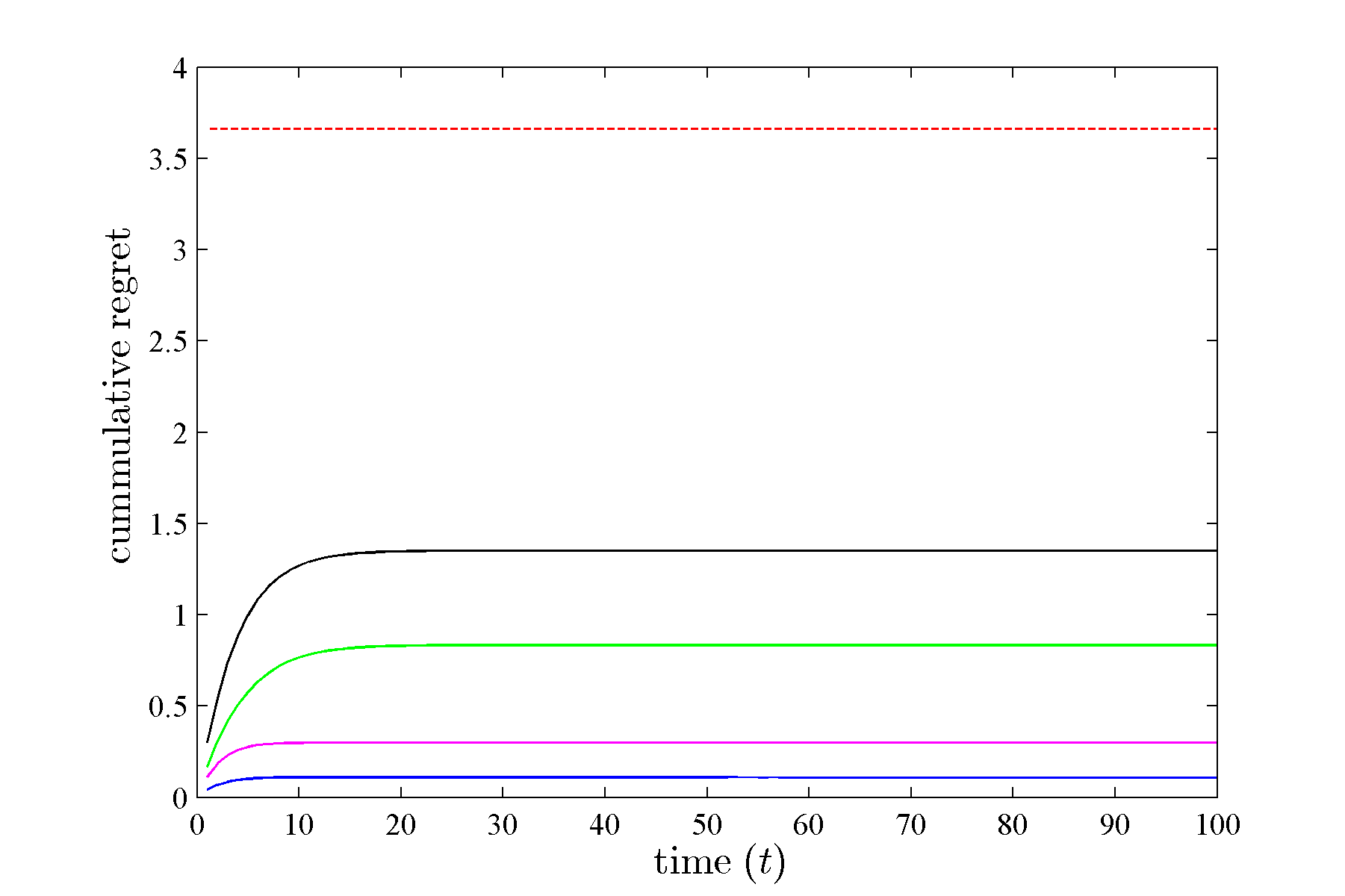}}
		\subfigure[$\eta = 0.5, \{E_t\}_{\text{set.}2}$]{\label{fig:p1d0_n1d0_e2c}
			\includegraphics[width=0.32\linewidth]{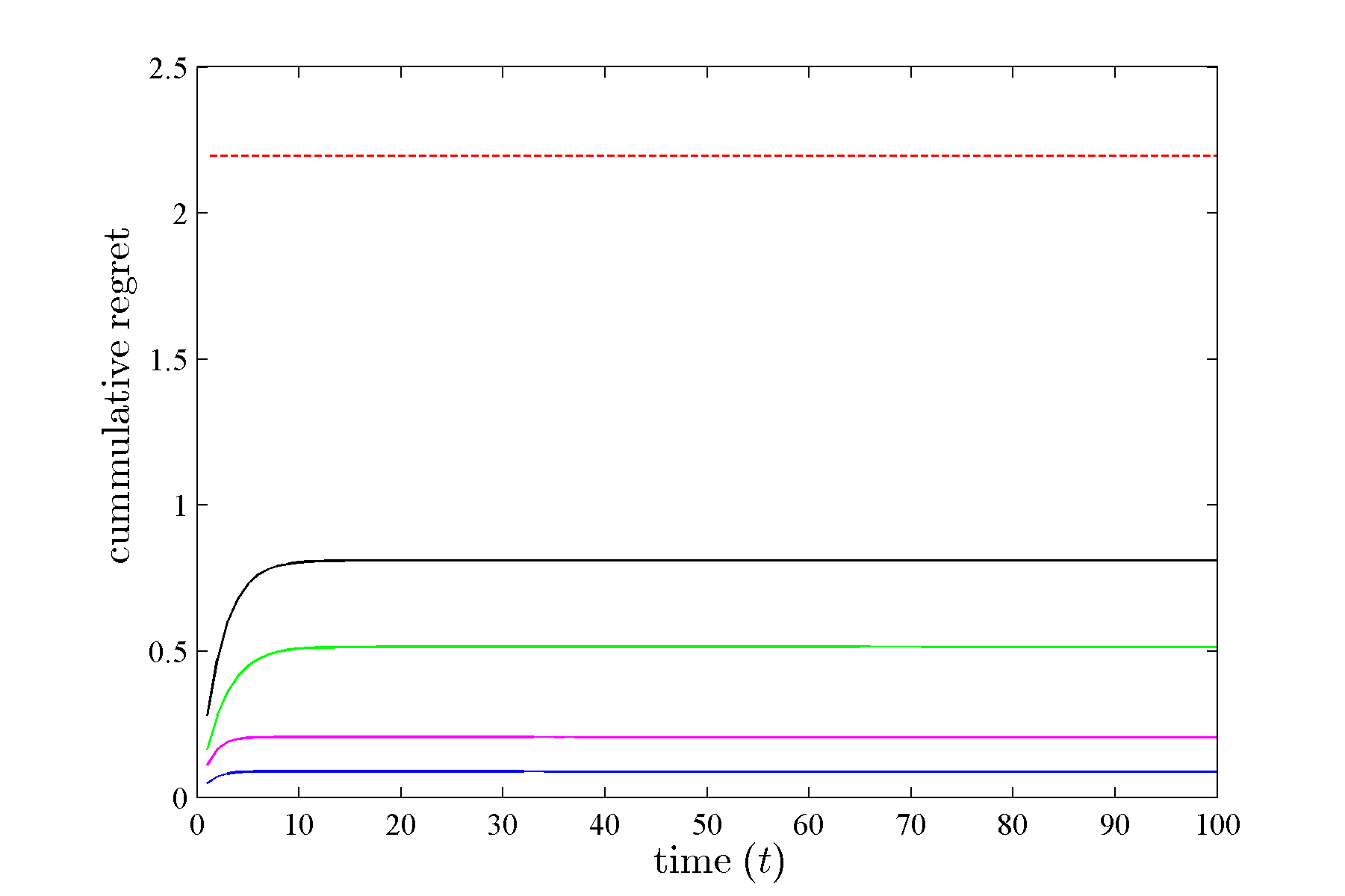}}
		
		\subfigure[$\eta = 0.1, \{E_t\}_{\text{set.}3}$]{\label{fig:p1d0_n1d0_e3a}
			\includegraphics[width=0.32\linewidth]{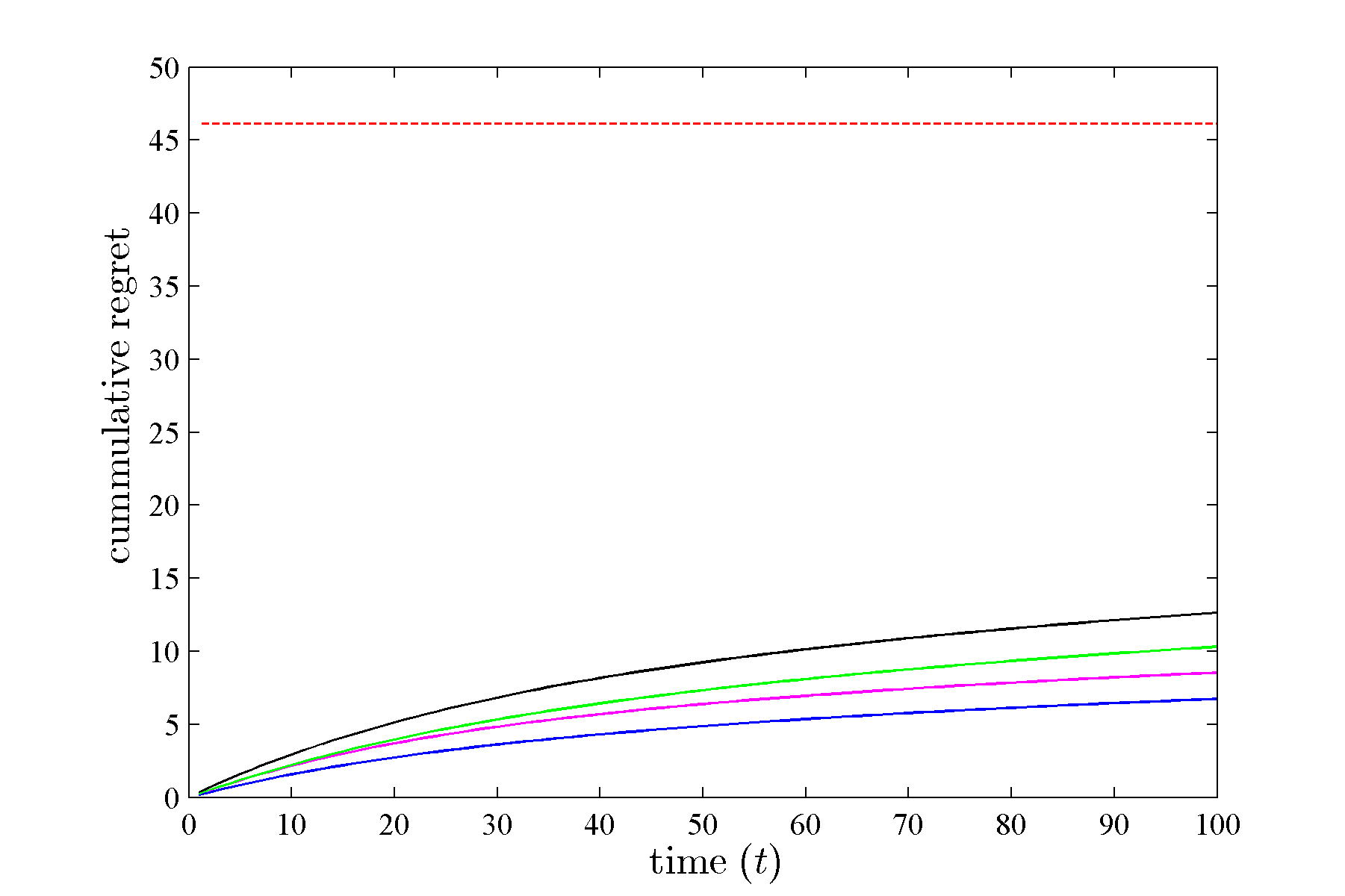}}
		\subfigure[$\eta = 0.3, \{E_t\}_{\text{set.}3}$]{\label{fig:p1d0_n1d0_e3b}
			\includegraphics[width=0.32\linewidth]{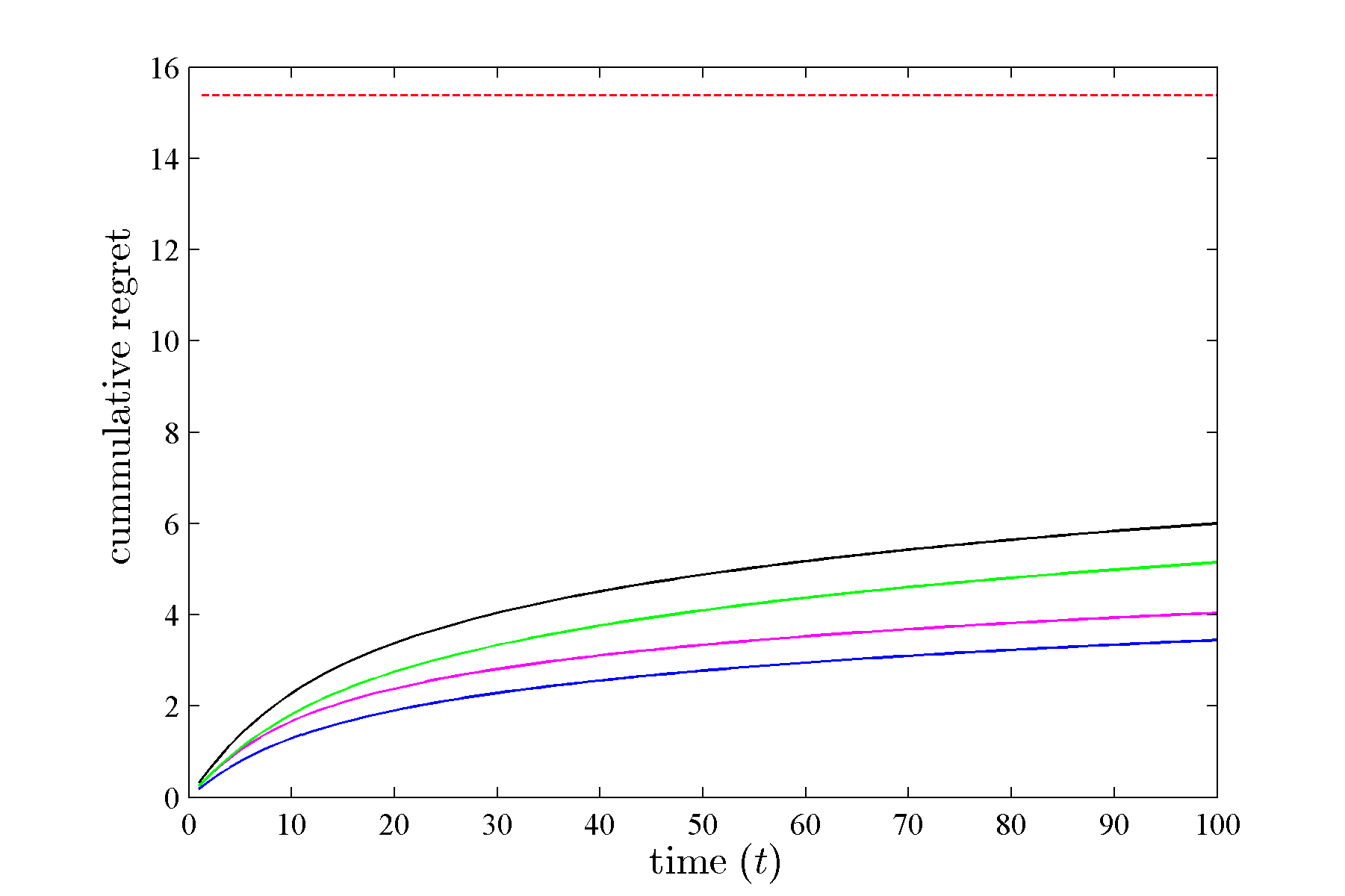}}
		\subfigure[$\eta = 0.5, \{E_t\}_{\text{set.}3}$]{\label{fig:p1d0_n1d0_e3c}
			\includegraphics[width=0.32\linewidth]{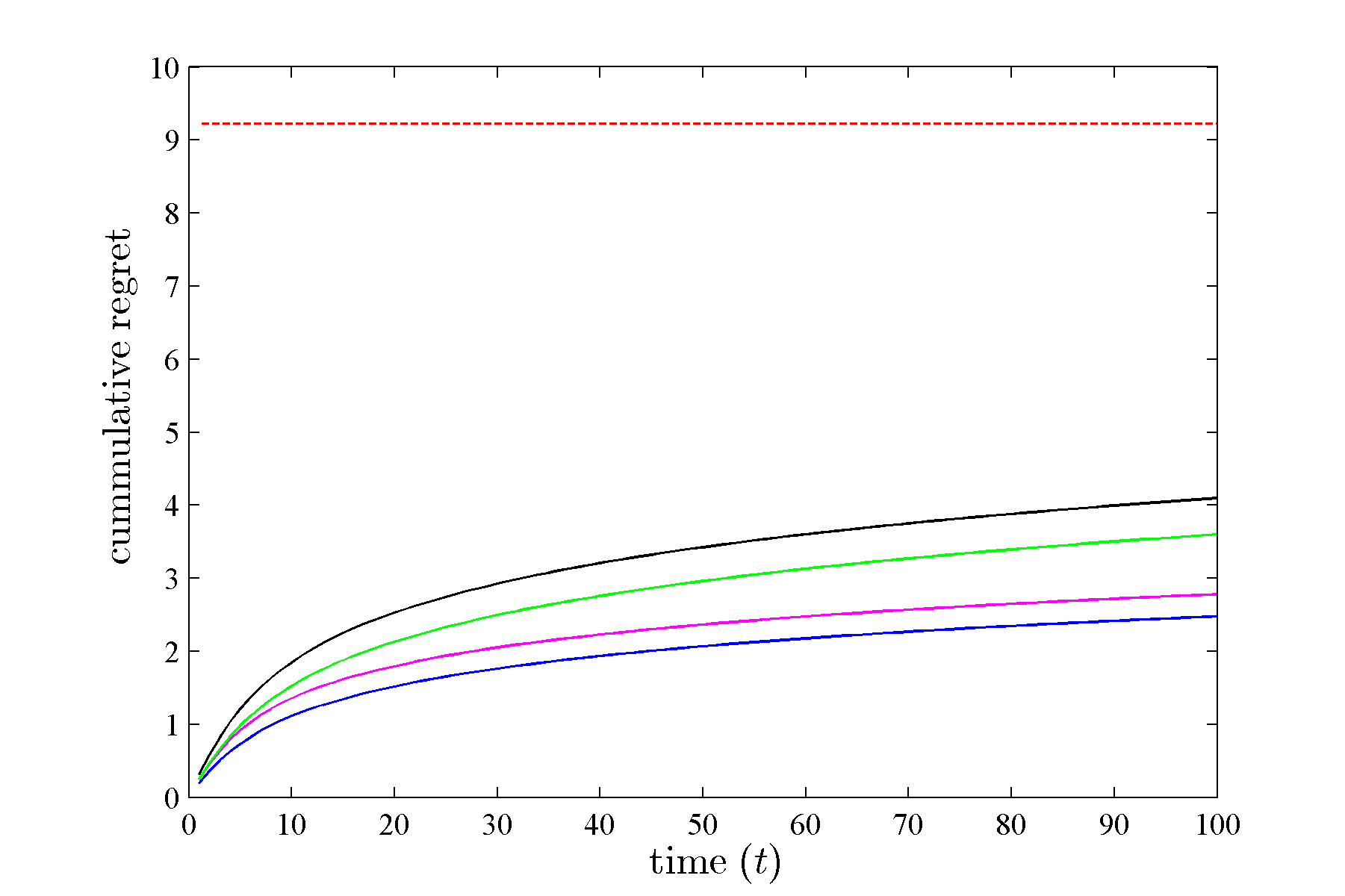}}
	}
\end{figure}

\begin{figure}[htbp]
	\caption[Cumulative regret of the Aggregating Algorithm over the football dataset for different choices of substitution functions]{Cumulative regret of the Aggregating Algorithm over the football dataset as used by \cite{vovk2009prediction}, for different choices of substitution functions (Best look ahead(\textcolor{blue}{---}), Worst look ahead(\textcolor{black}{---}), Inverse loss(\textcolor{green}{---}), and Weighted average(\textcolor{magenta}{---})) with learning rate $\eta$ (theoretical regret bound is shown by \textcolor{red}{- - -}).\label{fig:footballdata}}
	{
		\subfigure[$\eta = 0.1$]{\label{fig:real_n1d0_e1a}
			\includegraphics[width=0.32\linewidth]{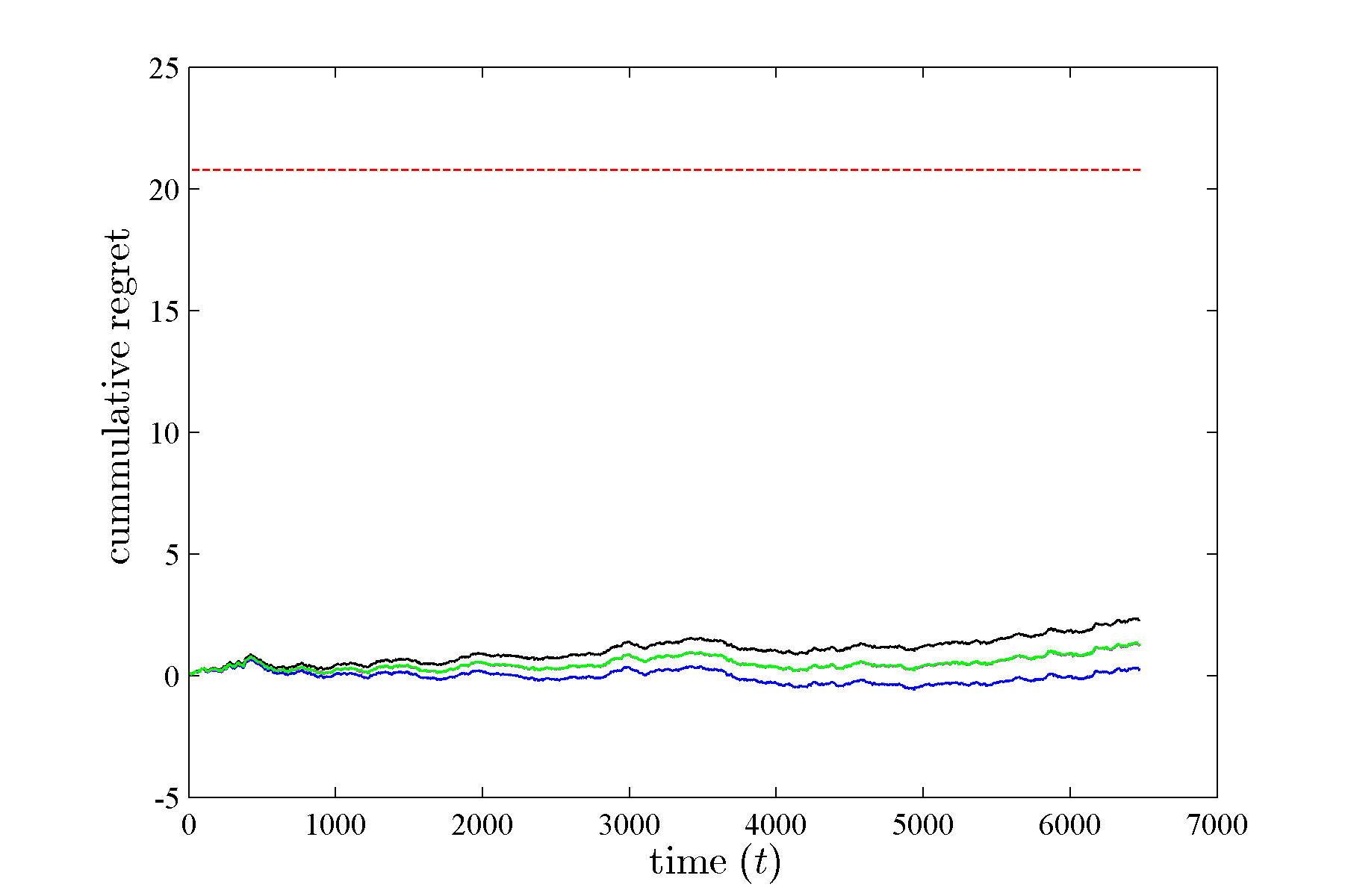}}
		\subfigure[$\eta = 0.3$]{\label{fig:real_n1d0_e1b}
			\includegraphics[width=0.32\linewidth]{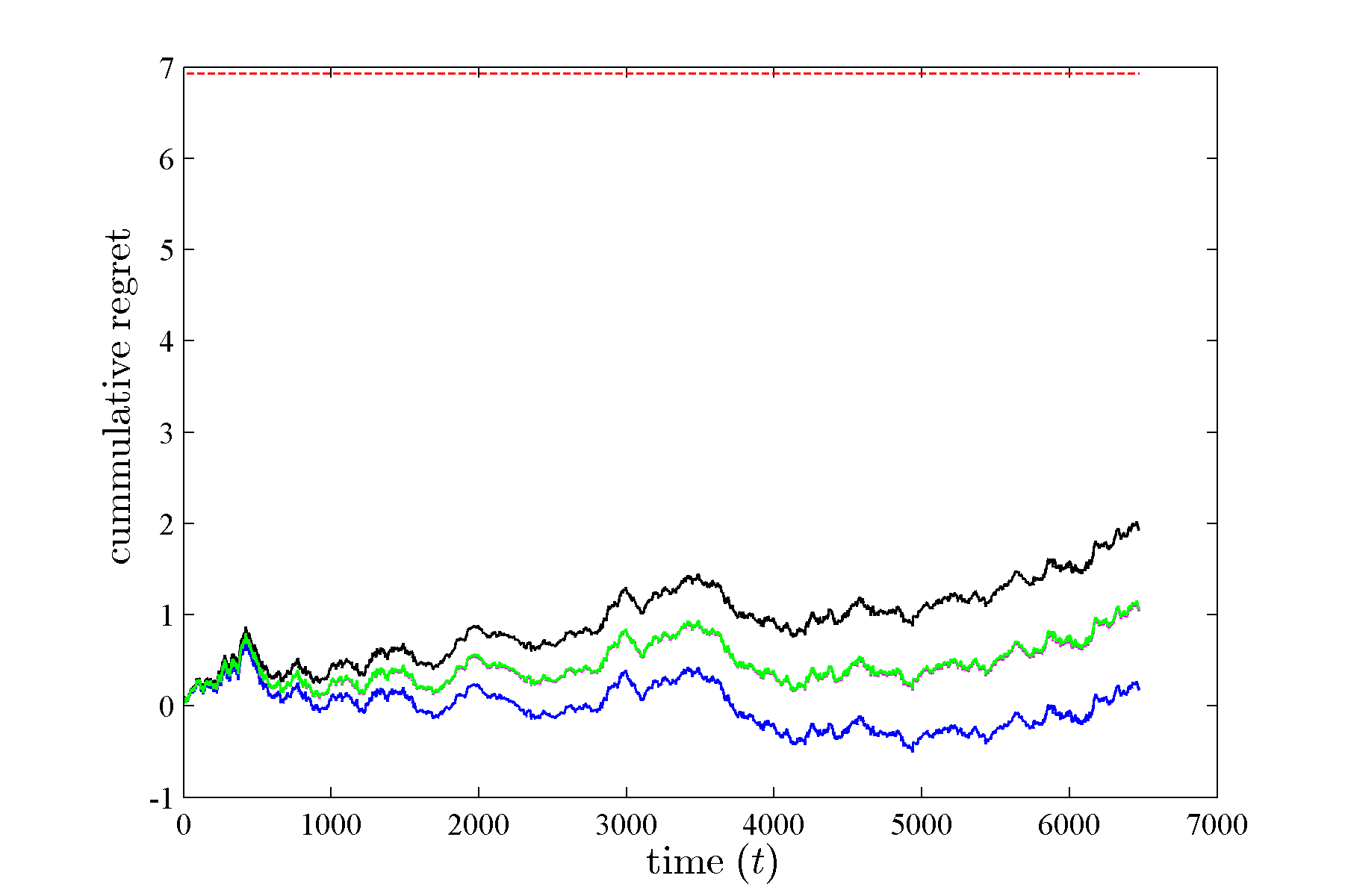}}
		\subfigure[$\eta = 0.5$]{\label{fig:real_n1d0_e1c}
			\includegraphics[width=0.32\linewidth]{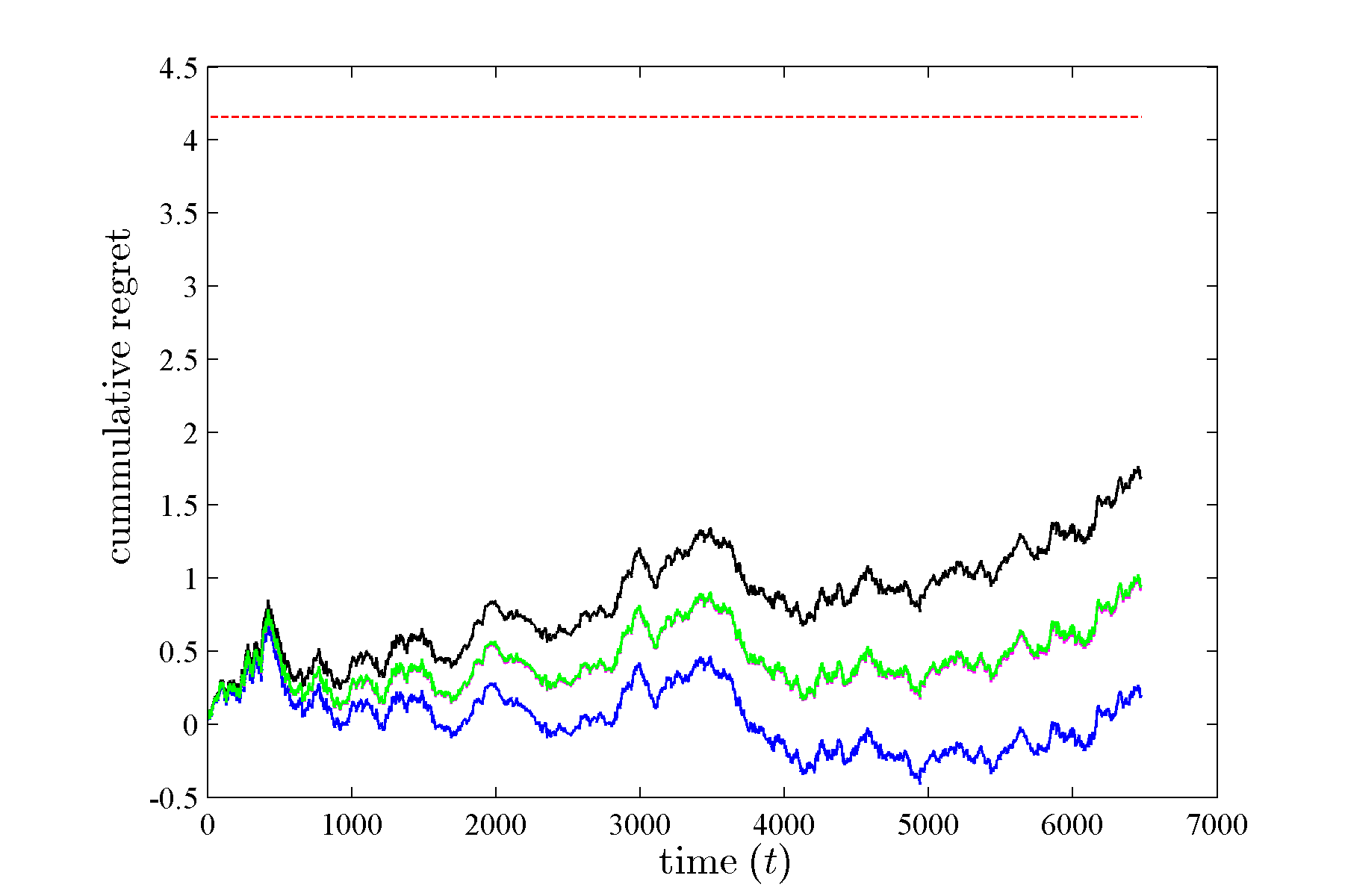}}
	}
\end{figure}

\subsection{Probability Games with Continuous outcome space}
\label{sec:probability}
We consider an important class of prediction problem called \textit{probability games} (as explained by \cite{vovk2001competitive}), in which the prediction $v$ and the outcome $y$ are probability distributions in some set (for example a finite set of the form $[n]$). A special class of loss functions called \textit{Bregman Loss Functions} (defined below) would be an appropriate choice for the probability games.

Given a differentiable convex function $\phi:\mathcal{S}\rightarrow \mathbb{R}$ defined on a convex set $\mathcal{S} \subset \mathbb{R}^d$ and two points $s_0,s \in \mathcal{S}$ the \textit{Bregman divergence} of $s$ from $s_0$ is defined as
\begin{equation*}
B_{\phi}(s,s_0) := \phi(s) - \phi(s_0) - (s - s_0)' \cdot \textnormal{\textsf{D}} \phi(s_0) ,
\end{equation*}
where $\textnormal{\textsf{D}} \phi(s_0)$ is the gradient of $\phi$ at $s_0$. For any strictly convex function $\phi : \tilde{\Delta}^n \rightarrow \mathbb{R}$, differentiable over the interior of $\tilde{\Delta}^n$, the \textit{Bregman Loss Function} (BLF, \cite{banerjee2005optimality}) $\ell_{\phi}:\Delta^n \times \Delta^n \rightarrow \RR_+$ with generator $\phi$ is given by
\begin{equation}
\label{blfdef}
\ell_{\phi}(y,v) := B_{\phi}(\tilde{y},\tilde{v}) = \phi(\tilde{y}) - \phi(\tilde{v}) - (\tilde{y} - \tilde{v})' \cdot \textnormal{\textsf{D}} \phi(\tilde{v}); \quad y,v \in \Delta^n,
\end{equation}
where $\tilde{y}=\Pi_\Delta(y)$, and $\tilde{v}=\Pi_\Delta(v)$. Since the conditional Bayes risk of a strictly proper loss is strictly concave, any differentiable strictly proper loss $\ell:\Delta^n \rightarrow \RR_+^n$ will generate a BLF $\ell_{\phi}$ with generator $\phi=-\Lubartil_\ell$. Further if $\ell$ is fair, $\ell_{i}(v)=\ell_{\phi}(e_i^n,v)$; i.e. reconstruction is possible. For example the \textit{Kullback-Leibler loss} given by $\ell_{KL}(y,v):=\sum_{i=1}^{n}{y(i)\log{\frac{y(i)}{v(i)}}}$, is a BLF generated by the log loss which is strictly proper.

The following lemma (multi-class extension of a result given by \cite{haussler1998sequential}) provides the mixability condition for probability games.
\begin{lemma}
	\label{mixctsn}
	For given $\ell:\Delta^n \times \Delta^n \rightarrow \RR_+$, assume that for all $\tilde{y},\tilde{v_1},\tilde{v_2} \in \tilde{\Delta}^n$ (let $y=\Pi_\Delta^{-1}(\tilde{y}), v_1=\Pi_\Delta^{-1}(\tilde{v_1})$, and $v_2=\Pi_\Delta^{-1}(\tilde{v_2})$), the function $g$ defined by
	\begin{equation}
	\label{gfunc}
	g(\tilde{y},\tilde{v_1},\tilde{v_2})=\frac{\beta}{c(\beta)} \ell(y,v_1) - \beta \ell(y,v_2)
	\end{equation}
	satisfies
	\begin{equation}
	\label{condmixcts}
	\textnormal{\textsf{H}}_{\tilde{y}} g(\tilde{y},\tilde{v_1},\tilde{v_2}) + \textnormal{\textsf{D}}_{\tilde{y}} g(\tilde{y},\tilde{v_1},\tilde{v_2}) \cdot (\textnormal{\textsf{D}}_{\tilde{y}} g(\tilde{y},\tilde{v_1},\tilde{v_2}))' \succcurlyeq 0.
	\end{equation}
	If
	\begin{equation}
	\label{mixreqgen}
	\exists{\tilde{v^*} \in \tilde{\Delta}^n} \text{ s.t. } \ell(y,v^*) \leq - \frac{c(\beta)}{\beta} \log{\int{e^{-\beta \ell(y,v)}P(d\tilde{v})}}
	\end{equation}
	holds for the vertices $\tilde{y} = e_i^{n-1}, i \in [n]$, then it holds for all values $\tilde{y} \in \tilde{\Delta}^n$ (where $y=\Pi_\Delta^{-1}(\tilde{y})$, $v^*=\Pi_\Delta^{-1}(\tilde{v^*})$ and $v=\Pi_\Delta^{-1}(\tilde{v})$).
\end{lemma}
\begin{proof}
	From \eqref{mixreqgen}
	\begin{equation*}
	\frac{\beta}{c(\beta)} \ell(y,v^*) + \log{\int{e^{-\beta \ell(y,v)}P(d\tilde{v})}} \leq 0.
	\end{equation*}
	By exponentiating both sides we get
	\begin{equation*}
	e^{\frac{\beta}{c(\beta)} \ell(y,v^*)} \cdot \int{e^{-\beta \ell(y,v)}P(d\tilde{v})} \leq 1.
	\end{equation*}
	Denoting the left hand side of the above inequality by $f(\tilde{y})$ we have
	\begin{equation*}
	f(\tilde{y}) = \int{e^{g(\tilde{y},\tilde{v^*},\tilde{v})}P(d\tilde{v})}.
	\end{equation*}
	Since the Hessian of $f$ w.r.t.\ $\tilde{y}$ given by
	\begin{equation*}
	\textsf{H}_{\tilde{y}} f(\tilde{y}) = \int{e^{g(\tilde{y},\tilde{v^*},\tilde{v})} \left( \textsf{H}_{\tilde{y}} g(\tilde{y},\tilde{v^*},\tilde{v}) + \textnormal{\textsf{D}}_{\tilde{y}} g(\tilde{y},\tilde{v^*},\tilde{v}) \cdot (\textnormal{\textsf{D}}_{\tilde{y}} g(\tilde{y},\tilde{v^*},\tilde{v}))' \right) P(d\tilde{v})}
	\end{equation*}
	is positive semi-definite (by \eqref{condmixcts}), $f(\tilde{y})$ is convex in $\tilde{y}$. So the maximum values of $f$ for $\tilde{y} \in \tilde{\Delta}^n$ occurs for some $\tilde{y} = e_i^{n-1}, i \in [n]$. And by noting that, \eqref{mixreqgen} is equivalent to $f(\tilde{y}) \leq 1$ for $\tilde{y} = e_i^{n-1}, i \in [n]$, the proof is completed.
\end{proof}

The next proposition shows that the mixability and exp-concavity of a strictly proper loss is carried over to the BLF generated by it.
\begin{proposition}
	\label{properblfconnect}
	For a strictly proper fair loss $\ell:\Delta^n \rightarrow \RR_+^n$, and the BLF $\ell_{\phi}:\Delta^n \times \Delta^n \rightarrow \RR_+$ generated by $\ell$ with $\phi=-\Lubartil_\ell$ , if $\ell$ is $\beta$-mixable (resp. $\alpha$-exp-concave), then $\ell_{\phi}$ is also $\beta$-mixable (resp. $\alpha$-exp-concave).
\end{proposition}
\begin{proof}
	From \eqref{gfunc} and \eqref{blfdef}, for the BLF $\ell_{\phi}$ we have
	\begin{align*}
	g(\tilde{y},\tilde{v_1},\tilde{v_2}) &= \frac{\beta}{c(\beta)} \{ \phi(\tilde{y}) - \phi(\tilde{v_1}) - (\tilde{y} - \tilde{v_1})' \cdot \textsf{D}\phi(\tilde{v_1}) \} \\
	& \quad \quad \quad \quad - \beta \{ \phi(\tilde{y}) - \phi(\tilde{v_2}) - (\tilde{y} - \tilde{v_2})' \cdot \textsf{D}\phi(\tilde{v_2}) \}, \\
	\textsf{D}_{\tilde{y}} g(\tilde{y},\tilde{v_1},\tilde{v_2}) &= \frac{\beta}{c(\beta)} \{ \textsf{D}\phi(\tilde{y}) - \textsf{D}\phi(\tilde{v_1}) \} - \beta \{ \textsf{D}\phi(\tilde{y}) - \textsf{D}\phi(\tilde{v_2}) \}, \\
	\textsf{H}_{\tilde{y}} g(\tilde{y},\tilde{v_1},\tilde{v_2}) &= \frac{\beta}{c(\beta)} \textsf{H}\phi(\tilde{y}) - \beta \textsf{H}\phi(\tilde{y}).
	\end{align*}
	And since $x \cdot x' \succcurlyeq 0, \forall{x \in \RR^n}$, \eqref{condmixcts} is satisfied for all $\tilde{y},\tilde{v_1},\tilde{v_2} \in \tilde{\Delta}^n$ when $c(\beta)=1$, which is the mixability condition (in addition requiring $\tilde{v^*}=\int{\tilde{v}P(d\tilde{v})}$ in \eqref{mixreqgen} is the exp-concavity condition). Then by applying Lemma~\ref{mixctsn} proof is completed.
\end{proof}
As an application of Proposition~\ref{properblfconnect}, we can see that both Kullback-Leibler loss and log loss are 1-mixable and 1-exp-concave.


\subsection{Proofs}
\label{sec:proof}

\begin{proof} \textbf{(Proposition \ref{geoprop})}
	We first prove that the set $\psi(\Delta^n)$ is convex.
	For any $p,q \in \Delta^n$, by assumption there exists $c \ge 0$ and $r \in \Delta^n$ such that
	$\frac{1}{2}(E_\beta(\ell(p))+E_\beta(\ell(q))) + c \vone_n =E_\beta(\ell(r))$.
	Therefore $\frac{1}{2}(\psi(p)+\psi(q))=\psi(r)$,
	which implies the convexity of the set $\psi(\Delta^n)$.
	
	Let $T:\RR^{n-1} \ni (e^{-\beta z_1}-e^{-\beta z_n},...,e^{-\beta z_{n-1}}-e^{-\beta z_n})' \rightarrow (e^{-\beta z_1},...,e^{-\beta z_n})' \in [0,1]^n$.
	Note this mapping from low dimension to high dimension is well defined because if there are two different $z$ and $\zbar$ in $\ell(\cV)$ such that $J E_\beta(z)$ = $J E_\beta(\zbar)$,
	then there must be $c \neq 0$ such that $E_\beta(z) + c \vone = E_\beta(\zbar)$.
	This means $z > \zbar$ or $z < \zbar$, which violates the strict properness of $\ell$.

	Since for any $v=\psi(p)=J E_\beta(\ell(p))$ we have $p=\ell^{-1}\left( E_\beta^{-1}(Tv)\right)$, the link $\psi$ is invertible ($\ell$ is invertible if it is strictly proper (\cite{vernet2011composite}), and $E_\beta$ is invertible for $\beta > 0$).
	
	Now $\ell \circ \psi^{-1}$ is $\beta$-exp-concave if for all $p,q \in \Delta^n$
	\begin{equation}
	\label{expproof}
	E_\beta \left( \ell \circ \psi^{-1} \left( \frac{1}{2}(\psi(p)+\psi(q)) \right) \right) \, \geq \, \frac{1}{2} E_\beta \left( \ell \circ \psi^{-1} (\psi(p)) \right) + \frac{1}{2}E_\beta \left( \ell \circ \psi^{-1} (\psi(q)) \right).
	\end{equation}
	The right-hand side is obviously $\frac{1}{2}\left(E_\beta(\ell(p))+E_\beta(\ell(q))\right)$. Let $r = \psi^{-1} \left( \frac{1}{2}(\psi(p)+\psi(q)) \right) \in \Delta^n$. Then
	\begin{equation*}
	J E_\beta(\ell(r))=\psi(r)=\frac{1}{2}\left( J E_\beta(\ell(p)) + J E_\beta(\ell(q)) \right).
	\end{equation*}
	Therefore $\frac{1}{2}\left( E_\beta(\ell(p)) + E_\beta(\ell(q)) \right) = E_\beta(\ell(r)) + c \vone_n$ for some $c \in \RR$. To establish \eqref{expproof}, it suffices to show $c \leq 0$. But this is guaranteed by the condition assumed.
\end{proof}

\begin{proof} \textbf{(Proposition \ref{exp_concave_approx})}
	We first show that for a half space $H_{-p}^{\gamma_p}$ defined in \eqref{eq:supp_T_ell} with $p \in \Delta^n_\epsilon$, $E_\beta^{-1}(H_{-p}^{\gamma_p} \cap \RR^n_+)$ must be the super-prediction set of a scaled and shifted log loss.
	In fact, as $p_i > 0$, clearly $\gamma_p = \min_{x \in T_\ell} x' \cdot (-p) < 0$.
	Define a new loss $\elltil^{\text{log}}_i(q) = -\frac{1}{\beta} \log (-\frac{\gamma_p}{p_i} q_i)$ over $q \in \Delta^n$.
	Then $S_{\elltil^{\log}} \subseteq E_\beta^{-1}(H_{-p}^{\gamma_p} \cap \RR^n_+)$ can be seen from
	\begin{align}
	\sum_i (-p_i) \exp(-\beta \elltil^{\text{log}}_i(q)) = \sum_i (-p_i) (-\frac{\gamma_p}{p_i} q_i) = \gamma_p.
	\end{align}
	Conversely, for any $u$ such that $u_i > 0$ and $(-p)'\cdot u = \gamma_p$,
	simply choose $q_i = -\frac{u_i p_i}{\gamma_p}$.
	Then $q \in \Delta^n$ and $E_\beta(\elltil^{\text{log}}(q)) = u$.
	In summary, $E_\beta^{-1}(H_{-p}^{\gamma_p} \cap \RR^n_+)$ is the super-prediction set of $\elltil^{\text{log}}$.
	
	To prove Proposition \ref{exp_concave_approx}, we first show that for any point $a \in T_\ell^{\epsilon}$ and any direction $d$ from the relative interior of the positive orthant (which includes the $\vone_n$ direction),
	the ray $\{a + r d : r \ge 0\}$ will be blocked by a boundary point of $T_\ell^{\epsilon}$.
	This is because by the definition of $T_\ell^\epsilon$ in \eqref{eq:ext_T_ell},
	the largest value of $r$ to guarantee $a + rd \in T_\ell^\epsilon$ can be computed by
	\begin{align}
	r^* := \sup \{r\ge 0: a+rd \in T_\ell^{\epsilon}\} = \sup \{r\ge 0: (a+rd)' \cdot (-p) \ge \gamma_p, \forall p \in \Delta^n_\epsilon \}
	\end{align}
	must be finite and attained.
	Denote $x = a + r^*d$, which must be on the boundary of $T_\ell^{\epsilon}$ because
	\begin{align}
	\label{eq:bd_1}
	-x' \cdot p &\ge \gamma_p, \text{ for all } p \in \Delta^n_\epsilon, \\
	\label{eq:bd_2}
	\text{ and } -x' \cdot p^* &= \gamma_{p^*} \text{ for some } p^* \in \Delta^n_\epsilon \text{ (not necessarily unique)}.
	\end{align}
	In order to prove the first statement of Proposition \ref{exp_concave_approx}, it suffices to show that for any point $x$ on the north-east boundary of $T_\ell^{\epsilon}$,
	there exists a $q \in \Delta^n$ such that $E_\beta(\elltil_\epsilon(q)) = x$.
	Suppose $x$ satisfies \eqref{eq:bd_1} and \eqref{eq:bd_2}.
	Then consider the (shifted/scaled) log loss $\elltil^{\log}$ that corresponds to $H_{p^*}^{\gamma_{p^*}}$.
	Because log loss is strictly proper, there must be a unique $q \in \Delta^n$ such that the hyperplane $H_0:=\{z: q' \cdot z = q' \cdot E_\beta^{-1}(x)\}$ supports the super-prediction set of $\elltil^{\log}$ (\ie\ $E_\beta^{-1}(H_{-p^*}^{\gamma_{p^*}} \cap \RR_+^n)$) at $E_\beta^{-1}(x)$.
	Since $E_\beta^{-1}(T_\ell^\epsilon \cap \RR_+^n)$ is a convex subset of $E_\beta^{-1}(H_{-p^*}^{\gamma_{p^*}} \cap \RR_+^n)$,
	this hyperplane also supports $E_\beta^{-1}(T_\ell^\epsilon \cap \RR_+^n)$ at $E_\beta^{-1}(x)$.
	Therefore $E_\beta^{-1}(x)$ is an optimal solution to the problem in the definition of $\elltil_\epsilon(q)$ in \eqref{def:elltil}.
	Finally observe that it must be the unique optimal solution,
	because if there were another solution which also lies on $H_0$, then by the convexity of the super-prediction set of $\elltil_\epsilon$,
	the line segment between them must also lie on the prediction set of $\elltil_\epsilon$.
	This violates the mixability condition of $\elltil_\epsilon$,
	because by construction its sub-exp-prediction set is convex.
	
	In order to check where $\ell(p) = \elltil_\epsilon(p)$, a sufficient condition is that the normal direction $d$ on the exp-prediction set evaluated at $E_\beta(\ell(p))$ satisfies $d_i / \sum_j d_j > \epsilon$.
	Simple calculus shows that $d_i \propto p_i \exp(\beta \ell_i(p))$.
	Therefore as long as $p$ is in the relative interior of $\Delta^n$,
	$d_i / \sum_j d_j > \epsilon$ can always be satisfied by choosing a sufficiently small $\epsilon$.
	And for each fixed $\epsilon$, the set $S_\epsilon$ mentioned in the theorem consists exactly of all such $p$ that satisfies this condition.
\end{proof} 

\begin{proof} \textbf{(Proposition \ref{complexversion})}
	When $n=2$, \eqref{binaryfirstder}, \eqref{binarysecder1} and \eqref{binarysecder2} and the positivity of $\tilde{\psi}'$ simplify \eqref{multiexpcondition} to the two conditions:
	\begin{eqnarray*}
		(1-\tilde{p}) \, k'(\tilde{p}) &\leq& k(\tilde{p}) - \alpha \, \tilde{\psi}'(\tilde{p}) \, k(\tilde{p})^2 \, (1-\tilde{p})^2, \\
		-\tilde{p} \, k'(\tilde{p}) &\leq& k(\tilde{p}) - \alpha \, \tilde{\psi}'(\tilde{p}) \, k(\tilde{p})^2 \, \tilde{p}^2,
	\end{eqnarray*}
	for all $\tilde{p} \in (0,1)$. These two conditions can be merged as follows
	\begin{equation*}
	-\frac{1}{\tilde{p}} + \alpha \, \tilde{\psi}'(\tilde{p}) \, k(\tilde{p}) \, \tilde{p} \enspace \leq \enspace \frac{k'(\tilde{p})}{k(\tilde{p})} \enspace \leq \enspace \frac{1}{1-\tilde{p}} - \alpha \, \tilde{\psi}'(\tilde{p}) \, k(\tilde{p}) \, (1-\tilde{p}), \quad \forall{\tilde{p} \in (0,1)}.
	\end{equation*}
	By noting that $k(\tilde{p})=\frac{w(\tilde{p})}{{\tilde{\psi}}'(\tilde{p})}$ and $k'(\tilde{p})=\frac{w'(\tilde{p}){\tilde{\psi}}'(\tilde{p})-w(\tilde{p}){\tilde{\psi}}''(\tilde{p})}{{{\tilde{\psi}}'(\tilde{p})}^2}$ completes the proof.
\end{proof}

\begin{proof} \textbf{(Proposition \ref{simplerversion})}
	Let $g(\tilde{p}) = \frac{1}{w(\tilde{p})}$ and so $g'(\tilde{p}) = - \frac{1}{w(\tilde{p})^2} w'(\tilde{p}), \, g(v) = \int_{\frac{1}{2}}^{v}g'(\tilde{p})d\tilde{p} + g(\frac{1}{2})$ and $g(\frac{1}{2}) = \frac{1}{w(\frac{1}{2})} = 1$. By dividing all sides of \eqref{maineq} by $-w(\tilde{p})$ and applying the substitution we get,
	\begin{equation}
	\label{temp1}
	\frac{1}{\tilde{p}}g(\tilde{p}) - \alpha \tilde{p} \enspace \geq \enspace  g'(\tilde{p}) - \Phi_{\tilde{\psi}}(\tilde{p}) g(\tilde{p}) \enspace \geq \enspace -\frac{1}{1-\tilde{p}}g(\tilde{p}) + \alpha (1-\tilde{p}), \quad \forall \tilde{p} \in (0,1),
	\end{equation}
	where $\Phi_{\tilde{\psi}}(\tilde{p}):=-\frac{\tilde{\psi}''(\tilde{p})}{\tilde{\psi}'(\tilde{p})}$. If we take the first inequality of \eqref{temp1} and rearrange it we obtain,
	\begin{equation}
	\label{temp2}
	-\alpha \geq \left( g'(\tilde{p})\frac{1}{\tilde{p}} - g(\tilde{p})\frac{1}{\tilde{p}^2} \right) - \Phi_{\tilde{\psi}}(\tilde{p}) \frac{g(\tilde{p})}{\tilde{p}} = \left( \frac{g(\tilde{p})}{\tilde{p}} \right)' - \Phi_{\tilde{\psi}}(\tilde{p}) \left( \frac{g(\tilde{p})}{\tilde{p}} \right), \quad \forall \tilde{p} \in (0,1).
	\end{equation}
	Multiplying \eqref{temp2} by $e^{-\int_{0}^{\tilde{p}}\Phi_{\tilde{\psi}}(t)dt}$ will result in,
	\begin{eqnarray}
	-\alpha e^{-\int_{0}^{\tilde{p}}\Phi_{\tilde{\psi}}(t)dt} &\geq& \left( \frac{g(\tilde{p})}{\tilde{p}} \right)' e^{-\int_{0}^{\tilde{p}}\Phi_{\tilde{\psi}}(t)dt} + \left( \frac{g(\tilde{p})}{\tilde{p}} \right) e^{-\int_{0}^{\tilde{p}}\Phi_{\tilde{\psi}}(t)dt} (- \Phi_{\tilde{\psi}}(\tilde{p})) \nonumber \\
	&=& \left( \frac{g(\tilde{p})}{\tilde{p}} e^{-\int_{0}^{\tilde{p}}\Phi_{\tilde{\psi}}(t)dt} \right)', \quad \forall \tilde{p} \in (0,1). \label{temp3}
	\end{eqnarray}
	Since
	\begin{equation*}
	- \int_{0}^{\tilde{p}}\Phi_{\tilde{\psi}}(t)dt = - \int_{0}^{\tilde{p}} - \frac{\tilde{\psi}''(t)}{\tilde{\psi}'(t)} dt = \int_{0}^{\tilde{p}}(\log{\tilde{\psi}'(t)})'dt = \log{\frac{\tilde{\psi}'(\tilde{p})}{\tilde{\psi}'(0)}},
	\end{equation*}
	\eqref{temp3} is reduced to
	\begin{eqnarray*}
		-\alpha \frac{\tilde{\psi}'(\tilde{p})}{\tilde{\psi}'(0)} &\geq& \left( \frac{g(\tilde{p})}{\tilde{p}} \frac{\tilde{\psi}'(\tilde{p})}{\tilde{\psi}'(0)} \right)', \quad \forall \tilde{p} \in (0,1) \\
		\Rightarrow -\alpha \tilde{\psi}'(\tilde{p}) &\geq& \left( \frac{g(\tilde{p})}{\tilde{p}} \tilde{\psi}'(\tilde{p}) \right)', \quad \forall \tilde{p} \in (0,1).
	\end{eqnarray*}
	For $v \geq \frac{1}{2}$ we thus have
	\begin{eqnarray*}
		-\alpha \int_{\frac{1}{2}}^{v}\tilde{\psi}'(\tilde{p}) d\tilde{p} &\geq& \int_{\frac{1}{2}}^{v}\left( \frac{g(\tilde{p})}{\tilde{p}} \tilde{\psi}'(\tilde{p}) \right)' d\tilde{p}, \quad \forall v \in [1/2,1) \\
		\Rightarrow -\alpha (\tilde{\psi}(v) - \tilde{\psi}(\frac{1}{2})) &\geq& \left( \frac{g(v)}{v} \tilde{\psi}'(v) - \frac{g(\frac{1}{2})}{\frac{1}{2}} \tilde{\psi}'(\frac{1}{2})\right) \\
		&=& \left( \frac{1}{w(v)v} \tilde{\psi}'(v) - 2 \tilde{\psi}'(\frac{1}{2})\right), \quad \forall v \in [1/2,1) \\
		\Rightarrow w(v) &\geq& \frac{\tilde{\psi}'(v)}{v (2\tilde{\psi}'(\frac{1}{2}) - \alpha (\tilde{\psi}(v) - \tilde{\psi}(\frac{1}{2})))}, \quad \forall v \in [1/2,1).
	\end{eqnarray*}
	Also by considering $v \leq \frac{1}{2}$ case as above, we get
	\begin{equation*}
	\frac{\tilde{\psi}'(v)}{v (2\tilde{\psi}'(\frac{1}{2}) - \alpha (\tilde{\psi}(v) - \tilde{\psi}(\frac{1}{2})))} \lesseqgtr w(v), \quad \forall v \in (0,1).
	\end{equation*}
	Finally by following the similar steps for the second inequality of \eqref{temp1}, the proof will be completed.
\end{proof}

Here we provide an integral inequalities related result (without proof) due to Beesack and presented in \cite{dragomir2000some}. 
\begin{theorem}
	\label{beesacktheo}
	Let $y$ and $k$ be continuous and $f$ and $g$ Riemann integrable functions on $J = [\alpha, \beta]$ with $g$ and $k$ nonnegative on $J$. If
	\begin{equation*}
	y(x) \enspace \geq \enspace f(x) + g(x) \int_{\alpha}^{x}y(t)k(t)dt, \quad x \in J,
	\end{equation*}
	then
	\begin{equation*}
	y(x) \enspace \geq \enspace f(x) + g(x) \int_{\alpha}^{x}f(t)k(t)\exp{\left( \int_{t}^{x}g(r)k(r)dr \right)}dt, \quad  x \in J.
	\end{equation*}
	The result remains valid if $\int_{\alpha}^{x}$ is replaced by $\int_{x}^{\beta}$ and $\int_{t}^{x}$ by $\int_{x}^{t}$ throughout.
\end{theorem}
Using the above theorem, we get the following simplified test for the conditions in Theorem~\ref{propexpsuff}: 
\[
-\alpha + \frac{\alpha}{2\tilde{p}^2} - \frac{2}{\tilde{p}^2} ~\leq~ a(\tilde{p}) ~\implies~ \left[ \frac{\alpha (1-\tilde{p})}{\tilde{p}} - \frac{2}{\tilde{p}(1-\tilde{p})} \right] + \frac{1}{\tilde{p}(1-\tilde{p})} \int_{\tilde{p}}^{1/2}a(t)dt ~\leq~ a(\tilde{p})
\]
and
\[
\frac{\alpha \tilde{p}}{(1-\tilde{p})} + \frac{2 \alpha \tilde{p} \! - \! \alpha \! - \! 4}{2 (1-\tilde{p})^2} ~\leq~ b(\tilde{p}) ~\implies~ \left[ \frac{\alpha \tilde{p}}{(1-\tilde{p})} - \frac{2}{\tilde{p}(1-\tilde{p})} \right] + \frac{1}{\tilde{p}(1-\tilde{p})} \int_{1/2}^{\tilde{p}}b(t)dt ~\leq~ b(\tilde{p}) .
\]
Above two identities are proved in the following proof of Theorem~\ref{propexpsuff}.

\begin{proof} \textbf{(Theorem \ref{propexpsuff})}
	The necessary \textit{and} sufficient condition for the exp-concavity of proper losses is given by \eqref{identitynscond}. But from \eqref{temp2}, we can see that \eqref{identitynscond} is equivalent to
	\begin{equation}
	\label{identitynscondequiv1}
	\left( \frac{g(\tilde{p})}{\tilde{p}} \right)' \leq -\alpha, \quad \forall \tilde{p} \in (0,1),
	\end{equation}
	and
	\begin{equation}
	\label{identitynscondequiv2}
	-\left( \frac{g(\tilde{p})}{1-\tilde{p}} \right)' \leq -\alpha, \quad \forall \tilde{p} \in (0,1),
	\end{equation}
	where $g(\tilde{p})=\frac{1}{w(\tilde{p})}$ with $w(\frac{1}{2})=1$; i.e. \eqref{identitynscond} if and only if \eqref{identitynscondequiv1} \& \eqref{identitynscondequiv2}.
	
	Now if we choose the weight function $w(\tilde{p})$ as follows
	\begin{equation}
	\label{weightchoicea}
	w(\tilde{p})=\frac{1}{\tilde{p} \left( 2 + \int_{1/2}^{\tilde{p}}a(t)dt \right)},
	\end{equation}
	such that $a(t) \leq -\alpha$, then \eqref{identitynscondequiv1} will be satisfied (since $\eqref{weightchoicea} \implies 2 + \int_{1/2}^{\tilde{p}}a(t)dt = \frac{1}{w(\tilde{p})\tilde{p}} = \frac{g(\tilde{p})}{\tilde{p}} \implies a(\tilde{p}) = \left( \frac{g(\tilde{p})}{\tilde{p}} \right)'$). Similarly the weight function $w(\tilde{p})$ given by
	\begin{equation}
	\label{weightchoiceb}
	w(\tilde{p})=\frac{1}{(1-\tilde{p}) \left( 2 - \int_{1/2}^{\tilde{p}}b(t)dt \right)},
	\end{equation}
	with $b(t) \leq -\alpha$ will satisfy \eqref{identitynscondequiv2} (since $\eqref{weightchoiceb} \implies 2 - \int_{1/2}^{\tilde{p}}b(t)dt = \frac{1}{w(\tilde{p})(1-\tilde{p})} = \frac{g(\tilde{p})}{1-\tilde{p}} \implies -b(\tilde{p}) = \left( \frac{g(\tilde{p})}{1-\tilde{p}} \right)'$). To satisfy both \eqref{identitynscondequiv1} and \eqref{identitynscondequiv2} at the same time (then obviously \eqref{identitynscond} will be satisfied), we can make the two forms of the weight function (\eqref{weightchoicea} and \eqref{weightchoiceb}) equivalent with the appropriate choice of $a(t)$ and $b(t)$. This can be done in two cases.
	
	In the first case, for $\tilde{p} \in (0,1/2]$ we can fix the weight function $w(\tilde{p})$ as given by \eqref{weightchoicea} and choose $a(t)$ such that,
	\begin{itemize}
		\item{$a(t) \leq -\alpha$ (then \eqref{identitynscondequiv1} is satisfied) and}
		\item{$\eqref{weightchoicea} = \eqref{weightchoiceb} \implies b(t) \leq -\alpha$ (then \eqref{identitynscondequiv2} is satisfied).}
	\end{itemize}
	But $\eqref{weightchoicea} = \eqref{weightchoiceb}$ for all $\tilde{p} \in (0,1/2]$ if and only if
	\begin{eqnarray*}
		&&\tilde{p} \left( 2 - \int_{\tilde{p}}^{1/2}a(t)dt \right)=(1-\tilde{p}) \left( 2 + \int_{\tilde{p}}^{1/2}b(t)dt \right), \quad \tilde{p} \in (0,1/2] \\
		&\iff&\frac{\tilde{p}}{1-\tilde{p}} \left( 2 - \int_{\tilde{p}}^{1/2}a(t)dt \right) - 2 = \int_{\tilde{p}}^{1/2}b(t)dt, \quad \tilde{p} \in (0,1/2] \\
		&\iff&\frac{\tilde{p}}{1-\tilde{p}}a(\tilde{p}) + \frac{1}{(1-\tilde{p})^2} \left( 2 - \int_{\tilde{p}}^{1/2}a(t)dt \right) = -b(\tilde{p}), \quad \tilde{p} \in (0,1/2],
	\end{eqnarray*}
	where the last step is obtained by differentiating both sides w.r.t $\tilde{p}$. Thus the constraint $\eqref{weightchoicea} = \eqref{weightchoiceb} \implies b(t) \leq -\alpha$ can be given as
	\begin{eqnarray*}
		&&\frac{\tilde{p}}{1-\tilde{p}}a(\tilde{p}) + \frac{1}{(1-\tilde{p})^2} \left( 2 - \int_{\tilde{p}}^{1/2}a(t)dt \right) \enspace \geq \enspace \alpha, \quad \tilde{p} \in (0,1/2] \\
		&\iff& a(\tilde{p}) \enspace \geq \enspace \left[ \frac{\alpha (1-\tilde{p})}{\tilde{p}} - \frac{2}{\tilde{p}(1-\tilde{p})} \right] + \frac{1}{\tilde{p}(1-\tilde{p})} \int_{\tilde{p}}^{1/2}a(t)dt, \quad \tilde{p} \in (0,1/2] \\
		&\iff& a(\tilde{p}) \enspace \geq \enspace f(\tilde{p}) + g(\tilde{p}) \int_{\tilde{p}}^{1/2}a(t)k(t)dt, \quad \tilde{p} \in (0,1/2],
	\end{eqnarray*}
	where $f(\tilde{p}) = \left[ \frac{\alpha (1-\tilde{p})}{\tilde{p}} - \frac{2}{\tilde{p}(1-\tilde{p})} \right]$, $g(\tilde{p}) = \frac{1}{\tilde{p}(1-\tilde{p})} = \frac{1}{\tilde{p}} + \frac{1}{(1-\tilde{p})}$ and $k(t)=1$. Now by applying Theorem~\ref{beesacktheo} we have
	\begin{equation*}
	a(\tilde{p}) \enspace \geq \enspace f(\tilde{p}) + g(\tilde{p}) \int_{\tilde{p}}^{1/2}f(t)k(t)\exp{\left( \int_{\tilde{p}}^{t}g(r)k(r)dr \right)}dt, \quad \tilde{p} \in (0,1/2].
	\end{equation*}
	Since
	\begin{equation*}
	\int_{\tilde{p}}^{t}g(r)k(r)dr = \int_{\tilde{p}}^{t}\frac{1}{r} + \frac{1}{(1-r)}dr = [\ln r - \ln (1-r)]_{\tilde{p}}^{t} = \ln \left(\frac{t}{(1-t)} \frac{(1-\tilde{p})}{\tilde{p}}\right),
	\end{equation*}
	\begin{eqnarray*}
		\int_{\tilde{p}}^{1/2}f(t)k(t)\exp{\left( \int_{\tilde{p}}^{t}g(r)k(r)dr \right)}dt &=& \int_{\tilde{p}}^{1/2}f(t)\frac{t}{(1-t)} \frac{(1-\tilde{p})}{\tilde{p}}dt \\
		&=& \frac{(1-\tilde{p})}{\tilde{p}} \int_{\tilde{p}}^{1/2}\left[ \frac{\alpha (1-t)}{t} - \frac{2}{t(1-t)} \right]\frac{t}{(1-t)} dt \\
		&=& \frac{(1-\tilde{p})}{\tilde{p}} \int_{\tilde{p}}^{1/2} \alpha - \frac{2}{(1-t)^2} dt \\
		&=& \frac{(1-\tilde{p})}{\tilde{p}} \left[ \alpha t - \frac{2}{1-t} \right]_{\tilde{p}}^{1/2} \\
		&=& \frac{(1-\tilde{p})}{\tilde{p}} \left[ \frac{\alpha}{2} - 4 - \alpha \tilde{p} + \frac{2}{1-\tilde{p}} \right],
	\end{eqnarray*}
	we get
	\begin{eqnarray*}
		a(\tilde{p}) &\geq& \left[ \frac{\alpha (1-\tilde{p})}{\tilde{p}} - \frac{2}{\tilde{p}(1-\tilde{p})} \right] + \frac{1}{\tilde{p}(1-\tilde{p})} \frac{(1-\tilde{p})}{\tilde{p}} \left[ \frac{\alpha}{2} - 4 - \alpha \tilde{p} + \frac{2}{1-\tilde{p}} \right], \quad \tilde{p} \in (0,1/2] \\
		&=& -\alpha+ \frac{\alpha}{2\tilde{p}^2} - \frac{2}{\tilde{p}^2}, \quad \tilde{p} \in (0,1/2].
	\end{eqnarray*}
	
	Similarly in the second case, for $\tilde{p} \in [1/2,1)$ we can fix the weight function $w(\tilde{p})$ as given by \eqref{weightchoiceb} and choose $b(t)$ such that,
	\begin{itemize}
		\item{$b(t) \leq -\alpha$ (then \eqref{identitynscondequiv2} is satisfied) and}
		\item{$\eqref{weightchoicea} = \eqref{weightchoiceb} \implies a(t) \leq -\alpha$ (then \eqref{identitynscondequiv1} is satisfied).}
	\end{itemize}
	But $\eqref{weightchoicea} = \eqref{weightchoiceb}$ for all $\tilde{p} \in [1/2,1)$ if and only if
	\begin{eqnarray*}
		&&\tilde{p} \left( 2 + \int_{1/2}^{\tilde{p}}a(t)dt \right)=(1-\tilde{p}) \left( 2 - \int_{1/2}^{\tilde{p}}b(t)dt \right), \quad \tilde{p} \in [1/2,1) \\
		&\iff&\int_{\tilde{p}}^{1/2}a(t)dt = \frac{1-\tilde{p}}{\tilde{p}} \left( 2 - \int_{1/2}^{\tilde{p}}b(t)dt \right) - 2, \quad \tilde{p} \in [1/2,1) \\
		&\iff&a(\tilde{p}) = -\frac{1-\tilde{p}}{\tilde{p}}b(\tilde{p}) - \frac{1}{\tilde{p}^2} \left( 2 - \int_{1/2}^{\tilde{p}}b(t)dt \right), \quad \tilde{p} \in [1/2,1),
	\end{eqnarray*}
	where the last step is obtained by differentiating both sides w.r.t $\tilde{p}$. Thus the constraint $\eqref{weightchoicea} = \eqref{weightchoiceb} \implies a(t) \leq -\alpha$ can be given as
	\begin{eqnarray*}
		&&\frac{1-\tilde{p}}{\tilde{p}}b(\tilde{p}) + \frac{1}{\tilde{p}^2} \left( 2 - \int_{1/2}^{\tilde{p}}b(t)dt \right) \enspace \geq \enspace \alpha, \quad \tilde{p} \in [1/2,1) \\
		&\iff& b(\tilde{p}) \enspace \geq \enspace \left[ \frac{\alpha \tilde{p}}{(1-\tilde{p})} - \frac{2}{\tilde{p}(1-\tilde{p})} \right] + \frac{1}{\tilde{p}(1-\tilde{p})} \int_{1/2}^{\tilde{p}}b(t)dt, \quad \tilde{p} \in [1/2,1) \\
		&\iff& b(\tilde{p}) \enspace \geq \enspace f(\tilde{p}) + g(\tilde{p}) \int_{1/2}^{\tilde{p}}b(t)k(t)dt, \quad \tilde{p} \in [1/2,1),
	\end{eqnarray*}
	where $f(\tilde{p}) = \left[ \frac{\alpha \tilde{p}}{(1-\tilde{p})} - \frac{2}{\tilde{p}(1-\tilde{p})} \right]$, $g(\tilde{p}) = \frac{1}{\tilde{p}(1-\tilde{p})} = \frac{1}{\tilde{p}} + \frac{1}{(1-\tilde{p})}$ and $k(t)=1$. Again by applying Theorem~\ref{beesacktheo} we have
	\begin{equation*}
	b(\tilde{p}) \enspace \geq \enspace f(\tilde{p}) + g(\tilde{p}) \int_{1/2}^{\tilde{p}}f(t)k(t)\exp{\left( \int_{t}^{\tilde{p}}g(r)k(r)dr \right)}dt, \quad \tilde{p} \in [1/2,1).
	\end{equation*}
	Since
	\begin{equation*}
	\int_{t}^{\tilde{p}}g(r)k(r)dr = \int_{t}^{\tilde{p}}\frac{1}{r} + \frac{1}{(1-r)}dr = [\ln r - \ln (1-r)]_{t}^{\tilde{p}} = \ln \left(\frac{\tilde{p}}{(1-\tilde{p})} \frac{(1-t)}{t}\right),
	\end{equation*}
	\begin{eqnarray*}
		\int_{1/2}^{\tilde{p}}f(t)k(t)\exp{\left( \int_{t}^{\tilde{p}}g(r)k(r)dr \right)}dt &=& \int_{1/2}^{\tilde{p}}f(t)\frac{\tilde{p}}{(1-\tilde{p})} \frac{(1-t)}{t}dt \\
		&=& \frac{\tilde{p}}{(1-\tilde{p})} \int_{1/2}^{\tilde{p}}\left[ \frac{\alpha t}{(1-t)} - \frac{2}{t(1-t)} \right]\frac{(1-t)}{t} dt \\
		&=& \frac{\tilde{p}}{(1-\tilde{p})} \int_{1/2}^{\tilde{p}} \alpha - \frac{2}{t^2} dt \\
		&=& \frac{\tilde{p}}{(1-\tilde{p})} \left[ \alpha t + \frac{2}{t} \right]_{1/2}^{\tilde{p}} \\
		&=& \frac{\tilde{p}}{(1-\tilde{p})} \left[ \alpha \tilde{p} + \frac{2}{\tilde{p}} - \frac{\alpha}{2} - 4 \right],
	\end{eqnarray*}
	we get
	\begin{eqnarray*}
		b(\tilde{p}) &\geq& \left[ \frac{\alpha \tilde{p}}{(1-\tilde{p})} - \frac{2}{\tilde{p}(1-\tilde{p})} \right] + \frac{1}{\tilde{p}(1-\tilde{p})} \frac{\tilde{p}}{(1-\tilde{p})} \left[ \alpha \tilde{p} + \frac{2}{\tilde{p}} - \frac{\alpha}{2} - 4 \right], \quad \tilde{p} \in [1/2,1) \\
		&=& \frac{\alpha \tilde{p}}{(1-\tilde{p})} + \frac{\alpha \tilde{p}}{(1-\tilde{p})^2} - \frac{\alpha}{2(1-\tilde{p})^2} - \frac{2}{(1-\tilde{p})^2}, \quad \tilde{p} \in [1/2,1).
	\end{eqnarray*}
\end{proof}

\begin{proof} \textbf{(Corollary \ref{specialcoro})}
	We have to show that $\tilde{\psi}_{\ell}^*$ will satisfy \eqref{maineq} with $\alpha = \beta$, for all $\beta$-mixable proper loss functions. Since
	\begin{equation}
	\label{explinkdesign}
	\frac{(\tilde{\psi}_{\ell}^*)''(\tilde{p})}{(\tilde{\psi}_{\ell}^*)'(\tilde{p})} = \frac{\frac{w_{\ell}'(\tilde{p}) w_{\ell^{\mathrm{log}}}(\tilde{p}) - w_{\ell}(\tilde{p}) w_{\ell^{\mathrm{log}}}'(\tilde{p})}{{w_{\ell^{\mathrm{log}}}(\tilde{p})}^2}}{\frac{w_{\ell}(\tilde{p})}{w_{\ell^{\mathrm{log}}}(\tilde{p})}} = \frac{w_{\ell}'(\tilde{p})}{w_{\ell}(\tilde{p})} - \frac{w_{\ell^{\mathrm{log}}}'(\tilde{p})}{w_{\ell^{\mathrm{log}}}(\tilde{p})} = \frac{w_{\ell}'(\tilde{p})}{w_{\ell}(\tilde{p})} - (\log{w_{\ell^{\mathrm{log}}}(\tilde{p})})' ,
	\end{equation}
	by substituting ${\tilde{\psi}} = \tilde{\psi}_{\ell}^*$ and $\alpha = \beta$ in \eqref{maineq} we have,
	\begin{eqnarray*}
		&&-\frac{1}{\tilde{p}} + \beta w_{\ell}(\tilde{p}) \tilde{p} \enspace \leq \enspace \frac{w_{\ell}'(\tilde{p})}{w_{\ell}(\tilde{p})} - \frac{(\tilde{\psi}_{\ell}^*)''(\tilde{p})}{(\tilde{\psi}_{\ell}^*)'(\tilde{p})} \enspace \leq \enspace \frac{1}{1-\tilde{p}} - \beta w_{\ell}(\tilde{p}) (1-\tilde{p}), \quad \forall \tilde{p} \in (0,1) \\
		&\iff&-\frac{1}{\tilde{p}} + \beta w_{\ell}(\tilde{p}) \tilde{p} \enspace \leq \enspace (\log{w_{\ell^{\mathrm{log}}}(\tilde{p})})' \enspace \leq \enspace \frac{1}{1-\tilde{p}} - \beta w_{\ell}(\tilde{p}) (1-\tilde{p}), \quad \forall \tilde{p} \in (0,1) \\
		&\iff&-\frac{1}{\tilde{p}} + \beta w_{\ell}(\tilde{p}) \tilde{p} \enspace \leq \enspace -\frac{1}{\tilde{p}} + \frac{1}{1-\tilde{p}} \enspace \leq \enspace \frac{1}{1-\tilde{p}} - \beta w_{\ell}(\tilde{p}) (1-\tilde{p}), \quad \forall \tilde{p} \in (0,1) \\
		&\iff&\beta \enspace \leq \enspace \frac{1}{\tilde{p}(1-\tilde{p})w_{\ell}(\tilde{p})} \enspace = \enspace \frac{w_{\ell^{\mathrm{log}}}(\tilde{p})}{w_{\ell}(\tilde{p})}, \quad \forall \tilde{p} \in (0,1),
	\end{eqnarray*}
	which is true for all $\beta$-mixable binary proper loss functions. From \eqref{explinkdesign}
	\begin{align*}
	\left( \log{(\tilde{\psi}_{\ell}^*)'(\tilde{p})} \right)' &= (\log{w_{\ell}(\tilde{p})})' - (\log{w_{\ell^{\mathrm{log}}}(\tilde{p})})' = \left( \log{\frac{w_{\ell}(\tilde{p})}{w_{\ell^{\mathrm{log}}}(\tilde{p})}} \right)', \\
	\Rightarrow \left[ \log{(\tilde{\psi}_{\ell}^*)'(\tilde{p})} \right]_{1/2}^{\tilde{p}} &= \left[ \log{\frac{w_{\ell}(\tilde{p})}{w_{\ell^{\mathrm{log}}}(\tilde{p})}} \right]_{1/2}^{\tilde{p}}, \\
	\Rightarrow \log{(\tilde{\psi}_{\ell}^*)'(\tilde{p})} - \log{(\tilde{\psi}_{\ell}^*)'\left(\frac{1}{2}\right)} &= \log{\frac{w_{\ell}(\tilde{p})}{w_{\ell^{\mathrm{log}}}(\tilde{p})} \cdot \frac{w_{\ell^{\mathrm{log}}}(\frac{1}{2})}{w_{\ell}(\frac{1}{2})}},
	\end{align*}
	it can be seen that a design choice of $(\tilde{\psi}_{\ell}^*)'\left(\frac{1}{2}\right) = 1$ is made in the construction of this link function.
\end{proof}

\subsection{Squared Loss}
\label{sec:square_loss}

\begin{figure}[htbp!]
	\centering
	\includegraphics[width=7cm]{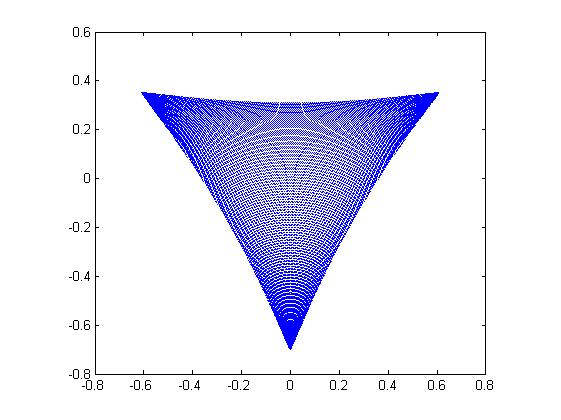}\\
	\caption[Projection of the exp-prediction set of square loss ($\beta=1$) along the $\vone_3$ direction.]{Projection of the exp-prediction set of square loss ($\beta=1$) along the $\vone_3$ direction.
		By the apparent lack of convexity of the projection, the condition $\partial_{\vone_n} \mathcal{B}_{\beta} \subseteq E_\beta (\ell(\mathcal{V}))$ in Proposition \ref{geoprop} does not hold in this case. \label{fig:proj_square}}
\end{figure}

In this section we will consider the multi-class squared loss with partial losses given by $\ell_{i}^{\mathrm{sq}}(p):=\sum_{j \in [n]}(\llbracket i=j \rrbracket - p_j)^2$. The Bayes risk of this loss is $\Lubartil_{\ell^{\mathrm{sq}}}(\tilde{p}) = 1 - \sum_{i=1}^{n-1}{p_i^2} - (1 - \sum_{i=1}^{n-1}{p_i})^2$. Thus the Hessian of the Bayes risk is given by
\begin{equation*}
\textnormal{\textsf{H}}\Lubartil_{\ell^{\mathrm{sq}}}(\tilde{p}) = 2
\begin{pmatrix}
-2 & -1 & \cdots & -1 \\
-1 & -2 & \cdots & -1 \\
\vdots  & \vdots  & \ddots & \vdots  \\
-1 & -1 & \cdots & -2
\end{pmatrix}.
\end{equation*}
For the identity link, from \eqref{kpeq} we get $k_{\mathrm{id}}(\tilde{p})=-\textnormal{\textsf{H}}\Lubartil_{\ell^{\mathrm{sq}}}(\tilde{p})$ and $\textnormal{\textsf{D}}_v[k_{\mathrm{id}}(\tilde{p})]=0$ since $\textnormal{\textsf{D}}\tilde{\psi}(\tilde{p})=I_{n-1}$. Thus from \eqref{multiexpcondition}, the multi-class squared loss is $\alpha$-exp-concave (with $\alpha > 0$) if and only if for all $\tilde{p} \in \mathring{\tilde{\Delta}}^n$ and for all $i \in [n]$
\begin{align}
0 &\preccurlyeq k_{\mathrm{id}}(\tilde{p}) - \alpha k_{\mathrm{id}}(\tilde{p}) \cdot (e_i^{n-1}-\tilde{p}) \cdot (e_i^{n-1}-\tilde{p})' \cdot k_{\mathrm{id}}(\tilde{p}) \nonumber \\
\iff k_{\mathrm{id}}(\tilde{p})^{-1} &\succcurlyeq \alpha (e_i^{n-1}-\tilde{p}) \cdot (e_i^{n-1}-\tilde{p})'. \label{idcancond1}
\end{align}
Similarly for the canonical link, from \eqref{canlinkcond1} and \eqref{canlinkcond2}, the composite loss is $\alpha$-exp-concave (with $\alpha > 0$) if and only if for all $\tilde{p} \in \mathring{\tilde{\Delta}}^n$ and for all $i \in [n]$
\begin{equation}
k_{\mathrm{id}}(\tilde{p})^{-1} = -[\textnormal{\textsf{H}}\Lubartil_{\ell^{\mathrm{sq}}}(\tilde{p})]^{-1} \succcurlyeq \alpha (e_i^{n-1}-\tilde{p}) \cdot (e_i^{n-1}-\tilde{p})'. \label{idcancond2}
\end{equation}
From \eqref{idcancond1} and \eqref{idcancond2}, it can be seen that for the multi-class squared loss the level of exp-concavification by identity link and canonical link are same. When $n=2$, since $k_{\mathrm{id}}(\tilde{p})=4$, the condition \eqref{idcancond1} is equivalent to
\begin{eqnarray*}
	&&\frac{1}{4} \geq \alpha (e_i^{n-1}-\tilde{p}) \cdot (e_i^{n-1}-\tilde{p})', \quad i \in [2], \forall{\tilde{p} \in (0,1)} \\
	&\iff&\alpha \leq \frac{1}{4\tilde{p}^2} \quad \text{and} \quad \alpha \leq \frac{1}{4(1-\tilde{p})^2}, \quad \forall{\tilde{p} \in (0,1)} \\
	&\iff&\alpha \leq \frac{1}{4}.
\end{eqnarray*}
When $n=3$, using the fact that a $2 \times 2$ matrix is positive semi-definite if its trace and determinant are both non-negative, it can be easily verified that the condition \eqref{idcancond1} is equivalent to $\alpha \leq \frac{1}{12}$.

For binary squared loss, the link functions constructed by geometric (Proposition \ref{geoprop}) and calculus (Corollary \ref{specialcoro}) approach are:
\begin{equation*}
\tilde{\psi}(\tilde{p}) = e^{-2(1-\tilde{p})^2} - e^{-2\tilde{p}^2} \quad \text{and} \quad \tilde{\psi}_{\ell}^*(\tilde{p}) = \frac{4}{4} \int_{0}^{\tilde{p}}{\frac{w_{\ell^{\mathrm{sq}}}(v)}{w_{\ell^{\mathrm{log}}}(v)}dv} = 4 \left(\frac{\tilde{p}^2}{2} - \frac{\tilde{p}^3}{3}\right),
\end{equation*}
respectively. By applying these link functions we can get 1-exp-concave composite squared loss.

\subsection{Boosting Loss}
\label{sec:boosting_loss}

\begin{figure}
	\centering
	\includegraphics[width=0.50\textwidth]{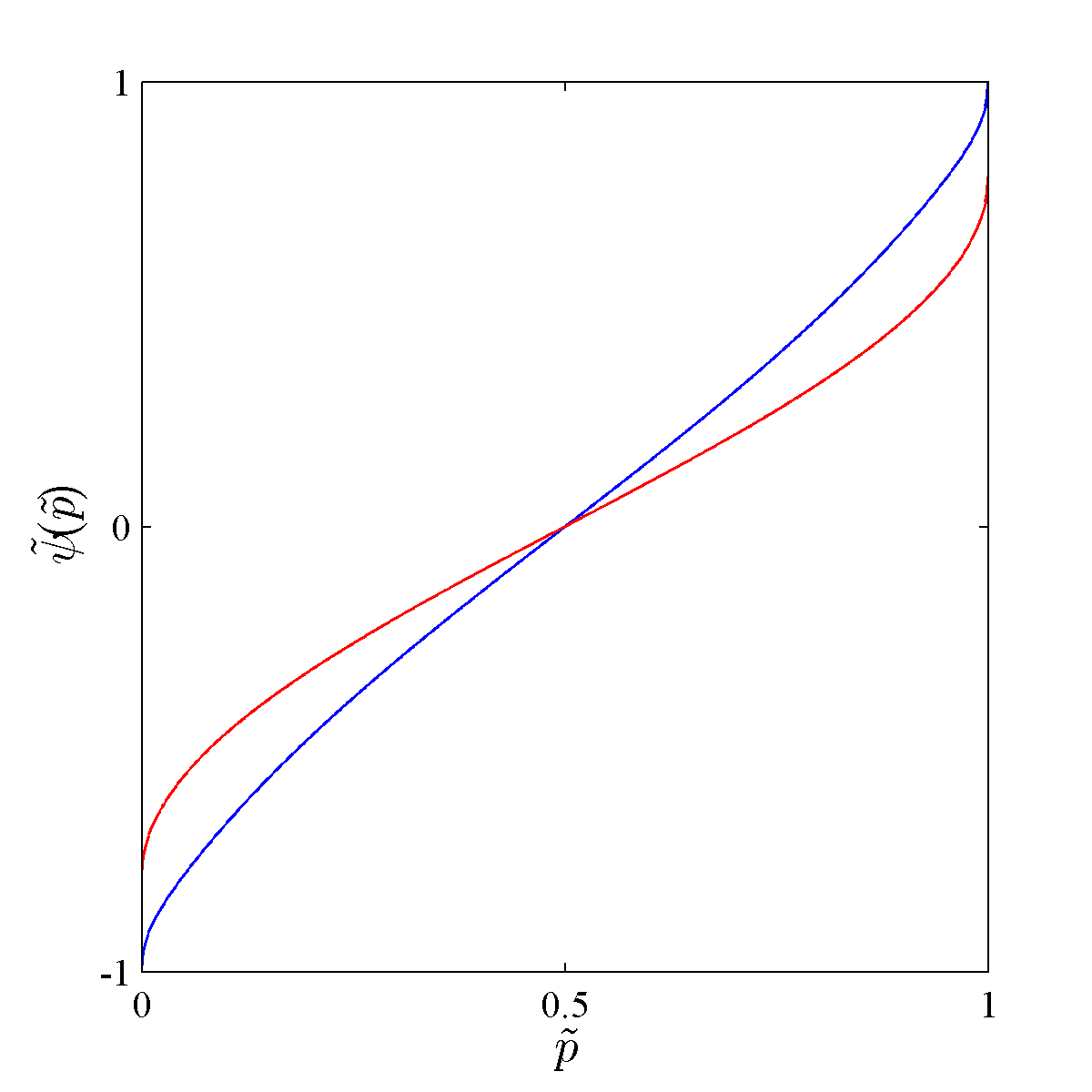}
	\caption[Exp-concavifying link functions for binary boosting loss]{Exp-concavifying link functions for binary boosting loss constructed by Proposition \ref{geoprop} (\textcolor{blue}{---}) and Corollary \ref{specialcoro} (\textcolor{red}{---}). \label{fig:expboost}}
\end{figure}

Consider the binary ``boosting loss'' (\cite{buja2005loss}) with partial losses given by
\begin{equation*}
\ell^{\mathrm{boost}}_1(\tilde{p}) = \frac{1}{2} \sqrt{\frac{1-\tilde{p}}{\tilde{p}}} \quad \text{and} \quad \ell^{\mathrm{boost}}_2(\tilde{p}) = \frac{1}{2} \sqrt{\frac{\tilde{p}}{1-\tilde{p}}}, \quad \forall \tilde{p} \in (0,1).
\end{equation*}
This loss has weight function
\begin{equation*}
w_{\ell^{\mathrm{boost}}}(\tilde{p}) = \frac{1}{4(\tilde{p} (1-\tilde{p}))^{3/2}}, \quad \forall \tilde{p} \in (0,1).
\end{equation*}
By applying the results of \cite{van2012mixability}, we can show that this loss is mixable with mixability constant 2 (since $\beta_{\ell}=\inf_{\tilde{p} \in (0,1)}{\frac{w_{\ell^{\mathrm{log}}}(\tilde{p})}{w_{\ell}(\tilde{p})}}$).

Now we can check the level of exp-concavification of this loss for different choices of link functions. By considering the identity link $\tilde{\psi}(\tilde{p}) = \tilde{p}$, from \eqref{identitynscond}
\begin{align*}
-\frac{1}{\tilde{p}} + \alpha w_{\ell^{\mathrm{boost}}}(\tilde{p}) \tilde{p} \enspace &\leq \enspace \frac{w_{\ell^{\mathrm{boost}}}'(\tilde{p})}{w_{\ell^{\mathrm{boost}}}(\tilde{p})}, \quad \forall \tilde{p} \in (0,1) \\
\Rightarrow -\frac{1}{\tilde{p}} + \alpha w_{\ell^{\mathrm{boost}}}(\tilde{p}) \tilde{p} \enspace &\leq \enspace 6 w_{\ell^{\mathrm{boost}}}(\tilde{p}) \sqrt{\tilde{p} (1-\tilde{p})} (2\tilde{p}-1), \quad \forall \tilde{p} \in (0,1) \\
\Rightarrow \alpha \tilde{p} - 6 \sqrt{\tilde{p} (1-\tilde{p})} (2\tilde{p}-1) \enspace &\leq \enspace \frac{1}{w_{\ell^{\mathrm{boost}}}(\tilde{p}) \tilde{p}}, \quad \forall \tilde{p} \in (0,1) \\
\Rightarrow \alpha \enspace &\leq \enspace 8 \sqrt{\frac{1-\tilde{p}}{\tilde{p}}} (\tilde{p}-1/4), \quad \forall \tilde{p} \in (0,1) \\
\Rightarrow \alpha \enspace &\leq \enspace 0,
\end{align*}
we see that the boosting loss is non-exp-concave. Similarly from \eqref{binarycanonical}
\begin{equation}
\label{boosteq}
\alpha \enspace \leq \enspace \frac{1}{w_{\ell^{\mathrm{boost}}}(\tilde{p}) \tilde{p}^2} \enspace = \enspace 4 \sqrt{\frac{1-\tilde{p}}{\tilde{p}}} (1-\tilde{p}), \quad \forall \tilde{p} \in (0,1)
\end{equation}
it can be seen that the RHS of \eqref{boosteq} approaches $0$ as $p \rightarrow 1$, thus it is not possible to exp-concavify (for some $\alpha > 0$) this loss using the canonical link. For binary boosting loss, the link functions constructed by geometric (Proposition \ref{geoprop}) and calculus (Corollary \ref{specialcoro}) approach are:
\begin{equation*}
\tilde{\psi}(\tilde{p}) = e^{-\sqrt{\frac{1-\tilde{p}}{\tilde{p}}}} - e^{-\sqrt{\frac{\tilde{p}}{1-\tilde{p}}}} \quad \text{and} \quad \tilde{\psi}_{\ell}^*(\tilde{p}) = \frac{4}{2} \int_{0}^{\tilde{p}}{\frac{w_{\ell^{\mathrm{boost}}}(v)}{w_{\ell^{\mathrm{log}}}(v)}dv} = \frac{1}{2} \arcsin(-1+2 \tilde{p}),
\end{equation*}
respectively (as shown in Figure~\ref{fig:expboost}). By applying these link functions we can get 2-exp-concave composite boosting loss.

\subsection{Log Loss}
\label{sec:log_loss}

By using the results from this paper and \cite{van2012mixability} one can easily verify that the multi-class log loss is both 1-mixable and 1-exp-concave.  For binary log loss, the link functions constructed by geometric (Proposition \ref{geoprop}) and calculus (Corollary \ref{specialcoro}) approach are:
\begin{equation*}
\tilde{\psi}(\tilde{p}) = e^{\log{\tilde{p}}} - e^{\log{1-\tilde{p}}} = 2 \tilde{p} - 1 \quad \text{and} \quad \tilde{\psi}_{\ell}^*(\tilde{p}) = \frac{4}{4} \int_{0}^{\tilde{p}}{\frac{w_{\ell^{\mathrm{log}}}(v)}{w_{\ell^{\mathrm{log}}}(v)}dv} = \tilde{p},
\end{equation*}
respectively.

\chapter{Accelerating Optimization for Easy Data}
\label{cha:oco}

The Online Convex Optimization (OCO) problem plays a key role in machine learning as it has interesting theoretical implications and important practical applications especially in the large scale setting where computational efficiency is the main concern. \citep{shalev2011online} provides a detailed analysis of the OCO problem setting and discusses several applications of this paradigm - online regression, prediction with expert advice, and online ranking.

Given a convex set $\Omega \subseteq \mathbb{R}^n$ and a set $\mathcal{F}$ of convex functions, the OCO problem can be formulated as a repeated game between a learner and an adversary. At each time step $t \in \sbr{T}$, the learner chooses a point $x_t \in \Omega$, then the adversary reveals the loss function $f_t \in \mathcal{F}$, and the learner suffers a loss of $f_t \del{x_t}$. The learner's goal is to minimize the regret (w.r.t.\ any $x^* \in \Omega$) which is given by
\begin{equation*}
R \del{\bc{f_t}_{t=1}^T, x^*} := \sum_{t=1}^{T}{f_t \del{x_t}} - \sum_{t=1}^{T}{f_t \del{x^*}}.
\end{equation*}
For example, consider the online linear regression problem (\citep{shalev2011online} (Example 2.1)). At each time step $t$, the learner receives a feature vector $x_t \in \bR^d$, and predicts $p_t \in \bR$. Then the adversary reveals the true value $y_t \in \bR$, and the learner pays the loss $\abs{y_t - p_t}$. When the learner's prediction is of the form $p_t = \tuple{w_t,x_t}$, and she needs to compete with the set of linear predictors, this problem can be cast in the OCO framework by setting $f_t \br{w_t} = \abs{\tuple{w_t,x_t} - y_t}$.

\cite{abernethy2008optimal} analyzed the OCO problem from a minimax perspective (where each player plays optimally for their benefit), and showed that $R \del{\bc{f_t}_{t=1}^T, x^*} \approx \Omega\br{\sqrt{T}}$ for arbitrary sequence of convex losses $\bc{f_t}_{t=1}^T$, and for any strategy of the learner. But the adversary choosing $f_t$ need not to be malicious always, for example the $f_t$ might be drawn from a distribution. 

There are two main classes of update rules which attain the above minimax regret bound $O\br{\sqrt{T}}$ (thus called minimax optimal updates), namely Follow The Regularized Leader (FTRL) and Mirror Descent. In this work we consider the latter class. Given a strongly convex function (formally defined later) $\psi$ and a learning rate $\eta > 0$, standard mirror descent update is given by
\begin{equation}
\label{minimx-omd}
x_{t+1}=\argmin_{x \in \Omega}{\eta \tuple{g_t,x_t}+\bregmanadaregpsi{x}{x_t}},
\end{equation}
where $g_t \in \partial f_t \del{x_t}$ and $\bregmanadaregpsi{\cdot}{\cdot}$ is Bregman divergence (formally defined later). \cite{shalev2011online} provides a comprehensive survey of analysis techniques for this non-adaptive algorithm family, where the learning rate is fixed for all rounds and chosen with knowledge of $T$. 

\paragraph*{Easy Data Instances:}
It is well understood that the minimax optimal algorithms achieve a regret bound of $O\del{\sqrt{T}}$, which cannot be improved for arbitrary sequences of convex losses \citep{zinkevich2003online}. But in practice there are several \textit{easy data} instances such as sparsity, predictable sequences and curved losses, in which much tighter regret bounds are achievable. These tighter bounds translate to much better performance in practice, especially for high dimensional but sparse problems (\cite{mcmahan2014analysis}). Even though minimax analysis gives robust algorithms, they are overly conservative on \textit{easy data}. Now we consider some of the existing algorithms that automatically adapt to the \textit{easy data} to learn faster while being robust to worst case as well.

\citep{duchi2011adaptive} replaced the single static regularizer $\psi$ in the standard mirror descent update~\ref{minimx-omd} by a data dependent sequence of regularizers. This is a fully adaptive approach as it doesn't require any prior knowledge about the bound on the term given by $\sum_{t=1}^{T}{\norm{g_t}^2}$ to construct the regularizers. Further for a particular choice of regularizer sequence they achieved a regret bound of the form 
\[ 
R \del{\bc{f_t}_{t=1}^T, x^*} = O\del*{\max\limits_t{\norm{x_t-x^*}_{\infty}} \sum_{i=1}^{n}{\sqrt{\sum_{t=1}^{T}{g_{t,i}^2}}}} , 
\] 
which is better than the minimax optimal bound ($O\br{G \sqrt{T}}$, where $G$ is the worst case magnitude of gradients) when the gradients of the losses are sparse and the prediction space is box-shaped.

\citep{chiang2012online,rakhlin2012online} have shown that an optimistic prediction $\tilde{g}_{t+1}$ of the next gradient $g_{t+1}$ at time $t$ can be used to achieve tighter regret bounds in the case where the loss functions are generated by some predictable process e.g. i.i.d losses with small variance and slowly changing gradients. For the general convex losses, the regret bound of this optimistic approach is $O \del*{\sqrt{\sum_{t=1}^{T}{\norm{g_t - \tilde{g}_t}_*^2}}}$. But this is a non-adaptive approach since one requires knowledge of the upper bound on $\sum_{t=1}^{T}{\norm{g_t - \tilde{g}_t}_*^2}$ to set the optimal value for the learning rate. Instead we can employ the standard doubling trick to obtain similar bound with slightly worst constants.

Online optimization with curved losses (strong-convex, exp-concave, mixable etc.) is easier than linear losses. When the loss functions are uniformly exp-concave or strongly convex, $O\del{\log{T}}$ regret bounds are achieved with appropriate choice of regularizers \citep{hazan2007logarithmic,hazan2007adaptive}. But this bound will become worse when the uniform lower bound on the convexity parameters is much smaller. In that case \citep{hazan2007adaptive} proposed an algorithm that can adapt to the convexity of the loss functions, and achieves $O\del{\sqrt{T}}$ regret bounds for arbitrary convex losses and $O\del{\log{T}}$ for uniformly strong-convex losses.

\paragraph*{Chapter Outline:}
Even though \citep{mcmahan2014analysis} has shown equivalence between mirror descent and a variant of FTRL (namely FTRL-Prox) algorithms with adaptive regularizers, no such mapping is available between optimistic mirror descent and optimistic FTRL updates. Recently \citep{2015arXiv150905760M} have combined adaptive FTRL and optimistic FTRL updates to achieve tighter regret bounds for sparse and predictable sequences. In section~\ref{improve-sparsity} we extend this unification to obtain adaptive and optimistic mirror descent updates. We obtained a factor of $\sqrt{2}$ improvement in the regret bound compared to that of \citep{2015arXiv150905760M}, because in their regret analysis they could not apply the strong FTRL lemma from \citep{mcmahan2014analysis}.

In section~\ref{improve-curvature} we consider the adaptive and optimistic mirror descent update with strongly convex loss functions. In this case we achieve tighter logarithmic regret bound without a priori knowledge about the lower bound on the strong-convexity parameters, in similar spirit of \citep{hazan2007adaptive}. We also present a curvature adaptive optimistic algorithm that interpolates the results for general convex losses and strongly-convex losses.

In practice the original convex optimization problem itself can have a regularization term associated with the constraints of the problem and generally it is not preferable to linearize those (possibly non-smooth) regularization terms. In section~\ref{extensions} we extend all our results to such composite objectives as well.

The main contributions of this chapter are:
\begin{itemize}
	\item An adaptive and optimistic mirror descent update that achieves tighter regret bounds for sparse and predictable sequences (Section~\ref{improve-sparsity}).
	\item Improved optimistic mirror descent algorithm that adapts to the curvature of the loss functions (Section~\ref{improve-curvature}).
	\item Extension of the unified update rules to the composite objectives (Section~\ref{extensions}).   
\end{itemize}

Omitted proofs are given in section~\ref{sec:proof-oco}. 

\section{Notation and Background}
\label{background}
We use the following notation throughout. For $n \in \mathbb{Z}^{+}$, let $[n]:=\{1,...,n\}$. The $i$th element of a vector $x \in \mathbb{R}^n$ is denoted by $x_i \in \mathbb{R}$, and for a time dependent vector $x_t \in \mathbb{R}^n$, the $i$th element is $x_{t,i} \in \mathbb{R}$. The inner product between two vectors $x,y \in \mathbb{R}^n$ is written as $\tuple{x,y}$. The gradient of a differentiable function $f$ at $x \in \mathbb{R}^n$ is denoted by $\nabla f \del{x}$ or $f' \del{x}$. A superscript $T$, $A^T$ denotes transpose of the matrix or vector $A$. Given $x \in \mathbb{R}^n$, $A=\mathrm{diag}\del{x}$ is the $n \times n$ matrix with entries $A_{ii}=x_i$ , $i \in [n]$ and $A_{ij}=0$ for $i \neq j$. Similarly given $B \in \mathbb{R}^{n \times n}$, $A=\mathrm{diag}\del{B}$ is the $n \times n$ matrix with entries $A_{ii}=B_{ii}$ , $i \in [n]$ and $A_{ij}=0$ for $i \neq j$. For a symmetric positive definite matrix $A \in S_{++}^n$, we have that $\forall{x \neq 0}, x^T A x > 0$. If $A-B \in S_{++}^n$, then we write $A \succ B$. The square root of $A \in S_{++}^n$ is the unique matrix $X \in S_{++}^n$ such that $XX=A$ and it is denoted as $A^{\half}$. We use the compressed summation notation $H_{a:b}$ as shorthand for $\sum_{s=a}^{b}{H_s}$, where $H_s$ can be a scalar, vector, matrix, or function. Given a norm $\norm{\cdot}$, its dual norm is defined as follows $\norm{y}_* := \sup\limits_{x: \norm{x} \leq 1}{\tuple{x,y}}$. For a time varying norm $\norm{\cdot}_{\del{t}}$, its dual norm is written as $\norm{\cdot}_{\del{t},*}$. The dual norm of the Mahalanobis norm $\norm{x}_A := \sqrt{x^T A x}$ is given by  $\norm{y}_{A^{-1}} = \sqrt{y^T A^{-1} y}$.

Given a convex set $\Omega \subseteq \mathbb{R}^n$ and a convex function $f:\Omega \rightarrow \mathbb{R}$, $\partial f \del{x}$ denotes the sub-differential of $f$ at $x$ which is defined as $\partial f \del{x}:=\setdel{g:f\del{y} \geq f\del{x} + \tuple{g,y-x}, \, \forall{y \in \Omega}}$. A function $f: \Omega \rightarrow \mathbb{R}$ is $\alpha$-strongly convex with respect to a general norm $\norm{\cdot}$ if for all $x,y \in \Omega$
\begin{equation*}
f \del{x} ~\geq~ f \del{y} + \tuple{g,x-y} + \frac{\alpha}{2} \norm{x-y}^2, \, g \in \partial f \del{y}.
\end{equation*}
The Bregman divergence with respect to a differentiable function $g$ is defined as follows
\begin{equation*}
\bregmangunc{x}{y} ~:=~ g \del{x} - g \del{y} - \tuple{\nabla g \del{y},x-y}.
\end{equation*}
Observe that the function $g$ is $\alpha$-strongly convex with respect to $\norm{\cdot}$ if and only if for all $x,y \in \Omega$ : $\bregmangunc{x}{y} \geq \frac{\alpha}{2} \norm{x-y}^2$. In this chapter we use the following properties of Bregman divergences
\begin{itemize}
	\item Linearity: $\mathcal{B}_{\alpha \psi + \beta \phi}\del{x,y} = \alpha \mathcal{B}_{\psi}\del{x,y} + \beta \mathcal{B}_{\phi}\del{x,y}$.
	\item Generalized triangle inequality: $\bregmanadaregpsi{x}{y} + \bregmanadaregpsi{y}{z} = \bregmanadaregpsi{x}{z} + \tuple{x-y,\nabla \psi \del{z} - \nabla \psi \del{y}}$.
\end{itemize}

The following proposition \citep{srebro2011universality,beck2003mirror} is handy in deriving explicit update rules for mirror descent algorithms.
\begin{proposition}
	\label{simple-form-omd}
	Suppose $\psi$ is strictly convex and differentiable, and $y$ satisfies the condition $\nabla \psi \del{y} = \nabla \psi \del{u} - g$. Then
	\begin{equation*}
	\minimize{\setdel*{\tuple{g,x} + \bregmanadaregpsi{x}{u}}} = \minimize{\bregmanadaregpsi{x}{y}}.
	\end{equation*}
\end{proposition}

\section{Adaptive and Optimistic Mirror Descent}
\label{improve-sparsity}
\begin{algorithm}[tb]
	\caption{Adaptive and Optimistic Mirror Descent}
	\label{algo:ada-opt-omd}
	\begin{algorithmic}
		\STATE {\bfseries Input:} regularizers $r_0, r_1 \geq 0$, scheme for selecting $r_t$ for $t \geq 2$.
		\STATE {\bfseries Initialize:} $x_1, \hat{x}_1 = 0 \in \Omega$.
		\FOR{$t=1$ {\bfseries to} $T$}
		\STATE Predict $\hat{x}_t$, observe $f_t$, and incur loss $f_t \del{\hat{x}_t}$. \STATE Compute $g_t \in \partial f_t \del{\hat{x}_t}$ and $\tilde{g}_{t+1} \del{g_1,...,g_t}$.
		\STATE Choose $r_{t+1}$ s.t. $r_{0:t+1}$ is $1$-strongly convex w.r.t.\ $\norm{\cdot}_{\del{t+1}}$.
		\STATE Update
		\begin{align}
		x_{t+1}
		~=~&
		\minimize{\tuple{g_t,x} + \bregmanadareg{x}{x_t}},
		\label{ada-omd-opt-eq1}
		\\
		\hat{x}_{t+1}
		~=~&
		\minimize{\tuple{\tilde{g}_{t+1},x} + \bregmanadaregcumalative{x}{x_{t+1}}{t+1}}.
		\label{ada-omd-opt-eq2}
		\end{align}
		\ENDFOR
	\end{algorithmic}
\end{algorithm}

When the sequence of losses $f_t$'s (in fact their sub-gradients $g_t$'s) are predictable, many authors have recently considered variance (regret) bounds (\cite{hazan2010extracting}) that depend only on the deviation of $g_t$ from its average, or path length (regret) bounds (\cite{chiang2012online}) in terms of $g_t - g_{t-1}$. \cite{rakhlin2012online} present an optimistic learning framework that yields such bounds for any mirror descent algorithm. In this framework, the learner is given a sequence of `hints' $\tilde{g}_{t+1} \del{g_1,...,g_t}$ of what $g_{t+1}$ might be. Then the learner chooses $x_{t+1}$ based on the optimistically predicted sub-gradient $\tilde{g}_{t+1}$ along with already observed sub-gradients $g_1,...,g_t$. For the optimistic sub-gradient prediction choices of $\tilde{g}_{t+1} = \frac{1}{t} \sum_{s=1}^{t}{g_s}$ (reasonable prediction when the adversary is iid) and $\tilde{g}_{t+1}=g_t$ (reasonable prediction for slow varying data), we obtain the variance bound and the path length bound respectively.


Given a $1$-strongly convex function $\psi$, and a learning rate $\eta > 0$, the optimistic mirror descent update is equivalent to the following two stage updates
\begin{align*}
x_{t+1} ~=~& \minimize{\eta\tuple{g_t,x}+\bregmanadaregpsi{x}{x_t}} \\
\hat{x}_{t+1} ~=~& \minimize{\eta\tuple{\tilde{g}_{t+1},x}+\bregmanadaregpsi{x}{x_{t+1}}}.
\end{align*}
Adaptive and Optimistic mirror descent update is obtained by replacing the static regularizer $\psi$ by a sequence of data dependent regularizers $r_t$'s, which are chosen such that $r_{0:t}$ is $1$-strongly convex with respect to $\norm{\cdot}_{\del{t}}$ (here we use the compressed summation notation $r_{0:t}\br{x} = \sum_{s=0}^{t}{r_s\br{x}}$). The unified update is given in Algorithm~\ref{algo:ada-opt-omd}. Note that the regularizer $r_{t+1}$ is constructed at time $t$ (based on the data observed only up to time $t$) and is used in the second stage update \eqref{ada-omd-opt-eq2}. Also observe that by setting $\tilde{g}_t=0$ for all $t$ in Algorithm~\ref{algo:ada-opt-omd} we recover a slightly modified adaptive mirror descent update given by $x_{t+1} = \minimize{\tuple{g_t,x}+\bregmanadareg{x}{x_t}}$, where $r_t$ can depend only on $g_1,...,g_{t-1}$.

In order to obtain a regret bound for Algorithm~\ref{algo:ada-opt-omd}, we first consider the \textit{instantaneous linear regret} (w.r.t.\ any $x^* \in \Omega$) of it given by $\tuple{\hat{x}_t - x^* , g_t}$. The following lemma is a generalization of Lemma 5 from \citep{chiang2012online} for time varying norms, which gives a bound on the instantaneous linear regret of Algorithm~\ref{algo:ada-opt-omd}. 
\begin{lemma}
	\label{ada-opt-omd-linear-regret-lemma}
	The instantaneous linear regret of Algorithm~\ref{algo:ada-opt-omd} w.r.t.\ any $x^* \in \Omega$ is bounded from above as follows
	\[
	\tuple{\hat{x}_t - x^* , g_t} ~\leq~ \bregmanadareg{x^*}{x_t} - \bregmanadareg{x^*}{x_{t+1}} + \half \norm{g_t-\tilde{g}_t}_{\del{t},*}^2.		
	\]
\end{lemma}
\begin{proof}
	Consider
	\begin{equation}
	\label{ada-opt-omd-linear-regret-lemma-proof1}
	\tuple{g_t,\hat{x}_t-x^*} 
	=
	\tuple{g_t - \tilde{g}_{t},\hat{x}_t-x_{t+1}} + \tuple{\tilde{g}_{t},\hat{x}_t-x_{t+1}}
	+ \tuple{g_t,x_{t+1}-x^*}.
	\end{equation}
	By the fact that $\tuple{a,b} \leq \norm{a} \norm{b}_* \leq \half \norm{a}^2 + \half \norm{b}_*^2$, we have 
	\begin{equation*}
	\tuple{g_t - \tilde{g}_{t},\hat{x}_t-x_{t+1}} ~\leq~ \half \norm{\hat{x}_t-x_{t+1}}_{\del{t}}^2 + \half \norm{g_t - \tilde{g}_{t}}_{\del{t},*}^2.
	\end{equation*}
	The first-order optimality condition \citep{boyd2004convex} for 
	\[
	x^* = \minimize{\tuple{g,x} + \bregmanadaregpsi{x}{y}}
	\] 
	is given by
	\begin{equation*}
	\tuple{x^* - z , g} ~\leq~ \bregmanadaregpsi{z}{y} - \bregmanadaregpsi{z}{x^*} - \bregmanadaregpsi{x^*}{y}, \forall{z \in \Omega}.
	\end{equation*}
	By applying the above condition for \eqref{ada-omd-opt-eq2} and \eqref{ada-omd-opt-eq1} we have respectively 
	\begin{align*}
	\tuple{\hat{x}_t-x_{t+1} , \tilde{g}_{t}} 
	~\leq~& \bregmanadareg{x_{t+1}}{x_{t}} - \bregmanadareg{x_{t+1}}{\hat{x}_t} 
	- \bregmanadareg{\hat{x}_t}{x_{t}},
	\\
	\tuple{x_{t+1} - x^* , g_t} 
	~\leq~& \bregmanadareg{x^*}{x_t} - \bregmanadareg{x^*}{x_{t+1}} 
	- \bregmanadareg{x_{t+1}}{x_t}.
	\end{align*}
	Thus by \eqref{ada-opt-omd-linear-regret-lemma-proof1} we have
	\begin{align*}
	&
	\tuple{g_t,\hat{x}_t-x^*}
	\\
	~\leq~&
	\half \norm{\hat{x}_t-x_{t+1}}_{\del{t}}^2 + \half \norm{g_t - \tilde{g}_{t}}_{\del{t},*}^2 
	+
	\bregmanadareg{x^*}{x_t} - \bregmanadareg{x^*}{x_{t+1}} - \bregmanadareg{x_{t+1}}{\hat{x}_t}
	\\
	~\leq~&
	\half \norm{g_t - \tilde{g}_{t}}_{\del{t},*}^2 + \bregmanadareg{x^*}{x_t} - \bregmanadareg{x^*}{x_{t+1}}
	\end{align*}
	where the second inequality is due to 1-strong convexity of $r_{0:t}$ w.r.t.\ $\norm{\cdot}_{\del{t}}$.
\end{proof}

The following lemma is already proven by \citep{chiang2012online} and used in the proof of our Theorem~\ref{ada-opt-omd-strong-cvx-theorem}.
\begin{lemma}
	\label{point-loss-connection}
	For Algorithm~\ref{algo:ada-opt-omd} we have, $\norm{\hat{x}_t - x_{t+1}}_{\del{t}} ~\leq~ \norm{g_t - \tilde{g}_t}_{\del{t},*}$.
\end{lemma}

The following regret bound holds for Algorithm~\ref{algo:ada-opt-omd} with a sequence of general convex functions $f_t$'s:
\begin{theorem}
	\label{ada-omd-optimistic-thm}
	The regret of Algorithm~\ref{algo:ada-opt-omd} w.r.t.\ any $x^* \in \Omega$ is bounded by 
	\begin{equation*}
	\sum_{t=1}^{T}{f_t \del{\hat{x}_t} - f_t \del{x^*}} ~\leq~ \half \sum_{t=1}^{T}{\norm{g_t - \tilde{g}_{t}}_{\del{t},*}^2} 
	+ \sum_{t=1}^{T}{\bregmanadaregsingle{x^*}{x_{t}}{t}} + \bregmanadaregsingle{x^*}{x_1}{0} - \bregmanadaregcumalative{x^*}{x_{T+1}}{T}.
	\end{equation*}
\end{theorem}
\begin{proof} 
	Consider
	\begin{align*}
	&
	\sum_{t=1}^{T}{f_t \del{\hat{x}_t} - f_t \del{x^*}}
	\\
	~\leq~&
	\sum_{t=1}^{T}{\tuple{g_t,\hat{x}_t-x^*}}
	\\
	~\leq~&
	\sum_{t=1}^{T}{\half \norm{g_t - \tilde{g}_t}_{\del{t},*}^2 + \bregmanadareg{x^*}{x_t} - \bregmanadareg{x^*}{x_{t+1}}},
	\end{align*}
	where the first inequality is due to the convexity of $f_t$ and the second one is due to Lemma~\ref{ada-opt-omd-linear-regret-lemma}. Then the following simplification of the sum of Bregman divergence terms completes the proof.
	\begin{align*}
	& 
	\sum_{t=1}^{T}{\bregmanadareg{x^*}{x_t} - \bregmanadareg{x^*}{x_{t+1}}}
	\\
	~=~&
	\bregmanadaregsingle{x^*}{x_1}{0} - \bregmanadaregcumalative{x^*}{x_{T+1}}{T} + \sum_{t=1}^{T}{\bregmanadaregsingle{x^*}{x_{t}}{t}} 
	\end{align*}
\end{proof}

Now we analyse the performance of Algorithm~\ref{algo:ada-opt-omd} with specific choices of regularizer sequences. First we recover the non-adaptive optimistic mirror descent \citep{chiang2012online} and its regret bound as a corollary of Theorem~\ref{ada-omd-optimistic-thm}.
\begin{corollary}
	\label{ada-omd-optimistic-cor}
	Given $1$-strongly convex (w.r.t.\ $\norm{\cdot}$) function $\psi$, define $\mathcal{R}_{\text{max}} \del{x^*} := \max_{x \in \Omega}{\bregmanadaregpsi{x^*}{x}} - \min_{x \in \Omega}{\bregmanadaregpsi{x^*}{x}} = \max_{x \in \Omega}{\bregmanadaregpsi{x^*}{x}}$. If $r_t$'s are given by $r_0 \del{x} = \frac{1}{\eta} \psi \del{x}$ (for $\eta > 0$) and $r_t \del{x} = 0, ~\forall{t \geq 1}$, then the regret of Algorithm~\ref{algo:ada-opt-omd} w.r.t.\ any $x^* \in \Omega$ is bounded as follows 
	\begin{equation*}
	\sum_{t=1}^{T}{f_t \del{\hat{x}_t} - f_t \del{x^*}} ~\leq~ \frac{\eta}{2} \sum_{t=1}^{T}{\norm{g_t - \tilde{g}_{t}}_{*}^2} + \frac{1}{\eta} \mathcal{R}_{\text{max}} \del{x^*}.
	\end{equation*}
	Further if $\sum_{t=1}^{T}{\norm{g_t - \tilde{g}_{t}}_{*}^2} \leq Q$, then by choosing $\eta = \sqrt{\frac{2 \mathcal{R}_{\text{max}} \del{x^*}}{Q}}$, we have 
	\begin{equation*}
	\sum_{t=1}^{T}{f_t \del{\hat{x}_t} - f_t \del{x^*}} ~\leq~ \sqrt{2 \mathcal{R}_{\text{max}} \del{x^*} Q}.
	\end{equation*}
\end{corollary}
\begin{proof} 
	For the given choice of regularizers, we have $r_{0:t} \del{x} = \frac{1}{\eta} \psi \del{x}$ and $\bregmanadareg{x}{y} = \frac{1}{\eta} \bregmanadaregpsi{x}{y}$. Since $r_{0:t}$ is $1$-strongly convex w.r.t.\ $\frac{1}{\sqrt{\eta}} \norm{\cdot}$, we have $\norm{\cdot}_{\del{t}} = \frac{1}{\sqrt{\eta}} \norm{\cdot}$ and $\norm{\cdot}_{\del{t},*} = \sqrt{\eta} \norm{\cdot}_{*}$. Then the corollary directly follows from Theorem~\ref{ada-omd-optimistic-thm}.
\end{proof}
In this non-adaptive case we need to know an upper bound of $\sum_{t=1}^{T}{\norm{g_t - \tilde{g}_{t}}_{*}^2}$ in advance to choose the optimal value for $\eta$. Instead we can employ the standard doubling trick to obtain similar bounds with slightly worst constants. 

By leveraging the techniques from \citep{duchi2011adaptive} we can adaptively construct regularizers based on the observed data. The following corollary describes a regularizer construction scheme for Algorithm~\ref{algo:ada-opt-omd} which is fully adaptive and achieves a regret guarantee that holds at anytime.
\begin{corollary}
	\label{diagonal-matrix-regret}
	Given $\Omega \subseteq \times_{i=1}^{n} \sbr{-R_i,R_i}$, let
	\begin{align}
	G_0 &~=~ 0
	\label{ada-mat-g0}
	\\
	G_1 &~=~ \gamma^2 I \,\text{ s.t. }\, \gamma^2 I \succcurlyeq \del{g_t - \tilde{g}_t} \del{g_t - \tilde{g}_t}^T, \, \forall{t} 
	\label{ada-mat-g1}
	\\
	G_t &~=~ \del{g_{t-1} - \tilde{g}_{t-1}} \del{g_{t-1} - \tilde{g}_{t-1}}^T, \, \forall{t \geq 2} 
	\label{ada-mat-gt}
	\\
	Q_{1:t} &~=~ \text{diag} \del*{\frac{1}{R_1},...,\frac{1}{R_n}} \text{diag} \del*{G_{1:t}}^{\half}.
	\nonumber
	\end{align}
	If $r_t$'s are given by $r_0 \del*{x} = 0$ and $r_t \del*{x} = \frac{1}{2 \sqrt{2}} \norm{x}_{Q_t}^2$, then the regret of Algorithm~\ref{algo:ada-opt-omd} w.r.t.\ any $x^* \in \Omega$ is bounded by 
	\begin{equation*}
	\sum_{t=1}^{T}{f_t \del{\hat{x}_t} - f_t \del{x^*}} \leq 2 \sqrt{2} \sum_{i=1}^{n}{R_i \sqrt{\gamma^2 + \sum_{t=1}^{T-1}{\del{g_{t,i} - \tilde{g}_{t,i}}^2}}}.
	\end{equation*}
\end{corollary}
\begin{proof} 
	By letting $\eta = \sqrt{2}$ for the given sequence of regularizers, we get $r_{0:t} \del*{x} = \frac{1}{2 \eta} \norm{x}_{Q_{1:t}}^2$. Since $r_{0:t}$ is $1$-strongly convex w.r.t.\ $\frac{1}{\sqrt{\eta}} \norm{\cdot}_{Q_{1:t}}$, we have $\norm{\cdot}_{\del*{t}} = \frac{1}{\sqrt{\eta}} \norm{\cdot}_{Q_{1:t}}$ and $\norm{\cdot}_{\del*{t},*} = \sqrt{\eta} \norm{\cdot}_{Q_{1:t}^{-1}}$. By using the facts that $\text{diag} \del*{\alpha_1,...,\alpha_n}^{\half} = \text{diag} \del*{\sqrt{\alpha_1},...,\sqrt{\alpha_n}}$ and $\text{diag} \del*{\beta_1,...,\beta_n} \cdot \text{diag} \del*{\gamma_1,...,\gamma_n} = \text{diag} \del*{\beta_1 \gamma_1,...,\beta_n \gamma_n}$, the $(i,i)$-th entry of the diagonal matrix $Q_{1:t}$ can be given as
	\begin{align*}
	\del*{Q_{1:t}}_{ii} =~& \frac{1}{R_i} \sqrt{\text{diag} \del*{\gamma^2 I + \sum_{s=1}^{t-1}{\del{g_s - \tilde{g}_s} \del{g_s - \tilde{g}_s}^T}}_{ii}}
	\\ 
	=~& \frac{1}{R_i} \sqrt{\gamma^2 + \sum_{s=1}^{t-1}{\del{g_{s,i} - \tilde{g}_{s,i}}^2}}.
	\end{align*}
	
	Now by Theorem~\ref{ada-omd-optimistic-thm} the regret bound of Algorithm~\ref{algo:ada-opt-omd} with this choice of regularizer sequence can be given as follows
	\begin{equation*}
	\sum_{t=1}^{T}{f_t \del{\hat{x}_t} - f_t \del{x^*}} ~\leq~ \half \sum_{t=1}^{T}{\norm{g_t - \tilde{g}_{t}}_{\del{t},*}^2} + \sum_{t=1}^{T}{\bregmanadaregsingle{x^*}{x_{t}}{t}}.
	\end{equation*}
	Consider
	\begin{align*}
	& \half \sum_{t=1}^{T}{\norm{g_t - \tilde{g}_{t}}_{\del{t},*}^2} 
	\\
	~=~&
	\half \sum_{t=1}^{T}{\eta \norm{g_t - \tilde{g}_{t}}_{Q_{1:t}^{-1}}^2}
	\\
	~=~&
	\frac{\eta}{2} \sum_{t=1}^{T}{\sum_{i=1}^{n}{\del*{g_{t,i} - \tilde{g}_{t,i}}^2} \del*{Q_{1:t}}_{ii}^{-1}}
	\\
	~=~&
	\frac{\eta}{2} \sum_{t=1}^{T}{\sum_{i=1}^{n}{\del*{g_{t,i} - \tilde{g}_{t,i}}^2} \frac{R_i}{\sqrt{\gamma^2 + \sum_{s=1}^{t-1}{\del{g_{s,i} - \tilde{g}_{s,i}}^2}}}}
	\\
	~\leq~&
	\frac{\eta}{2} \sum_{i=1}^{n}{R_i \sum_{t=1}^{T}{\frac{\del*{g_{t,i} - \tilde{g}_{t,i}}^2}{\sqrt{\sum_{s=1}^{t}{\del{g_{s,i} - \tilde{g}_{s,i}}^2}}}}}
	\\
	~\leq~&
	\eta \sum_{i=1}^{n}{R_i \sqrt{\sum_{t=1}^{T}{\del{g_{t,i} - \tilde{g}_{t,i}}^2}}}
	\\
	~\leq~&
	\eta \sum_{i=1}^{n}{R_i \sqrt{\gamma^2 + \sum_{t=1}^{T-1}{\del{g_{t,i} - \tilde{g}_{t,i}}^2}}},
	\end{align*}
	where the first and third inequalities are due to the fact that $\gamma^2 \geq \del{g_{t,i} - \tilde{g}_{t,i}}^2$ for all $t \in \sbr{T}$, and the second inequality is due to the fact that for any non-negative real numbers $a_1, a_2, ... , a_n$ : $\sum_{i=1}^{n}{\frac{a_i}{\sqrt{\sum_{j=1}^{i}{a_j}}}} \leq 2 \sqrt{\sum_{i=1}^{n}{a_i}}$.
	Also observing that
	\begin{align*}
	& \sum_{t=1}^{T}{\bregmanadaregsingle{x^*}{x_{t}}{t}} 
	\\
	~=~&
	\sum_{t=1}^{T}{\frac{1}{2 \eta} \norm{x^* - x_t}_{Q_t}^2}
	\\
	~=~&
	\frac{1}{2 \eta} \sum_{t=1}^{T}{\sum_{i=1}^{n}{\del*{x_i^* - x_{t,i}}^2 \del*{Q_t}_{ii}}}
	\\
	~\leq~&
	\frac{1}{2 \eta} \sum_{i=1}^{n}{\del*{2 R_i}^2 \sum_{t=1}^{T}{\del*{Q_t}_{ii}}}
	\\
	~=~&
	\frac{2}{\eta} \sum_{i=1}^{n}{R_i^2 \del*{Q_{1:T}}_{ii}}
	\\
	~=~&
	\frac{2}{\eta} \sum_{i=1}^{n}{R_i \sqrt{\gamma^2 + \sum_{t=1}^{T-1}{\del{g_{t,i} - \tilde{g}_{t,i}}^2}}}
	\end{align*}
	completes the proof.
\end{proof}

The regret bound obtained in the above corollary is much tighter than that of \citep{duchi2011adaptive} and \citep{chiang2012online} when the sequence of loss functions are sparse and predictable. Consider an adversary that is benign and sparse (having non-zero components in fixed locations). In this case, the predictor can learn the non-zero locations of the actual gradient after few iterations. Then $\tilde{g}_t$ will also be mostly zero in the locations where $g_t$ is zero.

Since we are using per-coordinate learning rates implicitly we get better bounds for the case where only certain coordinates of the gradients are accurately predictable as well. Even when the loss sequence is completely unpredictable, the above bound is not much worse than a constant factor of the bound in \citep{duchi2011adaptive}. For constructive examples confer \cite[Section~2.2]{mohri2016accelerating}. 

By using Proposition~\ref{simple-form-omd} we can derive explicit forms of the update rules given by \eqref{ada-omd-opt-eq1} and \eqref{ada-omd-opt-eq2} with regularizers constructed in Corollary~\ref{diagonal-matrix-regret}. For $y_{t+1} = x_{t} - \sqrt{2} Q_{1:t}^{-1} g_t$ and $\hat{y}_{t+1} = x_{t+1} - \sqrt{2} Q_{1:t+1}^{-1} \tilde{g}_{t+1}$, the updates \eqref{ada-omd-opt-eq1} and \eqref{ada-omd-opt-eq2} can be given as $x_{t+1} = \minimize \frac{1}{2} \norm{x - y_{t+1}}_{Q_{1:t}}^2$ and $\hat{x}_{t+1} = \minimize \frac{1}{2} \norm{x - \hat{y}_{t+1}}_{Q_{1:t+1}}^2$ respectively.

The next corollary explains a regularizer construction method with full matrix learning rates, which is an extension of Corollary~\ref{diagonal-matrix-regret}. But this approach is computationally not preferable, especially in high dimensions, as it costs $O\del{n^2}$ per round of operations.
\begin{corollary}
	\label{full-matrix-regret}
	Define $D := \maximizepairvalue{\norm{x-y}_2}$. Let $Q_{1:t} = \del*{G_{1:t}}^{\half}$, where $G_t$'s are given by \eqref{ada-mat-g0},\eqref{ada-mat-g1} and \eqref{ada-mat-gt}. If $r_t$'s are given by $r_0 \del*{x} = 0$ and $r_t \del*{x} = \frac{1}{\sqrt{2}D} \norm{x}_{Q_t}^2$, then the regret of Algorithm~\ref{algo:ada-opt-omd} w.r.t.\ any $x^* \in \Omega$ is bounded by 
	\begin{equation*}
	\sum_{t=1}^{T}{f_t \del{\hat{x}_t} - f_t \del{x^*}} \leq \sqrt{2} D \, \text{tr} \del*{Q_{1:T}}.
	\end{equation*}
\end{corollary}

The improvement is a bit more subtle in this case, and it is problem dependent as well. Since this method is not computationally efficient we haven't discussed it in detail. Please confer \cite[Section~1.3]{duchi2011adaptive} for an example.

\section{Optimistic Mirror Descent with Curved Losses}
\label{improve-curvature}
The following theorem provides a regret bound of Algorithm~\ref{algo:ada-opt-omd} for the case where $f_t$ is $H_t$-strongly convex with respect to some general norm $\norm{\cdot}$. Since this theorem is an extension of Theorem~2.1 from \citep{hazan2007adaptive} for the Optimistic Mirror Descent, this inherits the properties mentioned there such as : $r_t$'s can be chosen without the knowledge of uniform lower bound on $H_t$'s, and $O \del{\log{T}}$ bound can be achieved even when some $H_t \leq 0$ as long as $\frac{H_{1:t}}{t} > 0$. 

\begin{theorem}
	\label{ada-opt-omd-strong-cvx-theorem}
	Let $f_t$ is $H_t$-strongly convex w.r.t.\ $\norm{\cdot}$ and $H_t \leq \gamma$ for all $t \in \sbr{T}$. If $r_t$'s are given by $r_0 \del{x}=0$, $r_1 \del{x}=\frac{\gamma}{4} \norm{x}^2$, and $r_t \del{x}=\frac{H_{t-1}}{4} \norm{x}^2$ for all $t \geq 2$,	then the regret of Algorithm~\ref{algo:ada-opt-omd} w.r.t.\ any $x^* \in \Omega$ is bounded by
	\begin{equation*}
	\sum_{t=1}^{T}{f_t \del{\hat{x}_t} - f_t \del{x^*}} ~\leq~ 3 \sum_{t=1}^{T}{\frac{\norm{g_t - \tilde{g}_t}_*^2}{H_{1:t}}} + \frac{\gamma}{4} \norm{x^*-x_1}^2.
	\end{equation*}
\end{theorem}
\begin{proof}
	For the given choice of regularizers, we have $r_{0:t}\del{x}=\frac{H_{1:t-1} + \gamma}{4} \norm{x}^2$ and 
	\[
	\bregmanadaregcumalative{x}{y}{t} = \frac{H_{1:t-1} + \gamma}{4} \norm{x - y}^2 . 
	\] 
	Since $r_{0:t}$ is $1$-strongly convex w.r.t.\ $\sqrt{\frac{H_{1:t-1} + \gamma}{2}} \norm{\cdot}$, we have $\norm{\cdot}_{\del{t}} = \sqrt{\frac{H_{1:t-1} + \gamma}{2}} \norm{\cdot}$ and $\norm{\cdot}_{\del{t},*} = \sqrt{\frac{2}{H_{1:t-1} + \gamma}} \norm{\cdot}_*$. Thus for any $x^* \in \Omega$ we have
	\begin{align*}
	& f_t \del{\hat{x}_t} - f_t \del{x^*}
	\\
	~\leq~&
	\tuple{g_t,\hat{x}_t-x^*} - \frac{H_t}{2} \norm{\hat{x}_t - x^*}^2
	\\
	~\leq~&
	\half \norm{g_t - \tilde{g}_t}_{\del{t},*}^2 + \bregmanadareg{x^*}{x_t} - \bregmanadareg{x^*}{x_{t+1}} 
	- \frac{H_t}{2} \norm{\hat{x}_t - x^*}^2
	\\
	~=~&
	\frac{\norm{g_t - \tilde{g}_t}_*^2}{H_{1:t-1} + \gamma} + \frac{H_{1:t-1} + \gamma}{4} \norm{x^* - x_t}^2 
	- \frac{H_{1:t-1} + \gamma}{4} \norm{x^* - x_{t+1}}^2 - \frac{H_t}{2} \norm{\hat{x}_t - x^*}^2,
	\end{align*}
	where the first inequality is due to the strong convexity of $f_t$, and the second inequality is due to Lemma~\ref{ada-opt-omd-linear-regret-lemma}. Observe that
	\begin{align*}
	&
	\sum_{t=1}^{T}{\frac{H_{1:t-1} + \gamma}{4} \setdel*{\norm{x^* - x_t}^2 - \norm{x^* - x_{t+1}}^2}}
	\\
	~=~&
	\sum_{t=1}^{T}{\norm{x^* - x_{t+1}}^2 \setdel*{\frac{H_{1:t} + \gamma}{4} - \frac{H_{1:t-1} + \gamma}{4}}} 
	+ \frac{\gamma}{4} \norm{x^* - x_1}^2 - \frac{H_{1:T} + \gamma}{4} \norm{x^* - x_{T+1}}^2
	\\
	~\leq~&
	\sum_{t=1}^{T}{\frac{H_t}{4} \norm{x^* - x_{t+1}}^2} + \frac{\gamma}{4} \norm{x^* - x_1}^2,
	\end{align*}
	and 
	\begin{align*}
	&
	\sum_{t=1}^{T}{\frac{H_t}{4} \norm{x^* - x_{t+1}}^2 - \frac{H_t}{2} \norm{\hat{x}_t - x^*}^2}  
	\\
	~=~&
	\sum_{t=1}^{T}{\frac{H_t}{4} \setdel*{\norm{x^* - \hat{x}_t + \hat{x}_t - x_{t+1}}^2 - 2 \norm{x^* - \hat{x}_t}^2}}
	\\
	~\leq~&
	\sum_{t=1}^{T}{\frac{H_t}{2} \norm{\hat{x}_t - x_{t+1}}^2} 
	\\
	~\leq~&
	\sum_{t=1}^{T}{\frac{H_{1:t-1}+\gamma}{2} \norm{\hat{x}_t - x_{t+1}}^2} 
	\\
	~\leq~&
	2 \sum_{t=1}^{T}{\frac{\norm{g_t - \tilde{g}_t}_*^2}{H_{1:t-1}+\gamma}},
	\end{align*}
	where the first inequality is obtained by applying the triangular inequality of norms the fact that $(a+b)^2 \leq 2 a^2 + 2 b^2$, the second inequality is due to the facts that $H_t \leq \gamma$ and $H_{1:t-1} \geq 0$, and the third inequality is due to Lemma~\ref{point-loss-connection}.
	
	Now by summing up the instantaneous regrets and using the above observation we get
	\begin{align*}
	\sum_{t=1}^{T}{f_t \del{\hat{x}_t} - f_t \del{x^*}}
	~\leq~&
	3 \sum_{t=1}^{T}{\frac{\norm{g_t - \tilde{g}_t}_*^2}{H_{1:t-1} + \gamma}} + \frac{\gamma}{4} \norm{x^*-x_1}^2 
	\\
	~\leq~&
	3 \sum_{t=1}^{T}{\frac{\norm{g_t - \tilde{g}_t}_*^2}{H_{1:t}}} + \frac{\gamma}{4} \norm{x^*-x_1}^2, 
	\end{align*}
	where the last inequality is due to the fact that $H_t \leq \gamma$.
\end{proof}

In the above theorem if $H_t \geq H > 0$ and $\norm{g_t-\tilde{g}_t}_* \leq 1$ (w.l.o.g) for all $t$, then it obtain a regret bound of the form $O \del*{\log{\sum_{t=1}^{T}{\norm{g_t-\tilde{g}_t}_*^2}}}$, using the fact that if $a_t \leq 1$ for all $t \in \sbr{T}$, then $\sum_{t=1}^{T}{\frac{a_t}{t}} = O \del*{\log{\sum_{t=1}^{T}{a_t}}}$. When $H$ is small, however, this guaranteed regret can still be large.

Now instead of running Algorithm~\ref{algo:ada-opt-omd} on the observed sequence of $f_t$'s, we use the modified sequence of loss functions of the form
\begin{equation}
\label{prox-modification}
\tilde{f}_t \del{x} := f_t \del{x} + \frac{\lambda_t}{2} \norm{x - \hat{x}_t}^2, \, \lambda_t \geq 0,
\end{equation}
which is already considered in \citep{do2009proximal} for the non-optimistic mirror descent case. Given $f_t$ is $H_t$-strongly convex with respect to $\norm{\cdot}$, $\tilde{f}_t$ is $\del{H_t + \lambda_t}$-strongly convex. Also note that $\partial \tilde{f}_t \del{\hat{x}_t} = \partial f_t \del{\hat{x}_t}$ because the gradient of $\norm{x - \hat{x}_t}^2$ is $0$ when evaluated at $\hat{x}_t$ \citep{do2009proximal}. Thus in the updates \eqref{ada-omd-opt-eq1} and \eqref{ada-omd-opt-eq2} the terms $g_t$ and $\tilde{g}_{t+1}$ remain unchanged, only the regularizers $r_t$'s will change appropriately. By applying Theorem~\ref{ada-opt-omd-strong-cvx-theorem} for the modified sequence of losses given by \eqref{prox-modification} we obtain the following corollary. 

\begin{corollary}
	\label{prox-modification-corr}
	Let $2R = \maximizepairvalue{\norm{x-y}}$. Also let $f_t$ be $H_t$-strongly convex w.r.t.\ $\norm{\cdot}$, $H_t \leq \gamma$, and $\lambda_t \leq \delta$, for all $t \in \sbr{T}$. If Algorithm~\ref{algo:ada-opt-omd} is performed on the modified functions $\tilde{f}_t$'s with the regularizers $r_t$'s given by $r_0 \del{x}=0$, $r_1 \del{x}=\frac{\gamma + \delta}{4} \norm{x}^2$, and $r_t \del{x}=\frac{H_{t-1} + \lambda_{t-1}}{4} \norm{x}^2$ for all $t \geq 2$, then for any sequence $\lambda_1,...,\lambda_T \geq 0$, we get 
	\begin{equation*}
	\sum_{t=1}^{T}{f_t \del{\hat{x}_t} - f_t \del{x^*}} ~\leq~ 2R^2 \lambda_{1:T} 
	+ 3 \sum_{t=1}^{T}{\frac{\norm{g_t - \tilde{g}_t}_*^2}{\stronghlambda}} + \frac{\gamma + \delta}{4} \norm{x^*-x_1}^2.
	\end{equation*}
\end{corollary}

\begin{algorithm}[tb]
	\caption{Curvature Adaptive and Optimistic Mirror Descent}
	\label{algo:improved-ada-opt-omd}
	\begin{algorithmic}
		\STATE {\bfseries Input:} $r_0 \del{x}=0$ and $r_1 \del{x}=\frac{\gamma + \delta}{4} \norm{x}^2$.
		\STATE {\bfseries Initialize:} $x_1, \hat{x}_1 = 0 \in \Omega$.
		\FOR{$t=1$ {\bfseries to} $T$}
		\STATE Predict $\hat{x}_t$, observe $f_t$, and incur loss $f_t \del{\hat{x}_t}$. \STATE Compute $g_t \in \partial f_t \del{\hat{x}_t}$ and $\tilde{g}_{t+1} \del{g_1,...,g_t}$.
		\STATE Compute $\lambda_t = \frac{\sqrt{\del{\stronghlambdaminus}^2+\frac{6 \norm{g_t - \tilde{g}_t}_*^2}{R^2}}-\del{\stronghlambdaminus}}{2}$
		\STATE Define $r_{t+1} \del{x}=\frac{H_{t} + \lambda_{t}}{4} \norm{x}^2$.
		\STATE Update
		\begin{align*}
		x_{t+1}
		~=~&
		\minimize{\tuple{g_t,x} + \bregmanadareg{x}{x_t}},
		\\
		\hat{x}_{t+1}
		~=~&
		\minimize{\tuple{\tilde{g}_{t+1},x} + \bregmanadaregcumalative{x}{x_{t+1}}{t+1}}.
		\end{align*}
		\ENDFOR
	\end{algorithmic}
\end{algorithm}

In the above corollary if we consider the two terms that depend on $\lambda_t$'s, the first term increases and the second term deceases with the increase of $\lambda_t$'s. Based on the online balancing heuristic approach \citep{hazan2007adaptive}, the positive solution of $2R^2 \lambda_{t} = 3 \frac{\norm{g_t - \tilde{g}_t}_*^2}{\stronghlambda}$ is given by 
\begin{equation*}
\lambda_t = \frac{\sqrt{\del{\stronghlambdaminus}^2+\frac{6 \norm{g_t - \tilde{g}_t}_*^2}{R^2}}-\del{\stronghlambdaminus}}{2}.
\end{equation*}
The resulting algorithm with the above choice of $\lambda_t$ is given in Algorithm~\ref{algo:improved-ada-opt-omd}. By using the Lemma~3.1 from \citep{hazan2007adaptive} we obtain the following regret bound for Algorithm~\ref{algo:improved-ada-opt-omd}.

\begin{theorem}
	\label{curvature-adaptive-theorem}
	The regret of Algorithm~\ref{algo:improved-ada-opt-omd} on the sequence of $f_t$'s with curvature $H_t \geq 0$ is bounded by
	\begin{equation*}
	\sum_{t=1}^{T}{f_t \del{\hat{x}_t} - f_t \del{x^*}} ~\leq~ \frac{\gamma + \delta}{4} \norm{x^*-x_1}^2 
	+ 2 \inf_{\lambda_1^*,...,\lambda_T^*}{\setdel*{2R^2 \lambda_{1:T}^* + 3 \sum_{t=1}^{T}{\frac{\norm{g_t - \tilde{g}_t}_*^2}{H_{1:t}+\lambda_{1:t}^*}}}}.
	\end{equation*}
\end{theorem}
Thus the Algorithm~\ref{algo:improved-ada-opt-omd} achieves a regret bound which is competitive with the bound achievable by the best offline choice of parameters $\lambda_t$'s. From the above theorem we obtain the following two corollaries which show that Algorithm~\ref{algo:improved-ada-opt-omd} achieves intermediate rates between $O \del*{\sqrt{\sum_{t=1}^{T}{\norm{g_t-\tilde{g}_t}_*^2}}}$ and $O \del*{\log{\sum_{t=1}^{T}{\norm{g_t-\tilde{g}_t}_*^2}}}$ depending on the curvature of the losses. 

\begin{corollary}
	For any sequence of convex loss functions $f_t$'s, the bound on the regret of Algorithm~\ref{algo:improved-ada-opt-omd} is $O \del*{\sqrt{\sum_{t=1}^{T}{\norm{g_t-\tilde{g}_t}_*^2}}}$.
\end{corollary}
\begin{proof}
	Let $\lambda_{1}^* = \sqrt{\sum_{t=1}^{T}{\norm{g_t-\tilde{g}_t}_*^2}}$, and $\lambda_{t}^* = 0$ for all $t > 1$.
	\begin{align*}
	&
	2R^2 \lambda_{1:T}^* + 3 \sum_{t=1}^{T}{\frac{\norm{g_t - \tilde{g}_t}_*^2}{H_{1:t}+\lambda_{1:t}^*}} 
	\\
	~=~& 
	2R^2 \sqrt{\sum_{t=1}^{T}{\norm{g_t-\tilde{g}_t}_*^2}} + 3 \sum_{t=1}^{T}{\frac{\norm{g_t - \tilde{g}_t}_*^2}{0+\sqrt{\sum_{t=1}^{T}{\norm{g_t-\tilde{g}_t}_*^2}}}}
	\\
	~=~&
	\del*{2R^2 + 3} \sqrt{\sum_{t=1}^{T}{\norm{g_t-\tilde{g}_t}_2^2}}.
	\end{align*} 
\end{proof}

\begin{corollary}
	Suppose $\norm{g_t-\tilde{g}_t}_* \leq 1$ (w.l.o.g) and $H_t \geq H > 0$ for all $t \in \sbr{T}$. Then the bound on the regret of Algorithm~\ref{algo:improved-ada-opt-omd} is $O \del*{\log{\sum_{t=1}^{T}{\norm{g_t-\tilde{g}_t}_*^2}}}$.
\end{corollary}
\begin{proof}
	Set $\lambda_t^* = 0$ for all $t$.
	\begin{align*}
	2R^2 \lambda_{1:T}^* + 3 \sum_{t=1}^{T}{\frac{\norm{g_t - \tilde{g}_t}_*^2}{H_{1:t}+\lambda_{1:t}^*}} 
	~=~& 
	0 + 3 \sum_{t=1}^{T}{\frac{\norm{g_t - \tilde{g}_t}_*^2}{Ht+0}} 
	\\
	~=~& O \del*{\log{\sum_{t=1}^{T}{\norm{g_t-\tilde{g}_t}_*^2}}},
	\end{align*} 
	where the last inequality is due to the fact that if $a_t \leq 1$ for all $t \in \sbr{T}$, then $\sum_{t=1}^{T}{\frac{a_t}{t}} = O \del*{\log{\sum_{t=1}^{T}{a_t}}}$.
\end{proof}

The results obtained here can be extended to the applications discussed in \citep{do2009proximal,orabona2010online} to obtain much tighter results.

\section{Composite Losses}
\label{extensions}
Here we consider the case when observed loss function $f_t$ is composed with some non-negative (possibly non-smooth) convex regularizer term $\psi_t$ to impose certain constraints on the original problem. In this case we generally do not want to linearize the additional regularizer term, thus in the update rules given by \eqref{ada-omd-opt-eq1} and \eqref{ada-omd-opt-eq2} we include $\psi_t$ and $\psi_{t+1}$ respectively without linearizing them. This extension is presented in Algorithm~\ref{algo:comp-ada-opt-omd}. 

\begin{algorithm}[tb]
	\caption{Adaptive and Optimistic Mirror Descent with Composite Losses}
	\label{algo:comp-ada-opt-omd}
	\begin{algorithmic}
		\STATE {\bfseries Input:} regularizers $r_0, r_1 \geq 0$, composite losses $\setdel*{\psi_t}_t$ where $\psi_t \geq 0$.
		\STATE {\bfseries Initialize:} $x_1, \hat{x}_1 = 0 \in \Omega$.
		\FOR{$t=1$ {\bfseries to} $T$}
		\STATE Predict $\hat{x}_t$, observe $f_t$, and incur loss $f_t \del{\hat{x}_t} + \psi_t \del{\hat{x}_t}$. 
		\STATE Compute $g_t \in \partial f_t \del{\hat{x}_t}$ and $\tilde{g}_{t+1} \del{g_1,...,g_t}$.
		\STATE Construct $r_{t+1}$ s.t. $r_{0:t+1}$ is $1$-strongly convex w.r.t.\ $\norm{\cdot}_{\del{t+1}}$.
		\STATE Update
		\begin{align}
		x_{t+1}
		~=~&
		\minimize{\tuple{g_t,x} + \psi_t \del{x} + \bregmanadareg{x}{x_t}},
		\label{comp-ada-omd-opt-eq1}
		\\
		\hat{x}_{t+1}
		~=~&
		\minimize{\tuple{\tilde{g}_{t+1},x} +\psi_{t+1} \del{x} + \bregmanadaregcumalative{x}{x_{t+1}}{t+1}}.
		\label{comp-ada-omd-opt-eq2}
		\end{align}
		\ENDFOR
	\end{algorithmic}
\end{algorithm}

The following lemma provides a bound on the instantaneous regret of Algorithm~\ref{algo:comp-ada-opt-omd}. 
\begin{lemma}
	\label{comp-ada-omd-optimistic-lemma}
	The instantaneous regret of Algorithm~\ref{algo:comp-ada-opt-omd} w.r.t.\ any $x^* \in \Omega$ can be bounded as follows
	\begin{equation*}
	\setdel*{\compositefunc{\hat{x}_t}{\hat{x}_t}} - \setdel*{\compositefunc{x^*}{x^*}} 
	~\leq~ \half \norm{g_t-\tilde{g}_t}_{\del{t},*}^2 + \bregmanadareg{x^*}{x_t} - \bregmanadareg{x^*}{x_{t+1}}.
	\end{equation*}
\end{lemma}
\begin{proof}
	The instantaneous regret of the algorithm can be bounded as below using the convexity of $f_t$
	\begin{equation*}
	\setdel*{\compositefunc{\hat{x}_t}{\hat{x}_t}} - \setdel*{\compositefunc{x^*}{x^*}} 
	~\leq~ \tuple{g_t,\hat{x}_t-x^*} + \setdel*{\psi_t \del{\hat{x}_t} - \psi_t \del{x^*}}.
	\end{equation*}
	Now consider
	\begin{equation}
	\label{comp-ada-omd-optimistic-proof1}
	\tuple{g_t,\hat{x}_t-x^*} ~=~ \tuple{g_t - \tilde{g}_t,\hat{x_t}-x_{t+1}} + \tuple{\tilde{g}_t,\hat{x_t}-x_{t+1}} 
	+ \tuple{g_t,x_{t+1}-x^*}.
	\end{equation}
	By the fact that $\tuple{a,b} \leq \norm{a} \norm{b}_* \leq \half \norm{a}^2 + \half \norm{b}_*^2$, we have 
	\begin{equation*}
	\tuple{g_t - \tilde{g}_t,\hat{x_t}-x_{t+1}} ~\leq~ \half \norm{\hat{x_t}-x_{t+1}}_{\del{t}}^2 + \half \norm{g_t - \tilde{g}_t}_{\del{t},*}^2
	\end{equation*}
	The first-order optimality condition for $x^* = \minimize{\tuple{g,x} + f \del{x} + \bregmanadaregpsi{x}{y}}$ and for $z \in \Omega$, 
	\begin{equation*}
	\tuple{x^* - z , g} ~\leq~ \tuple{z - x^* , f' \del{x^*}} 
	+ \bregmanadaregpsi{z}{y} - \bregmanadaregpsi{z}{x^*} - \bregmanadaregpsi{x^*}{y}.
	\end{equation*}
	By applying the above condition for \eqref{comp-ada-omd-opt-eq2} and \eqref{comp-ada-omd-opt-eq1} we have respectively 
	\begin{equation*}
	\tuple{\hat{x_t} - x_{t+1} , \tilde{g}_t} ~\leq~ \tuple{\psi_t' \del{\hat{x_t}} , x_{t+1} - \hat{x_t}} 
	+ \bregmanadareg{x_{t+1}}{x_{t}} - \bregmanadareg{x_{t+1}}{\hat{x_t}} - \bregmanadareg{\hat{x_t}}{x_{t}}
	\end{equation*}
	\begin{equation*}
	\tuple{x_{t+1} - x^* , g_t} ~\leq~ \tuple{\psi_t' \del{x_{t+1}} , x^* - x_{t+1}} 
	+ \bregmanadareg{x^*}{x_t} - \bregmanadareg{x^*}{x_{t+1}} - \bregmanadareg{x_{t+1}}{x_t}.
	\end{equation*}
	Thus by \eqref{comp-ada-omd-optimistic-proof1} we have
	\begin{align*}
	&\tuple{g_t,\hat{x}_t-x^*} + \setdel*{\psi_t \del{\hat{x}_t} - \psi_t \del{x^*}}
	\\
	~\leq~&
	\half \norm{\hat{x_t}-x_{t+1}}_{\del{t}}^2 + \half \norm{g_t - \tilde{g}_t}_{\del{t},*}^2 
	+ \bregmanadareg{x^*}{x_t} - \bregmanadareg{x^*}{x_{t+1}} - \bregmanadareg{x_{t+1}}{\hat{x_t}} 
	\\
	&
	+ \psi_t \del{\hat{x}_t} - \psi_t \del{x^*} + \tuple{\psi_t' \del{\hat{x_t}} , x_{t+1} - \hat{x_t}} 
	+ \tuple{\psi_t' \del{x_{t+1}} , x^* - x_{t+1}}
	\\
	~\leq~&
	\half \norm{g_t - \tilde{g}_t}_{\del{t},*}^2 + \bregmanadareg{x^*}{x_t} - \bregmanadareg{x^*}{x_{t+1}} 
	\\
	&
	+ \psi_t \del{\hat{x}_t} + \tuple{\psi_t' \del{\hat{x_t}} , x_{t+1} - \hat{x_t}} 
	+ \tuple{\psi_t' \del{x_{t+1}} , x^* - x_{t+1}} - \psi_t \del{x^*}
	\\
	~\leq~&
	\half \norm{g_t - \tilde{g}_t}_{\del{t},*}^2 + \bregmanadareg{x^*}{x_t} - \bregmanadareg{x^*}{x_{t+1}} 
	\\
	&
	+ \psi_t \del{x_{t+1}} + \tuple{\psi_t' \del{x_{t+1}} , x^* - x_{t+1}} - \psi_t \del{x^*}
	\\
	~\leq~&
	\half \norm{g_t - \tilde{g}_t}_{\del{t},*}^2 + \bregmanadareg{x^*}{x_t} - \bregmanadareg{x^*}{x_{t+1}} 
	+ \psi_t \del{x^*} - \psi_t \del{x^*}
	\\
	~=~&
	\half \norm{g_t - \tilde{g}_t}_{\del{t},*}^2 + \bregmanadareg{x^*}{x_t} - \bregmanadareg{x^*}{x_{t+1}}
	\end{align*}
	where the second inequality is due to 1-strong convexity of $r_{0:t}$ w.r.t.\ $\norm{\cdot}_{\del{t}}$, and the third and fourth inequalities are due to the convexity of $\psi_t$ at $\hat{x}_t$ and $x_{t+1}$ respectively.
\end{proof}

From the above lemma we can observe that the instantaneous regret of Algorithm~\ref{algo:comp-ada-opt-omd} is exactly equal to that of the non-composite version (Algorithm~\ref{algo:ada-opt-omd}). Thus all the improvements that we discussed in the previous sections for the non-composite case are also applicable to composite losses as well.

\section{Discussion}

Early approaches to the OCO problem were conservative, in which the main focus was protection against the worst case scenario. But recently several algorithms have been developed for tightening the regret bounds in easy data instances such as sparsity, predictable sequences, and curved losses. We have unified some of these existing techniques to obtain new update rules for the cases when these easy instances occur together. First we have analysed an adaptive and optimistic update rule which achieves tighter regret bound when the loss sequence is sparse and predictable (Algorithm~\ref{algo:ada-opt-omd}). Then we have analysed an update rule that dynamically adapts to the curvature of the loss function and utilizes the predictable nature of the loss sequence as well (Algorithm~\ref{algo:improved-ada-opt-omd}). Finally we have extended these results to composite losses (Algorithm~\ref{algo:comp-ada-opt-omd}).

We also note that the regret bounds given in this chapter can be converted into convergence bounds for batch stochastic problems using online-to-batch conversion techniques \citep{cesa2004generalization,kakade2009generalization}.

\section{Appendix}

\subsection{Proofs}
\label{sec:proof-oco}
\begin{proof} \textbf{(Proposition~\ref{simple-form-omd})}
	Observe that
	\begin{align*}
	&
	\minimize{\bregmanadaregpsi{x}{y}}  
	\\
	~=~&
	\minimize{\psi \del{x} - \psi \del{y} - \tuple{\nabla \psi \del{y}, x - y}} 
	\\
	~=~&
	\minimize{\psi \del{x} - \tuple{\nabla \psi \del{y}, x}}
	\\
	~=~&
	\minimize{\psi \del{x} - \tuple{\nabla \psi \del{u} - g, x}}
	\\
	~=~&
	\minimize{\tuple{g, x} + \psi \del{x} - \psi \del{u} - \tuple{\nabla \psi \del{u}, x - u}}
	\\
	~=~&
	\minimize{\tuple{g, x} + \bregmanadaregpsi{x}{u}}.
	\end{align*}
\end{proof}

\begin{proof} \textbf{(Lemma~\ref{point-loss-connection})}
	Since $r_{0:t}$ is $1$-strongly convex w.r.t.\ $\norm{\cdot}_{\del{t}}$ we have
	\begin{align*}
	&
	\bregmanadareg{\hat{x}_t}{x_{t+1}} 
	\\
	~=~&
	r_{0:t} \del{\hat{x}_t} - r_{0:t} \del{x_{t+1}} - \tuple{\regderivative{x_{t+1}}, \hat{x}_t - x_{t+1}}
	\\
	~\geq~&
	\half \norm{\hat{x}_t - x_{t+1}}_{\del{t}}^2,
	\\
	\text{and}
	\\
	&
	\bregmanadareg{x_{t+1}}{\hat{x}_{t}}
	\\
	~=~&
	r_{0:t} \del{x_{t+1}} - r_{0:t} \del{\hat{x}_{t}} - \tuple{\regderivative{\hat{x}_{t}}, x_{t+1} - \hat{x}_{t}}
	\\
	~\geq~&
	\half \norm{x_{t+1} - \hat{x}_{t}}_{\del{t}}^2.
	\end{align*}
	Adding these two bounds, we obtain 
	\begin{equation}
	\label{point-loss-connection-proof1}
	\norm{\hat{x}_{t} - x_{t+1}}_{\del{t}}^2 ~\leq~ \tuple{\regderivative{\hat{x}_t} - \regderivative{x_{t+1}} , \hat{x}_{t} - x_{t+1}}.
	\end{equation}
	
	Suppose $y_{t+1}$ and $\hat{y}_{t}$ satisfy the conditions $\regderivative{y_{t+1}} = \regderivative{x_{t}} - g_t$ and $\regderivative{\hat{y}_{t}} = \regderivative{x_{t}} - \tilde{g}_{t}$ respectively. Then by applying Proposition~\ref{simple-form-omd} to the updates in \eqref{ada-omd-opt-eq2} and \eqref{ada-omd-opt-eq1} of Algorithm~\ref{algo:ada-opt-omd}, we obtain
	\begin{align*}
	x_{t+1} 
	~=~& \minimize{\bregmanadareg{x}{y_{t+1}}}
	\\
	\hat{x}_{t} 
	~=~& \minimize{\bregmanadareg{x}{\hat{y}_{t}}}.
	\end{align*}
	By applying the first order optimality condition for the above two optimization statements, we have
	\begin{align*}
	\tuple{\regderivative{x_{t+1}} - \regderivative{y_{t+1}} , \hat{x}_t - x_{t+1}} 
	~\geq~& 0
	\\
	\tuple{\regderivative{\hat{x}_t} - \regderivative{\hat{y}_t} , x_{t+1} - \hat{x}_t}
	~\geq~& 0,
	\end{align*}
	respectively. Combining these two bounds, we obtain
	\begin{equation*}
	\tuple{ \regderivative{\hat{y}_t} - \regderivative{y_{t+1}} , \hat{x}_t - x_{t+1}} 
	~\geq~ \tuple{ \regderivative{\hat{x}_t} - \regderivative{x_{t+1}} , \hat{x}_t - x_{t+1}}.
	\end{equation*}
	By combining the above result with \eqref{point-loss-connection-proof1}, we obtain
	\begin{align*}
	&
	\norm{\hat{x}_{t} - x_{t+1}}_{\del{t}}^2 
	\\
	~\leq~& 
	\tuple{ \regderivative{\hat{y}_t} - \regderivative{y_{t+1}} , \hat{x}_t - x_{t+1}} 
	\\
	~\leq~&
	\norm{\regderivative{\hat{y}_t} - \regderivative{y_{t+1}}}_{\del{t},*} \norm{\hat{x}_t - x_{t+1}}_{\del{t}},
	\end{align*}
	by a generalized Cauchy-Schwartz inequality. Dividing both sides by $\norm{\hat{x}_t - x_{t+1}}_{\del{t}}$, we have
	\begin{align*}
	&
	\norm{\hat{x}_{t} - x_{t+1}}_{\del{t}} 
	\\
	~\leq~&
	\norm{\regderivative{\hat{y}_t} - \regderivative{y_{t+1}}}_{\del{t},*}
	\\
	~=~&
	\norm{\del*{\regderivative{x_{t}} - \tilde{g}_t} - \del*{\regderivative{x_t} - g_t}}_{\del{t},*}
	\\
	~=~&
	\norm{g_t - \tilde{g}_t}_{\del{t},*}.
	\end{align*}
\end{proof}

\begin{proof} \textbf{(Corollary~\ref{full-matrix-regret})}
	By letting $\eta=\frac{D}{\sqrt{2}}$ for the given sequence of regularizers, we get $r_{0:t} \del*{x} = \frac{1}{2 \eta} \norm{x}_{Q_{1:t}}^2$. Since $r_{0:t}$ is $1$-strongly convex w.r.t.\ $\frac{1}{\sqrt{\eta}} \norm{\cdot}_{Q_{1:t}}$, we have $\norm{\cdot}_{\del*{t}} = \frac{1}{\sqrt{\eta}} \norm{\cdot}_{Q_{1:t}}$ and $\norm{\cdot}_{\del*{t},*} = \sqrt{\eta} \norm{\cdot}_{Q_{1:t}^{-1}}$. By Theorem~\ref{ada-omd-optimistic-thm} the regret bound of Algorithm~\ref{algo:ada-opt-omd} with this choice of regularizer sequence can be given as follows
	\begin{equation*}
	\sum_{t=1}^{T}{f_t \del{\hat{x}_t} - f_t \del{x^*}} ~\leq~ \half \sum_{t=1}^{T}{\norm{g_t - \tilde{g}_{t}}_{\del{t},*}^2} + \sum_{t=1}^{T}{\bregmanadaregsingle{x^*}{x_{t}}{t}}.
	\end{equation*}
	
	Consider
	\begin{align*}
	&\half \sum_{t=1}^{T}{\norm{g_t - \tilde{g}_{t}}_{\del{t},*}^2} 
	\\
	~=~&
	\half \sum_{t=1}^{T}{\eta \norm{g_t - \tilde{g}_{t}}_{Q_{1:t}^{-1}}^2}
	\\
	~=~&
	\frac{\eta}{2} \sum_{t=1}^{T}{\del*{g_t - \tilde{g}_{t}} Q_{1:t}^{-1} \del*{g_t - \tilde{g}_{t}}^T}
	\\
	~=~&
	\frac{\eta}{2} \sum_{t=1}^{T}{\del*{g_t - \tilde{g}_{t}} \del*{\gamma^2 I + G_{2:t}}^{-\half} \del*{g_t - \tilde{g}_{t}}^T}
	\\
	~\leq~&
	\frac{\eta}{2} \sum_{t=1}^{T}{\del*{g_t - \tilde{g}_{t}} \del*{G_{2:t+1}}^{-\half} \del*{g_t - \tilde{g}_{t}}^T}
	\\
	~\leq~&
	\eta \, \text{tr} \del*{G_{2:T+1}^{\half}}
	\\
	~\leq~&
	\eta \, \text{tr} \del*{\del*{\gamma^2 I + G_{2:T}}^{\half}}
	\\
	~=~&
	\eta \, \text{tr} \del*{Q_{1:T}},
	\end{align*}
	where the first inequality is due to the facts that $\gamma^2 I \succcurlyeq G_{t+1}$ and $A \succcurlyeq B \succcurlyeq 0 \Rightarrow A^{\half} \succcurlyeq B^{\half}$ and $B^{-1} \succcurlyeq A^{-1}$, the second inequality is due to the fact that $\sum_{t=1}^{T}{a_t^T \del*{\sum_{s=1}^{t}{a_s a_s^T}}^{-\half} a_t} \leq 2 \cdot \text{tr} \del*{\br{\sum_{t=1}^{T}{a_t a_t^T}}^{\frac{1}{2}}}$ (see Lemma~10 from \citep{duchi2011adaptive}), and the third inequality is due to the fact that $\gamma^2 I \succcurlyeq G_{T+1}$. Also observing that 
	\begin{align*}
	&\sum_{t=1}^{T}{\bregmanadaregsingle{x^*}{x_{t}}{t}} 
	\\
	~=~&
	\sum_{t=1}^{T}{\frac{1}{2 \eta} \norm{x^* - x_t}_{Q_t}^2}
	\\
	~\leq~&
	\frac{1}{2 \eta} \sum_{t=1}^{T}{\norm{x^* - x_t}_2^2 \lambda_{\text{max}} \del*{Q_t}}
	\\
	~\leq~&
	\frac{1}{2 \eta} \sum_{t=1}^{T}{\norm{x^* - x_t}_2^2 \text{tr} \del*{Q_t}}
	\\
	~\leq~&
	\frac{1}{2 \eta} \sum_{t=1}^{T}{D^2 \text{tr} \del*{Q_t}}
	\\
	~=~&
	\frac{D^2}{2 \eta} \text{tr} \del*{Q_{1:T}}.
	\end{align*}
	completes the proof.
\end{proof}

\subsection{Mirror Descent with $\beta$-convex losses}
Given a convex set $\Omega \subseteq \mathbb{R}^n$ and $\beta > 0$, a function $f:\Omega \rightarrow \mathbb{R}$ is $\beta$-convex, if for all $x,y \in \Omega$
\begin{equation*}
f \del{x} ~\geq~ f \del{y} + \tuple{g,x-y} + \beta \norm{x-y}_{gg^T}^2, \, g \in \partial f \del{y}.
\end{equation*}
As in Theorem~\ref{ada-opt-omd-strong-cvx-theorem}, we can obtain regret bound for the case when the loss function $f_t$ is $\beta_t$-convex (which is broader class than exp-concave losses) as well. But for the resulting bound we cannot apply Lemma~3.1 from \citep{hazan2007adaptive} to obtain a near optimal closed form solution of $\lambda_t$. 

\begin{algorithm}[tb]
	\caption{Adaptive Mirror Descent}
	\label{algo:ada-omd}
	\begin{algorithmic}
		\STATE {\bfseries Input:} regularizers $r_0 \geq 0$.
		\STATE {\bfseries Initialize:} $x_1 = 0 \in \Omega$.
		\FOR{$t=1$ {\bfseries to} $T$}
		\STATE Predict $x_t$, observe $f_t$, and incur loss $f_t \del{x_t}$. 
		\STATE Compute $g_t \in \partial f_t \del{x_t}$.
		\STATE Construct $r_{t}$ s.t. $r_{0:t}$ is $1$-strongly convex w.r.t.\ $\norm{\cdot}_{\del{t}}$.
		\STATE Update
		\begin{align}
		\label{ada-omd-eq}
		x_{t+1}
		~=~&
		\minimize{\tuple{g_t,x} + \bregmanadareg{x}{x_t}}. 
		\end{align}
		\ENDFOR
	\end{algorithmic}
\end{algorithm}

\begin{lemma}
	\label{ada-omd-linear-regret-lemma}
	The instantaneous linear regret of Algorithm~\ref{algo:ada-omd} w.r.t.\ any $x^* \in \Omega$ can be bounded as follows
	\begin{equation*}
	\tuple{x_t - x^* , g_t} ~\leq~ \bregmanadareg{x^*}{x_t} - \bregmanadareg{x^*}{x_{t+1}} + \half \norm{g_t}_{\del{t},*}^2.
	\end{equation*}
\end{lemma}
\begin{proof}
	By the first-order optimality condition for \eqref{ada-omd-eq} we have, 
	\begin{equation}
	\tuple{x - x_{t+1} , g_t + \regderivative{x_{t+1}} - \regderivative{x_t}} \quad \geq \quad 0
	\label{ada-omd-linear-regret-lemma-proof1}
	\end{equation}
	Consider
	\begin{align*}
	&
	\tuple{x_t - x^* , g_t} 
	\\
	~=~&
	\tuple{x_{t+1} - x^* , g_t} + \tuple{x_t - x_{t+1} , g_t} 
	\\
	~\leq~& 
	\tuple{x^* - x_{t+1} , \regderivative{x_{t+1}} - \regderivative{x_t}} 
	+ \tuple{x_t - x_{t+1} , g_t}  
	\\
	~=~& 
	\bregmanadareg{x^*}{x_t} - \bregmanadareg{x^*}{x_{t+1}} - \bregmanadareg{x_{t+1}}{x_t} 
	+ \tuple{x_t - x_{t+1} , g_t}
	\\
	~\leq~& 
	\bregmanadareg{x^*}{x_t} - \bregmanadareg{x^*}{x_{t+1}} - \bregmanadareg{x_{t+1}}{x_t} 
	+ \half \norm{x_t - x_{t+1}}_{\del{t}}^2 + \half \norm{g_t}_{\del{t},*}^2
	\\
	~\leq~& 
	\bregmanadareg{x^*}{x_t} - \bregmanadareg{x^*}{x_{t+1}} + \half \norm{g_t}_{\del{t},*}^2,
	\end{align*}
	where the first inequality is due to \eqref{ada-omd-linear-regret-lemma-proof1}, the second equality is due to the fact that $\tuple{\psiderivative{a} - \psiderivative{b},c-a} =  \bregmanadaregpsi{c}{b} - \bregmanadaregpsi{c}{a} - \bregmanadaregpsi{a}{b}$, the second inequality is due to the fact that $\tuple{a,b} \leq \norm{a} \norm{b}_* \leq \half \norm{a}^2 + \half \norm{b}_*^2$, and the third inequality is due to the $1$-strong convexity of $r_{0:t}$ w.r.t.\ $\norm{\cdot}_{\del{t}}$.
\end{proof}

\begin{theorem}
	\label{ada-omd-beta-cvx-theorem}
	Let $f_t$ is $\beta_t$-convex, $\forall t \in \sbr{T}$. If $r_t$'s are given by
	\begin{equation}
	\label{ada-omd-reg-choice-beta-t}
	r_t \del{x}=\norm{x}_{h_t}^2, \text{ where } h_0=I_{n \times n} \text{ and } h_t=\beta_t g_t g_t^T \text{ for } t \geq 1,
	\end{equation}
	then the regret of Algorithm~\ref{algo:ada-omd} w.r.t.\ any $x^* \in \Omega$ is bounded by
	\begin{equation*}
	\sum_{t=1}^{T}{f_t \del{x_t} - f_t \del{x^*}} ~\leq~ \frac{1}{4} \sum_{t=1}^{T}{\norm{g_t}_{h_{0:t}^{-1}}^2} + \norm{x^* - x_1}_2^2 .
	\end{equation*}
\end{theorem}
\begin{proof}
	For the choice of regularizer sequence $\setdel{r_t}$ given by \eqref{ada-omd-reg-choice-beta-t}, we have $r_{0:t}\del{x}=\norm{x}_{h_{0:t}}^2$ and $\bregmanadareg{x}{y}=\half \del*{\sqrt{2}\norm{x-y}_{h_{0:t}}}^2$. Since $r_{0:t}$ is $1$-strongly convex w.r.t.\ $\sqrt{2}\norm{\cdot}_{h_{0:t}}$, we have $\norm{\cdot}_{\del{t}}=\sqrt{2}\norm{\cdot}_{h_{0:t}}$ and $\norm{\cdot}_{\del{t},*}=\frac{1}{\sqrt{2}} \norm{\cdot}_{h_{0:t}^{-1}}$.
	
	For any $x^* \in \Omega$
	\begin{align*}
	&
	f_t \del{x_t} - f_t \del{x^*}
	\\
	~\leq~&
	\tuple{g_t,x_t - x^*} - \beta_t \norm{x^* - x_t}_{g_t g_t^T}^2 
	\\
	~\leq~&
	\bregmanadareg{x^*}{x_t} - \bregmanadareg{x^*}{x_{t+1}} 
	+ \half \norm{g_t}_{\del{t},*}^2 - \norm{x^* - x_t}_{\beta_t g_t g_t^T}^2
	\\
	~=~&
	\norm{x^* - x_t}_{h_{0:t}}^2 - \norm{x^* - x_{t+1}}_{h_{0:t}}^2 - \norm{x^* - x_t}_{h_t}^2 
	+ \half \norm{g_t}_{\del{t},*}^2, 
	\end{align*}
	where the first inequality is due to the $\beta_t$-convexity of $f_t \del{\cdot}$, and the second inequality is due to Lemma~\ref{ada-omd-linear-regret-lemma}. By summing all the instantaneous regrets we get
	\begin{align*}
	&
	\sum_{t=1}^{T}{f_t \del{x_t}} - \sum_{t=1}^{T}{f_t \del{x^*}} 
	\\
	~\leq~&
	\sum_{t=1}^{T}{\setdel*{\norm{x^* - x_t}_{h_{0:t}}^2 - \norm{x^* - x_t}_{h_{0:t-1}}^2 -\norm{x^* - x_t}_{h_{t}}^2}} \\
	& + \norm{x^* - x_1}_{h_{0}}^2 - \norm{x^* - x_{T+1}}_{h_{0:T}}^2 + \half \sum_{t=1}^{T}{\norm{g_t}_{\del{t},*}^2}
	\\
	~\leq~&
	\norm{x^* - x_1}_2^2 + \frac{1}{4} \sum_{t=1}^{T}{\norm{g_t}_{h_{0:t}^{-1}}^2}.
	\end{align*}
\end{proof}

Now instead of running Algorithm~\ref{algo:ada-omd} on the observed sequence of $f_t$'s, we use the modified sequence of loss functions of the form
\begin{equation}
\label{prox-modification-beta}
\tilde{f}_t \del{x} := f_t \del{x} + \lambda_t g\del{x}, \, \lambda_t \geq 0,
\end{equation}
where $g\del{x}$ is $1$-convex. By following the proof of Theorem~\ref{ada-omd-beta-cvx-theorem} for the modified sequence of losses given by \eqref{prox-modification-beta} we obtain the following corollary.

\begin{theorem}
	\label{curve-case-1-beta-thm1}
	Let $g\del{x}$ be a $1$-convex function, $A^2 = \maximizevalue{g \del{x}}$ and 
	\[
	B = \maximizevalue{\norm{g' \del{x}}_{\del*{g' \del{x} g' \del{x}^T}^{-1}}} .
	\]
	Also let $f_t$ be $\beta_t$-convex ($\beta_t \geq 0$), $\forall t \in \sbr{T}$. 
	If Algorithm~\ref{algo:ada-omd} is performed on the modified functions $\tilde{f}_t$'s with the regularizers $r_t$'s given by 
	\begin{equation}
	\label{curve-case-1-reg-choice-beta-t}
	r_t \del{x}=\norm{x}_{h_t}^2, \text{ where } h_0=I_{n \times n}, \text{ and }
	h_t=\beta_t g_t g_t^T + \lambda_t g' \del{x_t} g' \del{x_t}^T, \text{ for } t \geq 1,
	\end{equation} 
	then for any sequence $\lambda_1,...,\lambda_T \geq 0$, we get
	\begin{equation*}
	\sum_{t=1}^{T}{f_t \del{x_t} - f_t \del{x^*}} ~\leq~ \del*{A^2+\frac{B^2}{2}} \lambda_{1:T}
	+ \half \sum_{t=1}^{T}{\norm{g_t}_{h_{0:t}^{-1}}^2} + \norm{x^* - x_1}_2^2.
	\end{equation*}
\end{theorem}
\begin{proof}
	Since $f_t$ is $\beta_t$-convex and $g$ is $1$-convex, for any $x^* \in \Omega$ we have
	\begin{align*}
	&
	\setdel*{f_t \del{x_t} + \lambda_t g \del{x_t}} - \setdel*{f_t \del{x^*} + \lambda_t g \del{x^*}}
	\\
	~=~&
	f_t \del{x_t} - f_t \del{x^*} + \lambda_t \setdel*{g \del{x_t} - g \del{x^*}}
	\\
	~\leq~&
	\tuple{g_t,x_t - x^*} - \beta_t \norm{x^* - x_t}_{g_t g_t^T}^2 
	+ \lambda_t \setdel*{\tuple{g' \del{x_t},x_t - x^*} - \norm{x^* - x_t}_{g' \del{x_t} g' \del{x_t}^T}^2}
	\\
	~=~&
	\tuple{g_t + \lambda_t g' \del{x_t} , x_t - x^*} - \norm{x^* - x_t}_{\beta_t g_t g_t^T + \lambda_t g' \del{x_t} g' \del{x_t}^T}^2.  
	\end{align*}
	By following the similar steps from the proof of Theorem~\ref{ada-omd-beta-cvx-theorem} we get
	\begin{equation*}
	\sum_{t=1}^{T}{f_t \del{x_t} + \lambda_t g \del{x_t}} - \setdel*{\sum_{t=1}^{T}{f_t \del{x^*} + \lambda_t g \del{x^*}}} 
	~\leq~ 
	\frac{1}{4} \sum_{t=1}^{T}{\norm{g_t + \lambda_t g' \del{x_t}}_{h_{0:t}^{-1}}^2} + \norm{x^* - x_1}_2^2.
	\end{equation*}
	By using the facts that $\norm{x+y}_A^2 \leq 2 \norm{x}_A^2 + 2 \norm{y}_A^2$, $h_{0:t} \succcurlyeq h_t \succcurlyeq \lambda_t g' \del{x_t} g' \del{x_t}^T$, and $\norm{g' \del{x_t}}_{\del*{g' \del{x_t} g' \del{x_t}^T}^{-1}} \leq B$, we have
	\begin{align*}
	&
	\sum_{t=1}^{T}{f_t \del{x_t} + \lambda_t g \del{x_t}} - \setdel*{\sum_{t=1}^{T}{f_t \del{x^*} + \lambda_t g \del{x^*}}} 
	\\
	~\leq~&
	\half \sum_{t=1}^{T}{\setdel*{\norm{g_t}_{h_{0:t}^{-1}}^2 + \lambda_t^2 \norm{g' \del{x_t}}_{h_{0:t}^{-1}}^2}} + \norm{x^* - x_1}_2^2
	\\
	~\leq~&
	\half \sum_{t=1}^{T}{\setdel*{\norm{g_t}_{h_{0:t}^{-1}}^2 + \lambda_t^2 \norm{g' \del{x_t}}_{\del*{\lambda_t g' \del{x_t} g' \del{x_t}^T}^{-1}}^2}} 
	+ \norm{x^* - x_1}_2^2
	\\
	~\leq~&
	\half \sum_{t=1}^{T}{\norm{g_t}_{h_{0:t}^{-1}}^2} + \frac{B^2}{2} \lambda_{1:T} + \norm{x^* - x_1}_2^2.
	\end{align*}
	By neglecting the $g \del{x_t}$ terms in the L.H.S. and using the fact that $g\del{x^*} \leq A^2$ we get
	\begin{equation*}
	\sum_{t=1}^{T}{f_t \del{x_t}} ~\leq~ \sum_{t=1}^{T}{f_t \del{x^*}} + A^2 \lambda_{1:T} + \half \sum_{t=1}^{T}{\norm{g_t}_{h_{0:t}^{-1}}^2} 
	+ \norm{x^* - x_1}_2^2 + \frac{B^2}{2} \lambda_{1:T}.
	\end{equation*}
\end{proof}

But we cannot apply Lemma~3.1 from \citep{hazan2007adaptive} for the above regret bound to obtain a near optimal closed form solution to $\lambda_t$. One could employ an optimization algorithm to find the optimal $\lambda_t$.

\chapter{Conclusion}
\label{cha:conclussion}
This thesis studied the problem of bounding the performance of machine learning algorithms in both statistical and adversarial setting, and designing computationally efficient algorithms with strong theoretical guarantees.

The major contributions of this thesis are: 
\begin{itemize}
	\item An investigation of the influence of cost terms on the hardness of the cost-sensitive classification problem by extending the minimax lower bound analysis for balanced binary classification (Theorem~\ref{main-theo-cost-classi}). 
	\item A relationship between the contraction coefficient of a channel w.r.t.\ $c$-primitive $f$-divergence, and a generalized form of Dobrushin's coefficient (Theorem~\ref{contract-fc-theorem}).
	\item An increased understanding of contraction coefficients of binary symmetric channels w.r.t.\ any symmetric $f$-divergence (Section~\ref{subsec:bsc}).
	\item A complete characterization of the exp-concavity of any proper composite loss (Proposition~\ref{multiexpprop}). Using this characterization and the mixability condition of proper losses (\cite{van2012mixability}), we showed that it is possible to re-parameterize any $\beta$-mixable binary proper loss into a $\beta$-exp-concave composite loss with the same $\beta$ (Corollary~\ref{specialcoro}). 
	\item Analysis of unified update rules of the accelerated online convex optimization algorithms (Sections~\ref{improve-sparsity} and \ref{improve-curvature}). Improved regret bounds were achieved by exploiting the easy nature of the sequence of outcomes. 
\end{itemize}

By studying the geometry of prediction problems we have obtained insights into the factors which conspire to make the problem hard. These insights have enabled us to place bounds on the learning algorithm's ability to accurately predict the unseen data, and guided us in the design of better solutions. 

Throughout the thesis, we have pointed out a number of open questions which we feel are important. These include the following:
\begin{enumerate}
	\item Extend the study of the contraction coefficients of binary symmetric channels (Section~\ref{subsec:bsc}) w.r.t.\ symmetric $f$-divergences to $k$-ary symmetric channels (with $k > 2$) and general $f$-divergences.
	\item There are some divergences other than $f$-divergences which satisfy the weak data processing inequality, such as Neyman-Pearson $\alpha$-divergences (\citep{polyanskiy2010arimoto,raginsky2011shannon}). Thus it would be interesting to study strong data processing inequalities w.r.t.\ those divergences as well.  
	\item Recently people have attempted to relate several types of channel ordering to the strong data processing inequalities (\citep{makur2016comparison,polyanskiy2015strong}). It is worth to explore the relationship between the statistical deficiency based channel ordering (\cite{raginsky2011shannon}) and the strong data processing inequalities.
	\item It would be interesting to study the hardness of the cost-sensitive classification with example dependent costs (\citep{zadrozny2001learning,zadrozny2003cost}), and the binary classification problem w.r.t.\ generalized performance measures (\cite{koyejo2014consistent}) such as arithmetic, geometric and harmonic means of the true positive and true negative rates.
	\item Further study could be undertaken in applying the cost-sensitive privacy notion to some real-world problems. 
	\item We illustrated the impact of the choice of substitution function in Aggregating Algorithm with experiments conducted on a synthetic dataset and a number of different real-world data sets (Section~\ref{sec:subfunc}). A theoretical understanding of the choice of substitution functions would be very useful. 
	\item An efficiently computable $\beta$-exp-concavifying link function for $\beta$-mixable multi-class proper losses, is still not known. At least showing a negative result would be worth, for example, showing that it is not possible to exp-concavify a multi-class (with $n > 2$) square loss with same mixability constant.
	\item Develop adaptive and optimistic variants of second order online learning algorithms such as online Newton step (\cite{hazan2007logarithmic}), efficiently using sketching methods (\cite{woodruff2014sketching}).
\end{enumerate}
Exploring and exploiting the geometric structure of the learning problem will serve as a guiding light in this future work.



\backmatter

\bibliographystyle{anuthesis}
\bibliography{thesis}

\printindex

\end{document}